\documentclass[noshowframe,bj,authoryear]{imsart}

\usepackage[colorlinks, linkcolor=blue, citecolor=blue]{hyperref} 
\usepackage{amsfonts,amssymb,amsmath,amsthm}
\usepackage{url}            % simple URL typesetting
\usepackage[dvipsnames]{xcolor}
\usepackage{graphicx, tikz}
\usepackage{natbib}
\usepackage{enumitem}
\usepackage{float}
\usepackage{comment}
\usepackage{booktabs}
\usepackage{subcaption}

\usepackage{xcolor,cancel}

\usepackage{pgfplots}
\pgfplotsset{compat=1.18}
\usetikzlibrary{pgfplots.fillbetween}

\startlocaldefs
\theoremstyle{plain}
\newtheorem{thm}{Theorem}
\numberwithin{thm}{section}
\newtheorem{lemm}[thm]{Lemma}

\newtheorem{cor}[thm]{Corollary}
\theoremstyle{remark}
\newtheorem*{rmk}{Remark}
\newtheorem{defn}[thm]{Definition} 

\newtheorem*{asm}{Assumption}
\DeclareMathOperator*{\argmin}{arg\,min}

\newcommand{\E}{\mathbb{E}}

\renewcommand{\P}{\mathbb{P}}
\newcommand{\R}{\mathbb{R}}
\newcommand{\<}{\langle}
\renewcommand{\>}{\rangle}
\newcommand{\Tr}{{\rm tr}}
\newcommand{\Id}{\mathrm{Id}}
\newcommand{\eps}{\varepsilon}
\newcommand{\n}{\textsc{n}}
\newcommand{\cs}{\textsc{c}}
\newcommand{\Var}{{\rm Var}}
\newcommand{\bs}[1]{\boldsymbol{#1}}

\endlocaldefs
\begin{document}

\begin{frontmatter}
\title{Universality of Benign Overfitting in Binary Linear Classification}
%\title{A sample article title with some additional note\thanksref{T1}}
\runtitle{Universality of Benign Overfitting in Binary Linear Classification}
%\thankstext{T1}{A sample of additional note to the title.}

\begin{aug}
\author[A]{\inits{IH}\fnms{Ichiro}~\snm{Hashimoto}\ead[label=e1]{ichiro.hashimoto@mail.utoronto.ca}}
\author[B]{\inits{SV}\fnms{Stanislav}~\snm{Volgushev}\ead[label=e2]{stanislav.volgushev@utoronto.ca}}
\author[B]{\inits{PZ}\fnms{Piotr}~\snm{Zwiernik}\ead[label=e3]{piotr.zwiernik@utoronto.ca}}
%%%%%%%%%%%%%%%%%%%%%%%%%%%%%%%%%%%%%%%%%%%%%%
%% Addresses                                %%
%%%%%%%%%%%%%%%%%%%%%%%%%%%%%%%%%%%%%%%%%%%%%%
\address[A]{Department of Statistical Sciences, University of Toronto\printead[presep={,\ }]{e1}}

\address[B]{Department of Statistical Sciences, University of Toronto\printead[presep={,\ }]{e2,e3}}
\end{aug}

\begin{abstract}
The practical success of deep learning has led to the discovery of several surprising phenomena. One of these phenomena, that has spurred intense theoretical research, is ``benign overfitting'': deep neural networks seem to generalize well in the over-parametrized regime even though the networks show a perfect fit to noisy training data. It is now known that benign overfitting also occurs in various classical statistical models. For linear maximum margin classifiers, benign overfitting has been established theoretically in a class of mixture models with very strong assumptions on the covariate distribution. However, even in this simple setting, many questions remain open. For instance, most of the existing literature focuses on the noiseless case where all true class labels are observed without errors, whereas the more interesting noisy case remains poorly understood. We provide a comprehensive study of benign overfitting for linear maximum margin classifiers. We discover a phase transition in test error bounds for the noisy model which was previously unknown and provide some geometric intuition behind it. We further considerably relax the required covariate assumptions in both the noisy and noiseless cases. Our results demonstrate that benign overfitting of maximum margin classifiers holds in a much wider range of scenarios than was previously known and provide new insights into the underlying mechanisms.
\end{abstract}

\begin{keyword}
\kwd{Benign overfitting}
\kwd{binary classification}
\kwd{double descent}
\kwd{logistic regression}
\kwd{over-parametrization}
\end{keyword}

\end{frontmatter}

%%%%%%%%%%%%%%%%%%%%%%%%%%%%%%%%%%%%%%%%%%%%%%
%%%% Main text entry area:

\addtocontents{toc}{\protect\setcounter{tocdepth}{-100}} 

\section{Introduction}

The practical success of deep learning has led to the discovery of surprising theoretical phenomena. One such phenomenon, which has spurred intense theoretical research, occurs in the over-parameterized regime, where over-parameterized neural network models can achieve arbitrarily small test errors while perfectly interpolating training data—even in the presence of label noise (\cite{zhang2017understanding}, \cite{doi:10.1073/pnas.1903070116}). This observation, often referred to as benign overfitting or double descent, is surprising because it appears to conflict with the classical statistical understanding that there should be a trade-off between fitting the training data and generalization performance on the test data.

Although benign overfitting/double descent originated from empirical observations on deep learning models, it is now known to occur in several classical statistical models, such as linear regression ( \cite{bartlett2020benign}, \cite{9051968}, \cite{https://doi.org/10.1002/cpa.22008}, \cite{Hastie:2022aa}), ridge regression (\cite{10.5555/3648699.3648822}), and kernel-based estimators (\cite{Liang_2020}, \cite{pmlr-v125-liang20a}), to name a few. Benign overfitting has also been observed in binary linear classification (\cite{JMLR:v22:20-974,wang2022binary,NEURIPS2021_46e0eae7}), which will be the focus of our work.

\smallskip

\textbf{Benign overfitting in linear binary classification: existing results.} Existing results on benign overfitting for binary linear classification focus on a specific class of models. The observations $(\boldsymbol{x}_i, y_{\n,i}), i = 1,\dots, n$ are generated from (sub-)Gaussian mixtures with means $\boldsymbol{\mu}$ and $-\boldsymbol{\mu}$ for a deterministic vector $\boldsymbol{\mu}\in \R^p$. Specifically, let $y$ be a Rademacher variable with $\P(y=-1)=\P(y=1)=\tfrac12$, and let $\boldsymbol{\xi}=(\xi_1,\ldots,\xi_p)\in \R^p$ be a random vector  with independent (sub-)Gaussian entries ${\xi}_j$, \( j=1,\dots,p \), satisfying $\E[{\xi}_j] = 0, \Var({\xi}_j) = 1$. We also define $y_{\n}\in \{-1,1\}$ as a noisy version of $y$, with $\P(y_\n \neq y)=\eta$, where $\eta \in [0, 1/2)$. For a $p\times p$ positive definite matrix $\Sigma$, consider the model where the inputs consist of $n$ i.i.d. copies $(\boldsymbol{x}_i,  y_{\n, i}), i = 1,\dots, n$ of the pair $(\boldsymbol{x}, y_\n)$, where 
$$\boldsymbol{x} \;=\; y\boldsymbol{\mu} + \boldsymbol{z},\qquad \boldsymbol{z} = U\Lambda^{1/2}\boldsymbol{\xi}$$
with $\Sigma = U \Lambda U^\top$, $\Lambda = \mathrm{diag}\{\lambda_1, \lambda_2, \ldots, \lambda_p\}, \lambda_1\geq \lambda_2 \geq \ldots \geq \lambda_p$, and $U\in \R^{p\times p}$ is an orthonormal matrix consisting of eigenvectors of $\Sigma$. The model is called noisy when $\eta >0$ and noiseless when $\eta = 0$.

In this specific setting, \cite{JMLR:v22:20-974} studied both the noiseless and noisy model with sub-Gaussian mixtures under rather strong assumptions: they are limited to the nearly isotropic case $\Tr(\Sigma) \asymp p$, $\|\Sigma\|\asymp 1$ and only consider the small signal regime $\|\bs \mu\|\ll p/n$. In this specific setting, they established a test error bound in binary linear classification and obtained a sufficient condition for benign overfitting to occur showing that the test error approaches the noise rate $\eta$. 
 
This result was extended to more general anisotropic mixtures by \cite{wang2022binary} and \cite{NEURIPS2021_46e0eae7} for the noiseless model by relaxing assumptions and improving the test error bound. While \cite{wang2022binary} focused on Gaussian mixtures, their results not only include anisotropic mixtures for the noiseless model but also extend to the regime $\Tr(\Sigma)/n \ll \|\boldsymbol{\mu}\|^2$. Their work also strengthens the result of \cite{JMLR:v22:20-974} for the noisy model, although it remains limited to isotropic mixtures and the regime $\Tr(\Sigma)/n \gg \|\boldsymbol{\mu}\|^2$. \cite{NEURIPS2021_46e0eae7} extended these results to anisotropic sub-Gaussian mixtures but only for the noiseless model. Overall, the noisy model remains poorly studied compared to the noiseless model. The existing results are summarized in Table~\ref{summary-table}. 

A striking feature of this table is that benign overfitting is least well understood in the stronger signal regime. Intuitively, it seems that stronger signals make benign overfitting more likely and easier to occur. However, a closer look at the underlying geometry confirms this intuition only in the noiseless model, see Figure~\ref{fig:geom}. At the same time, noisy data are much more realistic in practice. For overfitting to be truly benign, it is reasonable to expect that it holds across a wide range of scenarios and does not break in cases which should be 'easy' due to a stronger signal. This underlines the importance of studying the strong signal regime, especially in the noisy model where geometric considerations do not provide a simple intuition on what to expect.

\begin{table}
\caption{Summary of existing studies.}
\label{summary-table}
\begin{center}
\resizebox{\textwidth}{!}{%
\begin{tabular}{c|c|c|c}
& \multicolumn{3}{c}{$\Sigma = $ Id}
\\ \toprule
& Small signal & Intermediate signal & Large signal
\\ 
& $\|\bs \mu\|^2 \ll p/n$ & $p/n \ll \|\boldsymbol{\mu}\|^2 \lesssim p^2/(n^2\log n)$ & $\max\{p/n, p^2/(n^2\log n)\} \ll \|\boldsymbol{\mu}\|^2$ \\ \midrule
noiseless     & \cite{JMLR:v22:20-974}, & \cite{wang2022binary},  & NA
\\
    & \cite{wang2022binary},  & \cite{NEURIPS2021_46e0eae7} &
\\
    & \cite{NEURIPS2021_46e0eae7}   & &
\\ \midrule
noisy     & \cite{JMLR:v22:20-974}, & NA  & NA
\\
    & \cite{wang2022binary} &  &
\\
\toprule
& \multicolumn{3}{c}{$\Sigma$ arbitrary}
\\
\toprule
& $\|\boldsymbol{\mu}\|^2 \ll \Tr(\Sigma)/n$ 
& $\Tr(\Sigma)/n \ll \|\boldsymbol{\mu}\|^2$ and 
& $\Tr(\Sigma)/n \ll \|\boldsymbol{\mu}\|^2$ and
\\
& & $ \|\boldsymbol{\mu}\|_\Sigma^2 \lesssim \Tr(\Sigma)^2/(n^2\log (n))$
& $ \Tr(\Sigma)^2/(n^2\log (n)) \ll \|\boldsymbol{\mu}\|_\Sigma^2$
\\ \midrule
noiseless    & \cite{JMLR:v22:20-974}, & \cite{wang2022binary}, & NA
\\
    & \cite{wang2022binary}, & \cite{NEURIPS2021_46e0eae7} &
\\
    & \cite{NEURIPS2021_46e0eae7} &  & 
\\ 
\midrule
noisy    & \cite{JMLR:v22:20-974} & NA & NA 
\\
& ($\Tr(\Sigma) \asymp p$, $\|\Sigma\|\asymp 1$) & & 
\\
\bottomrule
\end{tabular}
}
\end{center}
\end{table}

\smallskip

\textbf{Overview of our main contributions.} As is evident from the discussion of existing literature, benign overfitting has been established for binary linear classification in rather specialized settings. 

In the noisy setting, it is unclear whether benign overfitting is limited to sub-Gaussian, nearly isotropic features, small signals, or if it holds more widely. Even in the noiseless setting, all of the existing literature considers sub-Gaussian predictors. In addition to sub-Gaussianity of the predictors, all models considered in the existing literature lead to predictors with roughly equal norms. Again, it is unclear to what extent this model feature is crucial for benign overfitting or whether it is just a by-product of the specific model formulation. 

Our main contribution is to show that benign overfitting for binary classification is indeed more universal than existing results seem to suggest. Specifically, we show that mild moment conditions on the predictors are sufficient, that equal norms of the features are not important, and that benign overfitting can be observed in all cases which are left open in Table~\ref{summary-table}. In the isotropic case ($\Sigma = I_p$), we show that benign overfitting occurs as soon as the dimension is sufficiently large relative to the sample size under very mild assumptions on the signal strength. In the anisotropic case, the dimension is replaced by the proxy $\Tr(\Sigma)/\|\Sigma\|_{\textsc{f}}$. For the noiseless model, an interesting phase transition in the test error bound was observed by \cite{wang2022binary, NEURIPS2021_46e0eae7}. Our work shows that a similar phase transition also occurs for the noisy model. We also discover differences in the test error behavior between the noiseless and noisy model in the intermediate and strong signal regime. We further illustrate this difference from a geometric perspective. 

The remaining paper is organized as follows. In Section~\ref{sec:setting} we introduce the models and notation considered in this paper. Our main results are given in Section~\ref{sec:main} where we first provide new results in the noiseless case (Section~\ref{subsec:main-noiseless}) and later move on to the noisy case (Section~\ref{subsec:main-noisy}). Section~\ref{sec:sketch} contains a proof outline of our main results. Additional geometric intuition is given in Section~\ref{geometry}. In Appendix~\ref{sec:recoveringSubGauss} we illustrate how our theory can be used to recover some of the existing results in the literature. Detailed proofs are given in the supplement.

\section{Setting and Notations}\label{sec:setting}

Our first set of results will be for the following general model.

\begin{enumerate}[label=(M),ref=(M)]
\item The observations consist of $n$ i.i.d. copies $(\boldsymbol{x}_i, y_{\n,i})$ of the pair $(\bs x,y_\n)$. Here for a random variable $y\in \{-1,1\}$ satisfying $P(y = 1) = P(y = -1) = 1/2$, a random vector $\bs z \in \R^p$ independent of $y$, and a deterministic $\bs \mu \in \R^p$ we have \label{model:M}
\begin{equation}\label{eq:themodel}
\boldsymbol{x} \;=\; y\boldsymbol{\mu} + \boldsymbol{z}.    
\end{equation}
For $\eta \in [0, 1/2)$, generate $y_\n = - y$ w.p. $\eta$ and $y_\n = y$ w.p. $1-\eta$. 
\end{enumerate}

The Bayes risk (with the $0/1$-loss) in this model is at least the noise level $\eta$. We call this model \textit{noiseless} when $\eta = 0$. Under model \ref{model:M}, we will provide high-level conditions which ensure benign overfitting. Those high-level results will be further verified in the following model.

\begin{enumerate}[label=(EM), ref=(EM)]
\item The model \ref{model:M} is called the extended non-sub-Gaussian mixture model if $\boldsymbol{z} = gW \boldsymbol{\xi}$, where: \label{model:EM}
\begin{enumerate}
    \item [(i)] $g\in (0, \infty)$ is a random variable such that,  for some $\ell \in [2, \infty]$ and $k\in (2, 4]$, 
\begin{equation}\label{eq:condsg}
 \E g^2 = 1,\qquad  \E g^\ell < \infty,\qquad \E g^{-k} < \infty,
 \end{equation}
 \item [(ii)] $W \in \R^{p\times p}$, and $WW^\top = \Sigma \in \R^{p\times p} \setminus \{0\}$ is a deterministic positive semidefinite matrix,
 \item [(iii)] $\boldsymbol{\xi}=(\xi_1,\ldots,\xi_p) \in \R^p$ is a random vector with independent entries, independent of $g$ such that, for some  $r \in (2, 4]$ and $K>0$, 
\begin{equation}\label{eq:condsxi}
\E\xi_{j} = 0, \quad \E \xi_{j}^2 = 1,\quad  \E |\xi_{j}|^r \leq K, \quad \forall j=1,\dots,p.
\end{equation} 
\end{enumerate}
\end{enumerate}
In particular, $
{\rm cov}(\bs z)=\Sigma$. When $g\equiv 1$, this model incorporates as special cases several models that were considered in the literature. \cite{JMLR:v22:20-974}, \cite{NEURIPS2021_46e0eae7}, and \cite{minsker2025classification} studied the model when $\boldsymbol{z} = U\Lambda^{1/2} \boldsymbol{\xi}$ where $\Sigma = U \Lambda U^\top$, $\Lambda = \mathrm{diag}\{\lambda_1, \lambda_2, \ldots, \lambda_p\}, \lambda_1\geq \lambda_2 \geq \ldots \geq \lambda_p$, $U\in \R^{p\times p}$ is an orthonormal matrix consisting of eigenvectors of $\Sigma$  
and $\boldsymbol{\xi} \in \R^p$ is a (sub-)Gaussian vector with independent entries, $\E\boldsymbol{\xi}=0$, ${\rm Var}(\boldsymbol{\xi})=I_p$. This special structure simplifies some of the theoretical arguments. In contrast, we consider general matrices $W$ which include the standard choice $W = U\Lambda^{1/2}U^\top$. \cite{wang2022binary} studied this model when $\bs z \sim \mathcal{N}(0, \Sigma)$. 

A classifier is any mapping from the input space $\R^p$ to $\{-1,1\}$. Here we focus on linear classifiers of the form ${\rm sgn}(\<\bs w,\bs x\>)$ for some $\bs w\in \R^p$. Below we also call a classifier any method of choosing $\bs w$ based on the data. 

Similarly to the existing literature, we will only discuss homogeneous classifiers. However, some of our results are sufficiently general to include inhomogeneous classifiers after extending the feature vectors $\bs x_i$ through $\tilde{\bs x}_i = (\bs x_i,1)$. The formulation of model~\ref{model:M} is sufficiently general to accommodate such an extension, while previously considered models of the form \ref{model:EM} do not allow for vectors with constant non-zero entries. For the sake of brevity, we do not work out the details of such an extension in the remaining part of this paper.

 We say that a classifier exhibits \textit{benign overfitting} if it interpolates the training data with high probability and if its classification error converges to the Bayes risk in a suitable sense. 

A natural way to train the classifier is by minimizing the logistic loss:
\[
L(\boldsymbol{w}) \;=\; \frac{1}{n}\sum_{i=1}^n \log \{1 + \exp(-\langle \boldsymbol{w}, y_{\n,i}\boldsymbol{x}_i \rangle)\}.
\]
Numerical optimization is often done via gradient descent:
\begin{equation}\label{eq:gdsc}
\boldsymbol{w}_{t+1} = \boldsymbol{w}_t - a \nabla_{\!\boldsymbol{w}} L(\boldsymbol{w}_t), \quad \boldsymbol{w}_0 =0, \quad t = 0, 1, 2, \dots.     
\end{equation}

Our theoretical analysis will focus on the \emph{maximum margin classifier} $\boldsymbol{\hat w}$, that is the solution to the hard-margin support vector machine:
\begin{equation}\label{eq:mmc}
\boldsymbol{\hat w} \;=\; \argmin\|\boldsymbol{w}\|^2, \qquad \text{ subject to $\langle \boldsymbol{w}, y_{\n, i}\boldsymbol{x}_i \rangle \geq 1$\;\; for all $i= 1, 2, \dots, n$.}    
\end{equation}
This is motivated by the following result which establishes implicit bias (directional convergence) induced by the gradient descent \eqref{eq:gdsc}.
\begin{thm}[Theorem~3, \cite{soudry2022implicit}]\label{th:soudry}
    When the dataset is linearly separable ($\exists \bs w\in \mathbb{R}^p$ such that $\langle \boldsymbol{w}, y_i\boldsymbol{x}_i\rangle >0$ for all $i$), the linear classifier optimized by gradient descent \eqref{eq:gdsc},  with sufficiently small step size $a$,  converges in direction to the maximum margin classifier, that is
\[
\lim_{t\to \infty} \frac{\boldsymbol{w}_t}{\|\boldsymbol{w}_t\|} \;=\; \frac{\boldsymbol{\hat w}}{\|\boldsymbol{\hat w}\|}. 
\]
\end{thm}

We will prove that, under certain conditions, datasets generated by model \eqref{eq:themodel} are
linearly separable with high probability, making $\boldsymbol{\hat w}$ an interesting object to study. We note that the same classifier was analyzed in the existing literature (\cite{JMLR:v22:20-974, wang2022binary, NEURIPS2021_46e0eae7}).

Analyzing the properties of $\bs{\hat w}$ is difficult in general due to the lack of a closed form expression. Several existing works show that, in many relevant cases, $\bs{\hat w}$ equals the minimum-norm least squares estimator:
\begin{equation}\label{eq:ls}
    \bs w_{\rm LS} \;=\; \argmin \|\bs w\| \quad \text{s.t. $\displaystyle \sum_{i=1}^n (\langle \bs w, \bs x_i \rangle - y_i)^2$ is minimized.}
\end{equation}  
This includes the work of \cite{hsu2021proliferation, wang2022binary, NEURIPS2021_46e0eae7, minsker2025classification}. We derive conditions which ensure this equivalence in models~\ref{model:M} and \ref{model:EM} and use it to study the classification error of $\hat{\bs w}$. Another part of our analysis relies on purely geometric arguments, and covers regimes where this equivalence can fail.

\vspace{5mm}

\noindent \textbf{Notations.}  We use the standard asymptotic notations $\lesssim$,  $\gtrsim$, and $\asymp$. For non-negative sequences $a_n$ and $b_n$, if there exists a constant $C>0$  such that $a_n \leq Cb_n$ for sufficiently large $n$, then we write $a_n \lesssim b_n$. If $b_n \lesssim a_n$, we write $a_n \gtrsim b_n$. If $a_n \lesssim b_n$ and $a_n \gtrsim b_n$, then we write $a_n \asymp b_n$. Note also that when any of these asymptotic notations are used, we assume the constant hidden is a universal constant unless explicitly stated otherwise. We write $a_n \ll b_n$ if $a_n/b_n \to 0$ and $a_n \gg b_n$ if $b_n \ll a_n$. The Euclidean sphere in $\R^p$ with center at the origin and radius $1$ is denoted by $S^{p-1}$.

For a vector $\bs v\in \R^p$, we use $\|\bs v\|$ to denote the Euclidean norm of \bs $v$. For a matrix $A$ we use $\|A\|$ to denote the spectral norm and $\|A\|_{\textsc{f}}$ to denote the Frobenius norm of $A$. For a vector $\bs v\in \R^p$ and a positive semidefinite matrix $A$ we use $\|\bs v\|_A$ to denote $\sqrt{\bs v^\top A \bs v}$.

We use $\P_{\boldsymbol{u}}$ to denote the conditional probability given all variables except $\boldsymbol{u}$.

\section{Main Results}\label{sec:main}

In this section, we present our main results on test error bounds and conditions for benign overfitting of the maximum margin linear classifier \eqref{eq:mmc}. Throughout, we use the following notation
\[
\boldsymbol{y} = (y_1, \ldots, y_n)^\top \in \R^n, \qquad \boldsymbol{y}_\n = (y_{\n, 1}, \ldots, y_{\n, n})^\top \in \R^n,  
\]
\[
Z = (\bs z_1, \ldots, \bs z_n)^\top \in \R^{n\times p}, \qquad X = \boldsymbol{y}\boldsymbol{\mu}^\top + Z \in \R^{n\times p}. 
\]

When $\bs z_i \in \R^p \setminus \{\bs 0\}$ for all $i=1,\ldots, n$, denote by $\Delta (z)$ an $n\times n$ diagonal matrix whose diagonal entries are the Euclidean norms $\|\boldsymbol{z}_i\|$ for $i = 1, \dots, n$. For any vector $\bs v\in \R^n$, let $\check{\boldsymbol v} = \Delta(z)^{-1} \boldsymbol{v}$. We also set 
\begin{equation}\label{eq:defcheck}
\check{\bs y} := \Delta(z)^{-1} \bs y, \qquad \check{\bs y}_\n := \Delta(z)^{-1} \bs y_\n,  \qquad \check Z = \Delta(z)^{-1} Z   \qquad\mbox{and  \qquad} \check X = \Delta(z)^{-1} X.
\end{equation}

When $\bs{\check v}, \check Z$, or $\check X$ is present, we always assume $\bs z_i \in \R^p \setminus \{\bs 0\}$ for all $i=1,\ldots, n$ without stating it explicitly.

The main approach to study the risk of the maximum margin classifier is through a careful analysis of the Gram matrix $ZZ^\top$. The results of \cite{wang2022binary, NEURIPS2021_46e0eae7} heavily relied on the fact that, in their special model, $ZZ^\top$ is approximately a multiple of the identity $I_n$ with high probability, which in turn means $\boldsymbol{z}_i$'s are near orthogonal and have equal norms. In our work, we show that the equal norms assumption is not needed and near orthogonality suffices. More precisely, our general results will be established under high level assumptions involving parameters from the following events
\begin{align}
E_1(\varepsilon) &:= \big\{\|\check Z\check Z^\top - I_n\|\leq \varepsilon\big\},\label{eq:E1}
\\
E_2(\alpha_2, \alpha_\infty) &:= \big\{\|\check Z\boldsymbol{\mu}\| \leq \alpha_2 \|\boldsymbol{\mu}\| \text{ and } \|\check Z \boldsymbol{\mu}\|_\infty \leq \alpha_\infty \|\boldsymbol{\mu}\|\big\},\label{eq:E2}
\\
E_3(M) &:= \big\{\max_{i=1,\ldots,n} \|\boldsymbol{z}_i\|\leq M\big\},\label{eq:E3}
\\
E_4(\beta, \rho) & := \big\{|\|\check{\boldsymbol y}\|^2  - n\rho | \leq \beta n\rho\big\}  = \Big\{ \Big|\dfrac1n\sum_{i=1}^n \dfrac{1}{\|\boldsymbol{z}_i\|^{2}}-\rho \Big|\leq \beta\rho  \Big\},\label{eq:E4}
\\
E_5(\gamma, \rho) & := \big\{|\langle \check{\boldsymbol y}_{\n}, \check{\boldsymbol y} \rangle - (1 - 2\eta) n\rho| \leq \gamma n\rho\big\}
= \Big\{\Big|\dfrac1n\sum_{i=1}^n \dfrac{y_{\n, i} y_i}{\|\boldsymbol{z}_i\|^2} -(1-2\eta)\rho \Big|\leq \gamma\rho \Big\}.\label{eq:E5}
\end{align}
We often omit parameters and simply denote these events by $E_i$ when there is no risk of confusion. Note that the norm of the rows of $\check Z$ is $1$ and so the event $E_1$ controls the angle $\angle({\bs z}_i, {\bs z}_j)$ between ${\boldsymbol{z}}_i$ and ${\boldsymbol{z}}_j$ for all $i\neq j$; $|\cos(\angle({\bs z}_i, {\bs z}_j))|\leq \varepsilon$. When $\varepsilon$ is small, the $\boldsymbol{z}_i$ are near orthogonal to each other. Similarly, event $E_2$ controls the angle between $\boldsymbol{z}_i$'s and $\boldsymbol{\mu}$. The quantities $\alpha_2,\alpha_\infty$ provide two different ways of quantifying near orthogonality between all $\boldsymbol{z}_i$'s and $\boldsymbol{\mu}$. Event $E_3$ imposes an upper bound on the maximal norm of $\bs z_i$ and can be viewed as capturing a global scale. Events $E_4, E_5$ are concerned with concentration of averages and thus are related to moment assumptions on $1/\|\bs z_i\|$ rather than values of individual norms.

Before proceeding to our main results, we provide a concrete example of the general framework above in the case of model \ref{model:EM}. In the next statement and elsewhere, we will make repeated use of the following constants which depend only on the distribution of $g$ and on $\bs \xi$ through the constants $r, K$ from equation~\eqref{eq:condsxi}
\begin{equation} \label{eq:defconstantsEM}
\begin{aligned}
C_2(g) &:= 2^{2+2/k}\tfrac{\|g^{-2}\|_{L^{k/2}}}{\|g^{-2}\|_{L^1}},\\  C(r, K) &:= 2^{r/2-1}\left\{2 (K^{2/r} + 1)^{r/2}  + 2^{r/4} \right\},\\ 
C_1(r,K) &:= 4\left\{2C(r, K)\right\}^{2/r}. 
\end{aligned}
\end{equation}

\begin{lemm}\label{lemm5-ex1} Suppose model \ref{model:EM} holds and 
\begin{equation}\label{cond:lemm5-ex1}
\eps \;:=\; C_1(r,K) (\tfrac{n}{\delta})^{2/r} \max\{p^{2/r-1/2}, n^{2/r} \} \frac{\|\Sigma\|_{\textsc{f}}}{\Tr(\Sigma)} \;\leq\; \frac{1}{2}.   
\end{equation}
Then, with probability at least $1 - 5\delta$, event $\bigcap_{i=1}^5 E_i$ holds with $\eps$ as defined above and 
\[
\alpha_\infty \;=\; \alpha_2 \;=\; \frac{2\sqrt{n}\|\boldsymbol{\mu}\|_\Sigma}{\sqrt{\delta \Tr(\Sigma)}\|\boldsymbol{\mu}\|}, \qquad \beta = \gamma \;=\; \varepsilon + C_2(g)\delta^{-2/k} n^{-(1 - 2/k)},\] 
\[
\rho  \;=\; \E [g^{-2}]\Tr(\Sigma)^{-1}, \qquad M \;=\;    (1+\eps)\|g\|_{L^\ell}(\tfrac{n}{\delta})^{1/\ell}\sqrt{\Tr(\Sigma)},
\]
Moreover, with the same definitions, event $\cap_{i=1}^4 E_i$ holds with probability at least $1-4\delta$. 
\end{lemm}

Lemma~\ref{lemm5-ex1} provides concrete bounds on various quantities in the special setting of model \ref{model:EM}. To gain some further intuition about orders, set $\Sigma = I_p$. In that case $\Tr(\Sigma) = p, \|\Sigma\|_F = \sqrt{p}$ and~\eqref{cond:lemm5-ex1} holds provided that $p > C \max\{n^{8/r}\delta^{-4/r},n^{\tfrac{2}{r-2}}\delta^{-\tfrac{2}{r-2}}\}$ for a sufficiently large constant $C$. This is stronger than the over-parametrization requirements in existing work under sub-Gaussian assumptions due to the milder moment conditions we consider. For additional comparisons with existing work under sub-Gaussian assumptions, see Section~\ref{sec:detailedcomparison} in the Supplement. A simple computation shows that $\eps, \alpha_\infty, \alpha_2$ are small if $p$ is sufficiently large. Thus the key role of the high dimensionality of predictors is to ensure near orthogonality. Near orthogonality is also possible if $\Sigma$ is far from the identity, but clearly requires sufficiently many eigenvalues that are not too small and of roughly the same order. The proof of Lemma~\ref{lemm5-ex1} is provided in Section~\ref{sec:proof:lemm5-ex1} of the Supplement.

\subsection{Test error bounds and benign overfitting in the noiseless case}\label{subsec:main-noiseless}

In this section, we provide test error bounds in the noiseless case ($\eta = 0$) of models \ref{model:M} and \ref{model:EM}. We begin with a general result in model \ref{model:M}. 

\begin{thm}\label{thm:noiseless-main}
Consider model \ref{model:M} with $\eta = 0$. Suppose one of the following conditions holds:
\begin{enumerate}
\item[(i)] Event $\bigcap_{i=1}^4 E_i$ holds with $\beta \in [0, 1/2)$, $\|\boldsymbol{\mu}\|{\sqrt{(1-\beta)n\rho}} \geq C {\alpha_2}$ for sufficiently large constant $C$, $\alpha_2\|\boldsymbol{\mu}\|\sqrt{(1+ \beta)n\rho} \leq \tfrac14$,  $\varepsilon M \sqrt{(1+\beta)n\rho} \leq \tfrac14$, and $M\alpha_\infty \|\boldsymbol{\mu}\|(1+\beta)n\rho < \tfrac{3}{32}$.
\item[(ii)] Event $E_3(M)$ holds with $\|\boldsymbol{\mu}\| \geq CM$ for some constant $C>2$.  
\end{enumerate}
Then the gradient descent iterates converge in direction to $\boldsymbol{\hat w}$, and the following test error bound holds for a universal constant $c$
    \begin{equation}\label{eq:errboundnonoise}
        \P_{(\boldsymbol{x}, y)}(\langle \boldsymbol{\hat w}, y\boldsymbol{x}\rangle < 0) \;\leq\; c\|\E[\boldsymbol{z}\boldsymbol{z}^\top]\| \left(\frac{1}{\|\boldsymbol{\mu}\|^2} + \frac{1}{n\rho\|\boldsymbol{\mu}\|^4} \right).
    \end{equation}
\end{thm}
We provide the proof sketch of this result in Section~\ref{sec:sketch}, where we also state some fundamental results that underlie our arguments. The formal proof is given in Section~\ref{sec:proofgennoiselessthm}. There, we also provide an exponential bound on the test error when $\bs z_i$ are sub-Gaussian; see equation~\eqref{eq:errboundnonoise-exp}. 

Theorem~\ref{thm:noiseless-main} shows that near orthogonality of the vectors $\bs z_i$ and mild conditions on the range of their magnitudes suffice for benign overfitting. This implies that other special features of models that were studied so far, such as stringent moment assumptions or roughly equal norms of $\bs z_i$ are not required. In Section~\ref{sec:recoveringSubGauss} in the appendix, we discuss how the general framework of Theorem~\ref{thm:noiseless-main} can be used to recover existing results on sub-Gaussian mixtures by verifying its conditions in specific models. In addition, a more detailed comparison with existing work is given in Section~\ref{sec:detailedcomparison}.

We note that the cases (i) and (ii) in the above result correspond to two different regimes for the signal strength: in case (i), there are both upper and lower bounds on $\|\bs\mu\|$, while case (ii) only has a lower bound on $\|\bs\mu\|$. In both cases we obtain the same test error bound but the proof techniques differ; see Section~\ref{sec:proof} for additional details. Case (i) roughly corresponds to the small and intermediate signal strength regimes in Table~\ref{summary-table} which have been studied in the existing literature in special models. Case (ii) corresponds to the large signal regime in Table~\ref{summary-table} which has not been considered in the literature before.  Note that the derivation of case (ii) is purely geometric and is based on arguments that are completely different from approaches taken in prior work. In this scenario, equivalence between the max margin classifier and the least squares estimator can fail, see the discussion after equation~\eqref{eq:noisy-mmc-bound-sketch} for an explicit example.

The test error bound in~\eqref{eq:errboundnonoise} exhibits a phase transition: when $\|\bs\mu\| \ll (n\rho)^{-1/2}$, the test error bound is dominated by $\tfrac{\|\E[\boldsymbol{z}\boldsymbol{z}^\top]\|}{n\rho\|\boldsymbol{\mu}\|^4}$, while for $\|\bs\mu\| \gg (n\rho)^{-1/2}$ the bound simplifies to $\tfrac{\|\E[\boldsymbol{z}\boldsymbol{z}^\top]\|}{\|\boldsymbol{\mu}\|^2}$. The geometric intuition behind this transition will be discussed in more detail in Section~\ref{geometry}.

Next, we specialize the general results given above to model \ref{model:EM}. Our first Theorem~covers the small and intermediate signal case. 

\begin{thm}\label{thm:noiseless-ext-detail1}
Consider model \ref{model:EM} with $\eta=0$. Assume that $n\geq (\tfrac{4C_2(g)}{\delta^{2/k}})^\frac{k}{k-2}$. Then, there exists a constant $C$ depending only on the distribution of $g, \xi$ such that, provided that
\begin{equation*}
    \|\boldsymbol{\mu}\|^2 \geq C\delta^{-1/2}\|\boldsymbol{\mu}\|_\Sigma,
\end{equation*}
and 
\begin{align*}
\Tr(\Sigma) \geq C \left(\tfrac{n}{\delta}\right)^{1/\ell} & \max\left\{\left(\tfrac{n}{\delta}\right)^{2/r}n^{1/2}\|\Sigma\|_F \max\{p^{\frac{2}{r}-\frac{1}{2}},n^{2/r}\},  \sqrt{\tfrac{n}{\delta}}n\|\bs{\mu}\|_\Sigma\right\}
\end{align*}
hold, with probability at least $1-4\delta$, the gradient descent iterates converge in direction to $\boldsymbol{\hat w}$, and the following test error bound holds for a universal constant $c$:
    \begin{equation*}
        \P_{(\boldsymbol{x}, y)}(\langle \boldsymbol{\hat w}, y\boldsymbol{x}\rangle < 0) \;\leq\; c \|\Sigma\| \left(\frac{1}{\|\boldsymbol{\mu}\|^2} + \frac{\Tr(\Sigma)}{n \|\boldsymbol{\mu}\|^4 \E[g^{-2}]}\right).
    \end{equation*}
\end{thm}

The proof of Theorem~\ref{thm:noiseless-ext-detail1} is given in Section~\ref{sec:proof:EM:noiseless1}. In there, we provide a slightly sharper statement that explicitly discusses the precise form of the constant $C$.

Theorem~\ref{thm:noiseless-ext-detail1} substantially generalizes the existing literature which only considers linear transformations of sub-Gaussian predictors with independent entries. Interestingly, our test error bound exhibits similar features as bounds in the existing literature. More precisely, in the case $\Sigma = \Id, g \equiv 1$, the bound is increasing  in $ \tfrac{1}{\|\boldsymbol{\mu}\|^2} + \tfrac{p}{n \|\boldsymbol{\mu}\|^4 }$.

The assumption $\Tr(\Sigma) \geq  C\sqrt{\tfrac{n}{\delta}}n\|\bs{\mu}\|_\Sigma$ in Theorem~\ref{thm:noiseless-ext-detail1} prevents $\bs\mu$ from being arbitrarily large. The case of very large $\|\bs\mu\|$ is covered in the next result under substantially weaker assumptions on the other model parameters. This result corresponds to setting (ii) in the general Theorem~\ref{thm:noiseless-main}. Recall the definition of the constant $C_1(r,K)$ in~\eqref{eq:defconstantsEM}.

\begin{thm}\label{thm:noiseless-ext-detail2}
Consider model \ref{model:EM} with $\eta = 0$. Let $C$ be the constant in Theorem~\ref{thm:noiseless-main}(ii). 
For any $\delta \in (0, \tfrac12)$, if
\begin{equation}\label{cond:noiseless-ext-detail2}
\displaystyle \|\boldsymbol{\mu}\| \geq \tfrac{3}{2}C \|g\|_{L^\ell} (\tfrac{n}{\delta})^{1/\ell}\sqrt{\Tr(\Sigma)}
\end{equation}
for a constant $C>2$ and 
\begin{equation}\label{cond2:noiseless-ext-detail2}
\displaystyle \Tr(\Sigma) \geq 2C_1(r,K)(\tfrac{n}{\delta})^{2/r}\max\{p^{2/r-1/2}, n^{2/r}\}\|\Sigma\|_{\textsc{f}},
\end{equation}
then, with probability at least $1 - 2\delta$, the gradient descent iterates converge in direction to $\boldsymbol{\hat w}$, and the following test error bound holds:
    \begin{equation*}
        \P_{(x, y)}(\langle \boldsymbol{\hat w}, y \boldsymbol{x}\rangle < 0) \leq c \frac{\|\Sigma\|}{\|\boldsymbol{\mu}\|^2},
    \end{equation*}
    where $c$ is a universal constant.
\end{thm}

Theorem~\ref{thm:noiseless-ext-detail2} is proved in Section~\ref{sec:proof:EM:noiseless2}.

Theorem~\ref{thm:noiseless-ext-detail2} does not place any upper bounds on $\|\bs\mu\|$ and thus discusses a regime that was not considered in the existing literature, even for the very simple case of isotropic Gaussian $\bs z_i$. In the noiseless case, this result could be expected since for very large $\|\bs \mu\|$ there is a clear separation between the classes $y_i=1$ and $y_i=-1$.

We also note that the test error bound in this case is the same as one in the intermediate signal regime of Theorem~\ref{thm:noiseless-ext-detail1}. This is because the term $ \frac{\|\Sigma\|\Tr(\Sigma)}{n \|\boldsymbol{\mu}\|^4 \E[g^{-2}]}$ appearing in that bound is upper bounded by a constant times $\frac{\|\Sigma\|}{\|\boldsymbol{\mu}\|^2}$ under the assumption~\eqref{cond:noiseless-ext-detail2} on $\|\bs\mu\|$.

\subsection{Test error bounds and benign overfitting in the noisy model}\label{subsec:main-noisy}

In this section, we present our results for the noisy ($\eta  \in (0,1/2)$) version of models \ref{model:M} and \ref{model:EM}. We start with a general result which corresponds to the noisy version of Theorem~\ref{thm:noiseless-main}. It relies on the following assumption. 

\begin{enumerate}
\item[($N_C$)] For the constant $C$, the event $\bigcap_{i=1}^5 E_i$ holds with $\varepsilon M \sqrt{n\rho} \leq \tfrac{\eta}{2}$ and $\|\bs \mu\|\geq C\frac{\alpha_2}{\sqrt{n\rho}}$, and one of the following conditions holds
\begin{enumerate}
        \item[(i)] $\alpha_2 \|\boldsymbol{\mu}\| \sqrt{n\rho} \leq \tfrac{1}{30}$ and $\alpha_\infty \|\boldsymbol{\mu}\| M n\rho \leq \tfrac{1}{64}$, 
        \item [(ii)] $\displaystyle \|\boldsymbol{\mu}\| \geq C \alpha_\infty M$ and $\max\left\{\alpha_2^2, \alpha_2\alpha_\infty M\sqrt{n\rho} \right\} \leq {C}^{-1}$. 
    \end{enumerate}  
\end{enumerate}

We refer readers to the assumptions of Theorem~\ref{thm:noisy-ext-detail1-simple} for sufficient conditions for $N_{C_{2,\eta}}$ under model \ref{model:EM} with $C_{2,\eta} = \max\{\tfrac{22}{\eta}, \tfrac{17}{1-2\eta}\}$.

\begin{thm}\label{detail-noisy-main-1-simple}
Suppose that $(\bs x, y_\n)$ are generated from model \ref{model:M} with $\eta \in (0,1/2)$. If $(N_{C_{2,\eta}})$ holds with $C_{2,\eta} = \max\{\tfrac{22}{\eta}, \tfrac{17}{1-2\eta}\}$ and $\eps\vee \beta\vee \gamma \leq \frac{\min\{\eta, 1-2\eta\}}{8}$, then the gradient descent iterates converge in direction to $\boldsymbol{\hat w}$, and the following test error bound holds:
    \begin{equation*}
        \P_{(\boldsymbol{x}, y_\n)}(\langle \boldsymbol{\hat w}, y_\n\boldsymbol{x}\rangle < 0) \leq \eta + \frac{c_1 \|\E [\boldsymbol{z}\boldsymbol{z}^\top]\|}{(1 - 2\eta)^2} \left\{\eta n\rho + \frac{1}{\|\bs \mu\|^2} + \frac{1}{n\rho \|\boldsymbol{\mu}\|^4}\right\},
    \end{equation*}
    where $c_1$ is a universal constant. 
\end{thm}
The sketch of the proof of this Theorem~is provided in Section~\ref{sec:sketch}. The complete proof  is given in Section~\ref{sec:proof:detail-noisy-main-1-simple}. There, we also provide an exponential bound on the test error when $\bs z_i$ are sub-Gaussian and prove that the following simpler conditions are sufficient to ensure that the conditions of Theorem~\ref{detail-noisy-main-1-simple}
\begin{equation}\label{eq:condsimplenoisymain}
\|\boldsymbol{\mu}\| \geq \tfrac{17}{16}C_{2,\eta}\alpha_2 M, \quad  [\alpha_2 \vee \eps] M\sqrt{n\rho} \leq \tfrac{16}{{17} C_{2,\eta}},\quad \beta \vee \gamma \leq \tfrac{16}{{17} C_{2,\eta}}.
\end{equation}

Theorem~\ref{detail-noisy-main-1-simple} establishes an upper bound on the test error in the noisy model under a very general framework. The analysis in the noisy setting $\eta > 0$ is substantially more technical than the analysis of the noiseless setting $\eta = 0$, and this model is much less well understood. In fact, so far the only available results are for nearly isotropic sub-Gaussian predictors $\bs z_i$ in the small signal regime (\cite{JMLR:v22:20-974, wang2022binary}).

Theorem~\ref{detail-noisy-main-1-simple} not only generalizes the existing results by substantially relaxing the assumptions on $\bs z_i$, but also considers the intermediate and strong signal regimes which were not studied before. Interestingly, the behavior of the test error bound is similar to the noiseless model in the small signal regime $\|\bs\mu\|^2 \lesssim (n\rho)^{-1}$, but markedly different in the regime $\|\bs\mu\|^2 \gg (n\rho)^{-1}$. Specifically, when $\|\bs\mu\|^2 \gg (n\rho)^{-1}$, the test error bound involves the term $\|\E [\boldsymbol{z} \boldsymbol{z}^\top] \|n\rho$ which is independent of the signal strength $\|\bs\mu\|^2$. This is in stark contrast to the noiseless case where the test error bound in the large signal regime decays with $1/\|\bs \mu\|^2$. This suggests that the geometry and mechanisms underlying benign overfitting in the noisy and noiseless models are completely different when the signal is strong. This difference was not observed in the existing literature. Additional insights into the geometry of this phenomenon are provided in Section~\ref{sec:phasetransnoisy}.

To demonstrate that the phase transition discussed above is indeed a feature of the model and not a result of our proof technique, we establish, in a more specialized setting, lower bounds on the test error that exhibit the same phase transition.

\begin{thm}\label{thm:noisyphasegeneral}
In addition to the conditions of Theorem~\ref{detail-noisy-main-1-simple}, assume that $\|\bs{z}\|$ and $\bs z/\|\bs z\|$ are independent and that $\bs z/\|\bs z\|$ has a density $f$ with respect to the uniform distribution on the sphere that satisfies $0 < f_{min} \leq f \leq f_{max} < \infty$ for some constants $f_{min}, f_{max}$. Then the gradient descent iterates converge in direction to $\boldsymbol{\hat w}$, and we have
\begin{equation*}
f_{min}\cdot \kappa(1/\zeta_{\bs{\hat w}, \bs \mu}) \;\leq\; \P_{(\boldsymbol{x}, y_\n)}(\langle \boldsymbol{\hat w}, y_\n\boldsymbol{x}\rangle < 0) - \eta \;\leq\; f_{max}\cdot  \kappa(1/\zeta_{\bs{\hat w}, \bs \mu})
\end{equation*}
where the function $\kappa$ is defined as 
$
\kappa(t) := \tfrac{1-2\eta}{2} \P_{(\bs z, u_1)}\left(\|\bs z\| |u_1| > t \right) 
$
where $u_1$ is the first component of a uniform distribution on the sphere that is independent of $\|\bs z\|$ and $\zeta_{\bs{\hat w}, \bs \mu}>0$ satisfies
\[
\zeta_{\bs{\hat w}, \bs \mu}^2 \asymp \frac{1}{(1-2\eta)^2} \left\{\eta n\rho + \frac{1}{\|\bs \mu\|^2} + \frac{1}{n\rho \|\boldsymbol{\mu}\|^4}\right\}.
\]
\end{thm}

The proof of this result is given in Section~\ref{sec:proof:thm:noisyphasegeneral}.

Next, we specialize our general results from the previous theorems to model \ref{model:EM}, demonstrating that the high-level conditions we imposed are met in a rich class of models. Recall the definition of the constant $C_2(g)$ in~\eqref{eq:defconstantsEM}.

\begin{thm}\label{thm:noisy-ext-detail1-simple}
Consider model \ref{model:EM} with $\eta \in (0, \tfrac12)$. There exists a constant $C$ depending only on the distribution of $g, \bs \xi$ such that, if 
\begin{equation}\label{cond:noisy-ext-detail1-simple}
\|\boldsymbol{\mu}\|^2\; \geq C
\;\max\{\tfrac{1}{\eta},\tfrac{1}{1-2\eta}\}\delta^{-1/2}\|\bs{\mu}\|_\Sigma,
\end{equation} 
and
\begin{equation}\label{cond2:noisy-ext-detail1-simple}
\Tr(\Sigma)\; \geq\; \tfrac{C}{\eta}(\tfrac{n}{\delta})^{2/r+1/\ell} n^{1/2}\max\{p^{2/r-1/2}, n^{2/r}\}\|\Sigma\|_{\textsc{f}} 
\end{equation}
together with  one of the following holds:
\begin{enumerate}
    \item[(i)] $\Tr(\Sigma) \geq C(\tfrac{n}{\delta})^{1/2+1/\ell} n\|\bs{\mu}\|_\Sigma$,
    \item[(ii)] 
    $\|\boldsymbol{\mu}\|^2 \geq \tfrac{C}{\min(\eta,1-2\eta)}  (\tfrac{n}{\delta})^{1/2+1/\ell}\|\bs{\mu}\|_\Sigma \quad$   and  
$\quad \Tr(\Sigma) \geq \tfrac{C}{\min(\eta,1-2\eta)}(\tfrac{n}{\delta})^{1+1/\ell}n^{1/2}\frac{\|\bs{\mu}\|_\Sigma^2}{\|\boldsymbol{\mu}\|^2}$
\end{enumerate}
then, for $n \geq \delta^{-\frac{2}{k-2}}(\tfrac{16 C_2(g)}{\min\{\eta,1-2\eta\}})^{\tfrac{k}{k-2}}$, with probability at least $1 - 5\delta$, the gradient descent iterates converge in direction to $\boldsymbol{\hat w}$, and the following test error bound holds:
    \begin{equation*}
        \P_{(\boldsymbol{x}, y)}(\langle \boldsymbol{\hat w}, y\boldsymbol{x}\rangle < 0) \leq \eta + \frac{c \|\Sigma\|}{(1 - 2\eta)^2} \left(\eta\frac{\E[g^{-2}] n}{\Tr(\Sigma)} + \frac{1}{\|\bs \mu\|^2} + \frac{\Tr(\Sigma)}{\E[g^{-2}]n \|\boldsymbol{\mu}\|^4}\right),
    \end{equation*}
    where $c$ is a universal constant.
\end{thm}

The proof of Theorem~\ref{thm:noisy-ext-detail1-simple} is given in Section~\ref{sec:proof:thm:noisy-ext-detail1-simple} of the supplement.

Next, we briefly discuss the conditions of the above result in the simple setting $\Sigma = I_p$, which can be obtained from Theorem~\ref{thm:noisy-ext-detail1-simple} by plugging in $\Sigma = I_p, \ell = \infty, k=4, g \equiv 1$ and rewriting the conditions and the final bound in there: 

\begin{cor}\label{cor:sigmaid} Consider model \ref{model:EM} with $\eta \in (0,\tfrac12)$ and assume that $\Sigma = I_p$. Then a sufficient condition for benign overfitting is that 
\[
\|\bs \mu\|\gg (\tfrac{p}{n})^{1/4} \quad \text{and} \quad p\gtrsim \max\left\{n^\frac{4+(1+2/\ell)r}{2(r-2)}, n^{\frac{8}{r}+ 1 + \frac{2}{\ell}}\right\}.
\]    
\end{cor}
The form of the conditions in the above corollary implies that there will always be a non-empty range of possible values for $p$ provided that $\|\bs \mu\|$ is sufficiently large. If all entries of $\bs \mu$ take the same non-zero value, this boils down to requiring that $p$ be sufficiently large. A proof of this corollary is provided in Section~\ref{sec:proof:cor:sigmaid}.

We conclude this section by providing a more precise form for the test error under model \ref{model:EM}, implying lower and upper bounds for the test error in a specialized setting. Recall the definition of constant $C_2(g)$ in~\eqref{eq:defconstantsEM}.

\begin{thm}\label{thm:phasetransitionsimple}
Suppose \ref{model:EM} holds in model \ref{model:M} such that $\bs\xi\sim \mathcal{N}(0,I_p)$ and $\Sigma$ has largest and smallest eigenvalues $\lambda_{max}, \lambda_{min} > 0$, respectively. Define $C_{\eta,\delta} := \min\{\eta,(1-2\eta)\}^{-1} \delta^{-1-1/\ell} \sqrt{\log \delta^{-1}}$. For any $\delta \in (0, 1/5]$, there exists a sufficiently large constant $C$ depending only on the distribution of $g$ such that provided 
\[
p\geq C^2 C_{\eta,\delta}^2 \Big(\frac{ \lambda_{max}}{\lambda_{min}}\Big)^2 n^{2+2/\ell} \quad \text{and} \quad \|\bs\mu\| \geq C C_{\eta,\delta} \lambda_{max},
\]
we have for $n \geq \delta^{-\frac{2}{k-2}}\left(\tfrac{16C_2(g)}{\min\{\eta, 1-2\eta\}}\right)^{\frac{k}{k-2}}$, with probability at least $1-5\delta$, 
\[
\kappa(\lambda_{min}^{-1/2}\zeta_{\bs{\hat w}, \bs \mu}^{-1}) \;\leq\; \P_{(g, \bs \xi, y_\n)}(\langle \boldsymbol{\hat w}, y_\n\boldsymbol{x}\rangle < 0) - \eta \;\leq\; \kappa(\lambda_{max}^{-1/2}\zeta_{\bs{\hat w}, \bs \mu}^{-1}).
\]
Here 
$
\kappa(t) := \tfrac{1-2\eta}{2} \P\left(g |\xi_1| > t \right).
$ 
and $\zeta_{\bs{\hat w}, \bs \mu}>0$ satisfies
\[
c^{-1}\frac{\lambda_{min}\wedge 1}{\lambda_{max}\vee 1} \leq
\zeta_{\bs{\hat w}, \bs \mu}^2 \Big\{\frac{1}{1-2\eta}\Big(\eta \frac{n\E[g^{-2}]}{p} + \frac{1}{\|\bs\mu\|^2} + \frac{p}{n\E[g^{-2}]\|\bs \mu\|^4} \Big)\Big\}^{-1} \leq  c \frac{\lambda_{max}\vee 1}{\lambda_{min}\wedge 1} 
\]
for a universal constant $c$.
\end{thm}

The proof of this result is provided in Section~\ref{sec:proof:thm:phasetransitionsimple}. For simplicity, we have not included the precise form of the constant $C$ here. In the proof, a more precise statement, including an explicit form for $C$, is given. In the proof, we also provide a more general (and complex) statement under slightly weaker assumptions. We have opted for giving a simpler statement here for the sake of brevity. 

Theorem~\ref{thm:phasetransitionsimple} allows us to characterize precisely when benign overfitting occurs. It is also sufficiently general to identify cases when benign overfitting provably fails. In the following discussion, assume that $\eta \in (0,\tfrac12)$ is fixed, $g$ has a fixed distribution independent of $n$, that $\lambda_{min}, \lambda_{max}$ stay bounded away from zero and infinity, while $\bs \mu, p$ are interpreted as sequences indexed in $n$. Since the distribution of $g|\xi_1|$ remains fixed, the upper and lower bounds on the term $\P_{(g, \xi,y_\n)}\left(g|\xi_1| > \zeta_{\bs{\hat w}, \bs \mu}\right) - \eta$ converge to zero if and only if $\zeta_{\bs{\hat w}, \bs \mu} \to 0$. Under the assumption $p\geq C^2 C_{\eta,\delta}^2 \big(\tfrac{\lambda_{max}}{\lambda_{min}}\big)^2 n^{2+2/\ell}$ this holds if and only if $p/(n\|\bs\mu\|^4) \to 0$. To summarize, benign overfitting provably occurs if $p \gg n^{2+2/\ell}$ and $\|\bs\mu\| \gg (p/n)^{1/4}$ and provably fails if $p \gg n^{2+2/\ell}$ and $\|\bs\mu\| \lesssim (p/n)^{1/4}$.

\subsection{Proof Sketch of Theorem~\ref{thm:noiseless-main} and Theorem~\ref{detail-noisy-main-1-simple}}\label{sec:sketch}

In this subsection, we provide the proof sketch of Theorem~\ref{thm:noiseless-main} and \ref{detail-noisy-main-1-simple}. Complete proofs are given in Section~\ref{sec:proofgennoiselessthm} and \ref{sec:proof:detail-noisy-main-1-simple} in the supplement.

We start from the following Lemma~which provides a general test error bound not limited to the maximum margin classifier:

\begin{lemm} \label{testerror} 
If $\boldsymbol{w}\in \R^p$ is such that $\langle \boldsymbol{w}, \boldsymbol{\mu}\rangle >0$, then
\[
\mathbb{P}\left( \langle \boldsymbol{w}, y_\n \boldsymbol{x} \rangle  <0 \right) \;\; \leq\;\; \eta + \frac{1-\eta}{2}\|\E [\boldsymbol{z}\boldsymbol{z}^\top]\|\left(\frac{\|\boldsymbol{w} \|}{\langle \boldsymbol{w}, \boldsymbol{\mu} \rangle}\right)^2.
\]
If we further assume that $\bs z$ is a sub-Gaussian random vector, we have
\[
\mathbb{P} \left( \langle \boldsymbol{w}, y_\n \boldsymbol{x} \rangle  <0 \right)  \;\;\leq\;\; \eta + (1-\eta)\exp\left\{- c\left(\frac{\langle \boldsymbol{w}, \boldsymbol{\mu} \rangle}{\|\boldsymbol{w} \|\|\boldsymbol{z}\|_{\psi_2}}\right)^2\right\},
\]
where $\|\boldsymbol{z}\|_{\psi_2}$ is the sub-Gaussian norm\footnote{see Appendix~\ref{sec:recoveringSubGauss} for a formal definition of the sub-Gaussian norm} of $\bs z$ and $c$ is a universal constant.
\end{lemm}
Since the proof is elementary, we provide it here.
\begin{proof}
By a conditioning argument, we have
$\mathbb{P} \left( \langle \boldsymbol{w}, y_\n \boldsymbol{x} \rangle  <0 \right) \leq \eta + (1 - \eta)\mathbb{P}\left( \langle \boldsymbol{w}, y\boldsymbol{x} \rangle  <0 \right)$. Thus,
$$
    \mathbb{P} \left( \langle \boldsymbol{w}, y\boldsymbol{x} \rangle  <0 \right) \;=\;  \mathbb{P} \left( y\langle \boldsymbol{w}, \boldsymbol{z} \rangle  < - \langle \boldsymbol{w}, \boldsymbol{\mu} \rangle \right) \;=\; \frac{1}{2} \mathbb{P} \left( \langle \boldsymbol{w}, \boldsymbol{z} \rangle  < - \langle \boldsymbol{w}, \boldsymbol{\mu} \rangle \right) + \frac{1}{2} \mathbb{P} \left( -\langle \boldsymbol{w}, \boldsymbol{z} \rangle  < - \langle \boldsymbol{w}, \boldsymbol{\mu} \rangle \right).
$$
From this we conclude $\mathbb{P} \left( \langle \boldsymbol{w}, y\boldsymbol{x} \rangle  <0 \right)=\frac{1}{2}\mathbb{P} \left( |\langle \boldsymbol{w}, \boldsymbol{z} \rangle|  >  \langle \boldsymbol{w}, \boldsymbol{\mu} \rangle \right)$. By the Markov inequality, 
\begin{equation}\label{eq:markov-aux}
\mathbb{P} \left( \langle \boldsymbol{w}, y\boldsymbol{x} \rangle  <0 \right)\;\leq\; \frac{1}{2} \frac{\mathbb{E} \langle \boldsymbol{w},\boldsymbol{z} \rangle^2}{\langle \boldsymbol{w}, \boldsymbol{\mu} \rangle^2} \; =\; \frac{1}{2} \frac{\boldsymbol{w}^\top \E [\boldsymbol{z}\boldsymbol{z}^\top] \boldsymbol{w}}{\langle \boldsymbol{w}, \boldsymbol{\mu} \rangle^2} \;\leq\; \frac{1}{2} \frac{\|\boldsymbol{w} \|^2 \|\E [\boldsymbol{z}\boldsymbol{z}^\top]\|}{\langle \boldsymbol{w}, \boldsymbol{\mu} \rangle^2}.    
\end{equation}
If $\boldsymbol{z}$ is sub-Gaussian, instead, we have $\mathbb{P} \left( \langle \boldsymbol{w}, y\boldsymbol{x} \rangle  <0 \right) \;\leq\; \exp\left\{- c\left(\frac{\langle \boldsymbol{w}, \boldsymbol{\mu} \rangle}{\|\boldsymbol{w} \|\|\boldsymbol{z}\|_{\psi_2}}\right)^2\right\}$.
\end{proof}

With Lemma~\ref{testerror}, the proofs of Theorem~\ref{thm:noiseless-main} and Theorem~\ref{detail-noisy-main-1-simple} are reduced to obtaining bounds on $\tfrac{\|\boldsymbol{\hat w} \|}{\langle \boldsymbol{\hat w}, \boldsymbol{\mu} \rangle}$. However, the maximum margin classifier $\bs{\hat w}$ does not have a closed form in general, and hence it is difficult to analyze quantities that involve it. In the setting of Theorem~\ref{thm:noiseless-main}(i) and Theorem~\ref{detail-noisy-main-1-simple} this difficulty can be resolved by establishing the equivalence of the maximum margin classifier $\bs{\hat w}$ and the minimum norm least square estimator $\bs w_{\rm{LS}}$. In fact, we show in Lemma~\ref{lem:whatrepr} and Lemma~\ref{lem:whatreprnoisy} that, under near orthogonality of $\bs z_i$'s to each other and to $\bs \mu$, we have
\[
\bs{\hat w} = \bs w_{\rm{LS}} = X^\top(XX^\top)^{-1}\bs y_\n.
\]

Thus, our argument is further reduced to a careful analysis of the gram matrix $XX^\top$. This is done through the observation that 
\[
XX^\top = ZZ^\top + \bs y (Z\bs \mu)^\top + (Z\bs \mu)\bs y^\top + \|\bs \mu\|^2 \bs y \bs y^\top
\]
and so $XX^\top$ is obtained from $ZZ^\top$ by perturbing it with a low rank matrix. This observation enables us to obtain a closed form expression for $(XX^\top)^{-1}$. With the closed form expression, we can obtain the following useful expansion of $(XX^\top)^{-1}\bs y_\n$ which will be used repeatedly in the analysis of $\tfrac{\|\boldsymbol{\hat w} \|}{\langle \boldsymbol{\hat w}, \boldsymbol{\mu} \rangle}$:

\begin{lemm}\label{quad-decomp} 
If $XX^\top$ is invertible, then
$$
     (XX^\top)^{-1}\boldsymbol{y}_{\n} \;=\; d^{-1}A^{-1}\left[d  \boldsymbol{y}_{\n} -\left\{s_{\n} (\|\boldsymbol{\mu}\|^2 -t) + h_{\n} (1+h)\right\}  \boldsymbol{y}  + (h_{\n} s - s_{\n}(1+h)) \boldsymbol{\nu} \right].
$$
In particular, in the noiseless case, we have
\[ (X X^\top)^{-1}\bs y \;=\; d^{-1}A^{-1} \left[ (1 + h ) \boldsymbol{y} - s \boldsymbol{\nu}\right]. \]
\end{lemm}

Here, $A, d, s, s_\n, t, h, h_\n, \nu$ are defined by 
\begin{equation*}
    \begin{aligned}
    & A=ZZ^\top, \quad \boldsymbol{\nu}=Z\boldsymbol{\mu},\quad s = \boldsymbol{y}^\top A^{-1} \boldsymbol{y},\quad s_{\n} = \boldsymbol{y}_{\n}^\top A^{-1} \bs y, \quad t=\boldsymbol{\nu}^\top A^{-1}\boldsymbol{\nu},  \\
    & h = \boldsymbol{y}^\top A^{-1} \boldsymbol{\nu}, \quad h_{\n} = \boldsymbol{y}_{\n}^\top A^{-1} \boldsymbol{\nu}, \quad d = s(\|\bs \mu\|^2 - t) + (1+h)^2
    \end{aligned}
\end{equation*}
provided $A$ is invertible. An important ingredient of our proofs is the observation that many of the quadratic forms above are invariant under rescaling. For example, $s = \boldsymbol{y}^\top A^{-1} \boldsymbol{y} = \check{\boldsymbol{y}}^\top (\check Z \check Z^\top)^{-1} \check{\boldsymbol{y}}$ with the "check" notation as defined in \eqref{eq:defcheck}. This invariance was not utilized in prior works. It allows us to remove some of the assumptions on the individual norms $\|\bs z_i\|$ that were needed for previous results.

Therefore, the study of $\bs{\hat w}$ involves a careful analysis of all of these terms in order to obtain sharp lower and upper bounds of $\tfrac{\|\boldsymbol{\hat w} \|}{\langle \boldsymbol{\hat w}, \boldsymbol{\mu} \rangle}$. In Lemma~\ref{quad-bounds-s} - \ref{quad-bounds-nuae}, we establish bounds of each of these quantities under events $E_1, \ldots, E_5$. Those bounds are utilized in subsequent proofs in Section~\ref{sec:proof}. 

We also note that the expansion of $(XX^\top)^{-1}\bs y_\n$ (noisy case) is far more involved than $(XX^\top)^{-1}\bs y$ (noiseless case). This explains why the noisy case has been substantially less well studied in the existing literature. Establishing the $\bs{\hat w} = \bs w_{\rm LS}$ equivalence in noisy case requires precise concentration results, including upper and lower bounds, on some of the quantities $A, d, s, s_\n, t, h, h_\n, \nu$. For instance, while it is easy to see $s$ and $s_\n$ are of the same order, having the same order is not enough for a precise analysis in the noisy case. 

Lastly, a careful analysis of $\bs{\hat w}$ allows us to establish the following result under near orthogonality of $\bs z_i$'s to each other and to $\bs \mu$ (see Lemma~\ref{noiseless-mmc-bound1} under the assumptions of Theorem~\ref{thm:noiseless-main}(i)  and Lemma~\ref{noisy-mmc-bound1} under the assumptions of Theorem~\ref{detail-noisy-main-1-simple} for the precise statements and proofs): 

\begin{enumerate}
    \item[(Noiseless case)] 
    \begin{equation}\label{eq:noiseless-mmc-bound-sketch}
    \left(\frac{\|\boldsymbol{\hat w}\|}{\langle \boldsymbol{\hat w}, \boldsymbol{\mu} \rangle}\right)^2 \; \asymp \; \frac{1}{\|\boldsymbol{\mu}\|^2} + \frac{1}{n\rho \|\boldsymbol{\mu}\|^4}.
    \end{equation}
    \item[(Noisy case)]
    \begin{equation}\label{eq:noisy-mmc-bound-sketch}
    \left(\frac{\|\boldsymbol{\hat w}\|}{\langle \boldsymbol{\hat w}, \boldsymbol{\mu} \rangle}\right)^2 \; \asymp \; \frac{1}{(1-2\eta)^2}\left\{ n\rho\eta + \frac{1}{\|\bs\mu\|^2} + \frac{1}{n\rho \|\boldsymbol{\mu}\|^4} \right\}.
    \end{equation}
\end{enumerate}

For the large signal regime in the noiseless case, that is (ii) of Theorem~\ref{thm:noiseless-main}, equality between the maximum margin classifier and the least squares estimator cannot be guaranteed. In fact, by taking $\|\bs \mu\| > C$ and $\bs z \sim \rm{Unif}\{\pm \bs \mu/\|\bs \mu\|\}$, condition (ii) of Theorem~\ref{thm:noiseless-main} is met with probability one while $\bs{\hat w} \neq \bs w_{\rm LS}$ with high probability since $\bs{\hat w} = (1 - 1/\|\bs \mu\|)\bs \mu$ holds with probability $1-2^{-n}$ and $\bs w_{\rm LS} \neq (1 - 1/\|\bs \mu\|)\bs \mu$ unless $y_i\bs x_i = (1 - 1/\|\bs \mu\|)\bs \mu$ for all $i$. We also note that condition (ii) of Theorem~\ref{thm:noiseless-main} does not appear in previous work although the noiseless case has been relatively well studied compared to the noisy case. Moreover, the proof of the expansion~\eqref{eq:noiseless-mmc-bound-sketch} in this regime is based on a purely geometric argument that is different from arguments in the existing literature; see Section~\ref{sec:proofgennoiselessthm(ii)}. 

In summary, from Lemma~\ref{testerror} and the bounds \eqref{eq:noiseless-mmc-bound-sketch}, \eqref{eq:noisy-mmc-bound-sketch}, we obtain the bounds in Theorem~\ref{thm:noiseless-main} and Theorem~\ref{detail-noisy-main-1-simple}. 

We note that the majority of the arguments here are done in a deterministic manner under events $E_1, \ldots, E_5$. The separation of deterministic and probabilistic arguments is what allowed us to identify the key geometric aspects of the data structure, which resulted in significantly relaxing distributional assumption. See Section~\ref{sec:extendnonsG} for a detailed argument ensuring these events hold with high probability under model \ref{model:EM}. In establishing those results, we utilize probabilistic tools such as the Bahr-Esseen inequality (\cite{10.1214/aoms/1177700291}) that work under minimal moment assumptions and were not used in prior works. This allows us to significantly relax existing moment assumptions.

\section{Geometry behind Over-parametrization \& Phase Transition}\label{geometry}

In this section, we complement our results by offering geometric insights into benign overfitting under over-parametrization in a special case. We start with providing a coarse geometric picture which illustrates the differences and similarities between the noisy and noiseless cases on a high level, but fails to explain some of the more nuanced features of the problem which are discussed in later sections.

It is convenient to define $\bar{\bs x}_i := y_{\n,i} \bs x_i$ and $\bar{\bs z}_i := y_{\n, i}\bs z_i$. With this notation, we have $$
\bar{\bs x}_i \;=\; y_{\n, i}y_i \bs \mu + \bar{\bs z}_i \;=\; 
\begin{cases}
\bs \mu + \bar{\bs z}_i & \text{for all clean samples,} \\
-\bs \mu + \bar{\bs z}_i & \text{for all noisy samples,}
\end{cases}
$$
that is, all the clean samples spread around $\bs \mu$ and the noisy samples spread around $-\bs \mu$. Furthermore, the optimization problem \eqref{eq:mmc} becomes
\begin{equation}\label{eq:mmcgeom}
\bs{\hat w} = \argmin \|\bs w\|^2, \quad \text{subject to } \<\bs w, \bar{\bs x}_i\> \geq 1 \text{ for all $i=1,2, \dots, n.$}
\end{equation}

To simplify the arguments, we assume that $\|\bs z_i\|\approx \rho^{-1/2}$ for all $i$.\footnote{ This assumption is stronger than event $E_4$, which requires $\|\bs z_i\|\approx \rho^{-1/2}$ ``on average’’. However, this additional uniform norm assumption is likely to be met with high probability, for instance, in model \ref{model:EM} with $g\equiv 1$ (as shown in Lemma~\ref{lemm1-ex1}).} Consequently, all clean samples $\bar{\bs x}_i$ lie close to the sphere $\bs \mu + \rho^{-1/2}S^{p-1}$, while all noisy samples $\bar{\bs x}_i$ lie close to the sphere $-\bs \mu + \rho^{-1/2}S^{p-1}$. In this situation benign overfitting should occur if ``most'' of the sphere $\bs \mu + \rho^{-1/2}S^{p-1}$ lies in the set $\{\bs u \in \R^{p} : \bs u^\top \bs{\hat w}\geq 0\}$. More formally, ``most'' is in terms of the normalized surface area (if the $\bar{\bs z}_i$ are exactly uniformly distributed on the sphere $\rho^{-1/2}S^{p-1}$, this is an exact statement). For the geometric intuition that follows, it is useful to note that $\langle \hat{\bs w}, \bs\mu \rangle > 0$; this will become evident from the derivations in the sections below. We now consider the noiseless case and the noisy case separately. 

\smallskip

\noindent\textbf{Noiseless case:}

For the noiseless case, all samples lie near the sphere $\bs \mu + \rho^{-1/2}S^{p-1}$. Thus, it should be intuitively clear that benign overfitting occurs if the signal $\bs \mu$ is very large. When the signal is larger than the radius of the sphere, i.e. when $\|\bs \mu\|\gg \rho^{-1/2}$, the sphere $\bs \mu +\rho^{-1/2}S^{p-1}$ is away from the origin, and  hence $\bs \mu + \rho^{-1/2}S^{p-1}$ should lie in the set $\{\bs u \in \R^{p} : \bs u^\top \bs{\hat w}\geq 0\}$; see Figure~\ref{fig:geom}(a).

When the signal is very small compared to the radius of the sphere, i.e. $\|\bs \mu\| \ll \rho^{-1/2}$, the origin is located inside the sphere $\bs \mu + \rho^{-1/2}S^{p-1}$. Moreover, the origin and the center of the sphere are very close compared to the radius of the sphere. Intuition seems to suggest that non-negligible portion of the sphere lies outside of the set $\{\bs u \in \R^{p} : \bs u^\top \bs{\hat w}\geq 0\}$, no matter what the direction of $\bs{\hat w}$ is; see Figure~\ref{fig:geom}(b). It is thus surprising that benign overfitting can occur in this regime. 

\smallskip

\noindent\textbf{Noisy case:}

Next, we consider the noisy case. Recall that, for the noisy case, all clean samples lie near the sphere $\bs \mu + \rho^{-1/2}S^{p-1}$ while all noisy samples lie near the sphere $-\bs \mu + \rho^{-1/2}S^{p-1}$. Thus, for the noisy case, even if the signal $\bs \mu$ is very large, a non-negligible portion of the sphere $\bs \mu + \rho^{-1/2}S^{p-1}$ should lie "outside" of the set $\{\bs u \in \R^{p} : \bs u^\top \bs{\hat w}\geq 0\}$; see Figure~\ref{fig:geom}(d). This is because the maximum margin hyperplane $\{\bs u \in \R^p: \bs u^\top \bs{\hat w} = 1\}$ needs to intersect both spheres to interpolate all the clean and noisy samples, which lie on the sphere $\bs \mu + \rho^{-1/2}S^{p-1}$ and $-\bs \mu + \rho^{-1/2}S^{p-1}$, respectively. Moreover, it is also intuitive that $\bs{\hat w}$ becomes more and more orthogonal to $\bs \mu$ as $\|\bs \mu\|$ increases and the spheres move further apart. In this sense, it is even more surprising that benign overfitting can still occur. While this observation poses a somewhat counterintuitive picture, it does explain why the behavior of $\bs{\hat w}$, and hence the test error bound, is completely different from that of the noiseless case when the signal is strong.

When the signal $\bs \mu$ is very small, the behavior of $\bs{\hat w}$ is similar to the noiseless case since the spread of the spheres dominates the shift in their centers, so the samples are almost from one single sphere centered at the origin; see Figure~\ref{fig:geom}(c). This geometric intuition coincides with our theoretical findings that show similar test error bounds for the noisy and noiseless case in the weak signal regime. 
\begin{figure}[!htp]
    \centering
    \begin{subfigure}[b]{0.3\textwidth}
        \centering
        \begin{tikzpicture}
            \node[anchor=south west, inner sep=0] (image1) at (0,0) {\includegraphics[
            width=\linewidth,
            height=4cm,
            keepaspectratio
            ]{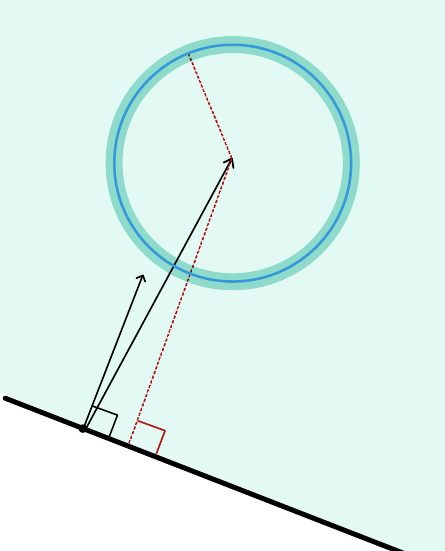}};
            \begin{scope}[overlay,x={(image1.south east)},y={(image1.north west)}]
                \node[anchor=north] at (0.23, 0.55) {\textcolor{black}{\large$\hat{\bs w}$}};
                \node[anchor=north] at (0.53, 0.88) {\textcolor{purple}{{\scriptsize$\rho^{-1/2}$}}};
                \node[anchor=north] at (0.16, 0.22) {\textcolor{black}{\large$\bs{0}$}};
                \node[anchor=north] at (0.6, 0.74) {\textcolor{black}{\large$\bs \mu$}};
                \node[anchor=north] at (0.5, 0.44) {\textcolor{purple}{$\frac{\<\hat{\bs w},\bs\mu\>}{\|\hat{\bs w}\|}$}};
                \node[anchor=north] at (0.75, 0.22) {\textcolor{teal}{$\<\hat{\bs w},\bs u\>\geq0$}};
            \end{scope}
        \end{tikzpicture}
        \caption{Noiseless \& Strong signal}
    \end{subfigure}
    \begin{subfigure}[b]{0.33\textwidth}
        \centering
        \begin{tikzpicture}
            \node[anchor=south west, inner sep=0] (image2) at (0,0) {\includegraphics[
            width=\linewidth,
            height=4cm,
            keepaspectratio
            ]{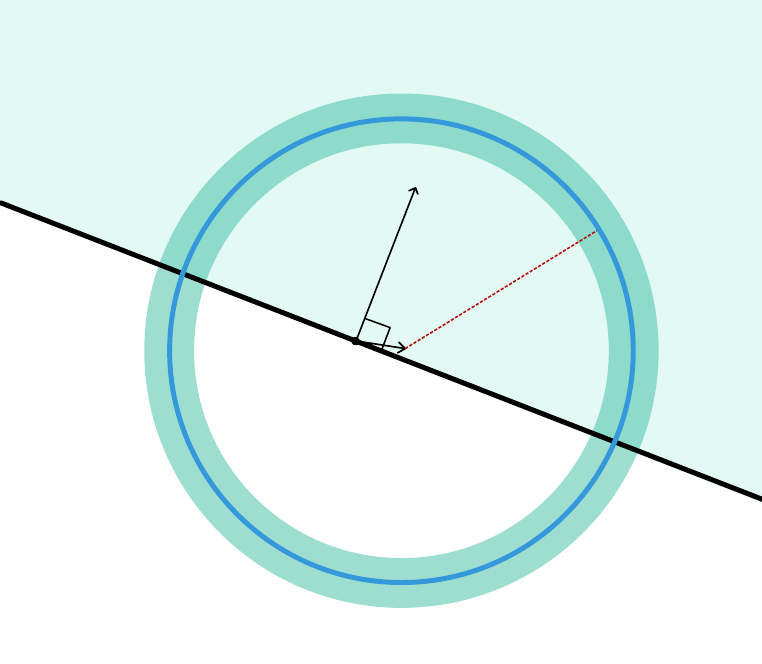}};
            \begin{scope}[overlay,x={(image2.south east)},y={(image2.north west)}]
                \node[anchor=north] at (0.59, 0.78) {\textcolor{black}{$\hat{\bs w}$}};
                \node[anchor=north] at (0.66, 0.68) {\textcolor{purple}{{ $\rho^{-1/2}$}}};
                \node[anchor=north] at (0.44, 0.5) {\textcolor{black}{$\bs{0}$}};
                \node[anchor=north] at (0.57, 0.52) {\textcolor{black}{$\bs \mu$}};
                \node[anchor=north] at (0.78, 0.95) {\textcolor{teal}{$\<\hat{\bs w},\bs u\>\geq0$}};
            \end{scope}
        \end{tikzpicture}
        \caption{Noiseless \& Weak signal}
    \end{subfigure}
    \begin{subfigure}[b]{0.33\textwidth}
        \centering
        \begin{tikzpicture}
            \node[anchor=south west, inner sep=0] (image3) at (0,0) {\includegraphics[
            width=\linewidth,
            height=4cm,
            keepaspectratio
            ]{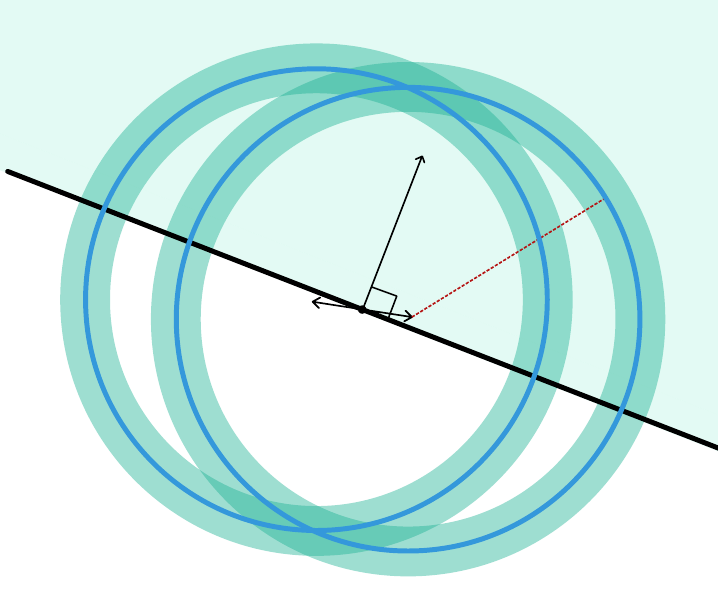}};
            \begin{scope}[overlay,x={(image3.south east)},y={(image3.north west)}]
                \node[anchor=north] at (0.63, 0.8) {\textcolor{black}{$\hat{\bs w}$}};
                \node[anchor=north] at (0.69, 0.69) {\textcolor{purple}{{ $\rho^{-1/2}$}}};
                \node[anchor=north] at (0.49, 0.5) {\textcolor{black}{$\bs{0}$}};
                \node[anchor=north] at (0.62, 0.52) {\textcolor{black}{$\bs \mu$}};
                \node[anchor=north] at (0.39, 0.54) {\textcolor{black}{$-\bs \mu$}};
                \node[anchor=north] at (0.8, 0.97) {\textcolor{teal}{$\<\hat{\bs w},\bs u\>\geq0$}};
            \end{scope}
        \end{tikzpicture}
        \caption{Noisy \& Weak signal}
    \end{subfigure}
    \begin{subfigure}[b]{0.9\textwidth}
        \centering
        \begin{tikzpicture}[scale=0.9]
            \node[anchor=south west, inner sep=0] (image4) at (0,0) {\includegraphics[
            width=\linewidth,
            height=4cm,
            keepaspectratio
            ]{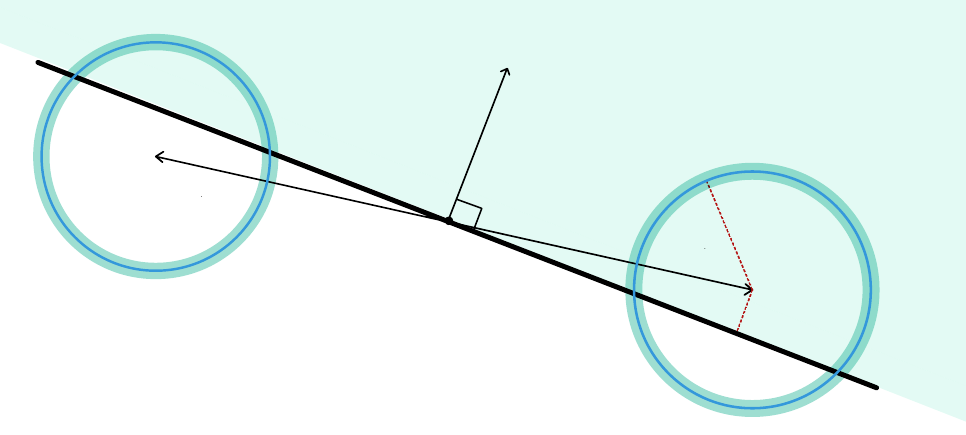}};
            \begin{scope}[overlay,x={(image4.south east)},y={(image4.north west)}]
                \node[anchor=north] at (0.55, 0.9) {\textcolor{black}{$\hat{\bs w}$}};
                \node[anchor=north] at (0.76, 0.57) {\textcolor{purple}{{ $\rho^{-1/2}$}}};
                \node[anchor=north] at (0.46, 0.49) {\textcolor{black}{\large$\bs{0}$}};
                \node[anchor=north] at (0.13, 0.68) {\textcolor{black}{\large$-\bs \mu$}};
                \node[anchor=north] at (0.8, 0.38) {\textcolor{black}{\large$\bs \mu$}};
                \node[anchor=north] at (0.85, .9) {\textcolor{teal}{\large$\<\hat{\bs w},\bs u\>\geq 0$}};
            \end{scope}
        \end{tikzpicture}
        \caption{Noisy \& Strong signal}
    \end{subfigure}
    \caption{The observations $\bar{\bs x}_i = y_i\bs x_i$ are concentrated near the spheres $\pm \bs\mu+\rho^{-1/2}S^{p-1}$. Except for the noiseless \& strong signal case, a non-negligible proportion of the sphere $\bs \mu + \rho^{-1/2}S^{p-1}$ seems to lie outside of the shaded half-space.}
    \label{fig:geom}
\end{figure}

The above naive analysis provides an intuition behind some of the main differences and similarities between the noisy and noiseless model. It seems to suggest that benign overfitting can not occur for smaller signals, but this seeming contradiction is resolved in Section~\ref{sec:blow-up} by arguments from high-dimensional probability. The discussion so far also does not explain why the test error bounds in both cases experience a phase transition at $\|\bs\mu\| \approx \rho^{-1/2}n^{-1/2}$ (see Theorems~\ref{thm:noiseless-main},~\ref{detail-noisy-main-1-simple}) instead of $\rho^{-1/2}$ which would be more intuitive. A more detailed geometric analysis of the reasons behind this phase transition is provided in Section~\ref{sec:phasetransenoiseless} for the noiseless and Section~\ref{sec:phasetransnoisy} for the noisy model.

\subsection{Blow up phenomenon}\label{sec:blow-up}

We begin by addressing the seemingly paradoxical observation made above, by explaining why benign overfitting could occur even when a non-negligible proportion of the sphere $\bs \mu + \rho^{-1/2}S^{p-1}$ seems to lie outside of the set $\{\bs u \in \R^{p} : \bs u^\top \bs{\hat w}\geq 0\}$.

For the following discussion, it is useful to recognize that, if $\<\hat{\bs w},\bs \mu\>\geq 0$, the amount $d_{\bs{\hat w}}:= \tfrac{\<\bs{\hat w}, \bs \mu\>}{\|\bs{\hat w}\|}$ represents the distance of $\boldsymbol{\mu}$ to the hyperplane $\{\bs u:\langle \boldsymbol{\hat w}, \bs u\rangle = 0\}$ (see Figure~\ref{fig:blowup}). It is thus intuitively clear that the test error will be governed by the behavior of $d_{\bs{\hat w}}$. We will repeatedly refer to this intuition in subsequent sections. Note that this intuition can be formalized, compare Lemma~\ref{testerror} where the test error bound depends precisely on $d_{\bs{\hat w}}$.

\begin{figure}[!htp]
\begin{tikzpicture}
  \node[anchor=south west, inner sep=0] (image) at (0,0) {\includegraphics[width=0.6\textwidth]{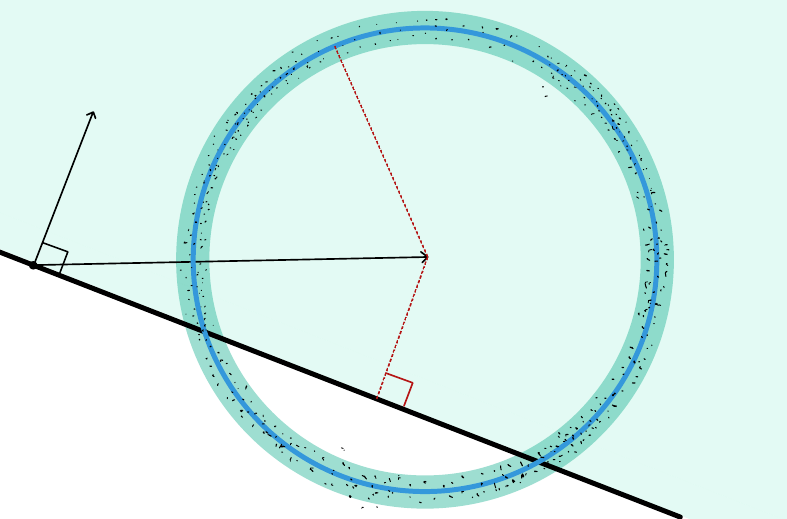}};
  \begin{scope}[x={(image.south east)},y={(image.north west)}]
    \node[anchor=north] at (0.53, 0.79) {\textcolor{purple}{{ \small$\rho^{-1/2}$}}};
    \node[anchor=north] at (0.03, 0.49) {\textcolor{black}{\large$\bs{0}$}};
    \node[anchor=north] at (0.15, 0.83) {\textcolor{black}{\large$\hat{\bs w}$}};
     \node[anchor=north] at (0.57, 0.54) {\textcolor{black}{\large$\bs \mu$}};
      \node[anchor=north] at (0.59, 0.41) {\textcolor{purple}{$d_{\bs{\hat w}}= \frac{\<\hat{\bs w},\bs\mu\>}{\|\hat{\bs w}\|}$}};
       \node[anchor=north] at (0.89, 0.14) {\textcolor{black}{\large$\<\hat{\bs w},\bs u\>=0$}};
  \end{scope}
\end{tikzpicture}
\caption{Blow-up phenomenon. The observations $y_i\bs x_i$ are concentrated around the sphere $\bs\mu+\rho^{-1/2}S^{p-1}$.  If $\tfrac{\<\bs{\hat w}, \bs \mu\>}{\|\bs{\hat w}\|}$ is big enough, a large proportion of the blue sphere lies in the shaded half-space. }
\label{fig:blowup}
\end{figure}

The goal of our analysis is to show that the sphere $\bs\mu+\rho^{-1/2}S^{p-1}$ almost entirely lies in the region defined by the inequality $\<\hat{\bs w},x\>\geq 0$ (the shaded half-space in Figure~\ref{fig:blowup}) even when this hyperplane is "close" to the origin relative to the radius of the sphere. We argue this by means of the "blow-up" phenomenon from high-dimensional probability.

\begin{lemm}[Blow-up, Lemma~5.1.7, \cite{vershynin_2018}]\label{vershynin_2018}
Let $A$ be a subset of the sphere $\sqrt{p}S^{p-1}$, and let $\sigma$ denote the normalized area on that sphere. If $\sigma(A) \geq 1/2$, then, for every $d\geq 0$, 
\[
\sigma(A_d) \geq 1 - 2 \exp (-cd^2), 
\]
where $c$ is a universal constant and the neighborhood $A_d$ is defined by
\[
A_d := \left\{\bs u \in \sqrt{p}S^{p-1}: \exists \bs v\in A \text{ such that } \|\bs u-\bs v\| \leq d \right\}.
\]
\end{lemm}
We note that the radius of the sphere is $\sqrt{p}$, and thus it is remarkable that the lower bound on $\sigma(A_d)$ does not involve the dimension $p$. As long as $A$ occupies at least half of the sphere and $d$ is large but independent of $p$, the neighborhood $A_d$ occupies the majority of the sphere no matter how large $p$ is.

Now, we will show the following which is an immediate corollary of the blow-up lemma:

\begin{cor}\label{cor:blowup}
Let $\tilde \sigma$ be the normalized area on the sphere $\bs\mu + \rho^{-1/2}S^{p-1}$. Then we have
\[
\tilde \sigma (\{\bs u: \langle \boldsymbol{\hat w}, \bs u \rangle \geq  0\}) \geq 1 - 2\exp(-cp\rho d_{\bs{\hat w}}^2),
\]
and $\tilde \sigma (\{\bs u: \langle \boldsymbol{\hat w}, \bs u \rangle \geq  0\}) \to 1$ as $\sqrt{p\rho}d_{\bs{\hat w}} \to \infty$.
\end{cor}
\begin{proof}
Let $\tilde A$ be the hemisphere on $\bs\mu + \rho^{-1/2}S^{p-1}$ defined by the half-plane $\{\bs u:\langle \bs{\hat w}, \bs u - \bs \mu\rangle \geq 0\}$. We have (cf. Figure~\ref{fig:blowup}) 
\[
H_{\bs{\hat w}} := \{\bs u\in \bs\mu+\rho^{-1/2}S^{p-1}: \langle \bs{\hat w},  \bs u\rangle \geq 0 \} \;\;\supset\;\; \tilde A_{d_{\bs{\hat w}}} :=\{\bs u\in \bs \mu+\rho^{-1/2}S^{p-1}:\;{\rm dist}(\bs u,\tilde A)\leq d_{\bs{\hat w}}\}.
\]
The affine map $\phi:  \bs u \mapsto  \sqrt{p\rho}( \bs u - \bs \mu)$ maps the set $\bs \mu + \rho^{-1/2}S^{p-1}$ to the set $ p^{1/2}S^{p-1}$. Defining $A = \phi(\tilde A)$ and $A_{d_{\bs{\hat w}}} = \phi(\tilde A_{d_{\bs{\hat w}}})$, we get 
\[
A_{d_{\bs{\hat w}}} \;=\;\{\bs u\in \sqrt{p} S^{p-1}:\;{\rm dist}(\bs u,A)\leq \sqrt{p\rho} d_{\bs{\hat w}}\}.
\]
Then, the blow-up Lemma~implies
\[
\tilde \sigma(H_{\bs{\hat w}}) \;\geq\; \tilde \sigma(\tilde A_{d_{\bs{\hat w}}}) \;=\; \sigma(A_{d_{\bs{\hat w}}}) \geq 1 - 2\exp(-cp\rho d_{\bs{\hat w}}^2)\to 1
\]
as $\sqrt{p\rho}d_{\bs{\hat w}} \to \infty$. 
\end{proof}

The condition $\sqrt{p\rho}d_{\bs{\hat w}} \to \infty$ can be equivalently stated as $d_{\bs{\hat w}}\gg \tfrac{1}{\sqrt{p}}\rho^{-1/2}$. This implies that, even when $d_{\bs{\hat w}}$ is much smaller than the radius $\rho^{-1/2}$, the majority of the sphere $\bs \mu + \rho^{-1/2}S^{p-1}$ lies in the set $\{\bs u \in \R^{p} : \bs u^\top \bs{\hat w}\geq 0\}$ as long as $d_{\bs{\hat w}}\gg \tfrac{1}{\sqrt{p}}\rho^{-1/2}$ is satisfied. This analysis illustrates that the seemingly paradoxical geometric observation was due to the failure of naive geometric intuition in high dimensions.

\subsection{Phase Transition: Noiseless Model}\label{sec:phasetransenoiseless}

In this section, we discuss the geometry behind the phase transition observed in Theorem~\ref{thm:noiseless-main} which deals with the noiseless model. Recall that the phase transition occurs at $\|\boldsymbol{\mu}\|\approx (n\rho)^{-1/2}$ and so it concerns only regime (i) in Theorem~\ref{thm:noiseless-main}. In this regime, we argued that $\Delta(y)  (XX^\top)^{-1}\bs y > \bs 0$ and consequently $\hat{\bs w}=\bs w_{\rm LS}=X^\top (XX^\top )^{-1}\bs y$ (cf. the proof of Lemma~\ref{lem:whatrepr}). This equality is the starting point of our geometric analysis in this subsection.

To show  phase transition in the fundamental quantity {$d_{\hat{\bs{w}}} = \tfrac{\<\hat{\bs w},\bs\mu\>}{\|\hat{\bs w}\|}$, we will first observe that, under complete orthogonality of $\bs z_i$'s and $\bs \mu$, the vector $\hat{\bs w}/\|\hat{\bs w}\|^2$ can be decomposed into two orthogonal vectors.  The analysis of this decomposition will explain the phase transition geometrically.

Recall that by definition we have $\bar{\bs x}_i=y_i \bs x_i$ and $\bar{\bs z}_i=y_i\bs z_i$ for $i=1,\ldots,n$ since $y_{\n, i} = y_i$ in the noiseless case. It is also convenient to let matrices $\bar X, \bar Z$ have $\bar{\bs x}_i, \bar{\bs z}_i$ as rows. With this notation $\bar{\bs x}_i=\bs \mu+\bar{\bs z}_i$ and $\hat{\bs w}=\bar X^\top (\bar X\bar X^\top)^{-1}\bs 1$, where $\bs 1\in \R^n$ is the vector of ones. The inequality $\Delta(y)(XX^\top)^{-1}\bs y> \bs 0$ translates to $(\bar X\bar X^\top)^{-1}\bs 1> \bs 0$. Note that  the hyperplane $H=\{\bs u\in \R^p:\;\<\hat{\bs w},\bs u\>=1\}$ contains all the points $\bar{\bs x}_i$  (cf. Figure~\ref{fig:blowup2}) and \begin{equation}\label{eq:wbynorm}
\frac{\hat{\bs w}}{\|\hat{\bs w}\|^2}\;=\;\frac{\bar X^\top (\bar X\bar X^\top)^{-1}\bs 1}{\bs 1^\top (\bar X\bar X^\top)^{-1}\bs 1}\;=\;\sum_{i=1}^n \alpha_i \bar{\bs x}_i,
\end{equation}
where $\alpha=(\alpha_1,\ldots,\alpha_n)=\tfrac{(\bar X\bar X^\top)^{-1}\bs 1}{\bs 1^\top (\bar X\bar X^\top)^{-1}\bs 1} $ is a vector with positive entries that sum to $1$ since $(\bar X\bar X^\top)^{-1}\bs 1> \bs 0$. In other words, the vector $\hat{\bs w}/\|\hat{\bs w}\|^2$ is a convex combination of the data points $\bar {\bs x}_i$.

\begin{figure}[!htp]
\begin{tikzpicture}
  \node[anchor=south west, inner sep=0] (image) at (0,0) {\includegraphics[width=0.6\textwidth]{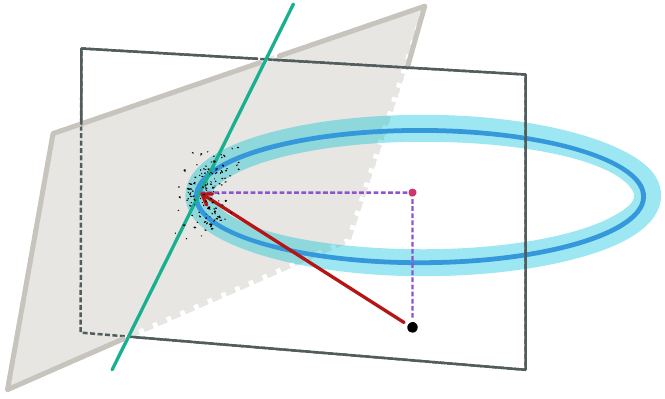}};
  \begin{scope}[x={(image.south east)},y={(image.north west)}]
    \node[anchor=north] at (0.65, 0.23) {\textcolor{black}{\Large $\bs 0$}};
    \node[anchor=north] at (0.23, 0.6) {\textcolor{black}{\Large $\textcolor{purple}{\frac{\hat{\bs w}}{\|\hat{\bs w}\|^2}}$}};
     \node[anchor=north] at (0.65, 0.57) {\textcolor{purple}{\Large $\bs \mu$}};
       \node[anchor=north] at (0.58, 1.07) {\textcolor{gray}{\large $\<\hat{\bs w},\bs u\>=1$}};
              \node[anchor=north] at (0.51, 0.605) {\textcolor{violet}{\Large $\bs z_\perp$}};
  \end{scope}
\end{tikzpicture}
\caption{{The observations $\bar{\bs x}_i=\bs \mu+\bar{\bs z}_i$ are concentrated on average around the sphere $\bs\mu+\rho^{-1/2}S^{p-1}$. Since $\bar{\bs z}_i$'s are all nearly orthogonal to $\bs \mu$, the data points actually concentrate around a smaller area depicted in blue. Since $n<\!\!\!< p$ they fill only a small subarea of this region and they all line on the hyperplane $\<\hat{\bs w},\bs u\>=1$. The decomposition $\frac{\hat{\bs w}}{\|\hat{\bs w}\|^2}=\bs\mu+\bs z_\perp$ depicted here is the fundamental geometric reason behind the phase transition.} }
\label{fig:blowup2}
\end{figure}

Looking at \eqref{eq:wbynorm} we make two observations. First, if $\|\bs \mu\|$ is much larger than the radius $\rho^{-1/2}$, we have $\bar{\bs x}_i=\bs\mu+y_i\bs z_i\approx \bs \mu$ for all $i$. In this case,   $\hat{\bs w}/\|\hat{\bs w}\|^2=\sum_{i=1}^n \alpha_i \bar{\bs x}_i$ is close to $\bs \mu$. Second, when $\|\boldsymbol{\mu}\|$ is negligible compared to the radius $\rho^{-1/2}$, then $\bar {\bs x}_i=\bs\mu+y_i\bs z_i\approx y_i\bs z_i=\bar{\bs z}_i$. In consequence, $\bar X\approx \bar Z$ and the vector $\boldsymbol{\hat w}/\|\boldsymbol{\hat w}\|^2$ should be reasonably approximated by the vector 
$$\bs z_\perp\;:=\;\frac{\bar Z^\top (\bar Z\bar Z^\top)^{-1}\bs 1}{\bs 1^\top (\bar Z\bar Z^\top)^{-1}\bs 1}.$$
Geometrically, $\bs z_\perp$ characterizes the maximum margin vector solving the optimization problem \eqref{eq:mmc} with $\bs x_i$'s replaced by $\bs z_i$'s; see Figure~\ref{fig:z_perp_decomposition} for an illustration.

We will now strengthen this geometric picture by showing that, in fact, $\boldsymbol{\hat w}/\|\boldsymbol{\hat w}\|^2$ is closely approximated by $\bs \mu+\bs z_\perp$, or equivalently, $\hat{\bs w}$ is closely approximated by $\bar{\bs w}:=\tfrac{\bs\mu+\bs z_\perp}{\|\bs\mu+\bs z_\perp\|^2}$. To show this, note that $$
\|\bs z_\perp\|^2\;=\;\frac{1}{\bs 1^\top (\bar Z\bar Z^\top)^{-1}\bs 1}\qquad\mbox{and}\qquad \bar Z\bs z_\perp\;=\;{\|\bs z_\perp\|^2}\bs 1. 
$$
The near orthogonality of $\boldsymbol{z}_i$'s  to $\boldsymbol{\mu}$ implies that $\bs z_\perp$ is also near orthogonal to $\boldsymbol{\mu}$. Consequently, $\|\bs\mu+\bs z_\perp\|^2\approx \|\bs\mu\|^2+\|\bs z_\perp\|^2$ and 
$$
\<\bs \mu+\bs z_\perp,\bar{\bs x}_i\>\;=\;\<\bs \mu+\bs z_\perp,\bs \mu+\bar{\bs z}_i\>\;\approx \;\|\bs\mu\|^2+\<\bs z_\perp,\bar{\bs z}_i\>\;=\;\|\bs\mu\|^2+\|\bs z_\perp\|^2.
$$

These approximations imply that the vector $\bar{\bs w}=\tfrac{\bs\mu+\bs z_\perp}{\|\bs \mu\|^2+\|\bs z_\perp\|^2}$ is close to a feasible point of the problem \eqref{eq:mmc}.

Under complete orthogonality we can actually prove $\bs{\hat w} = \bs{\bar w}$ rigorously via elementary arguments:

\begin{lemm}\label{lemm:whatdecomp}
Suppose that $\bs{\hat w} = X^\top (XX^\top)^{-1}\bs y$ and $\bs z_i \perp \bs z_j$ and $\bs z_i \perp \bs \mu$ for all $i\neq j$. Then, we have the following equalities 
\begin{equation*}
    \bs{\hat w}  = \bs{\bar w}, \quad \text{or equivalently~~~~$\frac{\bs{\hat w}}{\|\bs{\hat w}\|^2} = \bs \mu + \bs z_\perp$, and } %\\
    \left\|\frac{\bs{\hat w}}{\|\bs{\hat w}\|^2}\right\|^2  = \|\bs \mu\|^2 + \|\bs z_\perp\|^2.
\end{equation*}
\end{lemm}
\begin{proof}
Under the orthogonality condition $\bs z_i \perp \bs \mu$, we have $\bs z_\perp \perp \bs \mu$. Further recalling that $\bar Z \bs z_\perp = \|\bs z_\perp\|^2 \bs 1$, we have
$$
\<\bs \mu+\bs z_\perp,\bar{\bs x}_i\>\;=\;\<\bs \mu+\bs z_\perp,\bs \mu+\bar{\bs z}_i\>\;= \;\|\bs\mu\|^2+\<\bs z_\perp,\bar{\bs z}_i\>\;=\;\|\bs\mu\|^2+\|\bs z_\perp\|^2.
$$
This implies that the vector $\bar{\bs w}=\tfrac{\bs\mu+\bs z_\perp}{\|\bs \mu\|^2+\|\bs z_\perp\|^2}$ is a feasible point of the problem \eqref{eq:mmc}; $\<\bar{\bs w},\bar{\bs x}_i\>=1$ for all $i$. 

To show that $\bar{\bs w}$ is the optimal point $\hat{\bs w}$ it remains to confirm that $\|\tfrac{\hat{\bs w}}{\|\hat{\bs w}\|^2}\|^2 = \|\bs \mu\|^2+\|\bs z_\perp\|^2$. 

Denoting $D=\bar Z\bar Z^\top$ (diagonal matrix with $\|\bs z_i\|^2$ on the diagonal), we have $\|\bs z_{\perp}\|^{-2}=\bs 1^\top D^{-1}\bs 1$. Moreover,
$$
\|\tfrac{\hat{\bs w}}{\|\hat{\bs w}\|^2}\|^{-2}\;=\;\bs 1^\top (\bar X\bar X^\top)^{-1}\bs 1\;=\;\frac{\bs 1^\top D^{-1}\bs 1}{1+\|\bs\mu\|^2\bs 1^\top D^{-1}\bs 1},
$$
where to invert $\bar X\bar X^\top =\|\bs \mu\|^2\bs 1\bs 1^\top +D$ we used the Sherman–Morrison formula. This confirms that, under orthogonality, $\bs{\hat w} = \bs{\bar w}$ (the optimum of \eqref{eq:mmc} is uniquely defined). 
\end{proof}

This finishes our slightly informal argument that $\bs{\hat w}/\|\bs{\hat w}\|^2 \approx \bs \mu + \bs z_\perp$ (see Figure~\ref{fig:blowup2} for illustration). 

Now recall the beginning of Section~\ref{sec:blow-up} where we argued that the test error behavior is governed by $d_{\hat{\bs{w}}} = \tfrac{\<\hat{\bs w},\bs\mu\>}{\|\hat{\bs w}\|}$. To explain the phase transition in the test error behavior, we will thus show that a corresponding phase transition occurs in $d_{\hat{\bs{w}}}$. Given the derivations above, we get 
\[
\langle \bs{\hat w}, \bs \mu \rangle \;\approx\;\<\bar{\bs w},\bs\mu\>\;\approx\; \frac{\|\bs \mu\|^2}{\|\bs \mu\|^2 + \|\bs z_\perp\|^2}
\]
and hence
\[
d_{\hat{\bs{w}}}^2 \;=\; \left(\frac{\langle \bs{\hat w}, \bs \mu \rangle}{\|\bs{\hat w}\|}\right)^2 \;\approx\; \frac{\|\bs \mu\|^4}{\|\bs \mu\|^2 + \|\bs z_\perp\|^2}. 
\]
As $\|\bs\mu\|$ grows, this quantity exhibits a phase transition at $\|\bs \mu\| \approx \|\bs z_\perp\|$. Note that $\|\bs{z}_i\| \approx \rho^{-1/2}$ implies $\bs 1^\top D^{-1}\bs 1 \approx \sum_i \|\bs z_i\|^{-2}\approx n\rho$, or equivalently, $\|\bs z_\perp\|^{-2}\approx n\rho$. Thus the phase transition in $d_{\hat{\bs{w}}}$ occurs at $\|\bs z_\perp\| \approx (n\rho)^{-1/2}$, which matches our results in Theorem~\ref{thm:noiseless-main}. Under the orthogonality condition, we can formulate the above argument rigorously as:

\begin{lemm}
Let $\tilde \sigma$ be the normalized area on the sphere $\bs \mu + \rho^{-1/2}S^{p-1}$. Suppose $\bs{\hat w} = X^\top (XX^\top)^{-1}\bs y$ and $\bs z_i \perp \bs z_j$, $\bs z_i \perp \bs \mu$, and $\|\bs z_i\| = \rho^{-1/2}$ for all $i\neq j$. Then, we have 
\[
d_{\hat{\bs{w}}}^2 = \frac{\|\bs \mu\|^4}{\|\bs \mu\|^2 + (n\rho)^{-1}},
\]
and $\tilde \sigma (\{\bs u: \langle \boldsymbol{\hat w}, \bs u \rangle \geq  0\})\to 1$ as $\sqrt{p\rho}d_{\bs{\hat w}} \to \infty$.
\begin{proof}
By Lemma~\ref{lemm:whatdecomp}, we have $\bs{\hat w} = \bs{\bar w}$. Therefore, we have $
\langle \bs{\hat w}, \bs \mu \rangle \;=\; \frac{\|\bs \mu\|^2}{\|\bs \mu\|^2 + \|\bs z_\perp\|^2}$ and hence $d_{\hat{\bs{w}}}^2 = \frac{\|\bs \mu\|^4}{\|\bs \mu\|^2 + \|\bs z_\perp\|^2}.$ By noting $\|\bs z_i\| = \rho^{-1/2}$ and recalling $\|\bs z_\perp\|^2 = \bs 1^\top D^{-1 } \bs 1= \sum_i \|\bs z_i\|^{-2} = n\rho$, we see that the desired equality holds. The last part is due to Corollary~\ref{cor:blowup}.
\end{proof}
\end{lemm}

\begin{figure}[!htp]
\centering
  \begin{tikzpicture}
    \node[anchor=south west, inner sep=0] (image) at (0,0) {\includegraphics[width=0.7\textwidth]{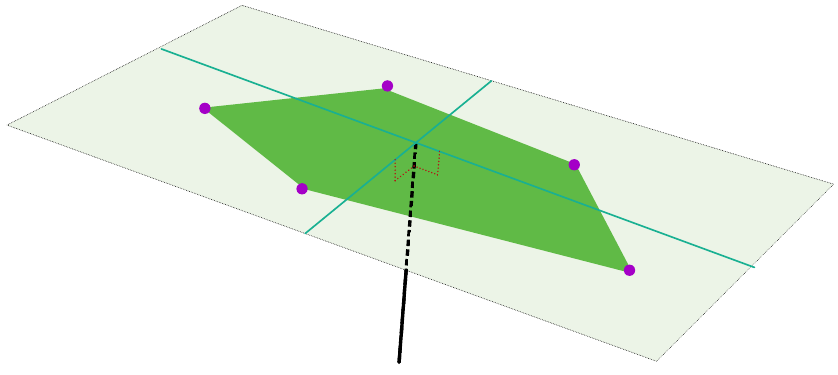}};
    \begin{scope}[x={(image.south east)},y={(image.north west)}]
      \node[anchor=north] at (0.43, 0.05) {\textcolor{black}{\large $\bs 0$}};
      \node[anchor=north] at (0.45, 0.72) {\textcolor{purple}{\large $\bs z_\perp$}};
      \node[anchor=north] at (0.23, 0.71) {\textcolor{black}{$\bar{\bs z}_2$}};
      \node[anchor=north] at (0.35, 0.5) {\textcolor{black}{$\bar{\bs z}_3$}};
      \node[anchor=north] at (0.72, 0.63) {\textcolor{black}{$\bar{\bs z}_4$}};
      \node[anchor=north] at (0.43, 0.88) {\textcolor{black}{$\bar{\bs z}_5$}};
      \node[anchor=north] at (0.78, 0.3) {\textcolor{black}{$\bar{\bs z}_1$}};
      \node[draw=purple, fill=purple, circle, minimum size=5pt, inner sep=0pt] at (0.497, 0.63) {};
      \node[draw=black, fill=black, circle, minimum size=5pt, inner sep=0pt] at (0.476, 0.045) {};
    \end{scope}
  \end{tikzpicture}
\caption{Illustration of $\bs z_\perp$ as a convex combination of $\bar{\bs z}_i$'s or, equivalently, as an orthogonal projection of the origin on the maximum margin hyperplane defined by them. }
\label{fig:z_perp_decomposition}
\end{figure}

\subsection{Phase Transition: Noisy Model}\label{sec:phasetransnoisy}

In this section we provide geometric intuition behind the phase transition on the noisy model observed in Theorem~\ref{detail-noisy-main-1-simple}. Our arguments are very similar to the previous subsection. The main idea will be to derive an approximate representation for $\hat{\bs w}$ and use it to exhibit a phase transition in $d_{\hat{\bs{w}}}$.}

We note that the formula in \eqref{eq:wbynorm} remains unchanged in the noisy case. Define $\nu_{\cs}=\sum_{i:{\rm clean}}\alpha_i$ and $\nu_\n=\sum_{i:{\rm noisy}}\alpha_i$ noting that $\nu_{\cs}+\nu_\n=1$. Then
\begin{eqnarray}
\nonumber\frac{\hat{\bs w}}{\|\hat{\bs w}\|^2}\;&=&\;\sum_{i=1}^n \alpha_i \bar{\bs x}_i\;=\;\sum_{i:{\rm clean}}\alpha_i(\bar{\bs z}_i+\bs \mu)+\sum_{i:{\rm noisy}}\alpha_i(\bar{\bs z}_i-\bs \mu)\\
\nonumber&=&\nu_{\cs}(\sum_{i:{\rm clean}}\tfrac{\alpha_i}{\nu_{\cs}}\bar{\bs z}_i+\bs \mu)+\nu_{\n}(\sum_{i:{\rm noisy}}\tfrac{\alpha_i}{\nu_{\n}}\bar{\bs z}_i-\bs \mu)\\
\label{eq:wovernormdecom}&=&\nu_{\cs}(\bs z_{\perp,\cs}+\bs \mu)+\nu_{\n}(z_{\perp,\n}-\bs \mu),
\end{eqnarray}
where we defined
\begin{equation}\label{eq:zperps}
\bs z_{\perp,\cs}:=\sum_{i:{\rm clean}}\tfrac{\alpha_i}{\nu_{\cs}}\bar{\bs z}_i\qquad \mbox{and}\qquad \bs z_{\perp,\n}:=\sum_{i:{\rm noisy}}\tfrac{\alpha_i}{\nu_{\n}}\bar{\bs z}_i    
\end{equation}
as convex combinations of $\bar{\bs z}_i$'s corresponding to clean samples and to the noisy samples respectively. 

Although this decomposition is universal, the vectors $\bs z_{\perp,\cs}+\bs\mu$ and $\bs z_{\perp,\n}-\bs\mu$ have a very clear geometric interpretation in the full orthogonal setting. To explain this in detail, let $\hat{\bs w}_{\cs}$ and $\hat{\bs w}_{\n}$ denote the maximal margin classifiers on the clean and on the noisy data respectively. We have the following result.

\begin{lemm}\label{lem:noisy_geom}
    Suppose $\bs z_i\perp \bs z_j$, $\bs z_i\perp\bs \mu$ for all $i\neq j$.  Then
$$
\bs z_{\perp,\cs}+\bs\mu\;=\;\frac{\hat{\bs w}_{\cs}}{\|\hat{\bs w}_{\cs}\|^2}\qquad\mbox{and}\qquad \bs z_{\perp,\n}-\bs\mu\;=\;\frac{\hat{\bs w}_{\n}}{\|\hat{\bs w}_{\n}\|^2}.
$$

Moreover, $\tfrac{\hat{\bs w}}{\|\hat{\bs w}\|^2}$ is the orthogonal projection of the origin on the line joining  $\tfrac{\hat{\bs w}_{\cs}}{\|\hat{\bs w}_{\cs}\|^2}$ and $\tfrac{\hat{\bs w}_{\n}}{\|\hat{\bs w}_{\n}\|^2}$ (see Figure~\ref{fig:noisy_decomp} for an illustration).

If we further assume
\begin{equation}\label{eq:normassump}
\sum_{i:\text{clean}}\|\bs z_i\|^{-2} \;=\; (1-\eta)n\rho\qquad \mbox{and}\qquad   \sum_{i:\text{noisy}}\|\bs z_i\|^{-2} \;=\;\eta n\rho,    
\end{equation}
then
\[
\nu_\cs = \frac{1-\eta + 2\eta(1-\eta)n\rho \|\bs \mu\|^2}{1+4\eta(1-\eta)n\rho\|\bs \mu\|^2}, \qquad \nu_\n = \frac{\eta + 2\eta(1-\eta)n\rho \|\bs \mu\|^2}{1+4\eta(1-\eta)n\rho\|\bs \mu\|^2}.
\]
\end{lemm}

Lemma~\ref{lem:noisy_geom} is proved in Section~\ref{sec:proof:lem:noisy_geom}. By Lemma~\ref{lem:noisy_geom}, in the full orthogonal case, we conclude that $\bs{\hat w}/\|\bs{\hat w}\|^2$ can be viewed as a weighted average of two vectors constructed from clean points and noisy points, respectively.

\begin{figure}[!htp]
\begin{tikzpicture}
  \node[anchor=south west, inner sep=0] (image) at (0,0) {\includegraphics[width=0.6\textwidth]{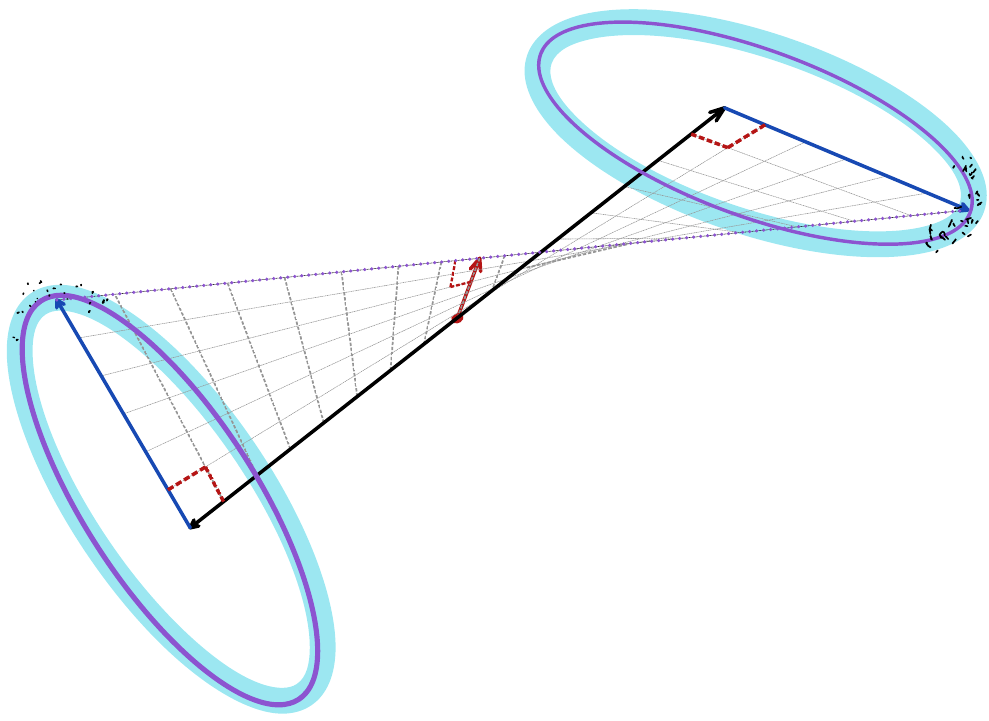}};
  \begin{scope}[x={(image.south east)},y={(image.north west)}]
    \node[anchor=north] at (0.47, 0.55) {\textcolor{black}{\Large $\bs 0$}};
    \node[anchor=north] at (0.48, 0.79) {\textcolor{black}{\Large $\textcolor{purple}{\frac{\hat{\bs w}}{\|\hat{\bs w}\|^2}}$}};
 \node[anchor=north] at (1.05, 0.75) {\textcolor{black}{\Large ${\frac{\hat{\bs w}_{\cs}}{\|\hat{\bs w}_{\cs}\|^2}}$}};
            \node[anchor=north] at (0.02, 0.75) {\textcolor{black}{\Large ${\frac{\hat{\bs w}_{\n}}{\|\hat{\bs w}_{\n}\|^2}}$}};
 \node[anchor=north] at (0.7, 0.9) {\textcolor{purple}{\Large $\bs \mu$}};
          \node[anchor=north] at (0.16, 0.27) {\textcolor{purple}{\Large $-\bs \mu$}};
\node[anchor=north] at (0.83, 0.78) {\textcolor{violet}{\large $\bs z_{\perp,\cs}$}};
\node[anchor=north] at (0.09, 0.44) {\textcolor{violet}{\large $\bs z_{\perp,\n}$}};
  \end{scope}
\end{tikzpicture}
\caption{{The clean observations $\bar{\bs x}_i=\bs \mu+\bar{\bs z}_i$ are concentrated around the sphere $\bs\mu+\rho^{-1/2}S^{p-1}$ and the noisy observations concentrate around $-\bs\mu+\rho^{-1/2}S^{p-1}$. The decomposition $\frac{\hat{\bs w}}{\|\hat{\bs w}\|^2}=\nu_\cs(\bs\mu+\bs z_{\perp,\cs})+\nu_\n(-\bs\mu+\bs z_{\perp,\n})$ depicted here is the fundamental geometric reason behind the phase transition in the noisy model.} }
\label{fig:noisy_decomp}
\end{figure}

Under the assumptions of Lemma~\ref{lem:noisy_geom}, the weights $\nu_\cs$ and $\nu_\n$ exhibit a phase transition at $\|\bs \mu\| \approx (n\rho)^{-1/2}$, and we also have  $\|\bs z_{\perp, \cs}\| = ((1-\eta)n\rho)^{-1/2}$ and $\|\bs z_{\perp, \n}\| = (\eta n\rho)^{-1/2}$.
From this, we can infer an interesting geometric characterization of $\bs{\hat w}/\|\bs{\hat w}\|^2$.

\textbf{Case 1: $\|\bs \mu\|\ll (n\rho)^{-1/2}$}. In this case we have
$
\nu_\cs \approx 1-\eta$ and $\nu_\n \approx \eta, 
$
and hence, 
$$
\frac{\bs{\hat w}}{\|\bs{\hat w}\|^2} \; \approx \;(1-2\eta) \bs \mu + (1-\eta )\bs z_{\perp, \cs} + \eta \bs z_{\perp, \n}
 \;=\; (1-2\eta) \bs \mu + \frac{1}{n\rho}\sum_{i=1}^n \frac{y_{\n, i}\bs z_i}{\|\bs z_i\|^2}  \;=\; (1-2\eta) \bs \mu + \bs z_\perp,
$$
where $\bs z_\perp$ is the orthogonal projection from the origin to the maximum margin separating hyperplane defined by $y_{\n, i}\bs z_i$'s $(i=1,\ldots, n)$. With this decomposition, we can see that the behavior in this small signal regime is essentially same as that of the noiseless model, and hence
\[
\left(\frac{\|\bs{\hat w}\|}{\langle \bs{\hat w}, \bs \mu \rangle}\right)^2 \;\approx\; \frac{1}{n\rho \|\bs \mu\|^4}.
\]

\textbf{Case 2: $\|\bs \mu\|\gg (n\rho)^{-1/2}$}. Now we have
$\nu_\cs \approx \tfrac{1}{2}$ and $\nu_\n \approx \tfrac{1}{2}$ and hence 
\begin{align*}
\frac{\bs{\hat w}}{\|\bs{\hat w}\|^2}  \;\approx\; \frac{1-2\eta}{4\eta(1-\eta)n\rho \|\bs \mu\|^2}\bs \mu + \frac{\bs z_{\perp, \cs} + \bs z_{\perp, \n}}{2}. %\\
\end{align*}
With this decomposition, we get
\[
\|\bs{\hat w}\|^2 \approx 4\eta(1-\eta)n\rho \quad \text{and} \quad \langle \bs{\hat w}, \bs \mu \rangle \approx 1 - 2\eta
\]
and hence
\[
\left(\frac{\|\bs{\hat w}\|}{\langle \bs{\hat w}, \bs \mu \rangle}\right)^2 \;\approx\; \eta n\rho.
\]

The above expansions confirm both the presence of a phase transition in $d_{\hat{\bs w}}$ at $\|\bs \mu\|\approx (n\rho)^{-1/2}$ and the different behavior of $d_{\hat{\bs w}}$ in the strong signal regime compared to the noiseless case.

It is worth noting that in both regimes, the $\bs \mu$ term is smaller than the remaining terms: if $\|\bs \mu\|\ll (n\rho)^{-1/2}$, $\bs{\hat w}/\|\bs{\hat w}\|^2$ is dominated by $\bs z_\perp$ and if $\|\bs \mu\|\gg (n\rho)^{-1/2}$, $\bs{\hat w}/\|\bs{\hat w}\|^2$ is dominated by $\frac{\bs z_{\perp, \cs} + \bs z_{\perp, \n}}{2}$. This observation explains why the behavior of $\bs{\hat w}/\|\bs{\hat w}\|^2$ is similar between the noiseless and noisy case when the signal is weak and very different when the signal is large. It also confirms the geometric intuition from the beginning of Section~\ref{geometry}: $\bs \mu$ indeed becomes more and more orthogonal to $\hat{\bs w}$ as $\|\bs \mu\|$ grows in the noisy case. The fact that benign overfitting can still occur in this setting is because $\hat{\bs w}$ still has a component $\tfrac{1-2\eta}{4\eta(1-\eta)n\rho \|\bs \mu\|^2}\bs \mu$ which is aligned with $\bs\mu$. By Corollary~\ref{cor:blowup}, the normalized surface area satisfies $\tilde \sigma(\{\bs u: \langle \bs{\hat w}, \bs u\rangle \geq 0\}) \to 1$ as $\sqrt{p\rho}d_{\bs{\hat w}}\to \infty$. Thus, together with the blow-up phenomenon in Section~\ref{sec:blow-up}, this enables such a seemingly paradoxical phenomenon.

\section{Discussion and future directions}

In this work, we showed that benign overfitting in binary linear classification occurs in a much wider range of scenarios than was previously known. Still, there are many interesting avenues for future research on benign overfitting in linear classification.

First, our analysis relies on finite $2+\eps$ moments for the predictors. While simulations in the supplement suggest that benign overfitting fails for features with less than two moments, providing a formal proof of this failure or studying the regime with exactly two finite moments remains an interesting open question. Second, our results heavily rely on linear separability of training data which holds with high probability under our assumptions. Extending the analysis to the case when linear separability is not guaranteed with high probability seems to require fundamentally different tools. Finally, our analysis is restricted to binary classification. Benign overfitting in multiclass classification is substantially less well studied than the binary case. A notable exception is \cite{wang2021multiclass}. The latter authors studied multiclass classification for noiseless isotropic Gaussian mixtures under certain near orthogonality conditions. Extending the analysis of the multiclass case beyond that setting or allowing for label noise is another interesting open problem. While some of the tools we developed in the present paper should be useful for this extension, the setting is quite different and a different analysis from ours will be required.

\begin{appendix}

\section{Recovering existing results for sub-Gaussian mixture models}\label{sec:recoveringSubGauss}

\textbf{Additional notation used in this section} 
For a random variable $X$, we use $\|X\|_{\psi_1}$ and $\|X\|_{\psi_2}$ to denote the sub-exponential norm and sub-Gaussian norm of $X$, respectively, which are defined as follows:
\begin{align*}
\|X\|_{\psi_1} & := \inf\{t>0 \,:\, \E \exp(|X|/t) \leq 2 \}, \\
\|X\|_{\psi_2} & := \inf\{t>0 \,:\, \E \exp(X^2/t^2) \leq 2 \}.
\end{align*}
For a random vector $X\in \R^p$, the sub-Gaussian norm of $X$ is denoted by $\|X\|_{\psi_2}$, whose definition is given by 
\[
\|X\|_{\psi_2} := \sup_{v\in S^{p-1}} \|\langle v,X \rangle\|_{\psi_2}.
\]
\medskip

In this subsection, we will show that our general framework (Theorem~\ref{thm:noiseless-main} and Theorem~\ref{detail-noisy-main-1-simple}) applied to the sub-Gaussian mixture model recovers existing results in the literature. Throughout this section, we make the following assumption. This includes essentially the sub-Gaussian mixture model that was considered in most of the existing literature. We note that our assumption is more general compared to the previous literature in terms of the choice of $W$ as noted in Section~\ref{sec:setting}.

\begin{enumerate}[label=(sG), ref=(sG)]
\item Suppose $\boldsymbol{z}= W\boldsymbol{\xi}$, where $\boldsymbol{\xi} \in \R^p$ has independent entries that have mean zero, unit variance, and $\|\xi_{j}\|_{\psi_2} \leq L$ for all $j= 1,\ldots, p$. \label{model:sG}
\end{enumerate}

By Theorem~\ref{detail-noisy-main-1-simple} (more precisely, the exponential bound~\eqref{eq:testerrornoisygenexp} in its proof) and Lemma~\ref{lemm:esG-noisy-what-bound}, we obtain the following result.

\begin{thm}[Noisy \& sub-Gaussian]\label{thm:sGnoisymain}
Let $\tilde C_1 $ be the constant in Lemma~\ref{lemm:esGOmega1}, $\tilde C_2 , \tilde C_3 $ be the constants in Lemma~\ref{lemm:esGOmega2}, and $\tilde C_4 $ be the constant in Lemma~\ref{lemm:esGOmega4} and define $C_{2,\eta} := \max\{\tfrac{22}{\eta},\tfrac{17}{1-2\eta}\}$. Suppose \ref{model:sG} holds in model \ref{model:M}, that
    \begin{align}
    \Tr(\Sigma) & \geq \frac{33\tilde C_1 L^2}{16\eta}   \max\left\{n^{3/2} \|\Sigma\|, n \|\Sigma\|_{\textsc{f}}\right\}, \label{eq:trcond} \\
    \|\boldsymbol{\mu}\|^2 & \geq 2\tilde C_2  C_{2,\eta}L \|\bs{\mu}\|_\Sigma, \nonumber 
    \end{align}
    and that further one of the following conditions holds:
    \begin{enumerate}
        \item[(i)] $\displaystyle \Tr(\Sigma) \geq \max\left\{60\tilde C_2  n, 132\tilde C_3 n\sqrt{\log (\tfrac{n}{\delta})} \right\} L    \|\bs{\mu}\|_\Sigma$, 
        \item [(ii)] $\displaystyle \|\boldsymbol{\mu}\|^2 \geq \tfrac{33}{16}\tilde C_3 C_{2,\eta} L \sqrt{\log (\tfrac{n}{\delta})}\|\bs \mu\|_\Sigma$ and \\       $\displaystyle \Tr(\Sigma) \geq 4\tilde C_2  C_{2,\eta} L^2 \max\left\{\tilde C_2  n, \tfrac{33}{32}\tilde C_3 n \sqrt{\log(\tfrac{n}{\delta})}\right\}\frac{\|\bs{\mu}\|_\Sigma^2}{\|\boldsymbol{\mu}\|^2}$.
    \end{enumerate}
    Then, we have for $n\geq  \log(\tfrac{1}{\delta}) \left( \tfrac{16\tilde C_4 }{\min\{\eta, 1-2\eta\}}\right)^2$, with probability at least $1-5\delta$, 
\[
\mathbb{P} \left( \langle \boldsymbol{w}, y_\n \boldsymbol{x} \rangle  <0 \right)  \;\;\leq\;\; \eta + (1-\eta)\exp\left\{- c \frac{(1-2\eta)^2}{\|\Sigma\|L^2}\left(\frac{ n }{\Tr(\Sigma)} + \frac{\Tr(\Sigma)}{ n\|\boldsymbol{\mu}\|^4} \right)^{-1}\right\},
\]
where $c$ is a universal constant.
\end{thm}

The proof of Theorem~\ref{thm:sGnoisymain} is given in Section~\ref{sec:proof:thm:sGnoisymain} in the supplement.

Theorem~\ref{thm:sGnoisymain} extends Theorem~4 in \cite{JMLR:v22:20-974} and Theorem~7 in \cite{wang2022binary}. 

\begin{sloppypar}
First, both existing results only consider specialized settings. Both of them hold only under $\Tr(\Sigma) \gtrsim n\|\boldsymbol{\mu}\|^2$, which can be regarded as a special case of the regime (i) in Theorem~\ref{thm:sGnoisymain}. 

Furthermore, compared to Theorem~4 in \cite{JMLR:v22:20-974}, our bound is tighter by a factor of $n$ in the exponential and holds under substantially weaker assumptions: they further assume $\Tr(\Sigma)\asymp p, \|\Sigma\|\asymp 1$, and $\Tr(\Sigma)\gtrsim n^2 \log n \|\Sigma\|$. The bound in Theorem~7 of \cite{wang2022binary} matches ours but is limited to isotropic Gaussian mixtures: $\bs z \sim \mathcal{N}(0, I_p)$.
\end{sloppypar}

Theorem~7 in \cite{wang2022binary} can also be regarded as a special case of Theorem~\ref{thm:sGnoisymain} under the regime (i). However, although their result is restricted to the regime $\Tr(\Sigma) \gtrsim n\|\boldsymbol{\mu}\|^2$, their assumption which corresponds to our assumption \eqref{eq:trcond} is weaker. This discrepancy comes from their isotropic Gaussian assumption, namely, they assume $\boldsymbol{z} \sim \mathcal{N}(0, I_p)$. In fact, they were able to establish a sharper result than Lemma~\ref{lemm:esGOmega1} using the special property of the inverse Wishart distribution and further properties of isotropic Gaussian distributions. The argument using properties of the inverse Wishart distribution can be found in \citet[][Lemma~21]{Muthkumar2020class}. 

Lastly, we note that \cite{frei23-COLT} studied benign overfitting in binary linear classification on a different data distribution model. A similar result is further established for shallow neural networks by \cite{pmlr-v178-frei22a} for the weak signal regime under a stronger data distribution assumption than sub-Gaussianity.

Similarly, we can also establish corresponding results for the noiseless sub-Gaussian mixture model by applying Theorem~\ref{thm:noiseless-main} and Lemma~\ref{lemm:esGboundsevents}.

\begin{thm}[Noiseless \& sub-Gaussian]\label{thm:noiseless-main-sG}
Let $C$ be the constant in (i) of Theorem~\ref{thm:noiseless-main}, $C^\prime$ be the constant in (ii) of Theorem~\ref{thm:noiseless-main}, $\tilde C_1 $ be the constant in Lemma~\ref{lemm:esGOmega1} and $\tilde C_2 , \tilde C_3 $ be the constants in Lemma~\ref{lemm:esGOmega2}. Suppose \ref{model:sG} holds in model \ref{model:M} with $\eta=0$ (the noiseless case), $n\geq \log(\tfrac{1}{\delta})$, and one of the following holds:
\begin{enumerate}
    \item[(i)] $\|\bs \mu\|^2 \geq 2\sqrt{2}C \tilde C_2 L\|\bs{\mu}\|_\Sigma$ and 
    \begin{align}
    \Tr(\Sigma) & \geq \tilde C_5 \max\left\{L^2n\|\Sigma\|_F, L^2n^{3/2} \|\Sigma\|, L n\sqrt{\log (\tfrac{n}{\delta})} \|\bs{\mu}\|_\Sigma \right\}, \label{eq:sGtrcond}
    \end{align}
    where $\tilde C_5 = \max\{3\sqrt{6}\tilde C_1, 4\sqrt{6}\tilde C_2, 48 \tilde C_3\}$, 
    \item[(ii)] $\|\bs \mu\| \geq \tfrac{3}{2}C^\prime \sqrt{\Tr(\Sigma)}$ and $\displaystyle \Tr(\Sigma) \geq 2\tilde C_1 L^2\max\left\{ \sqrt{n}\|\Sigma\|_F, n \|\Sigma\|\right\}$.
\end{enumerate}
Then, we have
    \begin{equation*}
        \P_{(\boldsymbol{x}, y)}(\langle \boldsymbol{\hat w}, y\boldsymbol{x}\rangle < 0) \leq \exp\left\{-\frac{c}{L^2 \|\Sigma\|}\left(\dfrac{1}{\|\boldsymbol{\mu}\|^2} + \frac{\Tr(\Sigma)}{n \|\boldsymbol{\mu}\|^4}\right)^{-1}\right\},
    \end{equation*}
where $c$ is a universal constant, with probability at least $1 - 4\delta$ under condition (i) and $1-2\delta$ under condition (ii).
\end{thm}

The proof of Theorem~\ref{thm:noiseless-main-sG} is given in Section~\ref{sec:proof:thm:noiseless-main-sG} of the supplement.

Theorem~\ref{thm:noiseless-main-sG} under condition (i) essentially recovers Theorem~3.1 in \cite{NEURIPS2021_46e0eae7}. The only difference is in the expression of the test error bound. The test error bound in Theorem~3.1 in \cite{NEURIPS2021_46e0eae7} has the following form
\[
\exp\left(-c\frac{n\|\boldsymbol{\mu}\|^4}{n\|\bs{\mu}\|_\Sigma^2 + \|\Sigma\|_{\textsc{f}}^2 + n\|\Sigma\|^2}\right)
\]
while ours has
\[
\exp\left(-c\frac{n\|\boldsymbol{\mu}\|^4}{\|\Sigma\|\left(n\|\boldsymbol{\mu}\|^2 + \Tr(\Sigma)\right)}\right).
\]

While \cite{NEURIPS2021_46e0eae7}'s bound is tighter under condition (i), our bound also holds under condition (ii) which was not studied in prior work. Moreover, our bound matches with theirs up to a constant factor when $\lambda_{max}/\lambda_{min} = O(1)$. The discrepancy can be seen in the inequality \eqref{eq:markov-aux}. Instead of using $\bs{w}^\top \E [\bs{z}\bs{z}^\top]\bs{w} \leq \|\bs w\|^2 \|\E [\bs{z}\bs{z}^\top]\|$, which holds for any $\bs w$, \cite{NEURIPS2021_46e0eae7} bounded $\bs{\hat w}^\top \E [\bs{z}\bs{z}^\top]\bs{\hat w}$ using the explicit form of $\bs {\hat w}$. While we could recover their result by following their argument instead of ~\eqref{eq:markov-aux}, we chose to state our result as presented here since our main motivation is to demonstrate that our theory holds in a much wider range of scenarios.

\end{appendix}

\begin{acks}[Acknowledgments]
We thank Mufan Li for pointing us to references in the very early stage of this work. The authors are further grateful to the Associate Editor and three anonymous Referees whose constructive comments led to an improved manuscript. We would also like to thank Prof. Christos Thrampoulidis for helpful discussions regarding the work \cite{wang2022binary}.
\end{acks}

\begin{funding}
The research of Stanislav Volgushev and Piotr Zwiernik was partially supported by a Discovery Grants from the Natural Sciences and Engineering Research Council of Canada (grants RGPIN-2024-05528 to SV and RGPIN-2023-03481 to PZ).
\end{funding}

\bibliographystyle{imsart-nameyear} % Style BST file (imsart-number.bst or imsart-nameyear.bst)
\bibliography{references}       % Bibliography file (usually '*.bib')

%% or include bibliography directly:
% \begin{thebibliography}{}
% \bibitem[\protect\citeauthoryear{???}{???}]{b1}
% \end{thebibliography}

\newpage

%%%%%%%%%%%%%%%%%%%%%%%%%%%%%%%%%%%%%%%%%
%%%%%%%%%%%%%%%%%%%%%%%%%%%%%%%%%%%%%%%%%
%   SUPPLEMET START HERE
%%%%%%%%%%%%%%%%%%%%%%%%%%%%%%%%%%%%%%%%%
%%%%%%%%%%%%%%%%%%%%%%%%%%%%%%%%%%%%%%%%%

\setcounter{section}{0}
\renewcommand*{\theHsection}{chX.\the\value{section}}
\setcounter{equation}{0}
\setcounter{figure}{0}
\renewcommand*{\theHfigure}{chX.\the\value{figure}}
\setcounter{table}{0}
\renewcommand*{\theHtable}{chX.\the\value{table}}

\renewcommand{\thesection}{S\arabic{section}}
\renewcommand{\theequation}{S.\arabic{equation}}
\renewcommand{\thefigure}{S\arabic{figure}}
\renewcommand{\thetable}{S\arabic{table}}

\section{Overview of supplement}

\addtocontents{toc}{\protect\setcounter{tocdepth}2} 

\tableofcontents

\section{General results and proofs of results in model \ref{model:M}}\label{sec:proof}

In this section, we provide proofs for our general results in model \ref{model:M}. Motivated by Theorem~\ref{th:soudry}, we study the maximum margin classifier and linear separability of data. 

Denote by $\Delta(\boldsymbol{y}_{\n})$ the diagonal matrix with $\boldsymbol{y}_{\n}=(y_{\n,1},\ldots,y_{\n,n})$ on the diagonal. A vector $\boldsymbol{w}\in \R^p$ linearly separates the data $(y_{\n,1},\bs x_1),\ldots,(y_{\n,n},\bs x_n)$ if $\<y_{\n, i} \boldsymbol{x}_i,\boldsymbol{w}\>>0$ for all $i=1,\ldots,n$. Equivalently, $\Delta(\boldsymbol{y}_{\n}) X \boldsymbol{w}$ has strictly positive entries. If $ X X^\top$ is invertible, then one example of a linearly separating vector is the minimum-norm least square estimator $\bs w_{\rm LS}$, which has the following expression:
\begin{equation}\label{eq:LS}
    \boldsymbol{w}_{\rm LS}\;=\; X^\top ( X X^\top)^{-1} \boldsymbol{y}_{\n}.
\end{equation}
Here, linear separation follows simply by noting that 
$\Delta(\boldsymbol{y}_{\n})X \boldsymbol{w}_{\rm LS} = \Delta(\boldsymbol{y}_{\n})\boldsymbol{y}_{\n}=\mathbf 1_n$. This also shows that $\boldsymbol{w}_{\rm LS}$ is a feasible point of \eqref{eq:mmc}. Explicit conditions on $\boldsymbol{w}_{\rm LS}$ to be equal to the optimal point $\boldsymbol{\hat w}$ of \eqref{eq:mmc} are also known.

\begin{lemm}\label{maxmargin=LS} 
If the vector $\Delta(\boldsymbol{y}_{\n})(XX^\top)^{-1}\boldsymbol{y}_{\n}$ has positive entries, then $\boldsymbol{\hat w} = \boldsymbol{w}_{\rm LS}$.
\end{lemm}

\begin{proof}
By Lemma~1, \cite{hsu2021proliferation}, the assumption implies that all $(\bs x_i, y_i), i = 1, \ldots, n$ are support vectors. Under this condition, the optimization problem defining $\bs{\hat w}$ is equivalent to
\begin{equation*}
\boldsymbol{\hat w} \;=\; \argmin\|\boldsymbol{w}\|^2, \qquad \text{ subject to $\langle \boldsymbol{w}, y_{\n, i}\boldsymbol{x}_i \rangle = 1$\;\; for all $i= 1, 2, \dots, n$,}    
\end{equation*}
and is also equivalent to the optimization problem defining $w_{\rm LS}$. Thus, we have $\bs{\hat w} = \bs{w}_{\rm LS}$.
\end{proof}

This Lemma~will play a key role in our proofs since it will allow us to verify that the maximum margin classifier equals the least square estimator: $\hat{\bs w} = \boldsymbol{w}_{\rm LS}$. Our analysis will then proceed by studying $\boldsymbol{w}_{\rm LS}$ through a careful analysis of $(XX^\top)^{-1}$.

\begin{asm}
In this section, we always assume the event $E_1(\eps)$ in \eqref{eq:E1} holds for some $\eps \in [0, \tfrac12]$. It is easy to see that this implies invertibility of $ZZ^\top$ {as $ZZ^\top$ is invertible if and only if $\check Z \check Z^\top$ is and on $E_1(\eps)$ we have $\check Z \check Z^\top = I_n + R$ with $\|R\| \leq \eps$.
} 
\end{asm}

\subsection{Bounding various quantities related to \texorpdfstring{$(XX^\top)^{-1}$}{the gram matrix} on events \texorpdfstring{$E_1,\dots,E_5$}{E1,...,E5}}

In Section~\ref{sec:sketch} we introduced some useful notation to provide in Lemma~\ref{quad-decomp} the formula for the matrix $(XX^\top)^{-1}$. We start by recalling this notation, introducing additional quantities that will be useful later, and proving this Lemma~formally. Let $A=Z Z^\top$ and, 
\begin{equation}\label{eq:notation}
    \begin{aligned}
    & \boldsymbol{\nu}=Z\boldsymbol{\mu},\quad s = \boldsymbol{y}^\top A^{-1} \boldsymbol{y},\quad t=\boldsymbol{\nu}^\top A^{-1}\boldsymbol{\nu}, \quad  h = \boldsymbol{y}^\top A^{-1} \boldsymbol{\nu},  \\
    & s_{\n} = \boldsymbol{y}_{\n}^\top A^{-1} \bs y, \quad s_{\n\n} = \boldsymbol{y}_{\n}^\top A^{-1} \boldsymbol{y}_{\n}, \quad h_{\n} = \boldsymbol{y}_{\n}^\top A^{-1} \boldsymbol{\nu}
\end{aligned}\end{equation}
{provided $A$ is invertible.} 
In the noiseless case $\boldsymbol{y_\n} = \boldsymbol{y}$ and so $s_{\n} = s_{\n\n} = s$ and $h_\n = h$.

\begin{lemm}\label{graminverse} 
Assume that event $E_1(\eps)$ holds for $\eps \in [0, \tfrac12]$. If $d := s(\|\boldsymbol{\mu}\|^2-t)+(1+h)^2\neq 0$, then $X X^\top$ is invertible and
\[( X X^\top)^{-1} = A^{-1} - d^{-1}A^{-1}\left[(1+h)(\bs y \boldsymbol{\nu}^\top +\boldsymbol{\nu} \boldsymbol{y}^\top)-s\boldsymbol{\nu}\boldsymbol{\nu}^\top+(\|\boldsymbol{\mu}\|^2-t)\bs y\boldsymbol{y}^\top \right]A^{-1}. \]
\end{lemm}
\begin{proof} Recall that $X=Z+\boldsymbol{y}\boldsymbol{\mu}^\top $. It follows that
$$   X X^\top\;=\;Z Z^\top + \boldsymbol{y} ( Z\boldsymbol{\mu})^\top + (Z\boldsymbol{\mu}) \boldsymbol{y}^\top +\|\boldsymbol{\mu}\|^2 \boldsymbol{y} \boldsymbol{y}^\top
$$
and so $XX^\top$ is obtained from $ Z Z^\top$ by perturbing it with a low rank matrix. Denoting
$$
\boldsymbol{u}\;:=\;\boldsymbol{y},\qquad \boldsymbol{v}\;:=\; Z\boldsymbol{\mu}+\tfrac{\|\boldsymbol{\mu}\|^2}{2}\boldsymbol{y}
$$
we get 
\begin{equation}\label{eq:XXtop}
X X^\top\;=\; Z Z^\top+\boldsymbol{u} \boldsymbol{v}^\top +\boldsymbol{v} \boldsymbol{u}^\top\;=\; Z Z^\top+\begin{bmatrix}
    \boldsymbol{u}& \boldsymbol{v}
\end{bmatrix}\begin{bmatrix}
    \boldsymbol{v}^\top\\ 
    \boldsymbol{u}^\top
\end{bmatrix}.\end{equation}
Recall that on the event $E_1(\eps)$ with $\eps \in [0, \tfrac12]$ the matrix $A = ZZ^\top$ is invertible. Denote $s_{uu}=\boldsymbol{u}^\top A^{-1} \boldsymbol{u}$, $s_{uv}=\boldsymbol{u}^\top A^{-1} \boldsymbol{v}$, and $s_{vv}=\boldsymbol{v}^\top A^{-1} \boldsymbol{v}$. Using the Woodbury matrix identity on \eqref{eq:XXtop} we get
\begin{eqnarray*}
( X X^\top)^{-1}&=&A^{-1}-A^{-1}\begin{bmatrix}
\boldsymbol{u} & \boldsymbol{v}    
\end{bmatrix}\left(I_2+\begin{bmatrix}
\boldsymbol{v}^\top \\ \boldsymbol{u}^\top    
\end{bmatrix}A^{-1}\begin{bmatrix}
\boldsymbol{u} & \boldsymbol{v}    
\end{bmatrix}\right)^{-1}\begin{bmatrix}
\boldsymbol{v}^\top \\ \boldsymbol{u}^\top    
\end{bmatrix}A^{-1}\\
&=& A^{-1}-A^{-1}\begin{bmatrix}
\boldsymbol{u} & \boldsymbol{v}    
\end{bmatrix}\begin{bmatrix}
1+s_{uv} & s_{vv} \\
s_{uu} & 1+s_{uv}    
\end{bmatrix}^{-1}\begin{bmatrix}
\boldsymbol{v}^\top \\ \boldsymbol{u}^\top    
\end{bmatrix}A^{-1}\\
&=& A^{-1}-\bar d^{-1}A^{-1}\begin{bmatrix}
\boldsymbol{u} & \boldsymbol{v}    
\end{bmatrix}\begin{bmatrix}
1+s_{uv} & -s_{vv} \\
-s_{uu} & 1+s_{uv}    
\end{bmatrix}\begin{bmatrix}
\boldsymbol{v}^\top \\ \boldsymbol{u}^\top    
\end{bmatrix}A^{-1},
\end{eqnarray*}
where $\bar d=(1+s_{uv})^2-s_{uu}s_{vv}$. Thus
\[( X X^\top)^{-1} = A^{-1} - \bar d^{-1}A^{-1}\left[(1+s_{uv})(\bs u \boldsymbol{v}^\top +\boldsymbol{v} \boldsymbol{u}^\top)-s_{uu}\boldsymbol{v} \boldsymbol{v}^\top-s_{vv}\boldsymbol{u} \boldsymbol{u}^\top \right]A^{-1}. \]
Using notation in \eqref{eq:notation}, we get $s_{uu}=s$, $s_{uv}=\tfrac{\|\boldsymbol{\mu}\|^2}{2}s+h$, and $s_{vv}=t+\frac{\|\boldsymbol{\mu}\|^4}{4}s+\|\boldsymbol{\mu}\|^2h$. After a bit of algebra we get $\bar d =d$ and the claimed equality.
\end{proof}

Note that Lemma~\ref{quad-decomp} follows directly from Lemma~\ref{graminverse}. Similar results  were also established in the existing literature (Lemma~SM8.1 in Supplementary Materials, \cite{wang2022binary} and Lemma~A.5 in \cite{NEURIPS2021_46e0eae7}).

{Provided we have $\|\boldsymbol{z}_i\|\neq 0$ for all $i$,} invertibility of $XX^\top$ and $A=ZZ^\top$ is equivalent to invertibility of  $\check X\check X^\top$ and $\check A=\check Z\check Z^\top$ respectively.   This follows since  $(ZZ^\top)^{-1}  = \Delta(z)^{-1}  (\check Z\check Z^\top)^{-1} \Delta(z)^{-1}$ and $(XX^\top)^{-1}  = \Delta(z)^{-1} (\check X \check X^\top)^{-1} \Delta(z)^{-1}$
We also use the following basic observation. For every $\bs a, \bs b$
\begin{align}
\label{eq:checkinv}\boldsymbol{a}^\top A^{-1} \boldsymbol{b} & \;=\; \check{\boldsymbol{a}}^\top \check A^{-1} \check{\boldsymbol{b}}%, \\
\end{align}

\begin{lemm}\label{eq:quad}%\label{invbounds}
Suppose that the event $E_1(\varepsilon)$, as given in \eqref{eq:E1}, holds for some $\varepsilon\in [0,\tfrac12]$. Then, for any $\boldsymbol{a},\boldsymbol{b}\in \R^p$
   \begin{equation}\label{eq:quadaa}
        \frac{\|\boldsymbol{a}\|^2}{1+\varepsilon} \;\leq\; \boldsymbol{a}^\top \check A^{-1} \boldsymbol{a} \;\leq\; \frac{\|\boldsymbol{a}\|^2}{1-\varepsilon}.
    \end{equation}
    and
    \begin{equation*}
\frac{\langle \boldsymbol{a}, \boldsymbol{b} \rangle - \varepsilon\|\boldsymbol{a}\|\|\boldsymbol{b}\|}{1 - \varepsilon^2} \;\leq\; \boldsymbol{a}^\top \check A^{-1} \boldsymbol{b}\;\leq \;\frac{\langle \boldsymbol{a}, \boldsymbol{b} \rangle + \varepsilon\|\boldsymbol{a}\|\|\boldsymbol{b}\|}{1 - \varepsilon^2} \;\leq\; \frac{\|\boldsymbol{a}\|\|\boldsymbol{b}\|}{1 - \varepsilon}.
\end{equation*}
\end{lemm}
\begin{proof}We first prove the bound on $\boldsymbol{a}^\top \check A^{-1} \boldsymbol{a}$ in \eqref{eq:quadaa}. Let $R = I_n - \check A $. Then, $\|R\| \leq \varepsilon \leq \tfrac12$ by the definition of $E_1(\varepsilon)$. Thus, $\check A = I - R$ is invertible and
\[
\|\check A^{-1}\| \leq 1 + \varepsilon + \varepsilon^2 + \dots \leq \frac{1}{1 - \varepsilon}.
\]
Now, \eqref{eq:quadaa} follows from $\|\check A\|\leq 1 + \varepsilon$ {(equiv. $\lambda_{\min}(\check A^{-1})\geq 1/(1+\varepsilon)$)} and $\|\check A^{-1}\|\leq \dfrac{1}{1 - \varepsilon}$ {(equiv. $\lambda_{\max}(\check A^{-1})\leq 1/(1-\varepsilon)$)}. To prove the bounds on $\boldsymbol{a}^\top \check A^{-1} \boldsymbol{b}$, by the polarization equality,
{\small \begin{align*}
\boldsymbol{a}^\top \check A^{-1} \boldsymbol{b} 
    & = \|\boldsymbol{a}\|\|\boldsymbol{b}\|\left(\frac{\boldsymbol{a}}{\|\boldsymbol{a}\|}\right)^\top \check A^{-1} \left(\frac{\boldsymbol{b}}{\|\boldsymbol{b}\|}\right) \\
    & = \frac{\|\boldsymbol{a}\|\|\boldsymbol{b}\|}{4}\left\{\left(\frac{\boldsymbol{a}}{\|\boldsymbol{a}\|} + \frac{\boldsymbol{b}}{\|\boldsymbol{b}\|}\right)^\top \check A^{-1} \left(\frac{\boldsymbol{a}}{\|\boldsymbol{a}\|} + \frac{\boldsymbol{b}}{\|\boldsymbol{b}\|}\right) \right. \\
    & \qquad \qquad \left.  - \left(\frac{\boldsymbol{a}}{\|\boldsymbol{a}\|} - \frac{\boldsymbol{b}}{\|\boldsymbol{b}\|}\right)^\top \check A^{-1}\left(\frac{\boldsymbol{a}}{\|\boldsymbol{a}\|} - \frac{\boldsymbol{b}}{\|\boldsymbol{b}\|}\right)\right\}.
\end{align*}}
Using \eqref{eq:quadaa}, we get
\begin{align*}
\boldsymbol{a}^\top \check A^{-1} \boldsymbol{b} 
    & \leq \frac{\|\boldsymbol{a}\|\|\boldsymbol{b}\|}{4}\left(\frac{\left\|\frac{\boldsymbol{a}}{\|\boldsymbol{a}\|} + \frac{\boldsymbol{b}}{\|\boldsymbol{b}\|}\right \|^2}{1 - \varepsilon} - \frac{\left\|\frac{\boldsymbol{a}}{\|\boldsymbol{a}\|} - \frac{\boldsymbol{b}}{\|\boldsymbol{b}\|}\right \|^2}{1 + \varepsilon}\right) \\
    & = \frac{\langle \boldsymbol{a}, \boldsymbol{b} \rangle + \varepsilon\|\boldsymbol{a}\|\|\boldsymbol{b}\|}{1 - \varepsilon^2} \\
    & \leq \frac{\|\boldsymbol{a}\|\|\boldsymbol{b}\|}{1 - \varepsilon} \quad \text{(by Cauchy-Schwartz inequality).}
\end{align*}

Similarly, 
$$
\boldsymbol{a}^\top \check A^{-1} \boldsymbol{b} \; \geq \;\frac{\|\boldsymbol{a}\|\|\boldsymbol{b}\|}{4}\left(\frac{\left\|\frac{\boldsymbol{a}}{\|\boldsymbol{a}\|} + \frac{\boldsymbol{b}}{\|\boldsymbol{b}\|}\right \|^2}{1 + \varepsilon} - \frac{\left\|\frac{\boldsymbol{a}}{\|\boldsymbol{a}\|} - \frac{\boldsymbol{b}}{\|\boldsymbol{b}\|}\right \|^2}{1 - \varepsilon}\right)  \;=\; \frac{\langle \boldsymbol{a}, \boldsymbol{b} \rangle - \varepsilon\|\boldsymbol{a}\|\|\boldsymbol{b}\|}{1 - \varepsilon^2},
$$
which concludes the proof.
\end{proof}

Recall that $s = \boldsymbol{y}^\top A^{-1}\boldsymbol{y}$, $s_{\n} = \boldsymbol{y}_{\n}^\top A^{-1} \boldsymbol{y}$, and $s_{\n\n} = \boldsymbol{y}_{\n}^\top A^{-1}\boldsymbol{y}_{\n}$. We have the following bounds on these quantities.

\begin{lemm}\label{quad-bounds-s}
Suppose events $E_1(\varepsilon)$ and $E_4(\beta,\rho)$ hold with $\eps \in [0, \tfrac12]$, $\beta, \rho > 0$. Then:
\begin{enumerate}[label=(\roman*), ref=(\roman*)]
\item $\displaystyle \frac{1-\beta}{1+\varepsilon}n\rho \;\leq \;\frac{\|\check{\boldsymbol y}\|^2}{1+\varepsilon} \; \leq\; s \;\leq \;\frac{\|\check{\boldsymbol y}\|^2}{1-\varepsilon}\; \leq \;\frac{1+\beta}{1-\varepsilon}n\rho$, 
\item $\displaystyle \frac{1-\beta}{1+\varepsilon}n\rho \;\leq \;\frac{\|\check{\boldsymbol y}\|^2}{1+\varepsilon} \;\leq \;s_{\n\n} \;\leq \;\frac{\|\check{\boldsymbol y}\|^2}{1-\varepsilon} \;\leq\; \frac{1+\beta}{1-\varepsilon}n\rho$. 
\end{enumerate}
If, in addition,  $E_5(\gamma,\rho)$ holds with $\gamma \in [0, \tfrac12)$ then:
\begin{enumerate}[label=(\roman*), ref=(\roman*)]\setcounter{enumi}{2} 
\item  $\displaystyle \frac{(1 - 2\eta) - (\gamma + \varepsilon(1 + \beta))}{1 - \varepsilon^2}n\rho \;\leq\; s_{\n} \; \leq \; \frac{(1 - 2\eta) + (\gamma + \varepsilon(1 + \beta))}{1 - \varepsilon^2}n\rho$.
\end{enumerate}
\end{lemm}

\begin{proof}
To bound $s=\bs y^\top A^{-1}\bs y$ in (i) we note that it is equal to $\check{\bs y}^\top \check A^{-1}\check{\bs y}$ by \eqref{eq:checkinv}. By Lemma~\ref{eq:quad}, this is bounded below by $\|\check{\bs y}\|^2/(1+\varepsilon)$ and above by $\|\check{\bs y}\|^2/(1-\varepsilon)$. The bounds $(1-\beta)n\rho/(1+\varepsilon)\leq \|\check{\bs y}\|^2/(1+\varepsilon)$ and $ \|\check{\bs y}\|^2/(1-\varepsilon)\leq (1+\beta)n\rho/(1+\varepsilon)$ hold because of the event $E_4$. The bound on $s_{\n\n}=\bs y_\n^\top A^{-1}\bs y_\n$ in (ii) is obtained in exactly the same way by noting that  $\|\check{\bs y}_\n\|=\|\check{\bs y}\|$. To bound $s_\n=\bs y_\n^\top A^{-1}\bs y$ in (iii) we note that it is equal to $\check{\bs y}_\n^\top \check A^{-1}\check{\bs y}$ by \eqref{eq:checkinv}. By Lemma~\ref{eq:quad}, 
$$
\frac{\<\check{\bs y}_\n,\check{\bs y}\>-\varepsilon\|\check{\bs y}\|^2}{1-\varepsilon^2}\;\leq\; \check{\bs y}_\n^\top \check A^{-1}\check{\bs y}\;\leq\;\frac{\<\check{\bs y}_\n,\check{\bs y}\>+\varepsilon\|\check{\bs y}\|^2}{1-\varepsilon^2},
$$
where we again used that $\|\check{\bs y}_\n\|=\|\check{\bs y}\|$. The conclusion follows under $E_4$ and $E_5$.
\end{proof}

Recall from \eqref{eq:notation} that $t = \boldsymbol{\nu}^\top A^{-1} \boldsymbol{\nu}$, $h = \boldsymbol{y}^\top A^{-1} \boldsymbol{\nu}$, and $h_{\n} = \boldsymbol{y}_{\n}^\top A^{-1} \boldsymbol{\nu}$. In the next Lemma~we also bound these quantities.
\begin{lemm}\label{quad-bounds-ht}
Suppose events $E_1(\varepsilon)$ and $E_2(\alpha_2,\alpha_\infty)$ hold with $\eps \in [0, \tfrac12]$ and $\alpha_2\geq \alpha_\infty>0$. Then 
\begin{enumerate}[label=(\roman*), ref=(\roman*)]
    \item $\displaystyle t  \leq \min\left\{\frac{\alpha_2^2}{1 - \varepsilon}, 1 \right\}\|\boldsymbol{\mu}\|^2$. \label{eq:t}
\end{enumerate}
If, in addition, $E_4(\beta,\rho)$ holds with $\beta, \rho > 0$. then: 
\begin{enumerate}[label=(\roman*), ref=(\roman*)]\setcounter{enumi}{1} 
    \item $\displaystyle |h|  \leq \frac{\|\check{\boldsymbol y}\|\alpha_2\|\boldsymbol{\mu}\|}{1 - \varepsilon} \leq \frac{\alpha_2\|\boldsymbol{\mu}\|\sqrt{(1+\beta)n\rho}}{1 - \varepsilon}$, \label{eq:h}
    \item $\displaystyle |h_{\n}|  \leq \frac{\|\check{\boldsymbol y}\|\alpha_2\|\boldsymbol{\mu}\|}{1 - \varepsilon} \leq \frac{\alpha_2\|\boldsymbol{\mu}\|\sqrt{(1+\beta)n\rho}}{1 - \varepsilon}$, \label{eq:hn}
\end{enumerate}
\end{lemm}
\begin{proof}
The bound $t \leq \|\bs\mu\|^2$ follows by observing that $t=\bs\mu^\top Z^\top(ZZ^\top)^{-1}Z\bs\mu$ is the squared norm of the projection of $\boldsymbol{\mu}$ onto the column space of $Z$. The remaining bounds follow now from Lemma~\ref{eq:quad} in the same way as in the proof of Lemma~\ref{quad-bounds-s}.
\end{proof}

Denote by $\bs e_1,\ldots,\bs e_n$ the canonical unit vectors in $\R^n$.

\begin{lemm}\label{quad-bounds-yay}
Suppose event $E_1(\varepsilon)$ holds with $\eps \in [0, \tfrac12]$. Then 
\begin{enumerate}[label=(\roman*), ref=(\roman*)]
    \item $ \dfrac{1 - \varepsilon \|\boldsymbol{z}_i\| \|\check{\boldsymbol y}\|}{(1 - \varepsilon^2)\|\boldsymbol{z}_i\|^2}  \leq \boldsymbol{y}^\top A^{-1} y_i\bs e_i \leq \dfrac{1 + \varepsilon \|\boldsymbol{z}_i\| \|\check{\boldsymbol y}\|}{(1 - \varepsilon^2)\|\boldsymbol{z}_i\|^2}$\label{eq:yye1}
    \item $ \dfrac{1 - \varepsilon \|\boldsymbol{z}_i\| \|\check{\boldsymbol y}\|}{(1 - \varepsilon^2)\|\boldsymbol{z}_i\|^2} \leq \boldsymbol{y}_{\n}^\top A^{-1} y_{\n,i}\bs e_i \leq \dfrac{1 + \varepsilon \|\boldsymbol{z}_i\| \|\check{\boldsymbol y}\|}{(1 - \varepsilon^2)\|\boldsymbol{z}_i\|^2} $ \label{eq:ynyne1}
    \item $\dfrac{1 - \varepsilon \|\boldsymbol{z}_i\| \|\check{\boldsymbol y}\|}{(1 - \varepsilon^2)\|\boldsymbol{z}_i\|^2}  \leq |\boldsymbol{y}^\top A^{-1} y_{\n,i}\bs e_i|  \leq \dfrac{1 + \varepsilon \|\boldsymbol{z}_i\| \|\check{\boldsymbol y}\|}{(1 - \varepsilon^2)\|\boldsymbol{z}_i\|^2} $ \label{eq:yyne1}
\end{enumerate}
If, in addition, $E_3(M)$ and $E_4(\beta,\rho)$ hold with $M>0$, $\beta >0$, and $\rho>0$.  Then: 
\begin{enumerate}[label=(\roman*), ref=(\roman*)]\setcounter{enumi}{3} 
    \item $\dfrac{1 - \varepsilon M \sqrt{(1+\beta)n\rho}}{(1 - \varepsilon^2)\|\boldsymbol{z}_i\|^2} \leq\boldsymbol{y}^\top A^{-1} y_i\bs e_i  \leq \dfrac{1 + \varepsilon M \sqrt{(1+\beta)n\rho}}{(1 - \varepsilon^2)\|\boldsymbol{z}_i\|^2}$ \label{eq:yye2}
    \item $\dfrac{1 - \varepsilon M \sqrt{(1+\beta)n\rho}}{(1 - \varepsilon^2)\|\boldsymbol{z}_i\|^2} \leq \boldsymbol{y}_{\n}^\top A^{-1} y_{\n,i}\bs e_i  \leq \dfrac{1 + \varepsilon M \sqrt{(1+\beta)n\rho}}{(1 - \varepsilon^2)\|\boldsymbol{z}_i\|^2}$ \label{eq:ynyne2}
    \item $\dfrac{1 - \varepsilon M \sqrt{(1+\beta)n\rho}}{(1 - \varepsilon^2)\|\boldsymbol{z}_i\|^2} \leq |\boldsymbol{y}^\top A^{-1} y_{\n,i}\bs e_i|  \leq  \dfrac{1 + \varepsilon M \sqrt{(1+\beta)n\rho}}{(1 - \varepsilon^2)\|\boldsymbol{z}_i\|^2}$ \label{eq:yyne2}
\end{enumerate}
\end{lemm}

\begin{proof}
The first two parts follow directly from Lemma~\ref{eq:quad}. For part (iii) (or (vi)), note that $|\boldsymbol{y}^\top A^{-1} y_{\n,i}\bs e_i| = |\boldsymbol{y}^\top A^{-1} y_{i}\bs e_i| |y_{\n,i} y_i| = |\boldsymbol{y}^\top A^{-1} y_{i}\bs e_i|$. Combined with part (i) (or (iv)), this proves the upper bound. The lower bound follows from (i) (or (iv)) if $1-\eps\|\boldsymbol{z}_i\| \|\check{\boldsymbol y}\| \geq 0$ and is clear if the latter expression is negative.
\end{proof}

\begin{lemm}\label{quad-bounds-nuae}
Suppose events $E_1(\varepsilon)$ and $E_2(\alpha_2,\alpha_\infty)$ hold with $\eps \in [0, \tfrac12]$ and $\alpha_2 \geq \alpha_\infty > 0$. Then:
\begin{equation}\label{eq:nue}
|\boldsymbol{\nu}^\top A^{-1} \bs e_i| \;\leq \;\frac{(\alpha_\infty + \varepsilon \alpha_2)\|\boldsymbol{\mu}\|}{(1 - \varepsilon^2)\|\boldsymbol{z}_i\|} \;\leq\; \frac{2 [\alpha_\infty \vee (\varepsilon \alpha_2)]\|\boldsymbol{\mu}\|}{(1 - \varepsilon^2)\|\boldsymbol{z}_i\|}.
\end{equation}
\end{lemm}
\begin{proof}
Note that $\bs\nu^\top A^{-1}\bs e_i=\check{\bs\nu}^\top \check A^{-1}\check{\bs e}_i$ and $\check{\bs \nu}=\check Z\bs\mu$. By event $E_2(\alpha_2,\alpha_\infty)$, $\|\check{\bs\nu}\|\leq \alpha_2\|\bs\mu\|$ and $\|\check{\bs\nu}\|_\infty\leq \alpha_\infty\|\bs\mu\|$. By Lemma~\ref{eq:quad}, we have
\begin{align*}
\bs\nu^\top A^{-1} \bs e_i 
    & \leq \frac{\langle \check{\bs\nu}, \check{\bs e_i}\rangle + \varepsilon\|\check{\bs\nu}\|\|\check{\bs e_i}\|}{1 - \varepsilon^2} \\
    & \leq \frac{\|\check{\bs\nu}\|_\infty + \varepsilon \alpha_2 \|\bs\mu\|}{(1 - \varepsilon^2)\|\bs z_i\|} \\
    & \leq \frac{(\alpha_\infty + \varepsilon \alpha_2)\|\bs\mu\|}{(1 - \varepsilon^2)\|\bs z_i\|}.
\end{align*}

By replacing $\bs e_i$ with $-\bs e_i$ and repeating the same argument, we also get
\[
 - \bs\nu^\top A^{-1} \bs e_i \leq \frac{(\alpha_\infty + \varepsilon \alpha_2)\|\bs\mu\|}{(1 - \varepsilon^2)\|\bs z_i\|}.
\]
\end{proof}

In addition to the bounds above, we also use the following bound frequently:

\begin{lemm}
Suppose events $E_3(M)$ and $E_4(\beta,\rho)$ holds with $M>0$, $\beta >0$, and $\rho>0$. Then we have
\begin{equation}\label{eq:M2rholowbound}
M^2\rho \;\geq\; \frac{1}{1+\beta}.
\end{equation}
\end{lemm}
\begin{proof}
It is easy to see this since the definitions of $E_3$ and $E_4$ imply
\[
n\rho(1+\beta) \;\geq\; \|\check{\boldsymbol y}\|^2\; =\; \sum_{i=1}^n \frac{1}{\|\boldsymbol{z}_i\|^2} \;\geq\; \frac{n}{M^2}. 
\]
\end{proof}

We now provide simple sufficient conditions for invertibility of $XX^\top$.
\begin{lemm}\label{invertibility}
    Suppose $\eps \in [0, \tfrac12]$ and either $E_1\cap E_2 \cap E_4$ holds with $\alpha_2\|\boldsymbol{\mu}\| \sqrt{(1+ \beta)n\rho} \leq \tfrac14$ or $E_1\cap E_2$ holds with $\alpha_2 < 1/\sqrt{2}$. Then
    \[
    d = s(\|\boldsymbol{\mu}\|^2 - t) + (1+h)^2 > 0
    \]
    holds, and hence $XX^\top$ is invertible. 
\end{lemm}
\begin{proof}
By Lemma~\ref{quad-bounds-ht} \ref{eq:t},  $t \leq \|\boldsymbol{\mu}\|^2$ and, by Lemma~\ref{quad-bounds-s}(i),  $s>0$. Thus, to show that $d>0$, it suffices to show that either $t<\|\bs\mu\|^2$ or $|h|<1$. If $E_1\cap E_2\cap E_4$ holds with $\alpha_2\|\boldsymbol{\mu}\| \sqrt{(1+ \beta)n\rho} \leq \tfrac14$, then Lemma~\ref{quad-bounds-ht} \ref{eq:h} implies
\[|h| \;\leq\; \frac{1}{4(1-\varepsilon)} \;\leq\; \frac{1}{2}\;<\;1.\]
Suppose now that $E_1\cap E_2$ holds with $\alpha_2 < 1/\sqrt{2}$. Since $\eps \in [0, \tfrac12]$, we get $\tfrac{\alpha_2^2}{1-\eps}<1$. Thus, by Lemma~\ref{quad-bounds-ht} \ref{eq:t}, $$t\;\leq\; \tfrac{\alpha_2^2}{1-\eps}\|\bs\mu\|^2\;<\; \|\bs\mu\|^2.$$
\end{proof}

\subsection{General Lower and upper bounds on the test error in a special case}

In this section we establish a useful Lemma~which provides lower and upper bounds on the test error in model \ref{model:M} under additional assumptions on the distribution of $\bs z_i$. This result will be used in the proof of Theorem~\ref{thm:noisyphasegeneral}. This can be seen as a complement to Lemma~\ref{testerror} which only provides an upper bound on the test error.

\begin{lemm}\label{test-error-spherical}
Assume that $(\bs x,y_{\n})$ is generated according to model \ref{model:M}. Suppose $\boldsymbol{z} \in \R^p$ in this model is a random vector such that $\bs z/\|\bs z\|, \|\bs z\|$ are independent and $\bs z/\|\bs z\|$ has a density $f$ with respect to the uniform distribution on the sphere $S^{p-1} := \{\|v\|=1:\, v\in \R^p\}$ such that $f_{min} \leq f \leq f_{max}$ for constants $0 < f_{\min} \leq f_{max} < \infty$.  
Suppose also that $\bs w\in \R^p$ is such that $\langle \boldsymbol{w}, \boldsymbol{\mu}\rangle >0$. Then, we have
\[
\frac{1-2\eta}{2} f_{min} \mathbb{P} \left( \|\bs z\| |u_1|   >  \frac{\langle \boldsymbol{w}, \boldsymbol{\mu} \rangle}{\|\boldsymbol{w} \|} \right) \leq \mathbb{P}\left( \langle \boldsymbol{w}, y_\n \boldsymbol{x} \rangle  <0 \right) - \eta \leq  \frac{1-2\eta}{2} f_{max} \mathbb{P} \left( \|\bs z\| |u_1|   >  \frac{\langle \boldsymbol{w}, \boldsymbol{\mu} \rangle}{\|\boldsymbol{w} \|} \right)
\]
where $u_1$ is the first entry of a random vector that has a uniform distribution on the sphere and is independent of $\|\bs z\|$.
\end{lemm}
\begin{proof} 
By the proof of Lemma~\ref{testerror},
\begin{align*}
\mathbb{P} \left( \langle \boldsymbol{w}, y_\n \boldsymbol{x} \rangle  <0 \right) &=  \eta \mathbb{P}\left( \langle \boldsymbol{w}, y\boldsymbol{x} \rangle  >0 \right) + (1 - \eta)\mathbb{P}\left( \langle \boldsymbol{w}, y\boldsymbol{x} \rangle  <0 \right)
\\
&= \eta + (1 - 2\eta)\mathbb{P}\left( \langle \boldsymbol{w}, y\boldsymbol{x} \rangle  <0 \right)
\end{align*}
and
\begin{align*} 
\mathbb{P}\left( \langle \boldsymbol{w}, y\boldsymbol{x} \rangle  <0 \right) 
 = \frac{1}{2}\mathbb{P} \left( |\langle \boldsymbol{w}, \boldsymbol{z} \rangle|  >  \langle \boldsymbol{w}, \boldsymbol{\mu} \rangle \right) 
 = \frac{1}{2}\mathbb{P} \left( \|\boldsymbol{z}\| \left|\left\langle \frac{\boldsymbol{w}}{\|\boldsymbol{w} \|}, \frac{\boldsymbol{z}}{\|\boldsymbol{z}\|} \right\rangle \right|  >  \frac{\langle \boldsymbol{w}, \boldsymbol{\mu} \rangle}{\|\boldsymbol{w} \|} \right).
\end{align*}

Next, observe that
\begin{align*}
&\mathbb{P} \left( \|\boldsymbol{z}\| \left|\left\langle \frac{\boldsymbol{w}}{\|\boldsymbol{w} \|}, \frac{\boldsymbol{z}}{\|\boldsymbol{z}\|} \right\rangle \right|  >  \frac{\langle \boldsymbol{w}, \boldsymbol{\mu} \rangle}{\|\boldsymbol{w} \|} \right)
\\
& = \int_{[0,\infty)}\int_{S^{p-1}} I\Big\{r\left|\left\langle \frac{\boldsymbol{w}}{\|\boldsymbol{w} \|}, \bs u \right\rangle \right|  >  \frac{\langle \boldsymbol{w}, \boldsymbol{\mu} \rangle}{\|\boldsymbol{w} \|} \Big\} f(\bs u) \mathrm{d}Q(\bs u)  \mathrm{d}P_{\|\bs z\|}(r)
\\
& \leq f_{max} \int_{[0,\infty)}\int_{S^{p-1}} I\Big\{r\left|\left\langle \frac{\boldsymbol{w}}{\|\boldsymbol{w} \|}, \bs u \right\rangle \right|  >  \frac{\langle \boldsymbol{w}, \boldsymbol{\mu} \rangle}{\|\boldsymbol{w} \|} \Big\} \mathrm{d}Q(\bs u)  \mathrm{d}P_{\|\bs z\|}(r)
\\
& = f_{max} \mathbb{P} \left( \|\bs z\| |u_1|   >  \frac{\langle \boldsymbol{w}, \boldsymbol{\mu} \rangle}{\|\boldsymbol{w} \|} \right)
\end{align*}
where $Q$ is the probability measure corresponding to a uniform distribution on $S^{p-1}$ and $u_1$ is the first entry of a random vector $\bs u$ that has a uniform distribution on the sphere and is independent of $\|\bs z\|$. The last equality holds because, for every rotation matrix $R\in O(p)$, 
\[
\Big\langle \frac{\bs w}{\|\bs w\|},\bs u\Big\rangle\;\overset{d}{=}\;\Big\langle\frac{\bs w}{\|\bs w\|},R\bs u\Big\rangle\;=\;\Big\langle R^\top \frac{\bs w}{\|\bs w\|},\bs u\Big\rangle
\]
and we can always choose $R$ so that $R^\top \tfrac{\bs w}{\|\bs w\|}=e_1$. A lower bound can be obtained similarly.
\end{proof}

\subsection{Proof of Theorem~\ref{thm:noiseless-main}} \label{sec:proofgennoiselessthm}

We provide separate proofs under the assumptions in (i) and (ii) of the theorem. As is evident from the proof, in addition, we obtain the following exponential test error bound assuming that $\bs z$ is sub-Gaussian with sub-Gaussian norm $\|\bs z\|_{\psi_2}$
\begin{equation}\label{eq:errboundnonoise-exp}
        \P_{(\boldsymbol{x}, y)}(\langle \boldsymbol{\hat w}, y\boldsymbol{x}\rangle < 0) \leq \;\exp\left\{- \frac{c}{\|\boldsymbol{z}\|_{\psi_2}^2}\left(\frac{1}{\|\bs \mu\|^2} + \frac{1}{n\rho \|\bs \mu\|^3}\right)^{-1}
\right\}
\end{equation}
in both regimes. Here, $c$ is a universal constant.

\subsubsection{Proof of Theorem~\ref{thm:noiseless-main} under regime (i)} \label{sec:proofgennoiselessthm(i)}
Our strategy for proving Part~(i) of this Theorem~is as follows. First, in Lemma~\ref{lem:whatrepr}, we give conditions under which the maximum margin classifier $\hat{\bs w}$ is equal to the least squares estimator. Then, in Lemma~\ref{noiseless-mmc-bound1}, we bound the quantity $\tfrac{\|\hat{\bs w}\|}{\<\hat{\bs w},\bs \mu\>}$. This gives us bounds on the test error by Lemma~\ref{testerror}. 

\begin{sloppypar}    
\begin{lemm}\label{lem:whatrepr}
In model \ref{model:M} with $\eta = 0$, suppose event $\bigcap_{i=1}^4 E_i$ holds with $(\alpha_2 \|\boldsymbol{\mu}\|\vee \varepsilon M) \sqrt{(1+ \beta)n\rho} \leq \tfrac14$, and $M\alpha_\infty\|\boldsymbol{\mu}\| (1+\beta)n\rho < \tfrac{3}{32}$.
    Then, $\boldsymbol{\hat w} = X^\top(XX^\top)^{-1} \boldsymbol{y}$. 
\end{lemm} 
\end{sloppypar}
\begin{proof}
Note that $\eps \in [0, \tfrac12]$ follows from the assumption $\varepsilon M\sqrt{(1+\beta)n\rho} \leq \tfrac14$ and \eqref{eq:M2rholowbound}. By Lemma~\ref{invertibility}, $XX^\top$ is invertible. By Lemma~\ref{maxmargin=LS}, it suffices to show that $\Delta(y)(XX^\top)^{-1}\boldsymbol{y}$ has positive entries. Using the fact that $d>0$, by Lemma~\ref{invertibility}, and the expression for $(XX^\top)^{-1}\boldsymbol{y}$ in Lemma~\ref{quad-decomp}, equivalently it suffices to show $(1+h)\boldsymbol{y}^\top A^{-1}y_i\bs e_i - s\boldsymbol{\nu}^\top A^{-1}y_i\bs e_i >0$ for each $i$.  Recall from the proof of Lemma~\ref{invertibility} that $\alpha_2\|\boldsymbol{\mu}\|\sqrt{(1+ \beta)n\rho} \leq \tfrac14$ implies $|h| \leq \tfrac12$. We also note that the assumptions $\alpha_2\|\boldsymbol{\mu}\| \sqrt{(1+ \beta)n\rho} \leq \tfrac14$ and $\varepsilon M\sqrt{(1+\beta)n\rho} \leq \tfrac14$ imply $M\eps\alpha_2\|\boldsymbol{\mu}\|(1+\beta)n\rho \leq \tfrac{1}{16}$.

By Lemma~\ref{quad-bounds-s}, Lemma~\ref{quad-bounds-ht}, Lemma~\ref{quad-bounds-yay}, and Lemma~\ref{quad-bounds-nuae}, we have
\begin{align*}
    & (1+h)\boldsymbol{y}^\top A^{-1}y_i\bs e_i - s\boldsymbol{\nu}^\top A^{-1}y_i\bs e_i \\
    & \qquad \qquad \geq \frac{1}{2}\frac{1 - \varepsilon \|\boldsymbol{z}_i\| \|\check{\boldsymbol y}\|}{(1 - \varepsilon^2)\|\boldsymbol{z}_i\|^2} - \frac{\|\check{\boldsymbol y}\|^2}{1-\varepsilon}\frac{2[\alpha_\infty \vee (\varepsilon\alpha_2)]\|\boldsymbol{\mu}\|}{(1 - \varepsilon^2)\|\boldsymbol{z}_i\|} \\
    & \qquad \qquad = \frac{1}{2(1 - \varepsilon^2)\|\boldsymbol{z}_i\|^2}\left\{1 - \varepsilon \|\boldsymbol{z}_i\|\|\check{\boldsymbol y}\| - \frac{4}{1 - \varepsilon} 
 \|\boldsymbol{z}_i\|\|\check{\boldsymbol y}\|^2[\alpha_\infty \vee (\varepsilon\alpha_2)]\|\boldsymbol{\mu}\| \right\} \\
    & \qquad \qquad \geq \frac{1}{2(1 - \varepsilon^2)\|\boldsymbol{z}_i\|^2} \left\{1 - \varepsilon M\sqrt{(1+\beta)n\rho} - \frac{4}{1 - \varepsilon}M[\alpha_\infty \vee (\varepsilon\alpha_2)]\|\boldsymbol{\mu}\| (1+\beta)n\rho \right\} \\
    & \qquad \qquad > 0, 
\end{align*}
where the last inequality is due to $M\eps\alpha_2\|\boldsymbol{\mu}\|(1+\beta)n\rho < \tfrac{3}{32}$ and the assumptions $\varepsilon M\sqrt{(1+\beta)n\rho} \leq \tfrac14$ and $M\alpha_\infty \|\boldsymbol{\mu}\| (1+\beta)n\rho < \tfrac{3}{32}$. 
\end{proof}

\noindent The next results gives us bounds on the fundamental quantity $\tfrac{\|\hat{\bs w}\|}{\<\hat{\bs w},\bs\mu\>}$.

\begin{lemm}\label{noiseless-mmc-bound1} 
In model \ref{model:M} with $\eta = 0$, suppose that the assumptions of Lemma~\ref{lem:whatrepr} hold and in addition suppose that $\beta\in [0, \tfrac12)$ and $\|\boldsymbol{\mu}\| \geq C \tfrac{\alpha_2}{\sqrt{(1-\beta)n\rho}}$ hold for sufficiently large constant $C$.  Then, we have 
\[
\left(\frac{\|\boldsymbol{\hat w}\|}{\langle \boldsymbol{\hat w}, \boldsymbol{\mu} \rangle}\right)^2 \; \asymp \; \frac{1}{\|\boldsymbol{\mu}\|^2} + \frac{1}{n\rho \|\boldsymbol{\mu}\|^4}.
\]
\end{lemm}
\begin{proof}
Since the assumptions of Lemma~\ref{lem:whatrepr} are satisfied, we have $\boldsymbol{\hat w} = X^\top(XX^\top)^{-1}\boldsymbol{y}$.

Note that the assumptions $\|\boldsymbol{\mu}\| \geq C \dfrac{\alpha_2}{\sqrt{(1-\beta)n\rho}}$ and $\alpha_2\|\boldsymbol{\mu}\|\sqrt{(1+ \beta)n\rho} \leq \tfrac14$ and Lemma~\ref{quad-bounds-ht} \ref{eq:t} imply that $\|\boldsymbol{\mu}\|^2 - t \asymp \|\boldsymbol{\mu}\|^2$ when $C$ is sufficiently large since
\[
t \;\leq\; \frac{\alpha_2^2\|\boldsymbol{\mu}\|^2}{(1 - \varepsilon)} \;\leq\; \frac{\alpha_2\|\boldsymbol{\mu}\|}{4(1 - \varepsilon)\sqrt{(1+\beta)n\rho}}\; \leq\; \frac{\sqrt{1 - \beta}}{4C(1 - \varepsilon)\sqrt{1+\beta}}\|\boldsymbol{\mu}\|^2.
\]

Furthermore, $s \asymp n\rho$ and $|h| \leq \tfrac12$ holds. Therefore,

$$
d   \; =\; s(\|\boldsymbol{\mu}\|^2 - t) + (1+h)^2 \; \asymp\; n\rho \|\boldsymbol{\mu}\|^2 + 1.
$$

Similarly, 

\begin{align*}
\|\boldsymbol{\hat w}\|^2 & = \boldsymbol{y}^\top (XX^\top)^{-1}XX^\top (XX^\top)^{-1}\boldsymbol{y} \\
    & = \boldsymbol{y}^\top (XX^\top)^{-1} \boldsymbol{y} \\
    & = d^{-1}\left\{(1+h)\boldsymbol{y}^\top A^{-1} - s \boldsymbol{\nu}^\top A^{-1}\right\}\boldsymbol{y} \quad \text{(by Lemma~\ref{quad-decomp})} \\
    & = d^{-1}\left\{(1+h)s - sh\right\} \\
    & = d^{-1}s \\
    & \asymp d^{-1}n\rho.
\end{align*}

Next, we obtain a bound for $\langle \boldsymbol{\hat w}, \boldsymbol{\mu} \rangle$. To do so, note that Lemma~\ref{quad-bounds-ht} \ref{eq:h} and the assumption $\|\boldsymbol{\mu}\| \geq C \dfrac{\alpha_2}{\sqrt{(1-\beta)n\rho}}$ imply
\begin{align*}
    |h| \;\leq\; \frac{\alpha_2\|\boldsymbol{\mu}\|\sqrt{(1+\beta)n\rho}}{1 - \varepsilon}\; \leq\; \frac{\sqrt{1 - \beta^2}}{(1 - \varepsilon)C}n\rho \|\boldsymbol{\mu}\|^2.
\end{align*}

Therefore, by taking $C$ sufficiently large, we have 
\begin{align*}
    \langle \boldsymbol{\hat w}, \boldsymbol{\mu} \rangle
        & = \boldsymbol{y}^\top(XX^\top)^{-1}X\boldsymbol{\mu} \\
        & = \|\boldsymbol{\mu}\|^2 \boldsymbol{y}^\top(XX^\top)^{-1}\boldsymbol{y} + \boldsymbol{y}^\top(XX^\top)^{-1} \boldsymbol{\nu} \\
        & = \|\boldsymbol{\mu}\|^2 d^{-1} s + d^{-1}\{(1+h)\boldsymbol{y}^\top A^{-1} - s \boldsymbol{\nu}^\top A^{-1}\} \boldsymbol{\nu} \quad \text{(by Lemma~\ref{quad-decomp})} \\
        & = d^{-1}\left\{s(\|\boldsymbol{\mu}\|^2 - t) + h(1+h)\right\} \\
        & \asymp d^{-1}n\rho \|\boldsymbol{\mu}\|^2.
\end{align*}

The proof of the expansion for $\tfrac{\|\boldsymbol{\hat w}\|}{\langle \boldsymbol{\hat w}, \boldsymbol{\mu} \rangle}$ follows from combining all the bounds derived above.
\end{proof}

By combining Lemma~\ref{testerror} and Lemma~\ref{noiseless-mmc-bound1} we obtain the statement of Theorem~\ref{thm:noiseless-main} in regime (i).

\subsubsection{Proof of Theorem~\ref{thm:noiseless-main} under regime (ii)} \label{sec:proofgennoiselessthm(ii)}
Next, we consider the regime $\|\boldsymbol{\mu}\| \gtrsim M$ to show (ii) of Theorem~\ref{thm:noiseless-main}. The main difference with the previous case is that here we cannot assure in general that $\boldsymbol{\hat w}$ is equal to the least squares estimator. We establish bounds on $\boldsymbol{\hat w}$ with an alternative geometric approach. An upper bound on $\|\boldsymbol{\hat w}\|$ will be obtained from the fact that every feasible point $\bs w_*$ leads to an upper bound on $\|\boldsymbol{\hat w}\|$. To obtain a lower bound, we show that the solution $\boldsymbol{\hat w}$ of \eqref{eq:mmc} must lie in a specific region $W\subset \R^p$ and establish a lower bound on $\|\boldsymbol{w} \|$ for every point $\boldsymbol{w} \in W$.

\begin{lemm}\label{noiseless-mmc-bound2} 
In model \ref{model:M} with $\eta = 0$, suppose event $E_3(M)$ holds with $\|\boldsymbol{\mu}\| \geq C M$ for some constant $C>2$. Then, the data are linearly separable and
\begin{enumerate}
    \item [(i)] $\displaystyle \frac{1}{\|\boldsymbol{\mu}\| + M}\; \leq \;\|\boldsymbol{\hat w}\| \;\leq \;\frac{1}{\|\boldsymbol{\mu}\| - M}$, 
    \item [(ii)] $\displaystyle \frac{\|\boldsymbol{\mu}\| - 2M}{\|\boldsymbol{\mu}\| - M} \;\leq \;\langle \boldsymbol{\hat w}, \boldsymbol{\mu} \rangle\; \leq \;\frac{\|\boldsymbol{\mu}\|}{\|\boldsymbol{\mu}\| - M}$,
    \item [(iii)] $\displaystyle \frac{\|\boldsymbol{\hat w}\|}{\langle \boldsymbol{\hat w}, \boldsymbol{\mu} \rangle} \; \asymp \; \frac{1}{\|\boldsymbol{\mu}\|}$.
\end{enumerate}
\end{lemm}
\begin{proof}
Let $\boldsymbol{w}_\ast = \left(\|\boldsymbol{\mu}\| - M\right)^{-1}\dfrac{\boldsymbol{\mu}}{\|\boldsymbol{\mu}\|}$. By the Cauchy-Schwarz inequality
\begin{align*}
\langle \boldsymbol{w}_\ast, y_i\boldsymbol{x}_i \rangle\; =\; \langle \boldsymbol{w}_\ast, \boldsymbol{\mu} \rangle + y_i\langle \boldsymbol{w}_\ast, \boldsymbol{z}_i \rangle \; \geq\; \frac{\|\boldsymbol{\mu}\|}{\|\boldsymbol{\mu}\| - M} - \frac{\|\boldsymbol{z}_i\|}{\|\boldsymbol{\mu}\| - M} \; \geq\; \frac{\|\boldsymbol{\mu}\| - M}{\|\boldsymbol{\mu}\| - M} \;=\; 1,
\end{align*}
which shows that $\bs w_*$ is a feasible point of the problem \eqref{eq:mmc}. This also shows that the data are linearly separable. By the definition of $\boldsymbol{\hat w}$, we have $\|\boldsymbol{\hat w}\|\; \leq \;\|\boldsymbol{w}_\ast\| \;=\; \tfrac{1}{\|\boldsymbol{\mu}\| - M}$.
To show the lower bound in Claim~(i), we note that $\|\boldsymbol{\hat w}\|^{-1}$ is the distance from the origin to the hyperplane $\langle \bs{\hat w}, y\bs x \rangle -1=0$ (this is due to the fact that the distance from a point $\bs v$ to a hyperplane $\bs w^\top \bs x + b =0$ is given by $|\bs w^\top \bs v + b|/\|\bs w\|$). Since all $y_i\bs x_i$'s are on the supporting hyperplane, we have 
\[\|\boldsymbol{\hat w}\|^{-1} \;\leq\;\min_i \|\boldsymbol{x}_i\| \;\leq \;\|\boldsymbol{\mu}\| + M. \]
To show Claim~(ii), let $\boldsymbol{x}_i$ be one of the support vectors which the maximum margin separating hyperplane passes through, that is, $\langle \boldsymbol{\hat w}, y_i\boldsymbol{x}_i \rangle = 1$ (It is clear that there exists at least one such point based on the optimization problem \eqref{eq:mmc}.). Then, we have
\[
\langle \boldsymbol{\hat w}, y_i\boldsymbol{x}_i \rangle = 1\quad \iff\quad \langle \boldsymbol{\hat w}, \boldsymbol{\mu} \rangle = 1 - y_i \langle \boldsymbol{\hat w}, \boldsymbol{z}_i \rangle.
\]
By Claim~(i) and the fact that $\|\boldsymbol{z}_i\| \leq M$, we have
\[
|\langle \boldsymbol{\hat w}, \boldsymbol{z}_i \rangle| \;\leq\; \|\boldsymbol{\hat w}\|\|\boldsymbol{z}_i\| \;\leq\; \frac{M}{\|\boldsymbol{\mu}\| - M}.
\]
This concludes the proof of Claim~(ii). Since $\|\boldsymbol{\mu}\| \geq CM$ for some constant $C>2$, the first and second results imply $\|\boldsymbol{\hat w}\| \asymp \|\boldsymbol{\mu}\|^{-1}$ and $\langle \boldsymbol{\hat w}, \boldsymbol{\mu} \rangle \asymp  1$, which conclude the proof of the last claim.
\end{proof}

The statement of Theorem~\ref{thm:noiseless-main} in regime (ii) now follows from Lemma~\ref{testerror} and Lemma~\ref{noiseless-mmc-bound2}. 

\subsection{Proof of Theorem~\ref{detail-noisy-main-1-simple}}\label{sec:proof:detail-noisy-main-1-simple}

Before proving Theorem~\ref{detail-noisy-main-1-simple}, we derive two key technical lemmas that make up most of the proof. The first Lemma~\ref{lem:whatreprnoisy} provides conditions which guarantee the equality $\boldsymbol{\hat w} = X^\top (XX^\top)^{-1} \boldsymbol{y}_{\n}$. We note that, in contrast to the noiseless case, where this can fail when $\|\bs\mu\|$ is very large, this equality always holds in the noisy regime. Using this characterization, we provide a sharp characterization of the order of $\|\boldsymbol{\hat w}\|/\langle \boldsymbol{\hat w}, \boldsymbol{\mu} \rangle$; see Lemma~\ref{noisy-mmc-bound1}. Combined with Lemma~\ref{testerror} this will prove the desired results.

\begin{lemm}\label{lem:whatreprnoisy}
For any $\eta \in(0,\tfrac12)$, if $(N_{C_{1,\eta}})$ holds with $C_{1,\eta} = \tfrac{22}{\eta}$ and $\eps\vee \beta\vee \gamma \leq \tfrac{\eta}{8}$, then $\boldsymbol{\hat w} = X^\top (XX^\top)^{-1} \boldsymbol{y}_{\n}$ holds.
\end{lemm}
\begin{proof}
First note that by Lemma~\ref{invertibility}, in both regimes (i) and (ii) in the specification of ($N_{C_{1,\eta}}$) we have $d>0$. Under (i) this follows because $\alpha_2\|\bs\mu\|\sqrt{(1+\beta)n\rho}\leq \tfrac{1}{30}\sqrt{1+\beta} \leq \tfrac{1}{30}\sqrt{\tfrac{17}{16}}<\tfrac14$ so that Lemma~\ref{invertibility} applies. Under (ii) this follows because $\alpha_2\leq C_{1,\eta}^{-\tfrac12}\leq \sqrt{\tfrac{\eta}{22}}\leq \tfrac{1}{\sqrt{44}} < \tfrac{1}{\sqrt{2}}$. Hence, again by Lemma~\ref{invertibility}, $XX^\top$ is invertible.  

By Lemma~\ref{maxmargin=LS}, to conclude that $\hat{\bs w}=X^\top(XX^\top)^{-1}\bs y_{\n}$, it suffices to show that each entry of the vector $\Delta(\bs y_{\n})(XX^\top )^{-1}\bs y_{\n}$ is strictly positive. Equivalently,  $d y_{\n,i}\bs y_{\n}^\top (XX^\top)^{-1} \bs e_i>0$ for all $i$. Plugging $d=s(\|\bs\mu\|^2-t)+(1+h)^2$ in the expression in Lemma~\ref{quad-decomp},  we alternatively get that  
\begin{equation}\label{eq:dyxe}
    \begin{aligned}
d y_{\n,i}\boldsymbol{y}_\n^\top (XX^\top)^{-1} \bs e_i
    & = \{s(\|\boldsymbol{\mu}\|^2 - t) + (1+h)^2\}(y_{\n,i}\boldsymbol{y}_\n^\top A^{-1} \bs e_i ) \\
    & \;\;- \{s_\n(\|\boldsymbol{\mu}\|^2 - t) + h_\n(1+h)\}(y_{\n,i}\boldsymbol{y}^\top A^{-1} \bs e_i) \\
    & \qquad \qquad + \{h_\n s - s_\n (1+h)\}(y_{\n,i}\boldsymbol{\nu}^\top A^{-1}\bs e_i) \\
    & = \{s (y_{\n,i}\boldsymbol{y}_\n^\top A^{-1} \bs e_i ) - s_\n (y_{\n,i}\boldsymbol{y}^\top A^{-1} \bs e_i)\}(\|\bs \mu\|^2-t) \\
    & \;\; +(1+h)^2(y_{\n,i}\boldsymbol{y}_\n^\top A^{-1} \bs e_i ) - h_\n(1+h)(y_{\n,i}\boldsymbol{y}^\top A^{-1} \bs e_i) \\
    & \qquad \qquad + \{h_\n s - s_\n (1+h)\}(y_{\n,i}\boldsymbol{\nu}^\top A^{-1}\bs e_i) 
\end{aligned}\end{equation}
needs to be positive for all $i$.
To show positivity of this expression, we start by establishing a lower bound of the coefficient of $\|\bs \mu\|^2 - t$ in \eqref{eq:dyxe}. By Lemma~\ref{quad-bounds-s}, Lemma~\ref{quad-bounds-yay}, and $\eps\vee \beta \vee \gamma \leq \tfrac{\eta}{8} \leq \tfrac{1}{16}$, we have

\begin{equation}\label{eq:sy-sy}
\begin{aligned}
&s (y_{\n,i}\boldsymbol{y}_\n^\top A^{-1} \bs e_i ) - s_\n (y_{\n,i}\boldsymbol{y}^\top A^{-1} \bs e_i)  
\\
\geq & \frac{1-\beta}{1+\eps}n\rho \frac{1-\eps M \sqrt{(1+\beta)n\rho}}{(1-\eps^2)\|\bs z_i\|^2} - \frac{(1-2\eta) + (\gamma + \eps(1+\beta))}{1-\eps^2}n\rho \frac{1+\eps M \sqrt{(1+\beta)n\rho}}{(1-\eps^2)\|\bs z_i\|^2}
\\
\geq & \frac{1-\beta}{1+\eps}n\rho \frac{1-\sqrt{1+\beta}\eta/2}{(1-\eps^2)\|\bs z_i\|^2} - \frac{(1-2\eta) + (\gamma + \eps(1+\beta))}{1-\eps^2}n\rho \frac{1+\sqrt{1+\beta}\eta/2}{(1-\eps^2)\|\bs z_i\|^2}
\\
= & \frac{n\rho}{(1-\eps^2)^2\|\bs z_i\|^2}\left\{(1-\eps)(1-\beta) - (1-\eps)(1-\beta)\sqrt{1+\beta}\eta/2 \right. 
\\
& \qquad \qquad \left. -[1-2\eta + \gamma + \eps(1+\beta)] - [(1-2\eta) + \gamma + \eps(1+\beta)]\sqrt{1+\beta}\eta/2\right\}
\\
= & \frac{n\rho}{(1-\eps^2)^2\|\bs z_i\|^2}\Big\{2\eta - \beta - \gamma - 2\eps \Big.
\\
& \qquad \qquad \Big. - [(1-\eps)(1-\beta) + 1 - 2\eta + \gamma + \eps(1+\beta)]\sqrt{1+\beta}\frac{\eta}{2}\Big\}
\\
\geq & \frac{n\rho}{(1-\eps^2)^2\|\bs z_i\|^2}\Big(\tfrac32 \eta-(1+1-2\eta+\tfrac{\eta}{8}+\tfrac{\eta}{8}(1+\tfrac{\eta}{8}))\sqrt{1+\tfrac{\eta}{8}}\tfrac{\eta}{2}\Big)\\
\geq & \frac{n\rho}{(1-\eps^2)^2\|\bs z_i\|^2}\Big(\tfrac32 \eta-\tfrac{\sqrt{1+\tfrac{\eta}{8}}}{2}(2-\frac74\eta+\tfrac{\eta^2}{64})\eta\Big)\\
\geq & \tfrac12 \eta \frac{n\rho}{(1-\eps^2)^2\|\bs z_i\|^2}\;\geq\; \tfrac12 \eta \frac{n\rho}{\|\bs z_i\|^2},
\end{aligned}
\end{equation}
where the last line follows since the function $\tfrac{\sqrt{1+\tfrac{\eta}{8}}}{2}(2-\frac74\eta+\tfrac{\eta^2}{64})$ is maximized at $\eta=0$.
Next, we bound each of the remaining values in \eqref{eq:dyxe}.

Using Lemma~\ref{quad-bounds-s} and the fact that $\eps\vee \beta \vee \gamma \leq \tfrac{\eta}{8} \leq \tfrac{1}{16}$ we get

\begin{equation}\label{eq:boundss}
\begin{aligned}
0.88n\rho \leq \tfrac{15}{17}n\rho \leq & s \text{(and $s_{\n\n}$)} \leq \tfrac{17}{15}n\rho \leq 1.14n\rho, \\
& s_\n \leq 1.14n\rho.
\end{aligned}
\end{equation}

Similarly, using Lemma~\ref{quad-bounds-ht} and the fact that $\eps\vee \beta \leq \tfrac{\eta}{8} \leq \tfrac{1}{16}$  we get
\begin{equation}\label{eq:hhnbound}
|h| \vee |h_\n| \;\leq\; 1.1\alpha_2\|\bs{\mu}\|\sqrt{n\rho}. 
\end{equation}

The lower bound 
\begin{equation}\label{eq:boundynn}
y_{\n,i}\,\boldsymbol{y}_\n^\top A^{-1} \bs e_i\;\geq \; 0.74\|\bs z_i\|^{-2}  
\end{equation}
follows by  Lemma~\ref{quad-bounds-yay} \ref{eq:ynyne2}, the inequality $\varepsilon M \sqrt{n\rho}\leq \tfrac{\eta}{2}$ and the fact that $\eps\vee \beta \leq \tfrac{\eta}{8} \leq \tfrac{1}{16}$.

The upper bound 
\begin{equation}\label{eq:boundyn}
|y_{\n,i}\,\boldsymbol{y}^\top A^{-1} \bs e_i |\;\leq\;  1.27\|\bs z_i\|^{-2}
\end{equation}
follows by Lemma~\ref{quad-bounds-yay} \ref{eq:yyne2}, the fact that $\varepsilon M\sqrt{n\rho}\leq \tfrac{\eta}{2}$, and again by the fact that $\eps\vee \beta \leq \tfrac{\eta}{8} \leq \tfrac{1}{16}$.

Finally, the bound
\begin{equation}\label{eq:boundnu}
|y_{\n,i}\,\boldsymbol{\nu}^\top A^{-1} \bs e_i| \;\leq\; 2.02 \frac{[\alpha_\infty\vee \varepsilon\alpha_2] \|\boldsymbol{\mu}\|}{\|\boldsymbol{z}_i\|} \leq 2.02 \frac{[\alpha_\infty\vee \varepsilon\alpha_2] \|\boldsymbol{\mu}\|M}{\|\boldsymbol{z}_i\|^2}
\end{equation}
follows by Lemma~\ref{quad-bounds-nuae} and the fact that $\eps \leq \tfrac{\eta}{8} \leq \tfrac{1}{16}$ and $\min_i \tfrac{M}{\|\bs z_i\|} \geq 1$.

Recalling $\|\bs \mu\|^2-t \geq 0$, \eqref{eq:dyxe} combined with~\eqref{eq:sy-sy}, \eqref{eq:boundynn}, \eqref{eq:boundyn}, and \eqref{eq:boundnu} yields

\begin{equation}\label{eq:dyxe2}
    \begin{aligned}
d y_{\n,i}\boldsymbol{y}_\n^\top (XX^\top)^{-1} \bs e_i
    \geq & 0.5\eta \frac{n\rho}{\|\bs z_i\|^2}(\|\bs \mu\|^2-t) + (1+h)^2\frac{0.74}{\|\bs z_i\|^2}
    \\
    & \;\;- |h_\n(1+h)|(1.27\|\bs z_i\|^{-2}) \\
    & \qquad \qquad + \{h_\n s - s_\n (1+h)\}(y_{\n,i}\boldsymbol{\nu}^\top A^{-1}\bs e_i) \\
     \geq& 0.5\eta \frac{n\rho}{\|\bs z_i\|^2}(\|\bs \mu\|^2-t) + (1+h)^2\frac{0.74}{\|\bs z_i\|^2}
    \\
    & \;\;- |h_\n(1+h)|(1.27\|\bs z_i\|^{-2}) \\
    & \qquad - \{|h_\n s| + |s_\n| (1+|h|)\}2.02 \frac{[\alpha_\infty\vee \varepsilon\alpha_2] \|\boldsymbol{\mu}\|M}{\|\boldsymbol{z}_i\|^2} 
\end{aligned}
\end{equation}
So far our bounds did not depend on the specific regime (i) or (ii). In the rest of the proof we take advantage of these additional restrictions.

\bigskip

\noindent\textbf{Regime (i):}  
Note that by \eqref{eq:hhnbound} and $\alpha_2\|\bs \mu\|\sqrt{n\rho}\leq \tfrac{1}{30}$, we can further bound $|h|\vee |h_\n|$ by 
\begin{equation}\label{eq:hhnbound(i)}
|h|\vee |h_\n| \leq 1.1\times \tfrac{1}{30}<0.04.
\end{equation}
With this, we obtain that $(1+h)^2 \geq 0.96^2$, $|h_\n(1+h)|\leq 0.04\cdot 1.04$, $|h_\n s-s_\n(1+h)|\leq (2\cdot 0.04\cdot 1.14+1.14)n\rho = 1.08\cdot 1.14n\rho$.

Further note that $\alpha_2\|\bs \mu\|\sqrt{n\rho}\leq \tfrac{1}{30}$ and $\eps M\sqrt{n\rho}\leq \tfrac{\eta}{2}$ imply
$$
\varepsilon\alpha_2\|\bs\mu\| M n \rho\;=\;(\varepsilon M \sqrt{n\rho})(\alpha_2\|\bs\mu\|\sqrt{n\rho})\;\leq\;\tfrac{\eta}{2} \tfrac{1}{30}\;<\;\tfrac{1}{64}.
$$
Thus, $[\alpha_\infty\vee (\eps \alpha_2)]\|\bs \mu\|Mn\rho \leq \tfrac{1}{64}$ holds.

With these, we obtain, from~\eqref{eq:dyxe2} after dropping the non-negative term $0.5\eta \tfrac{n\rho}{\|\bs z_i\|^2}(\|\bs \mu\|^2-t)$,
\begin{align*}
& d y_{\n,i}\boldsymbol{y}_\n^\top (XX^\top)^{-1}  \bs e_i \\
\geq & 0.96^2\cdot 0.74\|\bs z_i\|^{-2} - 0.04\cdot 1.04 \cdot 1.27\|\bs z_i\|^{-2} - 1.08\cdot 1.14 n\rho |\boldsymbol{\nu}^\top A^{-1}\bs e_i| \\
\geq & 0.96^2\cdot 0.74\|\bs z_i\|^{-2} - 0.04\cdot 1.04 \cdot 1.27\|\bs z_i\|^{-2} - 1.08\cdot 1.14 \cdot 0.04\|\bs z\|^{-2} \\
\geq & 0.57\|\bs z\|^{-2}.
\end{align*}

\bigskip

\noindent\textbf{Regime (ii):}. 
We look again at \eqref{eq:dyxe2}. Since the term $(1+h)^2\tfrac{0.74}{\|\boldsymbol{z}_i\|^2}$ is non-negative, we have
\begin{equation}\label{eq:dyxe3}
    \begin{aligned}
    d y_{\n,i}\boldsymbol{y}_\n^\top (XX^\top)^{-1} \bs e_i
    \geq & 0.5\eta \frac{n\rho}{\|\bs z_i\|^2}(\|\bs \mu\|^2-t)
    \\
    & \;\;- |h_\n(1+h)|(1.27\|\bs z_i\|^{-2}) \\
    & \qquad \qquad + \{h_\n s - s_\n (1+h)\}(y_{\n,i}\boldsymbol{\nu}^\top A^{-1}\bs e_i) 
\end{aligned}
\end{equation}

We will now show that the first term on the right-hand side of the equation dominates the other two. Note that under $N_{C_{1,\eta}}$, we have $\alpha_2^2 \leq \tfrac{1}{44}$, which combined with Lemma~\ref{quad-bounds-ht}\ref{eq:t} and $\eps \leq \tfrac{1}{16}$, implies $0\leq t \leq 0.03\|\bs \mu\|^2$, and hence 
\begin{equation}\label{eq:boundmu-t(ii)}
0.97\leq \|\bs \mu\|^2 - t\leq \|\bs \mu\|^2.
\end{equation}

Thus, we have
\begin{equation}\label{eq:lb1}
0.5\eta \frac{n\rho}{\|\bs z_i\|^2}(\|\bs \mu\|^2-t) \geq 0.48\eta \frac{n\rho\|\bs \mu\|^2}{\|\bs z_i\|^2}  
\end{equation}

Next we bound the other terms from above.
First, by~\eqref{eq:hhnbound}, 

\begin{align*}
|h_\n(1+h)|1.27\|\bs z_i\|^{-2} 
& \leq (1.1\alpha_2 \|\boldsymbol{\mu}\| \sqrt{n\rho} + (1.1\alpha_2 \|\boldsymbol{\mu}\| \sqrt{n\rho})^2)\cdot 1.27\|\bs z_i\|^{-2}\\
& \leq 1.4\frac{\alpha_2 \|\boldsymbol{\mu}\| \sqrt{n\rho}}{\|\bs z_i\|^2} + 1.54\frac{(\alpha_2 \|\boldsymbol{\mu}\| \sqrt{n\rho})^2}{\|\bs z_i\|^2}
\end{align*}

Second, by~\eqref{eq:boundss},~\eqref{eq:hhnbound}, and ~\eqref{eq:boundnu}, 
\begin{align*}
& \{|h_\n s| + |s_\n| (1+|h|)\}\cdot 2.02 \frac{[\alpha_\infty \vee (\eps \alpha_2)]\|\bs \mu\|M}{\|\bs z_i\|^2} \\
\leq & \left\{2\cdot 1.1\cdot 1.14\alpha_2 \|\bs \mu\| (n\rho)^{3/2} + 1.14n\rho\right\} \cdot 2.02 \frac{[\alpha_\infty \vee (\eps \alpha_2)]\|\bs \mu\|M}{\|\bs z_i\|^2} \\
\leq & \frac{n\rho\|\bs \mu\|^2}{\|\bs z_i\|^2}\left(5.07\alpha_2[\alpha_\infty \vee (\eps\alpha_2)]M\sqrt{n\rho} + 2.31 \frac{[\alpha_\infty \vee (\eps\alpha_2)]M}{\|\bs \mu\|}\right)
\end{align*}

In summary, we get

\begin{align*}
& d y_{\n,i}\boldsymbol{y}_\n^\top (XX^\top)^{-1} \bs e_i \\
\geq & \frac{n\rho\|\bs \mu\|^2}{\|\bs z_i\|^2}\left\{0.48\eta - 1.4\frac{\alpha_2}{\|\bs \mu\|\sqrt{n\rho}} - 1.54 \alpha_2^2 \right. \\
& \qquad \qquad \left. - 5.07\alpha_2[\alpha_\infty \vee (\eps\alpha_2)]M\sqrt{n\rho} - 2.31 \frac{[\alpha_\infty \vee (\eps\alpha_2)]M}{\|\bs \mu\|} \right\} \\
\geq & \frac{n\rho\|\bs \mu\|^2}{\|\bs z_i\|^2}\Big\{0.48\eta - 1.4C_{1,\eta}^{-1} - 1.54C_{1,\eta}^{-1} \Big. \\
&\qquad \qquad \Big. - 5.07\left[C_{1,\eta}^{-1}\vee \eps\alpha_2^2M\sqrt{n\rho}\right] - 2.31\left[C_{1,\eta}^{-1}\vee \frac{\eps \alpha_2 M }{\|\bs \mu\|}\right] \Big\}\\
\geq & \frac{n\rho\|\bs \mu\|^2}{\|\bs z_i\|^2}\Big\{0.48\eta - 1.4C_{1,\eta}^{-1} - 1.54C_{1,\eta}^{-1} - 5.07C_{1,\eta}^{-1} - 2.31C_{1,\eta}^{-1} \Big\}\\
= & \frac{n\rho\|\bs \mu\|^2}{\|\bs z_i\|^2}\Big(0.48\eta - 10.32C_{1,\eta}^{-1} \Big) = \frac{n\rho\|\bs \mu\|^2}{\|\bs z_i\|^2}\Big(0.48 - 10.32\cdot \frac{1}{22} \Big)\eta>0,
\end{align*}
where the second inequality follows by the assumptions $\|\boldsymbol{\mu}\| \geq C_{1,\eta} \alpha_\infty M$, $\alpha_2 \leq \|\bs\mu\|\sqrt{n\rho}C_{1,\eta}^{-1}$ and $\max\left\{\alpha_2^2, \alpha_2\alpha_\infty M\sqrt{n\rho} \right\} \leq {C_{1,\eta}}^{-1}$, and the third inequality follows by noting that $\|\boldsymbol{\mu}\| \geq C_{1,\eta}\eps \alpha_2M$ and $\eps \alpha_2^2M\sqrt{n\rho}\leq C_{1,\eta}^{-1}$ follow from $\|\boldsymbol{\mu}\| \geq C_{1,\eta}\tfrac{\alpha_2}{\sqrt{n\rho}}$, $\alpha_2^2 \leq C_{1,\eta}^{-1}$, and the assumption $\eps M\sqrt{n\rho}\leq \tfrac{\eta}{2}$. 
\end{proof}

\begin{lemm}\label{noisy-mmc-bound1}
For any $\eta \in(0,\tfrac12)$, if $(N_{C_{2,\eta}})$ holds with $C_{2,\eta} = \max\{\tfrac{22}{\eta}, \tfrac{17}{1-2\eta}\}$ and $\eps\vee \beta\vee \gamma \leq \frac{\min\{\eta, 1-2\eta\}}{8}$, then we have $\langle \boldsymbol{\hat w}, \boldsymbol{\mu} \rangle >0$ and 
\[
\left(\frac{\|\boldsymbol{\hat w}\|}{\langle \boldsymbol{\hat w}, \boldsymbol{\mu} \rangle}\right)^2 \asymp \frac{1}{(1-2\eta)^2}\left\{ \eta n\rho + \frac{1}{\|\bs\mu\|^2} + \frac{1}{n\rho \|\boldsymbol{\mu}\|^4} \right\}.
\]
\end{lemm}
\begin{proof} 
We first note that since $C_{2,\eta}\geq C_{1,\eta}$ and $\eps\vee \beta \vee \gamma \leq \tfrac{\min\{\eta, 1-2\eta\}}{8}\leq \tfrac{\eta}{8}$, we can use all the results that were established in the proof of Lemma~\ref{lem:whatreprnoisy}.

We start by establishing additional useful results which hold in both regimes. First, we prove that in both regimes we have
\begin{equation}\label{eq:boundmu-t}
0.97\|\bs \mu\|^2 \leq \|\bs \mu\|^2-t \leq \|\bs \mu\|^2,
\end{equation}
which has already been proved in regime (ii) as \eqref{eq:boundmu-t(ii)}. In regime (i) by $\|\boldsymbol{\mu}\| \geq C_{2,\eta} \tfrac{\alpha_2}{\sqrt{n\rho}}$, $\alpha_2 \|\boldsymbol{\mu}\| \sqrt{n\rho} \leq \tfrac{1}{30}$, and Lemma~\ref{quad-bounds-ht} \ref{eq:t} 
we have
\[
t \;\leq\; \frac{\alpha_2^2\|\boldsymbol{\mu}\|^2}{1-\varepsilon}\; \leq\; \frac{\alpha_2 \|\boldsymbol{\mu}\|\sqrt{n\rho}}{(1-\varepsilon)C_{2,\eta}}\|\boldsymbol{\mu}\|^2 \;\leq\; \frac{1}{30(1-\varepsilon)C_{2,\eta}}\|\boldsymbol{\mu}\|^2 ,
\]
which implies~\eqref{eq:boundmu-t} by $C_{2,\eta} \geq 44$ and $\eps \leq \tfrac{1}{16}$.

Furthermore, we have
\begin{equation}\label{eq:boundsn}
    0.74(1-2\eta)n\rho \leq s_\n \leq 1.27 (1-2\eta)n\rho.
\end{equation}
For the upper bound, we can see this by observing that $\eps \vee \beta \vee \gamma \leq \frac{1-2\eta}{8}$ and Lemma~\ref{quad-bounds-s} (iii) imply 
\[
s_\n \leq \frac{1-2\eta + \frac{1-2\eta}{8}+ \frac{1-2\eta}{8}(1+1/16)}{1-1/16^2}n\rho \leq 1.27(1-2\eta)n\rho.
\]
We can similarly obtain the lower bound.

Next, first give a useful expression for $\|\boldsymbol{\hat w}\|^2$ in order to later obtain explicit bounds. Using Lemma~\ref{quad-decomp}, notation introduced in \eqref{eq:notation}, and $d = s(\|\boldsymbol{\mu}\|^2-t) + (1+h)^2$, we get 
\begin{equation}\label{eq:repw}
\begin{aligned}
\|\boldsymbol{\hat w}\|^2 & = \boldsymbol{y}_\n^\top (XX^\top)^{-1} \boldsymbol{y_\n} \\
    & = \tfrac1d\Big[d s_{\n\n}  - \{s_\n(\|\boldsymbol{\mu}\|^2 -t) + h_\n(1+h) \} s_\n + (h_\n s - hs_\n - s_\n) h_\n\Big] \\
    & = \tfrac1d\{(s s_{\n\n} - s_\n^2 )(\|\boldsymbol{\mu}\|^2 -t) + (1+h)^2 s_{\n\n} - 2 s_\n h_\n(1 + h) + h_\n^2 s \}.
\end{aligned}
\end{equation}

By Lemma~\ref{quad-bounds-s} and $\eps\vee \beta \vee \gamma \leq \tfrac{\eta}{8}\leq \tfrac{1}{16}$, the lower bound of $ss_{\n\n} - s_\n^2$ can be obtained as follows:

\begin{align*}
& ss_{\n\n} - s_\n^2 \\
\geq & \left(\frac{1-\beta}{1+\eps}n\rho\right)^2 - \left\{\frac{1-2\eta + \gamma + \eps(1+\beta)}{1-\eps^2}n\rho\right\}^2 \\
= & \frac{(n\rho)^2}{(1-\eps^2)^2}\big\{(1-\beta)(1-\eps) + 1-2\eta +\gamma + \eps(1+\beta)\big\} \\
& \qquad \qquad \times \big\{(1-\beta)(1-\eps) - 1+2\eta -\gamma - \eps(1+\beta)\big\} \\
= & \frac{(n\rho)^2}{(1-\eps^2)^2}\big\{1 + 1-2\eta + \gamma -\beta + 2\eps\beta\big\}\big\{2\eta - 2\eps - \beta - \gamma \big\} \\
\geq & \frac{(n\rho)^2}{(1-\eps^2)^2}(1+1-2\eta  - \beta)\cdot \frac{3\eta}{2} \\
\geq & (n\rho)^2\left(1 + \frac{7(1-2\eta)}{8}\right) \cdot \frac{3\eta}{2} \\
\geq & \frac{3}{2} \eta (n\rho)^2.
\end{align*}

Similarly, 
\begin{align*}
& ss_{\n\n} - s_\n^2 \\
\leq & \left(\frac{1+\beta}{1-\eps}n\rho\right)^2 - \left\{\frac{1-2\eta - \gamma - \eps(1+\beta)}{1-\eps^2}n\rho\right\}^2 \\
= & \frac{(n\rho)^2}{(1-\eps^2)^2}\big\{(1+\beta)(1+\eps) + 1-2\eta -\gamma - \eps(1+\beta)\big\} \\
& \qquad \qquad \times \big\{(1+\eps)(1+\beta) - 1 + 2\eta + \gamma + \eps(1+\beta) \big\} \\
= & \frac{(n\rho)^2}{(1-\eps^2)^2}(1+1-2\eta + \beta - \gamma)\cdot (2\eta + 2\eps + \beta + \gamma + 2\beta \eps) \\
\leq & \frac{(n\rho)^2}{(1-1/16^2)^2}2 \cdot \frac{161\eta}{64} \\
\leq & 5.08 \eta (n\rho)^2.
\end{align*}

Together with \eqref{eq:boundmu-t}, we have 
\begin{equation}\label{eq:bound(ss-s2)mu}
1.45 \eta (n\rho\|\bs \mu\|)^2 \leq (ss_{\n\n}- s_\n^2)(\|\bs \mu\|^2 -t ) \leq 5.08 \eta (n\rho \|\bs \mu\|)^2.
\end{equation}

We now provide an explicit expression for $\langle \boldsymbol{\hat w}, \boldsymbol{\mu} \rangle$. From the representation in Lemma~\ref{quad-decomp}, some tedious but straightforward algebraic manipulations show that 
\begin{equation}\label{eq:noisyrepr}
\langle \hat{\bs w}, \bs \mu \rangle = \boldsymbol{y}_\n^\top 
= d^{-1}\{s_\n (\|\boldsymbol{\mu}\|^2 - t) + (1 + h)h_\n\}. 
\end{equation}
This is a purely algebraic statement and does not use any of the special assumptions made in this lemma. 

From here on we consider the two regimes from Assumption ($N_{C_{2,\eta}}$) separately and provide lower and upper bounds on $d, \|\hat{\bs{w}}\|^2, \langle \boldsymbol{\hat w}, \boldsymbol{\mu} \rangle$ in each regime separately.

\bigskip

\textbf{Regime (i): } 
By~\eqref{eq:boundss}, \eqref{eq:hhnbound(i)}, and \eqref{eq:boundmu-t}, we have
\begin{equation}\label{eq:boundd(i)}
\begin{aligned}
0.85(n\rho\|\bs \mu\|^2+1)\leq & 0.88\cdot 0.97n\rho \|\bs \mu\|^2  + 0.96^2 \\
 \leq & d = s(\|\bs \mu\|^2 - t) + (1+h)^2 \\
 \leq & 1.14n\rho \|\bs \mu\|^2 +1.04^2 \leq 1.14(n\rho \|\bs \mu\|^2 + 1).
\end{aligned}
\end{equation}

By~\eqref{eq:boundss} and \eqref{eq:hhnbound(i)}, we also have
\[
0.81n\rho \leq 0.96^2 \cdot 0.88 n\rho \leq (1+h)^2s_{\n\n} \leq 1.04^2\cdot 1.14n\rho \leq 1.24 n\rho.
\]

With~\eqref{eq:boundss}, \eqref{eq:hhnbound(i)} and \eqref{eq:boundsn}, we have
\[
|2s_\n h_\n(1+h)| \leq 2\cdot 1.27(1-2\eta)n\rho \cdot 0.04 \cdot 1.04 \leq 0.11 n\rho,
\]
\[
0\leq h_\n^2 s \leq 0.04^2\cdot 1.14n\rho \leq 0.01n\rho.
\]

Recalling the representation for $\|\boldsymbol{\hat w}\|^2$ in~\eqref{eq:repw} and also $\|\boldsymbol{\mu}\|^2 - t \geq 0.97 \|\boldsymbol{\mu}\|^2$, we have 
\[
0.7d^{-1}n\rho \left\{\eta n\rho \|\bs \mu\|^2 +1 \right\}\leq \|\bs{\hat w}\|^2 \leq 5.08 d^{-1}n\rho \left\{\eta n\rho \|\bs \mu\|^2 +1 \right\}.
\]

Note that $\|\boldsymbol{\mu}\| \geq C_{2,\eta} \tfrac{\alpha_2}{\sqrt{n\rho}}$ implies 
\begin{equation}\label{eq:somebound}
\alpha_2\|\boldsymbol{\mu}\|\sqrt{n\rho} \leq C_{2,\eta}^{-1}n\rho \|\boldsymbol{\mu}\|^2.  
\end{equation}

Since $|h_\n|\leq 1.1\alpha_2\|\bs \mu\|\sqrt{n\rho}$ by~\eqref{eq:hhnbound}, by \eqref{eq:somebound}, we have $|h_\n|\leq 1.1 C_{2,\eta}^{-1}n\rho \|\bs \mu\|^2$. Together with $|h|\leq 0.04$ by~\eqref{eq:hhnbound(i)} and $C_{2,\eta}\geq \frac{17}{1-2\eta}$, we have
\[
|(1+h)h_\n| \leq 1.04\cdot 1.1 C_{2,\eta}^{-1}n\rho \|\bs \mu\|^2 \leq 0.07 (1-2\eta)n\rho \|\bs \mu\|^2.
\]
Together with \eqref{eq:boundmu-t},  \eqref{eq:boundsn} and~\eqref{eq:noisyrepr} we have
\begin{align*}
    \langle \boldsymbol{\hat w}, \boldsymbol{\mu}\rangle 
    & \leq d^{-1}\left\{1.27(1-2\eta)n\rho\|\bs \mu\|^2 + 0.07(1-2\eta)n\rho \|\bs \mu\|^2\right\} \\
    & = 1.34(1-2\eta)d^{-1}n\rho \|\bs \mu\|^2,
\end{align*}
and
\begin{align*}
\langle \boldsymbol{\hat w}, \boldsymbol{\mu}\rangle & \geq d^{-1}\left\{0.74\cdot 0.97(1-2\eta)n\rho\|\bs \mu\|^2 - 0.07(1-2\eta)n\rho \|\bs \mu\|^2\right\} \\
& \geq  0.64(1-2\eta)d^{-1}n\rho \|\bs \mu\|^2.
\end{align*}

The proof of the expansion for $\tfrac{\|\boldsymbol{\hat w}\|}{\langle \boldsymbol{\hat w}, \boldsymbol{\mu} \rangle}$ follows now from combining all bounds derived above.

\medskip

\noindent\textbf{Regime (ii)} 

By \eqref{eq:hhnbound} and $\|\boldsymbol{\mu}\| \geq C_{2,\eta}\frac{\alpha_2}{\sqrt{n\rho}}$, we have
\[
|h|\vee |h_\n|\leq 1.1C_{2,\eta}^{-1}n\rho \|\bs\mu\|^2.
\]
Furthermore, since $\alpha_2^2\leq C_{2,\eta}^{-1}$, again using \eqref{eq:hhnbound}
\[
|h|^2\vee |h_\n|^2 \leq 1.21\alpha_2^2n\rho\|\bs \mu\|^2 \leq 1.21C_{2,\eta}^{-1}n\rho \|\bs \mu\|^2.
\]

In summary, we can assume throughout the remaining proof

\begin{equation}\label{eq:boundshhnh2hn2}
\begin{aligned}
|h|\vee |h_\n|\vee |h|^2\vee |h_\n|^2 & \leq 1.21C_{2,\eta}^{-1}n\rho \|\bs \mu\|^2 \\
& = 1.21\min\left\{\frac{\eta}{22}, \frac{1-2\eta}{17}\right\}n\rho \|\bs \mu\|^2 \\
& \leq \min\left\{0.06\eta, 0.08(1-2\eta)\right\}n\rho \|\bs \mu\|^2.
\end{aligned}
\end{equation}

Therefore, by \eqref{eq:boundss}, \eqref{eq:boundmu-t}, and \eqref{eq:boundshhnh2hn2}, noting that $\min\left\{0.06\eta, 0.08(1-2\eta)\right\} \leq 0.03 $
we have, recalling that $d = s(\|\bs \mu\|^2 - t) + (1+h)^2$

\[
0.85(n\rho \|\bs \mu\|^2 +1) \leq d \leq 1.23(n\rho \|\bs \mu\|^2 +1)
\]

Similarly, by \eqref{eq:boundss}, \eqref{eq:boundmu-t}, and \eqref{eq:boundshhnh2hn2}, we obtain the following useful bounds: 
\begin{align*}
0.88n\rho - 0.11\eta(n\rho\|\bs \mu\|)^2 \leq (1+h)^2s_{\n\n} &\leq 1.14n\rho + 0.21\eta(n\rho\|\bs \mu\|)^2,
\\
|2s_\n h_\n (1+h)| &\leq 0.28\eta (n\rho\|\bs \mu\|)^2,
\\
0\leq h_\n^2 s &\leq 0.07\eta(n\rho\|\bs \mu\|)^2.
\end{align*}
With these, the representation of $\|\bs{\hat w}\|^2$ in~\eqref{eq:repw} and \eqref{eq:bound(ss-s2)mu}, we have 
\[
0.88d^{-1}n\rho(\eta n\rho \|\bs \mu\|^2 +1) \leq \|\bs{\hat w}\|^2 \leq 5.64d^{-1}n\rho(\eta n\rho \|\bs \mu\|^2 +1).
\]

Finally, by \eqref{eq:boundss}, \eqref{eq:boundmu-t}, \eqref{eq:boundsn}, \eqref{eq:noisyrepr}, and \eqref{eq:boundshhnh2hn2}, we have
\[
0.55 (1-2\eta) d^{-1}n\rho\|\bs \mu\|^2 \leq \langle \boldsymbol{\hat w}, \boldsymbol{\mu} \rangle \leq 1.43(1-2\eta)d^{-1}n\rho\|\bs \mu\|^2.
\]
Combining all the bounds above proves the claim of the Lemma~in the second regime. 
\end{proof}

\begin{proof}[Proof of Theorem~\ref{detail-noisy-main-1-simple} and an exponential bound on the test error]
The claim of the Theorem~follows directly from Lemma~\ref{testerror} and Lemma~\ref{noisy-mmc-bound1}. If we further assume $\bs z$ is sub-Gaussian, an application of the second part of Lemma~\ref{testerror} yields the bound
\begin{equation}\label{eq:testerrornoisygenexp}
\P_{(\boldsymbol{x}, y_\n)}(\langle \boldsymbol{\hat w}, y_\n\boldsymbol{x}\rangle < 0) \;\;\leq\;\; \eta + (1-\eta)\exp\left\{-c_2 \frac{(1 - 2\eta)^2}{ \|\boldsymbol{z}\|_{\psi_2}^2}\left( \eta n\rho + \frac{1}{\|\bs \mu\|^2} + \frac{1}{n\rho \|\boldsymbol{\mu}\|^4}\right)^{-1}\right\},
\end{equation}
where $c_2$ is a universal constant. 
\end{proof}

\begin{proof}[Proof that~\eqref{eq:condsimplenoisymain} implies the conditions of Theorem~\ref{detail-noisy-main-1-simple}] 
We will show that the conditions $\eps M\sqrt{n\rho}\leq \tfrac{\eta}{2}$, $\|\boldsymbol{\mu}\| \geq C_{2,\eta} \frac{\alpha_2}{\sqrt{n\rho}}$, (ii) in ($N_{C_{2,\eta}}$), and $\eps\vee \beta \vee \gamma \leq \tfrac{\min\{\eta, 1-2\eta\}}{8}$ follow from the following simpler conditions:
\[
\|\boldsymbol{\mu}\| \geq C\alpha_2 M, \quad  [\alpha_2 \vee \eps] M\sqrt{n\rho} \leq C^{-1},\quad \beta \vee \gamma \leq C^{-1}
\]
by taking $C= \tfrac{17}{16}C_{2,\eta}$.

We first note $C^{-1} = \tfrac{16}{17}C_{2,\eta}^{-1} = \tfrac{16}{17}\min\{\tfrac{\eta}{22}, \tfrac{1-2\eta}{17}\}\leq \tfrac{\min\{\eta, 1-2\eta\}}{8}$.

With this and $\alpha_\infty \leq \alpha_2$, it is easy to see
\begin{align*}
& \eps M\sqrt{n\rho}  \leq C^{-1} \leq \tfrac{\eta}{2}, \\
& \alpha_\infty M  \leq \alpha_2 M \leq C^{-1}\|\bs \mu\| \leq C_{2,\eta}^{-1}\|\bs \mu\|.
\end{align*}

To show the remaining conditions, recall $M\sqrt{\rho} \geq \tfrac{1}{\sqrt{1+\beta}}\geq \sqrt{\tfrac{16}{17}}$ by~\eqref{eq:M2rholowbound}.
With this, $\alpha_2 M\sqrt{n\rho}\leq C^{-1}$ and $C \geq 1$ imply
\[
\alpha_2^2 \leq \tfrac{C^{-2}}{M^2 n\rho} \leq \tfrac{1+\beta}{n}C^{-1} \leq \tfrac{17}{16} \tfrac{16}{17}C_{2,\eta}^{-1} = C_{2,\eta}^{-1}
\]
Furthermore, $\alpha_\infty \leq \alpha_2$, $\alpha_2^2 \leq C_{2,\eta}^{-1}(<1)$ and $\alpha_2 M\sqrt{n\rho}\leq C^{-1}$ again imply
\[\alpha_2\alpha_\infty M\sqrt{n\rho} \leq \alpha_\infty C^{-1} \leq C_{2,\eta}^{-1/2} C_{2,\eta}^{-1} \leq C_{2,\eta}^{-1}.
\]
Next, by $M\sqrt{\rho} \geq \sqrt{\tfrac{16}{17}}$ again, $C\alpha_2 M \leq \|\bs \mu\|$ and $\alpha_2 M\sqrt{n\rho}\leq C^{-1}$ implies
\[
\frac{\alpha_2}{\sqrt{n\rho}} \;=\; \frac{\alpha_2 M}{M \sqrt{n\rho}} \leq \sqrt{\tfrac{17}{16}}\frac{C^{-1}\|\bs \mu\|}{\sqrt{n}} \leq \sqrt{\tfrac{17}{16}} \tfrac{16}{17}C_{2,\eta}^{-1} \|\bs \mu\| \leq C_{2,\eta}^{-1} \|\bs \mu\|.
\]
Finally, to see that $\eps \leq \tfrac{\min\{\eta, 1-2\eta\}}{8}$, note that 
\[
\eps \leq C^{-1} \frac{1}{M\sqrt{n\rho}} \leq C^{-1} \frac{1}{M\sqrt{\rho}} \leq \sqrt{\tfrac{17}{16}}\tfrac{16}{17}\min\{\tfrac{\eta}{22}, \tfrac{1-2\eta}{17}\} \leq \tfrac{\min\{\eta, 1-2\eta\}}{8}.
\]
\end{proof}

\subsection{Proof of Theorem~\ref{thm:noisyphasegeneral} }\label{sec:proof:thm:noisyphasegeneral}
\begin{proof}[Proof of Theorem~\ref{thm:noisyphasegeneral}]
Let $\zeta_{\bs{\hat w}, \bs \mu} := \tfrac{\|\bs{\hat w} \|}{\langle \bs{\hat w}, \boldsymbol{\mu} \rangle}$. The expansion for $\zeta_{\bs{\hat w}, \bs \mu}$ under the assumptions of Theorem~\ref{detail-noisy-main-1-simple} is established in Lemma~\ref{noisy-mmc-bound1}. By Lemma~\ref{test-error-spherical} we obtain
\[
\mathbb{P}\left( \langle \boldsymbol{w}, y_\n \boldsymbol{x} \rangle  <0 \right) - \eta \;\leq\; \frac{1-2\eta}{2} 
f_{max} \mathbb{P} \left( \|\bs z\| |u_1|   >  \frac{\langle \boldsymbol{w}, \boldsymbol{\mu} \rangle}{\|\boldsymbol{w} \|} \right) = \frac{1-2\eta}{2} 
f_{max} \mathbb{P} \left( \|\bs z\| |u_1|   > \zeta_{\bs{\hat w}, \bs \mu}^{-1}  \right)
\]
where $u_1$ is the first component of a random vector that follows a uniform distribution on the sphere and is independent of $\|z\|$. The lower bound is obtained by the same arguments.
\end{proof}

\section{Proofs for model \ref{model:EM}: Extended non-sub-Gaussian Mixtures} \label{sec:extendnonsG}

In this section, we apply the general results in model \ref{model:M}, Theorem~\ref{thm:noiseless-main} and Theorem~\ref{detail-noisy-main-1-simple},  to prove results in the extended non-sub-Gaussian mixture model given in \ref{model:EM}. We start by relating events $E_1,\dots,E_5$ from \eqref{eq:E1}--\eqref{eq:E5} to events $\Omega_1,\dots, \Omega_{4,2}$ defined below in~\eqref{eq:omega1}-\eqref{eq:omega42} which utilize the special structure of model \ref{model:EM}. In Section~\ref{sec:modelEM:Omega} we bound the probability of events $\Omega_1,\dots, \Omega_{4,2}$. This allows us to prove Lemma~\ref{lemm5-ex1} which studies events $E_1,\dots,E_5$ from \eqref{eq:E1}--\eqref{eq:E5}. This is the key ingredient to proving the main results in model \ref{model:EM}: Theorems~\ref{thm:noiseless-ext-detail1}, \ref{thm:noiseless-ext-detail2}, \ref{thm:noisy-ext-detail1-simple} and \ref{thm:phasetransitionsimple}.

\begin{rmk}
We note the following useful inequalities, which we later use repeatedly: 
\begin{equation}\label{eq:gnorm}
\|g\|_{L^\ell} \geq 1, \quad  \|g^{-1}\|_{L^2} \geq 1.    
\end{equation}
It is easy to see this since $\|g\|_{L^\ell} \geq \|g\|_{L^2} = 1$ and $\E [g^{-2}] = \E[(g^2)^{-1}]\geq (\E[g^2])^{-1} = 1$ by Jensen's inequality.
\end{rmk}

\subsection{General results for events \texorpdfstring{$E_1,\dots, E_5$}{E1, ..., E5}} \label{sec:genboundsevents}

The main goal of this section is to collect general results that will allow us to provide sufficient conditions for the events $E_1,\dots,E_5$ (c.f.~\eqref{eq:E1}--\eqref{eq:E5}) in terms of more elementary events related to the variables $g_i, W\bs\xi_i$ in model \ref{model:EM}.

Consider model \ref{model:EM} with $\bs z=gW\bs \xi$. We start by introducing useful notation that will be used throughout this section. Let $\boldsymbol{v}_i = W \boldsymbol{\xi}_i$ for $i=1,\ldots,n$ and $V = (\boldsymbol{v}_1,\ldots, \boldsymbol{v}_n)^\top \in \R^{n\times p}$. Let $\Delta(\boldsymbol v)$ be a diagonal matrix whose diagonal entries consist of $\|\boldsymbol{v}_i\|, i=1,\ldots, n$, and let $\check V = \Delta (\boldsymbol v)^{-1} V$. Moreover, define $\Delta (g)$ to be a diagonal matrix whose diagonal entries consist of $g_i, i= 1, \ldots, n$. Note that $\boldsymbol{z}=g \boldsymbol{v}$ and so $Z=\Delta(g)V$ and $\Delta(\bs z)=\Delta(g)\Delta(\boldsymbol v)$. It follows that
$$
\check Z\;=\;\Delta(\boldsymbol z)^{-1}Z\;=\;(\Delta(g)\Delta(\boldsymbol v))^{-1}\Delta(g) V\;=\;\Delta(\boldsymbol v)^{-1}V\;=\;\check V
$$
and thus $$\check A = \check Z \check Z^\top = \check V \check V^\top.$$
Define the following events
\begin{align}
\Omega_1(\eps) &= \{\|VV^\top - \Tr(\Sigma)I_n\| \leq \tfrac\eps4 \Tr(\Sigma) \},  \label{eq:omega1}
\\
\Omega_{2,2}(b_2) &= \{\|V \bs\mu\| \leq b_2\}, \label{eq:omega22}
\\
\Omega_{2,\infty}(b_\infty) &= \{\|V \bs\mu\|_\infty \leq \label{eq:omega2inf}b_\infty\},
\\
\Omega_3(D) &= \{\max_i |g_i| \leq D\}, \label{eq:omega3}
\\
\Omega_{4,1}(B_1) &= \Big\{\Big|\sum_i g_i^{-2} - \E[g_i^{-2}]\Big| \leq B_1\Big\}, \label{eq:omega41}
\\
\Omega_{4,2}(B_2) &= \Big\{\Big|\sum_i y_i y_{i,\n}g_i^{-2} - (1-2\eta)\E[g_i^{-2}]\Big| \leq B_2\Big\}. \label{eq:omega42}
\end{align}

Then we have the following Lemma. 
\begin{lemm}\label{lem:genboundsevents} 
Let $\eps\in [0,\tfrac12]$, $M = D(1+\eps)\sqrt{\Tr(\Sigma)}$, $\rho = \E[g^{-2}]/\Tr(\Sigma)$ and
\begin{align*}
\alpha_2 = \frac{2b_2}{\|\bs \mu\|\sqrt{\Tr(\Sigma)}}, \qquad \alpha_\infty = \frac{2b_\infty}{\|\bs \mu\|\sqrt{\Tr(\Sigma)}},
\\
\beta = \eps + \frac{2B_1}{n \E[g^{-2}]}, \qquad \gamma = \eps + \frac{B_1+B_2}{n\E[g^{-2}]}.
\end{align*}
Then $\Omega_{1}(\eps)  \subseteq E_1(\eps)$ and
\begin{align*}
\Omega_1(\eps)  \cap \Omega_{2,2}(b_2) \cap \Omega_{2,\infty}(b_\infty) &\subseteq E_2(\alpha_2,\alpha_\infty),
\\
\Omega_1(\eps) \cap \Omega_{3}(D)   &\subseteq E_3(M),
\\
 \Omega_1(\eps)  \cap \Omega_{4,1}(B_1) &\subseteq   E_4(\beta,\rho)
\\
\Omega_1(\eps)  \cap \Omega_{4,1}(B_1)  \cap \Omega_{4,2}(B_2)  &\subseteq   E_5(\gamma,\rho).
\end{align*}

\end{lemm}

\begin{proof}
For a square symmetric matrix $A = (a_{ij})$ we have $\max_{i} |a_{i,i}| \leq \|A\|$. Thus on $\Omega_1(\eps)$ we get 
\begin{equation}\label{eq:boundnormv1}
\max_i \left|\|\boldsymbol{v}_i\|^2 - \Tr(\Sigma)\right|\; \leq\; \frac{\varepsilon}{4} \Tr(\Sigma).   
\end{equation}
This implies
\begin{equation}\label{eq:boundnormv2}
\max_{i=1,\dots,n} \frac{\Tr(\Sigma)}{\|\boldsymbol{v}_i\|^2} \leq \frac{1}{1 - \varepsilon/4} \quad\quad \mbox{and } \quad\quad\max_{i=1,\dots,n} \left|\frac{\Tr(\Sigma)}{\|\boldsymbol{v}_i\|^2} - 1 \right|\leq \frac{\varepsilon/4}{1 - \varepsilon/4}. 
\end{equation} 
Those facts will be used repeatedly throughout the proof.

\bigskip

\textbf{Proof of the inclusion $\Omega_{1}(\eps)  \subseteq E_1(\eps)$ }
\noindent
From the bound on $\|VV^\top - \Tr(\Sigma)I_n\|$, we have 
\[
\|\check V \check V^\top - \Tr(\Sigma)\Delta(\boldsymbol v)^{-2} \| \;\leq\; \frac{\varepsilon}{4} \Tr(\Sigma)\|\Delta(\boldsymbol v)^{-2}\|,
\]
where $\|\Delta(\boldsymbol v)^{-2}\|=\max_i \tfrac{1}{\|v_i\|^2}$. By the triangle inequality we get that
\begin{align*}
\|\check V \check V^\top - I_n\| & \;\leq\; \|\check V \check V^\top - \Tr(\Sigma)\Delta(\boldsymbol v)^{-2} \| + \|\Tr(\Sigma)\Delta(\boldsymbol v)^{-2}  - I_n\| \\
& \;\leq\; \frac{\varepsilon}{4} \Tr(\Sigma)\|\Delta(\boldsymbol v)^{-2}\| + \|\Tr(\Sigma)\Delta(\boldsymbol v)^{-2}  - I_n\|.
\\
& \;\leq\; \max_{i=1,\dots,n} \frac{\Tr(\Sigma)\varepsilon}{4 \|\boldsymbol{v}_i\|^2} + \max_{i=1,\dots,n} \left|\frac{\Tr(\Sigma)}{\|\boldsymbol{v}_i\|^2}  - 1 \right|. 
\end{align*}
Combined with~\eqref{eq:boundnormv2} it follows that
\[
\|\check A - I_n\| \;=\; \|\check V \check V^\top - I_n\| \;\leq\; 2 \frac{\varepsilon/4}{1 - \varepsilon/4}\;\leq\;  4\frac{\varepsilon}{4} \;=\; \varepsilon.
\]
\bigskip 

\textbf{Proof of the inclusion $\Omega_{2,2}(b_2)\cap \Omega_{2,\infty}(b_\infty)\cap \Omega_1(\eps) \subseteq E_2(\alpha_2,\alpha_\infty)$ }

\noindent Recalling that $\check Z = \check V = \Delta(\bs v)^{-1} V$ we have, by~\eqref{eq:boundnormv2} and the assumption $\eps \le 1/2$, 
\[
\|\check Z\bs \mu\| = \|\check V\bs \mu\| \leq \max_{i=1,\ldots, n} \frac{\|V\bs\mu\|}{\|\bs v_i\|} \leq \frac{b_2}{\sqrt{(1 - \eps/4)\Tr(\Sigma)}} \leq \frac{2b_2}{\sqrt{\Tr(\Sigma)}}=\alpha_2\|\bs\mu\|.
\]
Similarly
\[
\|\check Z\bs \mu\|_\infty = \|\check V\bs \mu\|_\infty \leq \max_{i=1, \ldots, n} \frac{\|V\bs\mu\|_\infty}{\|\bs v_i\|} \leq \frac{b_\infty}{\sqrt{(1 - \eps/4)\Tr(\Sigma)}} \leq \frac{2b_\infty}{\sqrt{\Tr(\Sigma)}}=\alpha_\infty\|\bs\mu\|
\]
showing that $E_2(\alpha_2,\alpha_\infty)$ holds. 
\bigskip

\textbf{Proof of the inclusion $\Omega_1(\eps)   \cap  \Omega_{3}(D) \subseteq E_3(M)$ }

\noindent By~\eqref{eq:boundnormv1} we have
\[
\max_{i} \|\bs z_i\| \leq (\max_{i} |g_i|)(\max_i \|\bs v_i\|) \leq \sqrt{1+\tfrac\eps4} \sqrt{\Tr(\Sigma)} D \leq (1+\eps) D \sqrt{\Tr(\Sigma)} =M.
\]
\bigskip

\textbf{Proof of the inclusion $\Omega_1(\eps) \cap \Omega_{4,1}(B_1) \subseteq E_4(\beta,\rho)$ }

\noindent
On event $\Omega_1(\eps)$ we have, using~\eqref{eq:boundnormv2},
\begin{equation}\label{eq:conc_v2}
\max_i|\|\boldsymbol{v}_i\|^{-2} - \Tr(\Sigma)^{-1}|\leq \frac{\varepsilon/4}{(1 - \varepsilon/4)\Tr (\Sigma)} \leq \frac{\eps}{\Tr(\Sigma)}.
\end{equation}  

For generic i.i.d. random variables $\zeta_i$, it follows that 
\begin{equation}\label{eq:auxlemm3}
\begin{aligned}
&\left|\frac1n\sum_{i=1}^n \zeta_i\|\boldsymbol{v}_i\|^{-2} -  \E [\zeta]\Tr(\Sigma)^{-1} \right| 
\\
&\qquad \leq \left|\frac1n\sum_{i=1}^n \zeta_i (\|\boldsymbol{v}_i\|^{-2} - \Tr(\Sigma)^{-1})\right| + \left|\frac{1}{n\Tr(\Sigma)}\sum_{i=1}^n (\zeta_i  - \E [\zeta]) \right| 
\\
&\qquad \leq \frac{\varepsilon}{n\Tr(\Sigma)}\sum_{i=1}^n |\zeta_i| + \frac{1}{n\Tr(\Sigma)} \left|\sum_{i=1}^n (\zeta_i  - \E [\zeta]) \right|
\\
&\qquad \leq \frac{\varepsilon}{n\Tr(\Sigma)}\left(n\E|\zeta| + \left|\sum_{i=1}^n (|\zeta_i|  - \E|\zeta|) \right| \right) + \frac{1}{n\Tr(\Sigma)} \left|\sum_{i=1}^n (\zeta_i  - \E [\zeta]) \right|.
\end{aligned}\end{equation}
where going from the second to the third line we used the H\"older inequality together with \eqref{eq:conc_v2} to bound the first term. Going from the third to the fourth line we write $|\zeta_i|=|\zeta_i|-\E|\zeta_i|+\E |\zeta_i|$. 

Setting $\zeta_i = g_i^{-2}$ and noting that in this case $\zeta_i = |\zeta_i|$, the last line in~\eqref{eq:auxlemm3} is bounded from above by
\[
\frac{\varepsilon}{n\Tr(\Sigma)}\left(n\E[g^{-2}] + B_1 \right) + \frac{B_1}{n\Tr(\Sigma)} = \frac{\E[g^{-2}]}{\Tr(\Sigma)}\Big(\eps + (1+\eps)\frac{B_1}{n\E[g^{-2}]}\Big) \leq \rho \beta.
\]

\bigskip

\textbf{Proof of the inclusion $\Omega_1(\eps) \cap \Omega_{4,1}(B_1)  \cap \Omega_{4,2}(B_2) \subseteq   E_5(\gamma,\rho)$ }

\noindent Apply~\eqref{eq:auxlemm3} with $\zeta = g^{-2}yy_\n$ to obtain
\begin{align*}
&\left|\frac1n\sum_{i=1}^n \zeta_i\|\boldsymbol{v}_i\|^{-2} -  \E [\zeta]\Tr(\Sigma)^{-1} \right|
\\
&\qquad \leq \frac{\varepsilon}{n\Tr(\Sigma)}\left(n\E[g^{-2}] + \left|\sum_{i=1}^n (g_i^{-2}  - \E[g^{-2}]) \right| \right)
\\
&\qquad \qquad + \frac{1}{n\Tr(\Sigma)} \left|\sum_{i=1}^n (g_i^{-2}y_iy_{i,\n}  - (1-2\eta)\E|g^{-2}|) \right|.
\\
&\qquad \leq \frac{\varepsilon}{n\Tr(\Sigma)}\left(n\E|g^{-2}| + B_1\right) + \frac{B_2}{n\Tr(\Sigma)} 
\\
&\qquad = \frac{\E[g^{-2}]}{\Tr(\Sigma)}\Big(\eps + \frac{\eps B_1 +  B_2}{n\E[g^{-2}]} \Big) \leq \frac{\E[g^{-2}]}{\Tr(\Sigma)}\Big(\eps + \frac{B_1 +  B_2}{n\E[g^{-2}]} \Big) = \rho \gamma,
\end{align*}
where the first inequality is due to $\E[y_i y_{i,\n}] = 1-2\eta$. This completes the proof of the lemma.
\end{proof}

\newpage

\subsection{Bounds on the probability of \texorpdfstring{$\Omega_1,\dots, \Omega_{4,2}$}{O1,...,O4,5}}\label{sec:modelEM:Omega}

\subsubsection{Bound on \texorpdfstring{$\P(\Omega_1)$}{P(Omega 1)}, implying event \texorpdfstring{$E_1$}{E1} holds with high probability}

We recall the following useful result.

\begin{thm}[Bahr-Esseen inequality, \cite{10.1214/aoms/1177700291}]\label{beinq} 
Let $X_i, i = 1, 2, \dots, n$ be a sequence of random variables with $\mathbb{E} |X_i|^\nu < \infty$ for $1 \leq \nu \leq 2$ and put $\displaystyle S_n = \sum_{i=1}^n X_i$. If $X_i$'s satisfy that
\begin{equation}\label{eq:BEcondition}
\mathbb{E}(X_{m+1}| S_m) = 0 \text{ a.s. for all $m$,}    
\end{equation}
then 
\[
\mathbb{E} |S_n|^\nu \leq 2 \sum_{i=1}^n \mathbb{E} |X_i|^\nu
\]
\end{thm}

Note that the assumption \eqref{eq:BEcondition} is satisfied when $X_i$'s are independent and mean zero. By Bahr-Esseen inequality, we can show the following lemma.
\begin{lemm}\label{normbound} 
Suppose $\boldsymbol{\xi}=(\xi_1,\ldots,\xi_p)$ is a random vector satisfying \eqref{eq:condsxi} for some $r\in (2,4]$ and $K>0$. For any symmetric matrix $A \in \mathbb{R}^{p\times p}$, we have
\[
\mathbb{E} |\boldsymbol{\xi}^\top A \boldsymbol{\xi} - \Tr(A)|^{r/2}\; \leq \;C(r, K) p^{1 - r/4} \|A\|_{\textsc{f}}^{r/2}, 
\]
where $C(r, K)$ was defined in~\eqref{eq:defconstantsEM}.
\end{lemm}
\begin{proof}
Since $\tfrac{r}{2}\in (1,2]$, Jensen's inequality gives that for any $a,b\in \R$ 
$$\frac{1}{2^{r/2}}| a+ b|^{r/2}\;=\;|\tfrac12 a+\tfrac12 b|^{r/2}\;\leq\; \tfrac12|a|^{r/2}+\tfrac12 |b|^{r/2}$$
or, in other words, $| a+ b|^{r/2}\leq 2^{r/2-1}(|a|^{r/2}+|b|^{r/2})$. Exploiting this, we get
\begin{eqnarray}\label{expbound}
\mathbb{E}\left | \boldsymbol{\xi}^\top A \boldsymbol{\xi} - \Tr(A) \right|^{r/2}
	& =& \mathbb{E}\left  |\sum_{k=1}^p A_{kk} (\xi_{k}^2-1) + \sum_{k\neq l} A_{kl} \xi_{k} \xi_{l} \right|^{r/2}   \\
\nonumber	& \leq& 2^{r/2-1} \mathbb{E}\left| \sum_{k=1}^p A_{kk} (\xi_{k}^2-1) \right|^{r/2}  +\; 2^{r/2-1} \mathbb{E}\left| \sum_{k\neq l} A_{kl} \xi_{k} \xi_{l} \right|^{r/2}. %(\ref{expbound}    
\end{eqnarray}
The first term in (\ref{expbound}) can be bounded using Bahr-Esseen inequality as follows: 
\begin{align*}
\mathbb{E}\left| \sum_{j=1}^p A_{jj} (\xi_{j}^2-1) \right|^{r/2}
	& \leq 2 \sum_{j=1}^p |A_{jj}|^{r/2} \mathbb{E} |\xi_{j}^2 -1|^{r/2} \\
        & = 2 \sum_{j=1}^p |A_{jj}|^{r/2} \|\xi_{j}^2 -1\|_{L^{r/2}}^{r/2} \\
        & \leq 2 \sum_{j=1}^p |A_{jj}|^{r/2} \left(\|\xi_{j}^2\|_{L^{r/2}} + 1\right)^{r/2} \, \text{(by the triangle inequality)}\\
	& \leq 2 (K^{2/r} + 1)^{r/2}\sum_{j=1}^p |A_{jj}|^{r/2}
 \\
	& = 2 (K^{2/r} + 1)^{r/2}\sum_{j=1}^p (A_{jj}^2)^{r/4} \\
	& \leq 2 (K^{2/r} + 1)^{r/2} p^{1-r/4} \left( \sum_{j=1}^p A_{jj}^2 \right)^{r/4} \, \text{(by Jensen's inequality)} \\
	& \leq 2 (K^{2/r} + 1)^{r/2} p^{1-r/4} \|A \|_{\textsc{f}}^{r/2}.
\end{align*}
By noting that $\E |X|^{r/2}=\E(|X|^2)^{r/4}\leq (\E|X|^2)^{r/4}$ by Jensen's inequality, the second term in (\ref{expbound}) can be bounded as follows:
\begin{align*}
\mathbb{E}\left| \sum_{i\neq j} A_{ij} \xi_{i} \xi_{j} \right|^{r/2}
	& \leq \left( 2\sum_{i\neq j} A_{ij}^2 \mathbb{E}[ \xi_{i}^2 \xi_{j}^2] + \sum_{\{i,j\}\neq \{u,v\}, i\neq j, u\neq v} A_{ij} A_{uv} \mathbb{E} [\xi_{i} \xi_{j} \xi_{u} \xi_{v}] \right)^{r/4} \\
	& \leq  2^{r/4}\left( \|A \|_{\textsc{f}}^2 \right)^{r/4} = 2^{r/4}\|A \|_{\textsc{f}}^{r/2},
\end{align*}
where the last line follows since all terms in the second sum are zero by the assumption $\E[\xi_{j}] = 0, $ and independence. With (\ref{expbound}), we have
\begin{align*}
\mathbb{E}\left | \boldsymbol{\xi}^\top A \boldsymbol{\xi} - \Tr(A) \right|^{r/2}
	& \leq 2^{r/2-1}\left(2 (K^{2/r} + 1)^{r/2}  + \frac{2^{r/4}}{p^{1-r/4}} \right)p^{1-r/4} \|A \|_{\textsc{f}}^{r/2} \\
	& \leq 2^{r/2-1}\left\{2 (K^{2/r} + 1)^{r/2}  + 2^{r/4} \right\}p^{1-r/4} \|A \|_{\textsc{f}}^{r/2}.
\end{align*}
\end{proof}

\begin{lemm}\label{lemm1-ex1} 
Suppose model \ref{model:EM} holds. Then for
\begin{align}\label{eq:varepsilon}
\varepsilon & \;:=\; C_1(r,K)\delta^{-2/r} \!\!\max\{n^{2/r} p^{2/r-1/2}, n^{4/r} \} \frac{\|\Sigma\|_{\textsc{f}}}{\Tr(\Sigma)} % \;\leq\; \frac{1}{2}.
\end{align}
we have $\P(\Omega_1(\eps)) \geq 1-\delta$ where $C_1(r,K)$ was defined in~\eqref{eq:defconstantsEM}.
\end{lemm}

\begin{proof}
Recall the notation for $V, \bs v_i$ from the beginning of Section~\ref{sec:genboundsevents}. We will rely on the following ancillary inequality: For $a_i \geq 0, i=1, 2, \dots, n$ and $0 < s \leq 1$,
\[
\left ( \sum_{i=1}^n a_i \right)^s \leq \sum_{i=1}^n a_i^s,
\]
which can be proved by induction. 
Using this inequality with $s=\tfrac{r}{4}$, we can bound
\begin{align*}
\mathbb{E}\|VV^\top - \Tr(\Sigma) I_n \|_{\textsc{f}}^{r/2} 
	& =  \mathbb{E} \left (\sum_{i}  (\|\boldsymbol{v}_i\|^2 - \Tr(\Sigma) )^2 + \sum_{i\neq j} \langle \boldsymbol{v}_i, \boldsymbol{v}_j \rangle^2 \right)^{r/4} \\
	& \leq \sum_{i}  \mathbb{E} | \|\boldsymbol{v}_i\|^2 - \Tr(\Sigma) |^{r/2} + \sum_{i \neq j} \mathbb{E} | \langle \boldsymbol{v}_i, \boldsymbol{v}_j \rangle |^{r/2} \\
    & = \sum_{i}  \mathbb{E} | \boldsymbol{\xi}_i^\top W^\top W\boldsymbol{\xi}_i - \Tr(W^\top W) |^{r/2} + \sum_{i \neq j} \mathbb{E} | \langle \boldsymbol{v}_i, \boldsymbol{v}_j \rangle |^{r/2} \quad \text{by $\Tr(W^\top W) = \Tr(\Sigma)$} \\
	& \leq  C(r, K) np^{1-r/4} \|W^\top W\|_{\textsc{f}}^{r/2} + \sum_{i \neq j} \mathbb{E} | \langle \boldsymbol{v}_i, \boldsymbol{v}_j \rangle |^{r/2} \quad \text{(by Lemma~\ref{normbound})} \\
    & = C(r, K) np^{1-r/4} \|\Sigma\|_{\textsc{f}}^{r/2} + \sum_{i \neq j} \mathbb{E} | \langle \boldsymbol{v}_i, \boldsymbol{v}_j \rangle |^{r/2}. \quad \text{by $\|W^\top W\|_\textsc{f} = \|\Sigma\|_\textsc{f}$}
\end{align*}
The second term can be further bounded as follows:
\begin{align*}
\sum_{i \neq j} \mathbb{E} \left| \langle \boldsymbol{v}_i, \boldsymbol{v}_j \rangle \right|^{r/2}
	& \leq \sum_{i \neq j} \left\{\mathbb{E} \left| \langle \boldsymbol{v}_i, \boldsymbol{v}_j \rangle \right|^2 \right\}^{r/4} \\
	& = \sum_{i \neq j}\left\{\mathbb{E}\, \Tr(\boldsymbol{\xi}_i^\top W^\top W \boldsymbol{\xi}_j \boldsymbol{\xi}_j^\top W^\top W \boldsymbol{\xi}_i)  \right\}^{r/4} \\
	& = \sum_{i \neq j}\left\{\mathbb{E}\, \Tr(\boldsymbol{\xi}_i \boldsymbol{\xi}_i^\top W^\top W \boldsymbol{\xi}_j \boldsymbol{\xi}_j^\top W^\top W)  \right\}^{r/4} \\
    & = \sum_{i \neq j}\left\{\Tr(W^\top W W^\top W)  \right\}^{r/4} \\
	& = \sum_{i \neq j}\left\{\Tr(\Sigma^2)  \right\}^{r/4} \\
	& = \sum_{i \neq j}\left\{\|\Sigma\|_{\textsc{f}}^2  \right\}^{r/4} \\
	& \leq n^2 \|\Sigma\|_{\textsc{f}}^{r/2}.
\end{align*}

Noting that $C(r, K)\geq 5$ since $K\geq 1$, we have shown that 
\[
\mathbb{E}\|VV^\top - \Tr(\Sigma) I_n \|_{\textsc{f}}^{r/2} \;\leq\; 2C(r,K) \max\{n p^{1-r/4}, n^2\}\|\Sigma\|_{\textsc{f}}^{r/2}.
\]

By Markov inequality, this implies that, with probability at least $1 - \delta$,
\begin{equation*}
\|VV^\top - \Tr(\Sigma) I_n\|_{\textsc{f}} \;\leq\; \left\{\frac{2C(r,K)}{\delta}\right\}^{2/r}\!\!\! \max\{n^{2/r} p^{2/r-1/2}, n^{4/r} \}\|\Sigma\|_{\textsc{f}} \;=\; \frac{\varepsilon}{4} \Tr(\Sigma),
\end{equation*}
where $\varepsilon$ was defined in \eqref{eq:varepsilon}. This shows $\P(\Omega_1(\eps)) \geq 1-\delta$. 
\end{proof}

\subsubsection{Events \texorpdfstring{$E_2-E_5$}{E2 - E5} hold with high probability too}

We now show that the other events also hold with high probability. Recall the definitions of events $\Omega_{2,2}(b_2)$ in~\eqref{eq:omega22} and $\Omega_{2,\infty}(b_\infty)$ in \eqref{eq:omega2inf}. 

\begin{lemm}\label{lemm2-ex1}
\end{lemm}
\begin{proof}
Recall the notation for $V, \bs v_i$ from the beginning of Section~\ref{sec:genboundsevents}. Observe that
\[
\mathbb{E} \|V\boldsymbol{\mu} \|^2 = \mathbb{E} \sum_{i=1}^n \langle W \bs\xi_i, \boldsymbol{\mu} \rangle^2 = \sum_{i=1}^n (W^\top \boldsymbol{\mu})^\top (\mathbb{E} \bs\xi_i \bs\xi_i^\top) W^\top \boldsymbol{\mu}  = n \bs \mu^\top WW^\top \bs \mu = n \|\bs{\mu}\|_\Sigma^2.
\]
By Markov inequality, with probability at least $1 - \delta$, 
\[
\|V\boldsymbol{\mu}\| \;\leq\; \|\bs{\mu}\|_\Sigma\sqrt{\frac{n}{\delta}}.
\]
Since $\|V\boldsymbol{\mu}\|_\infty \leq \|V\boldsymbol{\mu}\|$ the claim follows.
\end{proof}

\begin{rmk} 
We can also show that the event $\Omega_{2,\infty}(b_\infty)$ holds with 
\[
b_\infty = \sqrt{2}\left(\frac{C(r, K)}{\delta}\right)^{1/r} n^{1/r}p^{1/r - 1/4}\|\bs{\mu}\|_\Sigma,
\]
where $C_2(r, K) = 2C(r,K)^{1/r}$ where the latter constant was defined in~\eqref{eq:defconstantsEM}. Note that  this is indeed smaller than $\|\bs{\mu}\|_\Sigma\sqrt{\tfrac{n}{\delta}}$ when $r=4$, but not necessarily smaller otherwise due to the presence of $p^{1/r-1/4}$. To show this, note that
\begin{align*}
\|\langle \boldsymbol{v}, \boldsymbol{\mu}\rangle\|_{L^{r}}^{2}
    & = \|\langle \boldsymbol{v}, \boldsymbol{\mu}\rangle^2\|_{L^{r/2}} \\
    & \leq \left\|\langle \boldsymbol{v}, \boldsymbol{\mu}\rangle^2 - \|\bs{\mu}\|_\Sigma^2 \right\|_{L^{r/2}} + \|\bs{\mu}\|_\Sigma^2 \\
    & = \left\|\boldsymbol{\xi}^\top (W^\top \boldsymbol{\mu}\boldsymbol{\mu}^\top W) \boldsymbol{\xi} - \Tr(W^\top \boldsymbol{\mu}\boldsymbol{\mu}^\top W) \right\|_{L^{r/2}} + \|\bs{\mu}\|_\Sigma^2 \\
    & \leq \left(C(r,K)p^{1-r/4}\|W^\top\boldsymbol{\mu}\boldsymbol{\mu}^\top W\|_{\textsc{f}}^{r/2}\right)^{2/r} + \|\bs{\mu}\|_\Sigma^2 \\
    & = \left\{\left(C(r,K)p^{1-r/4}\right)^{2/r} + 1 \right\}\|\bs{\mu}\|_\Sigma^2 \\
    & \leq 2 C(r, K)^{2/r} p^{2/r - 1/2} \|\bs{\mu}\|_\Sigma^2,
\end{align*}
where the first inequality is due to the triangle inequality, the second inequality is due to Lemma~\ref{normbound} and $\|\bs \mu\|_\Sigma^2 = \mu^\top W W^\top \bs \mu = \Tr(W^\top \bs \mu \bs \mu^\top W)$, the third equality is due to $\|W^\top \bs \mu \bs \mu^\top W\|_\textsc{f} = \|\bs \mu\|_\Sigma^2$, and the last inequality is due to $C(r,K)>5$ and $r \in (2, 4]$.

By the Markov inequality and the union bound argument, we have
\[
\P\left(\max_{i=1,\ldots,n}|\langle \boldsymbol{v}_i, \boldsymbol{\mu}\rangle| \geq t \right) \leq n\frac{2^{r/2}C(r, K)p^{1 - r/4}\|\bs{\mu}\|_\Sigma^r}{t^r}.
\]

Therefore, we have with probability at least $1-\delta$
\[
\|V\boldsymbol{\mu}\|_\infty = \max_i |\langle \boldsymbol{v}_i, \boldsymbol{\mu}\rangle| \leq \sqrt{2}\left(\frac{C(r, K)}{\delta}\right)^{1/r} n^{1/r}p^{1/r - 1/4}\|\bs{\mu}\|_\Sigma.
\]
\end{rmk}

\begin{lemm}\label{lemm3-ex1} 
Suppose model \ref{model:EM} holds. Set 
\[
B_1 = B_2 = 2^{1+2/k} (\tfrac{n}{\delta})^{2/k}\|g^{-2}\|_{L^{k/2}}.
\]
{Then for the events $\Omega_{4,1}, \Omega_{4,2}$ defined in~\eqref{eq:omega41}, \eqref{eq:omega42} we have $\P(\Omega_{4,1}(B_1)) \geq 1-\delta$ and $ \P(\Omega_{4,2}(B_2)) \geq 1-\delta$.}
\end{lemm}

\begin{proof}
We begin by proving a useful intermediate result. Let $\zeta_i$ denote i.i.d. copies of a random variable $\zeta$ with $\E[|\zeta|^a] < \infty$ for some $a \in (1,2]$. By the Lyapunov inequality ($\|X\|_{L^s}\leq \|X\|_{L^t}$ if $1\leq s \leq t$)
\[
\|\zeta - \E[\zeta]\|_{L^a}\; \leq \; \|\zeta\|_{L^a} + |\E[\zeta]| \; \leq \; \|\zeta\|_{L^a} + \|\zeta\|_{L^1} \; \leq\;  2 \|\zeta\|_{L^a}.
\]
Thus, by the Bahr-Esseen inequality (Theorem~\ref{beinq}), we have
\begin{align*}
\E \Big[\Big|\sum_{i=1}^n (\zeta_i - \E [\zeta_i])\Big|^{a}\Big] 
    & \leq 2\sum_{i=1}^n \E [\left|\zeta_i - \E [\zeta_i]\right|^{a}] 
     = 2n \|\zeta - \E [\zeta]\|_{L^{a}}^{a} 
     \leq 2n \left(2\|\zeta\|_{L^a}\right)^{a}.
\end{align*}
The Markov inequality implies
\[
\P \Big(\Big|\sum_{i=1}^n (\zeta_i - \E [\zeta_i]) \Big| \geq t\Big) \;\leq\; \frac{2n \left(2\|\zeta\|_{L^a}\right)^{a}}{t^{a}}
\]
and thus we have with probability at least $1-\delta$
\begin{equation}\label{eq:boundzeta}
\Big|\sum_{i=1}^n (\zeta_i - \E [\zeta_i]) \Big| \;\leq\; 2^{1+1/a} \|\zeta\|_{L^a}(\tfrac{n}{\delta})^{1/a}.
\end{equation}

Applying~\eqref{eq:boundzeta} with $\zeta = g^{-2}$ and $a = k/2$ shows that $\P(\Omega_{4,1}(B_1)) \geq 1-\delta$. Similarly, applying~\eqref{eq:boundzeta} with $\zeta = yy_{\n}g^{-2}$ and noting that $\E[yy_{\n}g^{-2}] = (1-2\eta)\E[g^{-2}]$ by independence between $yy_{\n}$ and $g$ shows that $\P(\Omega_{4,2}(B_2)) \geq 1-\delta$.

\end{proof}

\begin{lemm} \label{lemm4-ex1}
\sloppypar{Suppose model \ref{model:EM} holds. Set $D = (\tfrac{n}{\delta})^{1/\ell} \|g\|_{L^\ell}$. Then for the event $\Omega_3$ defined in~\eqref{eq:omega3} we have $\P(\Omega_3(D)) \geq 1-\delta$. When $\ell = \infty$, we interpret $(n/\delta)^{1/\infty} = 1$.}  
\end{lemm}

\begin{proof}
For $\ell<\infty$, the result follows from the bound
\[
\E[\max_i |g_i|^\ell] \leq n \max_i \E[|g_i|^\ell] = n \E[|g|^\ell],
\]    
and the Markov inequality. For $\ell = \infty$, $\P(\Omega_3(D)) = 1$ holds with $D = \|g\|_{L^\infty}$.
\end{proof}

\subsection{Proof of Lemma~\ref{lemm5-ex1}}\label{sec:proof:lemm5-ex1}

\begin{proof}[Proof of Lemma~\ref{lemm5-ex1}]
From Lemma~\ref{lem:genboundsevents} we see that the event $\bigcap_{j=1}^5 E_j$ holds on the intersection of all events in~\eqref{eq:omega1}--\eqref{eq:omega42}. From Lemma~\ref{lemm1-ex1}--Lemma~\ref{lemm4-ex1} we find that the probability of the latter intersection is bounded below by $1-5\delta$. The precise forms of $\eps,\beta, \gamma, \rho, \alpha_1,\alpha_2$ follow from those lemmas after straightforward algebraic manipulations. Similarly, the event $\bigcap_{j=1}^4 E_j$ holds on the intersection of the events from ~\eqref{eq:omega1}--\eqref{eq:omega41}, that is, it does not require $\Omega_{4,2}$ and hence this intersection has probability bounded from below by $1-4\delta$. \end{proof}

\subsection{Proof of Theorem~\ref{thm:noiseless-ext-detail1} }\label{sec:proof:EM:noiseless1}

\begin{proof}[Proof of Theorem~\ref{thm:noiseless-ext-detail1}] 

Below, we will prove that the following conditions are sufficient to obtain the conclusion: let $\tilde C$ be the constant in Theorem~\ref{thm:noiseless-main}(i). Assume that $n\geq (\tfrac{4C_2(g)}{\delta^{2/k}})^\frac{k}{k-2}$, that
\begin{equation}\label{cond:noiseless-ext-detail1_1}
    \|\boldsymbol{\mu}\|^2 \geq 2\sqrt{2}\tilde C\|g^{-1}\|_{L^2}^{-1}\delta^{-1/2}\|\bs{\mu}\|_\Sigma,
\end{equation}
and that 
\begin{align} \label{cond:noiseless-ext-detail1_2}
\Tr(\Sigma) \geq \|g\|_{L^\ell}\|g^{-1}\|_{L^2}\left(\tfrac{n}{\delta}\right)^{1/\ell} & \max\left\{5\sqrt{\tfrac{3}{2}}C_1(r,K)\left(\tfrac{n}{\delta}\right)^{2/r}n^{1/2}\|\Sigma\|_F \max\{p^{\frac{2}{r}-\frac{1}{2}},n^{2/r}\}, \right. \\
& \qquad \qquad \left. 40\|g^{-1}\|_{L^2}\sqrt{\tfrac{n}{\delta}}n\|\bs{\mu}\|_\Sigma\right\}. \notag
\end{align}
By taking $C = \max\{2\sqrt{2}\tilde C\|g^{-1}\|_{L^2}^{-1},\|g\|_{L^\ell}\|g^{-1}\|_{L^2}\max\{5\sqrt{\tfrac{3}{2}}C_1(r,K),40\|g^{-1}\|_{L^2}\} \}$ we obtain the claim of theorem.

We note that the assumption~\eqref{cond:noiseless-ext-detail1_2} implies 
\begin{equation}\label{eq:noiseless-ext-tr-cond1}
\Tr(\Sigma) \geq 5\sqrt{\tfrac{3}{2}} C_1(r,K)\|g\|_{L^\ell}\|g^{-1}\|_{L^2}(\tfrac{n}{\delta})^{\frac{2}{r}+\frac{1}{\ell}}\max\{p^{\frac{2}{r} - \frac{1}{2}}, n^{\frac{2}{r}}\}\sqrt{n}\|\Sigma\|_F.
\end{equation}
This together with \eqref{eq:gnorm} yields
\[
1 \geq 5\sqrt{\tfrac{3}{2}}\|g\|_{L^\ell}\|g^{-1}\|_{L^2} \sqrt{n}\eps \geq 5\sqrt{\tfrac{3}{2}}\sqrt{n}\eps.
\]
Thus $\eps \leq \tfrac{1}{5}<\tfrac{1}{4}$, and the conditions of Lemma~\ref{lemm5-ex1} hold. The condition $n\geq (\tfrac{4C_2(g)}{\delta^{2/k}})^\frac{k}{k-2}$ implies $C_2(g)\delta^{-2/k}n^{-(1-2/k)} \leq \tfrac{1}{4}$ and this in turn yields $\beta <\tfrac{1}{2}$. We now check carefully that the other conditions of Theorem~\ref{thm:noiseless-main}(i) are also satisfied. 

The condition $\|\bs\mu\|\sqrt{(1-\beta)n\rho}\geq \tilde C\alpha_2$ is equivalent to 
$$
\|\bs\mu\|\sqrt{(1-\beta)n\tfrac{\E[g^{-2}]}{\Tr(\Sigma)}}\;\geq \tilde C \frac{2\sqrt{n}\|\bs{\mu}\|_\Sigma}{\sqrt{\delta\Tr(\Sigma)}\|\bs\mu\|},
$$
which, after some algebra, becomes
$$
\|\bs\mu\|^2\;\geq\;\frac{2}{\sqrt{1-\beta}}\tilde C\|g^{-1}\|^{-1}_{L^2}\delta^{-1/2}\|\bs{\mu}\|_\Sigma.
$$
Since $\beta< \tfrac12$, the above inequality follows by the assumption~\eqref{cond:noiseless-ext-detail1_1}.

The condition $\eps M\sqrt{(1+\beta)n\rho}\leq \tfrac{1}{4}$ is equivalent to
\[
\Tr(\Sigma) \geq 4\sqrt{1+\beta}(1+\eps)C_1(r,K)\|g\|_{L^\ell}\|g^{-1}\|_{L^2}(\tfrac{n}{\delta})^{\frac{2}{r}+\frac{1}{\ell}}\max\{p^{\frac{2}{r} - \frac{1}{2}}, n^{\frac{2}{r}}\}\sqrt{n}\|\Sigma\|_F,
\]
which follows from~\eqref{eq:noiseless-ext-tr-cond1} and the bounds $\eps < \tfrac{1}{4}, \beta<\tfrac{1}{2}$.

Next, note that the assumption~\eqref{cond:noiseless-ext-detail1_2} implies
\begin{equation}\label{eq:noiseless-ext-tr-cond3}
\Tr(\Sigma)\geq 40\|g\|_{L^\ell}\|g^{-1}\|_{L^2}^2\left(\tfrac{n}{\delta}\right)^{\frac{1}{2}+\frac{1}{\ell}}n\|\bs{\mu}\|_\Sigma. 
\end{equation}

The condition $\alpha_2\|\bs \mu\|\sqrt{(1+\beta)n\rho}\leq \tfrac{1}{4}$ is equivalent to
\[
\Tr(\Sigma) \geq 8\sqrt{1+\beta}\|g^{-1}\|_{L^2}\sqrt{n}\left(\tfrac{n}{\delta}\right)^{1/2}\|\bs{\mu}\|_\Sigma.
\]
This follows from~\eqref{eq:noiseless-ext-tr-cond3} and $\beta < \tfrac12$. 

Finally, $M\alpha_\infty \|\bs \mu\|(1+\beta)n\rho <\tfrac{3}{32}$ is equivalent to
\[
\Tr(\Sigma) > \tfrac{64}{3}(1+\eps)(1+\beta)\|g\|_{L^2}\|g^{-1}\|_{L^2}\left(\tfrac{n}{\delta}\right)^{\frac{1}{2}+\frac{1}{\ell}}n\|\bs{\mu}\|_\Sigma.
\]
This follows from~\eqref{eq:noiseless-ext-tr-cond3} and $\eps < \tfrac14, \beta < \tfrac12$. This completes the verification of the conditions of Theorem~\ref{thm:noiseless-main} (i). The conclusion follows from this result combined with the fact that $\rho = \E[g^{-2}]\Tr(\Sigma)^{-1}$.  
\end{proof}

\subsection{Proof of Theorem~\ref{thm:noiseless-ext-detail2} }\label{sec:proof:EM:noiseless2}

\begin{proof}[Proof of Theorem~\ref{thm:noiseless-ext-detail2}] 
By Lemma~\ref{lem:genboundsevents} event $E_3(M)$ holds with $M = D(1+\eps)\sqrt{\Tr(\Sigma)}$ on $\Omega_1(\eps) \cap \Omega_3(D)$. Use Lemma~\ref{lemm4-ex1} to see that $\P(\Omega_3(D)) \geq 1-\delta$ for $D = (\tfrac{n}{\delta})^{1/\ell} \|g\|_{L^\ell}$ and Lemma~\ref{lemm1-ex1} to get $\P(\Omega_1(\eps)) \geq 1-\delta$ for $\eps = C_1(r,K)\delta^{-2/r} \!\!\max\{n^{2/r} p^{2/r-1/2}, n^{4/r} \} \frac{\|\Sigma\|_{\textsc{f}}}{\Tr(\Sigma)}$. Thus $\P(E_3(M)) \geq \P(\Omega_1(\eps) \cap \Omega_3(D)) \geq 1-2\delta$. The condition~\eqref{cond2:noiseless-ext-detail2} implies that $\eps \leq \tfrac12$. Now the inequality $\|\bs\mu\| \geq CM$ follows by condition~\eqref{cond:noiseless-ext-detail2} and $\eps \leq \tfrac{1}{2}$. In consequence, the conditions of Theorem~\ref{thm:noiseless-main}(ii) hold and the conclusion follows. 
\end{proof}

\subsection{Proof of Theorem~\ref{thm:noisy-ext-detail1-simple}} \label{sec:proof:thm:noisy-ext-detail1-simple}

\begin{proof}[Proof of Theorem~\ref{thm:noisy-ext-detail1-simple}] 

We will prove that the following conditions are sufficient to obtain the conclusion of Theorem~\ref{thm:noisy-ext-detail1-simple}. From those conditions, it is clear how large the unspecified constant in the statement of Theorem~\ref{thm:noisy-ext-detail1-simple} needs to be.

Recall the definition $C_{2,\eta}=\max\{\tfrac{22}{\eta},\tfrac{17}{1-2\eta}\}$. For any $\delta \in (0, \tfrac{1}{5})$, if 
\begin{equation}\label{cond:noisy-ext-detail1}
    \|\boldsymbol{\mu}\|^2\; \geq \;2C_{2,\eta}\|g^{-1}\|_{L^2}^{-1}\delta^{-1/2}\|\bs{\mu}\|_\Sigma,
\end{equation} 
and
\begin{equation}\label{cond2:noisy-ext-detail1}
\Tr(\Sigma)\; \geq\; \tfrac{33}{16\eta}C_1(r,K)\|g\|_{L^\ell}\|g^{-1}\|_{L^2}(\tfrac{n}{\delta})^{2/r+1/\ell} n^{1/2}\max\{p^{2/r-1/2}, n^{2/r}\}\|\Sigma\|_{\textsc{f}}   
\end{equation}
together with  one of the following holds:
\begin{enumerate}
    \item[(i)] $\Tr(\Sigma) \geq 132 \|g\|_{L^\ell}\|g^{-1}\|_{L^2}^2(\tfrac{n}{\delta})^{1/2+1/\ell} n\|\bs{\mu}\|_\Sigma$, 
    \item[(ii)] $\|\boldsymbol{\mu}\|^2 \geq \tfrac{33}{16}C_{2,\eta}\|g\|_{L^l} (\tfrac{n}{\delta})^{1/2+1/\ell}\|\bs{\mu}\|_\Sigma$
    and \\
    $\displaystyle \Tr(\Sigma) \geq \tfrac{33}{8}C_{2,\eta} \|g\|_{L^l}\|g^{-1}\|_{L^2}(\tfrac{n}{\delta})^{1+1/\ell}n^{1/2}\frac{\|\bs{\mu}\|_\Sigma^2}{\|\boldsymbol{\mu}\|^2}$. 
\end{enumerate}
Assume further that $n \geq \delta^{-\frac{2}{k-2}}(\tfrac{16 C_2(g)}{\min\{\eta,1-2\eta\}})^{\tfrac{k}{k-2}}$. We now show that the above conditions are sufficient to obtain the conclusion of Theorem~\ref{thm:noisy-ext-detail1-simple}.

First, note that $C_2(g) \geq 4$ and $\tfrac{k}{k-2} \geq 2$ combined with the assumed lower bound on $n$ imply $n \geq \left(\tfrac{16}{\min\{\eta,1-2\eta\}}\right)^2 > 8^2$. 

Consider the definitions from Lemma~\ref{lemm5-ex1}. For $\eps$ defined therein, condition \eqref{cond2:noisy-ext-detail1} is equivalent to
\[
1 \geq \eps \frac{33}{16\eta}\|g\|_{L^\ell} \|g^{-1}\|_{L^2} (\tfrac{n}{\delta})^{1/\ell} n^{1/2} > 4 \sqrt{n} \eps
\]
where the last inequality uses \eqref{eq:gnorm}. The above inequality and $n>8^2$ imply $\eps < \tfrac{1}{32}$ which in turn shows that condition \eqref{cond:lemm5-ex1} in Lemma~\ref{lemm5-ex1} holds.
Consequently, with probability $1-5\delta$ the event $\bigcap_{i=1}^5 E_i$ holds with the parameters defined therein. The bound $\eps < \tfrac{1}{4\sqrt{n}}$ and $n \geq \left(\tfrac{16}{\min\{\eta,1-2\eta\}}\right)^2$ further show $\eps < \tfrac{\min\{\eta,1-2\eta\}}{16}$. By definition of $\beta, \gamma$, the assumption on $n$ implies $\eps\vee \beta\vee \gamma\leq \min\{\eta,1-2\eta\}/8$.

We now verify that the remaining conditions of  Theorem~\ref{detail-noisy-main-1-simple} hold by proving the conditions in $(N_{C_{2,\eta}})$. We first verify the condition $\eps M\sqrt{n\rho}\leq \tfrac{\eta}{2}$, which holds by \eqref{cond2:noisy-ext-detail1} and the fact that $\eps< \tfrac{1}{32}$. The condition $\|\bs\mu\|\geq C_{2,\eta}\tfrac{\alpha_2}{\sqrt{n\rho}}$ is equivalent to \eqref{cond:noisy-ext-detail1}. 

We now show that the conditions in (i) in the present Theorem~imply the conditions in (i) of $(N_{C_{2,\eta}})$.
Condition $\alpha_2\|\bs\mu\|\sqrt{n\rho}\leq \tfrac{1}{30}$ in part (i) of the definition of the event $(N_{C_{2,\eta}})$ is equivalent to $\Tr(\Sigma)\geq \tfrac{60}{\sqrt{\delta}}n\|g^{-1}\|_{L^2}\|\bs{\mu}\|_\Sigma$. Condition $\alpha_\infty\|\bs\mu\|Mn\rho\leq\tfrac{1}{64}$ is equivalent to 
$$
\Tr(\Sigma)\;\geq\; 128 (1+\eps)\|g\|_{L^\ell}\|g^{-1}\|^2_{L^2} n (\tfrac{n}{\delta})^{1/\ell+1/2}\|\bs{\mu}\|_\Sigma.
$$
It is easy to see both of these conditions hold by condition (i) in the theorem, the fact that $\eps< \tfrac{1}{32}$ and \eqref{eq:gnorm}. This completes the verification of (i) in the definition of ($N_{C_{2,\eta}}$).

We now verify part (ii) in the definition of ($N_{C_{2,\eta}}$) under (ii) of the present theorem. Condition $\|\bs\mu\|\geq C_{2,\eta}\alpha_\infty M$ is equivalent to 
$$
\|\bs\mu\|^2\;\geq\;2 C_{2,\eta}(1+\eps)\|g\|_{L^\ell}(\tfrac{n}{\delta})^{1/\ell+1/2}\|\bs{\mu}\|_\Sigma,
$$
which follows by the first condition in (ii) and the fact that $\eps < \tfrac{1}{32}$. Finally, to check the condition $\max\{\alpha_2^2,\alpha_2\alpha_\infty M\sqrt{n\rho}\}\leq C_{2,\eta}^{-1}$ is equivalent to 
\[
\Tr(\Sigma) \geq 4C_{2,\eta} \frac{n}{\delta}\max\left\{1,  (1+\eps) \|g\|_{L^\ell}\|g^{-1}\|_{L^2}\left(\frac{n}{\delta}\right)^{1/\ell}\sqrt{n}\right\}\frac{\|\bs{\mu}\|_\Sigma^2}{\|\bs \mu\|^2},
\]
which follows by the second condition in (ii), the fact that $\eps < \tfrac{1}{32}$ and \eqref{eq:gnorm}.

The conclusion now follows by Theorem~\ref{detail-noisy-main-1-simple}. 
\end{proof}

\subsection{Proof of Corollary~\ref{cor:sigmaid}}\label{sec:proof:cor:sigmaid}

\begin{proof}[Proof of Corollary~\ref{cor:sigmaid}] 
The conditions \eqref{cond:noisy-ext-detail1-simple} and \eqref{cond2:noisy-ext-detail1-simple} of Theorem~\ref{thm:noisy-ext-detail1-simple} can be summarized as
\[
\|\bs \mu\|\gtrsim 1 \quad \text{and} \quad p\gtrsim \max\left\{n^\frac{4+(1+2/\ell)r}{2(r-2)}, n^{\frac{8}{r}+ 1 + \frac{2}{\ell}}\right\}.
\]
Conditions (i) and (ii) of Theorem~\ref{thm:noisy-ext-detail1-simple} can be summarized as follows:
\begin{enumerate}
    \item[(i)] $p\gtrsim n^{\frac{3}{2} + \frac{1}{\ell}}\|\bs \mu\|$, 
    \item[(ii)] $\|\bs \mu\|\gtrsim n^{\frac{1}{2}+\frac{1}{\ell}}, p\gtrsim n^{\frac{3}{2} + \frac{1}{\ell}}$.
\end{enumerate}

Noting $r\in (2,4]$, we have $\frac{8}{r}+ 1 + \frac{2}{\ell} > \frac{3}{2} + \frac{1}{\ell}$. Therefore, the conditions of Theorem~\ref{thm:noisy-ext-detail1-simple} follow from:
\begin{enumerate}
    \item[(i)] $1 \lesssim \|\bs \mu\|\lesssim \frac{p}{n^{\frac{3}{2} + \frac{1}{\ell}}}$ and $p\gtrsim \max\left\{n^\frac{4+(1+2/\ell)r}{2(r-2)}, n^{\frac{8}{r}+ 1 + \frac{2}{\ell}}\right\}$, or 
    \item[(ii)] $\|\bs \mu\|\gtrsim n^{\frac{1}{2}+\frac{1}{\ell}}$ and $p\gtrsim \max\left\{n^\frac{4+(1+2/\ell)r}{2(r-2)}, n^{\frac{8}{r}+ 1 + \frac{2}{\ell}}\right\}$.
\end{enumerate}
Finally, noting that $\frac{p}{n^{\frac{3}{2} + \frac{1}{\ell}}} \gtrsim n^{\frac{1}{2}+\frac{1}{\ell}}$ holds under $p\gtrsim n^{\frac{8}{r}+ 1 + \frac{2}{\ell}}$, the valid regions for $\|\bs \mu\|$ (i) and (ii) actually overlap. Therefore, we can simply state the conditions as
\[
\|\bs \mu\|\gtrsim 1 \quad \text{and} \quad p\gtrsim \max\left\{n^\frac{4+(1+2/\ell)r}{2(r-2)}, n^{\frac{8}{r}+ 1 + \frac{2}{\ell}}\right\}.
\] 
The upper bound on the test error reduces to
\[
\P_{(\boldsymbol{x}, y)}(\langle \boldsymbol{\hat w}, y\boldsymbol{x}\rangle < 0) \leq \eta + \frac{c }{(1 - 2\eta)^2} \left(\eta\frac{\E[g^{-2}] n}{p} + \frac{1}{\|\bs \mu\|^2} + \frac{p}{\E[g^{-2}]n \|\boldsymbol{\mu}\|^4}\right).
\]
Under the assumption $p\gtrsim \max\left\{n^\frac{4+(1+2/\ell)r}{2(r-2)}, n^{\frac{8}{r}+ 1 + \frac{2}{\ell}}\right\}$ this bound converges to $\eta$ provided that $\|\bs\mu\| \gg (p/n)^{1/4}$. \end{proof}

\subsection{Proof of Theorem~\ref{thm:phasetransitionsimple}} \label{sec:proof:thm:phasetransitionsimple}

We start by providing lower and upper bounds for the probability $\mathbb{P}\left( \langle \boldsymbol{w}, y_\n \boldsymbol{x} \rangle  <0 \right)$ in a special setting. It can be seen as a version of Lemma~\ref{test-error-spherical} that is tailored to model \ref{model:EM}.

\begin{lemm}\label{lemm:lowupboundEM}
Suppose model \ref{model:EM} holds with $\bs z = gW\bs \xi$ with $\|\bs \xi\|$ and $\bs \xi/\|\bs \xi\|$ independent and $\bs \xi/ \|\bs \xi\|$ has a density $f$ with respect to the uniform distribution on the sphere $S^{p-1}$ such that $f_{min} \leq f \leq f_{max}$. Then we have
\begin{align*}
\frac{1-2\eta}{2}f_{min}\P \left(g\|\bs \xi\||u_1| > \frac{\langle \bs w, \bs \mu \rangle}{\lambda_{min}^{1/2}\| \bs w\|} \right)
\leq \mathbb{P}\left( \langle \boldsymbol{w}, y_\n \boldsymbol{x} \rangle  <0 \right)  - \eta \leq 
\frac{1-2\eta}{2} f_{max}\P \left(g\|\bs \xi\||u_1| > \frac{\langle \bs w, \bs \mu \rangle}{\lambda_{max}^{1/2}\| \bs w\|} \right),
\end{align*}
where $\lambda_{min}$ and $\lambda_{max}$ are the smallest and the largest eigenvalue of $\Sigma$, respectively and $u_1$ is the first component of a random vector $\bs u$ that has a uniform distribution on the sphere $S^{p-1}$ and is independent of $\|\bs\xi\|$ and $g$.
\end{lemm}
\begin{proof}[Proof of Lemma~\ref{lemm:lowupboundEM}]
The proof is the same as that of Lemma~\ref{test-error-spherical} except the way we handle $\P \left(|\langle \bs w, \bs z \rangle| > \langle \bs w, \bs \mu \rangle \right)$:
\begin{align*}
 \P \left(|\langle \bs w, \bs z \rangle| > \langle \bs w, \bs \mu \rangle \right) 
= & \P \left(g\|\bs \xi\| \left|\left\langle \bs w, W \frac{\bs \xi}{\|\bs \xi\|}  \right\rangle \right| > \langle \bs w, \bs \mu \rangle \right) \\
= & \P \left(g\|\bs \xi\| \left|\left\langle W^\top \bs w,  \frac{\bs \xi}{\|\bs \xi\|}  \right\rangle \right| > \langle \bs w, \bs \mu \rangle \right) \\
= & \P \left(g\|\bs \xi\| \left|\left\langle \frac{W^\top \bs w}{\|W^\top \bs w\|},  \frac{\bs \xi}{\|\bs \xi\|}  \right\rangle \right| > \frac{\langle \bs w, \bs \mu \rangle}{\|W^\top \bs w\|} \right) \\
= & \P \left(g\|\bs \xi\| \left|\left\langle \frac{W^\top \bs w}{\|W^\top \bs w\|},  \frac{\bs \xi}{\|\bs \xi\|}  \right\rangle \right| > \frac{\langle \bs w, \bs \mu \rangle}{\|\bs w\|_\Sigma} \right). \quad \text{(by $\|W^\top \bs w\| = \|\bs w\|_\Sigma$)}
\end{align*} 
By following the same arguments as in the proof of Lemma~\ref{test-error-spherical}, we have 
\begin{align*}
f_{min}\P \left(g\|\bs \xi\||u_1| > \frac{\langle \bs w, \bs \mu \rangle}{\lambda_{min}^{1/2}\| \bs w\|} \right)
\leq & f_{min}\P \left(g\|\bs \xi\||u_1| > \frac{\langle \bs w, \bs \mu \rangle}{\|\bs w\|_\Sigma} \right) \\
\leq & \P \left(|\langle \bs w, \bs z \rangle| > \langle \bs w, \bs \mu \rangle \right) \\
\leq & f_{max}\P \left(g\|\bs \xi\||u_1| > \frac{\langle \bs w, \bs \mu \rangle}{\|\boldsymbol{w}\|_\Sigma} \right)
\leq f_{max}\P \left(g\|\bs \xi\||u_1| > \frac{\langle \bs w, \bs \mu \rangle}{\lambda_{max}^{1/2}\| \bs w\|} \right).
\end{align*}

\end{proof}

We will now prove the following result, which implies Theorem~\ref{thm:phasetransitionsimple} as special case as we will explain below.

\begin{thm}\label{thm:phasetransitionsimple-complicated}
Suppose \ref{model:EM} holds in model \ref{model:M} such that there exists an $L<\infty$ with $\|\xi_j\|_{\psi_2} \leq L$ and $\Sigma$ has largest and smallest eigenvalues $\lambda_{max}, \lambda_{min} > 0$, respectively. Assume that moreover $\|\bs \xi\|$ and $\bs \xi/\|\bs \xi\|$ are independent and $\bs \xi/ \|\bs \xi\|$ has a density $f$ with respect to the uniform distribution on the sphere $S^{p-1}$ such that $f_{min} \leq f \leq f_{max}$.

Define $C_{\eta,\delta} := \min\{\eta,(1-2\eta)\}^{-1} \delta^{-1-1/\ell} \sqrt{\log \delta^{-1}}$ and 
\[
\kappa(t) := \tfrac{1-2\eta}{2} \P\left(g \|\bs\xi\| |u_1| > t \right)
\]
where $u_1$ has the same distribution as the first entry of a uniform distribution on the sphere and is independent of $g, \bs\xi$.

For any $\delta \in (0, 1/5]$, there exists a sufficiently large constant $C$ depending only on $L$ and the distribution of $g$ such that provided 
\[
p\geq C^2 C_{\eta,\delta}^2 \Big(\frac{ \lambda_{max}}{\lambda_{min}}\Big)^2 n^{2+2/\ell} \quad \text{and} \quad \|\bs\mu\| \geq C C_{\eta,\delta} \lambda_{max} 
\]
we have for $n \geq \delta^{-\frac{2}{k-2}}\left(\tfrac{16C_2(g)}{\min\{\eta, 1-2\eta\}}\right)^{\frac{k}{k-2}}$, with probability at least $1-5\delta$, 
\[
f_{min}\cdot \kappa(\lambda_{min}^{-1/2}\zeta_{\bs{\hat w}, \bs \mu}^{-1}) \;\leq\; \P_{(g, \bs \xi, y_\n)}(\langle \boldsymbol{\hat w}, y_\n\boldsymbol{x}\rangle < 0) - \eta \;\leq\; f_{max}\cdot  \kappa(\lambda_{max}^{-1/2}\zeta_{\bs{\hat w}, \bs \mu}^{-1}),
\]
where $\zeta_{\bs{\hat w}, \bs \mu}>0$ satisfies
\[
c^{-1}\frac{\lambda_{min}\wedge 1}{\lambda_{max}\vee 1} \leq
\zeta_{\bs{\hat w}, \bs \mu}^2 \Big\{\frac{1}{1-2\eta}\Big(\eta \frac{n\E[g^{-2}]}{p} + \frac{1}{\|\bs\mu\|^2} + \frac{p}{n\E[g^{-2}]\|\bs \mu\|^4} \Big)\Big\}^{-1} \leq  c \frac{\lambda_{max}\vee 1}{\lambda_{min}\wedge 1} 
\]
for a universal constant $c$.
\end{thm}

\begin{proof}[Proof of Theorem~\ref{thm:phasetransitionsimple-complicated}]

We will prove a slightly sharper result that is explicit with respect to all constants. Specifically, we will show that the following assumptions suffice:

Let $\tilde C_1 $ be the constant in Lemma~\ref{lemm:esGOmega1}, $\tilde C_2 , \tilde C_3 $ be the constants in Lemma~\ref{lemm:esGOmega2}, $C_2(g)$ as defined in~\eqref{eq:defconstantsEM}, and $C_{2,\eta}:=\max\{\tfrac{22}{\eta},\tfrac{17}{1-2\eta}\}$. Suppose \ref{model:EM} holds in model \ref{model:M} such that there exists an $L<\infty$ with $\|\xi_j\|_{\psi_2} \leq L$ and with $\Sigma$ such that the smallest and largest eigenvalues are given by $\lambda_{min}, \lambda_{max}$, respectively. Assume that
    \begin{align*}
    p & \geq \left(\frac{\lambda_{max}}{\lambda_{min}}\frac{33\tilde C_1 L^2}{16\eta}  \|g\|_{L^\ell}\|g^{-1}\|_{L^2} (\tfrac{n}{\delta})^{1/\ell}\right)^2 n^2 \log (\tfrac{1}{\delta}),  \\
    \|\boldsymbol{\mu}\| & \geq 2 \lambda_{max}\tilde C_2 C_{2,\eta}\frac{L}{\|g^{-1}\|_{L^2}} \sqrt{\log(\tfrac{1}{\delta})},
    \end{align*}
    and further one of the following conditions holds:
    \begin{enumerate}
        \item[(i)] $\displaystyle p \geq \max\left\{60\tilde C_2  n\sqrt{\log(\tfrac{1}{\delta})}, 132\tilde C_3 \|g^{-1}\|_{L^2}\|g\|_{L^\ell}(\tfrac{n}{\delta})^{1+1/\ell}\sqrt{\log n}\right\} L\|g^{-1}\|_{L^2}    \|\boldsymbol{\mu}\|\tfrac{\lambda_{max}^{1/2}}{\lambda_{min}}$, 
        \item [(ii)] $\displaystyle \|\boldsymbol{\mu}\| \geq \lambda_{max}^{1/2}\tfrac{33}{16}\tilde C_3 C_{2,\eta}L\|g\|_{L^\ell}\delta^{-1}(\tfrac{n}{\delta})^{1/\ell}\sqrt{\log n}$ and \\       $\displaystyle p \geq 4\tfrac{\lambda_{max}}{\lambda_{min}}\tilde C_2  C_{2,\eta} L^2 \max\left\{\tilde C_2  n\log (\tfrac{1}{\delta}), \tfrac{33}{32}\tilde C_3\|g\|_{L^\ell}\|g^{-1}\|_{L^2}(\tfrac{n}{\delta})^{\frac{1}{2}+\frac{1}{\ell}}\sqrt{\tfrac{1}{\delta}\log(\tfrac{1}{\delta})\log n}\right\}$.
    \end{enumerate}

Before proceeding, we note that by a combination of Lemma~\ref{noisy-mmc-bound1}, Lemma~\ref{lemm:esGboundsevents} and Lemma~\ref{lemm:esG-noisy-what-bound} under the above conditions we have for $n \geq \delta^{-\frac{2}{k-2}}\left(\tfrac{16C_2(g)}{\min\{\eta, 1-2\eta\}}\right)^{\frac{k}{k-2}}$, with probability at least $1-5\delta$, $\langle \boldsymbol{\hat w}, \boldsymbol{\mu} \rangle>0$ and 
\begin{multline}\label{eq:lowupbounds}
\frac{c^{-1}}{(1-2\eta)^2}\left\{ \eta \frac{n \E[g^{-2}]}{p\lambda_{max}} + \frac{1}{\|\bs\mu\|^2} + \frac{p\lambda_{min}}{n \E[g^{-2}] \|\boldsymbol{\mu}\|^4} \right\} \leq \left(\frac{\|\boldsymbol{\hat w}\|}{\langle \boldsymbol{\hat w}, \boldsymbol{\mu} \rangle}\right)^2
\\
\leq \frac{c}{(1-2\eta)^2}\left\{ \eta \frac{n \E[g^{-2}]}{p\lambda_{min}} + \frac{1}{\|\bs\mu\|^2} + \frac{p\lambda_{max}}{n \E[g^{-2}] \|\boldsymbol{\mu}\|^4} \right\} 
\end{multline}
where $c$ is a universal constant. 

We will now show that the conditions mentioned above hold if we pick the constant $C$ in the statement of Theorem~\ref{thm:phasetransitionsimple-complicated} sufficiently large. Assume without loss of generality that $C \geq 1$. The condition $p\geq C^2 C_{\eta,\delta}^2 \Big(\frac{\lambda_{max}}{ \lambda_{min}}\Big)^2 n^{2+2/\ell}$ implies both of the following assumptions
\begin{align*}
p & \geq \left(\frac{\lambda_{max}}{\lambda_{min}}\frac{33\tilde C_1 L^2}{16\eta}  \|g\|_{L^\ell}\|g^{-1}\|_{L^2} (\tfrac{n}{\delta})^{1/\ell}\right)^2 n^2 \log (\tfrac{1}{\delta})
\\
p &\geq 4\tfrac{\lambda_{max}}{\lambda_{min}}\tilde C_2  C_{2,\eta} L^2 \max\left\{\tilde C_2  n\log (\tfrac{1}{\delta}), \tfrac{33}{32}\tilde C_3\|g\|_{L^\ell}\|g^{-1}\|_{L^2}(\tfrac{n}{\delta})^{\frac{1}{2}+\frac{1}{\ell}}\sqrt{\tfrac{1}{\delta}\log(\tfrac{1}{\delta})\log n}\right\},
\end{align*}
while the condition $\|\bs\mu\| \geq C C_{\eta,\delta}\lambda_{max}$ yields $\|\boldsymbol{\mu}\| \geq 2 \lambda_{max}\tilde C_2 C_{2,\eta}\frac{L}{\|g^{-1}\|_{L^2}} \sqrt{\log(\tfrac{1}{\delta})}$. 
Hence it suffices to establish that either the assumption in (i) or the assumption on $\|\bs\mu\|$ in (ii) holds. If the assumption $\|\boldsymbol{\mu}\| \geq \lambda_{max}^{1/2}\tfrac{33}{16}\tilde C_3 C_{2,\eta}L\|g\|_{L^\ell}\delta^{-1}(\tfrac{n}{\delta})^{1/\ell}\sqrt{\log n}$ fails, we must have 
\begin{align*}
\|\boldsymbol{\mu}\| &< \lambda_{max}^{1/2}\tfrac{33}{16}\tilde C_3 C_{2,\eta}L\|g\|_{L^\ell}\delta^{-1}(\tfrac{n}{\delta})^{1/\ell}\sqrt{\log n}
\\
&\leq \lambda_{max}^{1/2}C C_{\eta,\delta} n^{1/\ell}\sqrt{\log n}
\\
&\leq \lambda_{max}^{1/2}\Big(\frac{\lambda_{min}}{\lambda_{max}}\Big)^2\frac{p C C_{\eta,\delta} n^{1/\ell}\sqrt{\log n}}{C^2 C_{\eta,\delta}^2 n^{2+2/\ell}}
\\
& \leq \frac{\lambda_{min}}{\lambda_{max}^{1/2}}\frac{p \sqrt{\log n}}{C C_{\eta,\delta} n^{2+1/\ell}}.
\end{align*}
Thus $p \geq \frac{\lambda_{max}^{1/2}}{\lambda_{min}}C C_{\eta,\delta} n^{2+1/\ell} \|\boldsymbol{\mu}\| (\log n)^{-1/2}$. Since $\log n \leq n$ for $n \geq 1$, it follows that (i) is met.

In summary, we have established that the bounds in~\eqref{eq:lowupbounds} hold. The assumptions we placed on $\bs\xi$ imply the conditions of Lemma~\ref{lemm:lowupboundEM}, and the conclusion of Theorem~\ref{thm:phasetransitionsimple-complicated} follows after some elementary calculations.  

\end{proof}

\textbf{Proof of Theorem~\ref{thm:phasetransitionsimple}}
Theorem~\ref{thm:phasetransitionsimple-complicated} implies Theorem~\ref{thm:phasetransitionsimple} is because for $\bs \xi \sim N(0,I_p)$ we have independence between $\|\bs \xi\|$ and $\bs\xi/\|\bs \xi\|$. Moreover $\bs\xi/\|\bs \xi\|$ follows a uniform distribution on the sphere, so we can set $f_{min} = f_{max} = 1$. Moreover, for $\bs u$ a random vector with uniform distribution on the sphere that is independent of $g, \|\bs \xi\|$, the random variable $g \|\bs \xi\| |u_1|$ has the same distribution as $g|\xi_1|$. \hfill $\Box$

\subsection{Proof of Lemma~\ref{lem:noisy_geom}}\label{sec:proof:lem:noisy_geom}

\begin{proof}[Proof of Lemma~\ref{lem:noisy_geom}]
    We prove the expression for $\bs z_{\perp,\cs}+\bs\mu$ with the formula for $\bs z_{\perp,\n}-\bs\mu$ having an analogous proof.  Denote $D=\bar Z\bar Z^\top$ (diagonal matrix). Let $\bs e_{\cs}$ and $\bs e_{\textsc{n}}$ be  $0/1$-vectors with $1$'s corresponding to the clean and noisy samples respectively. Also, let $s=\bs 1^\top D^{-1}\bs 1$, $s_{\cs}=\bs e_{\cs}^\top D^{-1}\bs 1$, and $s_\n = \bs e_\n^\top D^{-1} \bs 1=s- s_\cs$.
    Using for the clean subsample the same argument as in the previous subsection we get that $$\frac{\hat{\bs w}_{\cs}}{\|\hat{\bs w}_{\cs}\|^2}\;=\;\bs \mu+\frac{\bar Z_{\cs}^\top(\bar Z_{\cs}\bar Z_{\cs}^\top)^{-1}\bs 1}{\bs 1^\top (\bar Z_{\cs}\bar Z_{\cs}^\top)^{-1}\bs 1},$$
    where $\bar Z_{\cs}$ is a submatrix of $\bar Z$ with rows corresponding to clean samples. Thus, to show $\bs z_{\perp,\cs}+\bs\mu=\frac{\hat{\bs w}_{\cs}}{\|\hat{\bs w}_{\cs}\|^2}$, it is enough to show that the last term is equal to $\bs z_{\perp,\cs}=\sum_{i:\rm{clean}} \tfrac{\alpha_i}{\nu_{\cs}}\bar{\bs z}_i$. Since $\bar{\bs z}_i$'s are orthogonal, it is equivalent to show that the vectors ${(\bar Z_{\cs}\bar Z_{\cs}^\top)^{-1}\bs 1}$ and $[(\bar X\bar X^\top)^{-1}\bs 1]_{\cs}$ are scalar multiple of each others.  Note that $\bar X=(\bs e_{\cs}-\bs e_{\n})\bs \mu^\top+\bar Z$ and so, using the fact that $\bs z_i\perp \bs \mu$ for all $i$, 
    $$
    \bar X\bar X^\top =\|\bs \mu\|^2(\bs e_{\cs}-\bs e_{\textsc{n}})(\bs e_{\cs}-\bs e_{\textsc{n}})^\top +D
    $$
    and so, by the Sherman–Morrison formula,
\begin{equation}\label{eq:XXinvgeom}
    (\bar X\bar X^\top)^{-1}\bs 1 =D^{-1}\bs 1-\frac{\|\bs\mu\|^2(s_{\cs}-s_{\n})}{1+\|\bs \mu\|^2 (\bs e_{\cs}-\bs e_{\textsc{n}})^\top D^{-1}(\bs e_{\cs}-\bs e_{\textsc{n}})}D^{-1}(\bs e_{\cs}-\bs e_{\textsc{n}})  
\end{equation}
    giving that $[(\bar X\bar X^\top)^{-1}\bs 1]_{\cs}\propto [D^{-1}\bs e_{\cs}]_{\cs}={(\bar Z_{\cs}\bar Z_{\cs}^\top)^{-1}\bs 1}$. This proves the first statement. 
    
    To prove that $\tfrac{\bs{\hat w}}{\|\bs{\hat w}\|^2}$ is the orthogonal projection of the origin on the line joining $\tfrac{\bs{\hat w}_\cs}{\|\bs{\hat w}_\cs\|^2}$ and $\tfrac{\bs{\hat w}_\n}{\|\bs{\hat w}_\n\|^2}$, we show that the coefficients $\nu_{\cs}$ and $\nu_\n$ in decomposition \eqref{eq:wovernormdecom} satisfy 
    \begin{equation}\label{eq:geomnu}
    \nu_{\cs}\;=\;\frac{\|\bs\mu\|^2+\|\bs z_{\perp,\n} - \bs \mu\|^2}{(\|\bs\mu\|^2+\|\bs z_{\perp,\cs}+\bs \mu\|^2)+(\|\bs\mu\|^2+\|\bs z_{\perp,\textsc{n}} - \bs \mu\|^2)}\;=\;1-\nu_{\n}.    
    \end{equation}
    To see this, simply use the facts that $\<\hat{\bs w}_{\cs},\bs \mu+\bar{\bs z}_i\>=1$ and $\<\bs z_{\perp, \n} -\bs \mu, \bs \mu + \bar z_i \> = -\|\bs \mu\|^2$ for all clean samples, $\<\hat{\bs w}_{\textsc{n}},-\bs \mu+\bar{\bs z}_i\>=1$ and $\<\bs z_{\perp, \cs} +\bs \mu, -\bs \mu + \bar z_i \> = -\|\bs \mu\|^2$ for all noisy samples, and $\nu_\cs + \nu_\n = 1$. Moreover, $\<\hat{\bs w},\bs \mu+\bar{\bs z}_i\>=1$ for all clean samples and $\<\hat{\bs w},-\bs \mu+\bar{\bs z}_i\>=1$ for all noisy samples. The formula for $\nu_{\cs}$ now follows using the decomposition in \eqref{eq:wovernormdecom} and a little algebra. It is straightforward now that $\lambda=\nu_\cs$ is the optimum for which the norm $\|\lambda \frac{\hat{\bs w}_{\cs}}{\|\hat{\bs w}_{\cs}\|^2}+(1-\lambda) \frac{\hat{\bs w}_{\textsc{n}}}{\|\hat{\bs w}_{\textsc{n}}\|^2}\|$ is minimal.

    Lastly, suppose further \eqref{eq:normassump}. Note that the full orthogonality of $\bs z_i$'s and $\bs z_{\perp, \cs} = \tfrac{\bar Z_{\cs}^\top(\bar Z_{\cs}\bar Z_{\cs}^\top)^{-1}\bs 1}{\bs 1^\top (\bar Z_{\cs}\bar Z_{\cs}^\top)^{-1}\bs 1}$ imply that $\|\bs z_{\perp, \cs}\|^{-2} = \sum_{i:\rm{clean}} \|\bs z_i\|^{-2}$. Therefore, \eqref{eq:normassump} implies $\|\bs z_{\perp, \cs}\|^{-2} = (1-\eta)n\rho$. Similarly, \eqref{eq:normassump} implies $\|\bs z_{\perp, \n}\|^{-2} = \eta n\rho$. We can now use \eqref{eq:geomnu} to directly verify the last formula.
\end{proof}
\begin{rmk}
        In our arguments above we always assumed that $\hat{\bs w}$, $\hat{\bs w}_{\cs}$, and $\hat{\bs w}_\n$ are equal to the corresponding least squares estimators. In the orthogonal case this is easy to verify; see the discussion of the geometry of the optimization problem \eqref{eq:mmc} in the previous subsection.
\end{rmk}

\section{Extended Sub-Gaussian Mixture Models}\label{sec:proofs:extSG}

Throughout this section, we will consider the extended sub-Gaussian mixture model defined below. 

\begin{enumerate}[label=(EsG), ref=(EsG)]
    \item \label{model:EsG} The model \ref{model:M} is called the \textit{extended sub-Gaussian mixture model} if model \ref{model:EM} holds and, additionally, $\|\xi_j\|_{\psi_2}\leq L$. (We note that the mean zero and unit variance assumption implies $L \geq 1$ since by the definition of sub-Gaussian norm, $2 \geq \E \exp(\xi_{j}^2/L^2) \geq 1 + \E \xi_{j}^2/L^2 = 1 + 1/L^2$).
\end{enumerate}

The results in this section serve two purposes. First, in the case $g \equiv 1$, this reduces to the sub-Gaussian mixture model which has been considered in most of the existing literature on benign overfitting in linear classification (\cite{JMLR:v22:20-974, NEURIPS2021_46e0eae7, wang2022binary}). The technical results in here pave the way for the proof of Theorem~\ref{thm:sGnoisymain} and Theorem~\ref{thm:noiseless-main-sG} which are given at the end. Second, the results, in particular Lemma~\ref{lemm:esG-noisy-what-bound}, are utilized in the proof of Theorem~\ref{thm:phasetransitionsimple} which provides lower bounds on the test error.

First, note that the assumptions on $g$, $y, y_\n$ are the same in the extended sub-Gaussian mixture model~\ref{model:EsG} as they were in the non-sub-Gaussian mixture model~\ref{model:EM}. Hence we can use Lemma~\ref{lemm3-ex1} and Lemma~\ref{lemm4-ex1} to bound the probability of $\Omega_3(D)$, $\Omega_{4,1}(B_1)$, and $\Omega_{4,2}(B_2)$. 

\subsection{Bounds on the probability of events \texorpdfstring{$\Omega_1,\Omega_2,\Omega_{3}$}{O1,...,O3}}

The next Lemma~establishes the probability bound for $\Omega_1(\eps)$.

\begin{lemm}\label{lemm:esGOmega1}
Suppose model \ref{model:EsG} holds with $n\geq \log (\tfrac{1}{\delta})$. Then there exists a universal constant $\tilde C_1 $ such that for
\[
\eps := \tilde C_1  L^2\max\left\{\sqrt{n}\|\Sigma\|_F, n\|\Sigma\| \right\}/\Tr(\Sigma)
\]
we have for the event $\Omega_1(\eps)$ defined in~\eqref{eq:omega1}: $\P(\Omega_1(\eps)) \geq 1-\delta.$
\end{lemm}

\begin{proof}
Consider the $n\times p$ matrix with rows $\bs \xi_1, \ldots, \bs \xi_n$ and denote its columns by $\bs \zeta_1, \ldots, \bs \zeta_p$;  $(\bs \xi_1, \ldots, \bs \xi_n)^\top = (\bs \zeta_1, \bs \zeta_2, \ldots, \bs \zeta_p)$. Since $\bs \xi_i, i = 1, \dots, n$ are i.i.d and have mean-zero, independent entries, $\bs \zeta_j \in \R^n, j = 1, \ldots, p$ are independent random vectors with mean-zero, independent entries. Then, for any $\bs w \in S^{n-1}$, we have
\[
V^\top \bs w = W (\bs \xi_1, \ldots, \bs \xi_n) \bs w = W (\bs \zeta_1^\top \bs w, \ldots, \bs \zeta_p^\top \bs w)^\top.
\]
Letting, $\bs \zeta_{\bs w} = (\bs \zeta_1^\top \bs w, \ldots, \bs \zeta_p^\top \bs w)^\top$, we have
\[
\bs w^\top VV^\top \bs w = \bs \zeta_{\bs w}^\top W^\top W \bs \zeta_{\bs w}.
\]
Since $\bs \zeta_j$'s are mean-zero, independent random vectors, we note that $\bs \zeta_{\bs w}$ has mean-zero, independent entries. Furthermore, since $\xi_{ij}, i=1,\ldots, n$ are independent mean-zero sub-Gaussian random variables with $\|\xi_{ij}\|_{\psi_2} \leq L$, Proposition~2.6.1 in \cite{vershynin_2018} implies that
\[
\|(\bs{\zeta_w})_j\|_{\psi_2}^2 = \left\|\sum_{i=1}^n w_i \xi_{ij}\right\|_{\psi_2}^2 \leq c_1 \sum_{i=1}^n w_i^2 \|\xi_{ij}\|_{\psi_2}^2 \leq c_1 L^2 \|\bs w\|^2 = c_1 L^2, 
\]
where $c_1$ is a universal constant. Thus, we conclude that $\|(\bs{\zeta_w})_j\|_{\psi_2} \leq c_2 L$, where $c_2$ is a universal constant. 

Recalling $\E VV^\top = \Tr(\Sigma)I_n$, we have 
\[
\bs w^\top (VV^\top - \Tr(\Sigma) I_n)\bs w = \bs \zeta_{\bs w}^\top W^\top W \bs \zeta_{\bs w} - \E[\bs \zeta_{\bs w}^\top W^\top W \bs \zeta_{\bs w}].
\]
Since $\bs{\zeta_w}$ is a random vector with independent, mean-zero entries with $\|(\bs{\zeta_w})_j\|_{\psi_2} \leq c_2 L$, Theorem~6.2.1 (Hanson-Wright inequality) in \cite{vershynin_2018} implies that
\begin{equation}\label{eq:hansonwright}
\begin{aligned}
\P\left\{|\bs w^\top (VV^\top - \Tr(\Sigma)I_n)\bs w| \geq t \right\} 
& \leq 2 \exp \left\{-c \min\left(\frac{t^2}{c_2^4L^4\|W^\top W\|_F^2}, \frac{t}{c_2^2L^2\|W^\top W\|}\right) \right\} \\
& = 2 \exp \left\{-c \min\left(\frac{t^2}{c_2^4L^4\|\Sigma\|_F^2}, \frac{t}{c_2^2L^2\|\Sigma\|}\right) \right\},
\end{aligned}
\end{equation}
where $c$ is a universal constant.

Now let $\mathcal N$ be a $\tfrac{1}{4}$-net on the sphere $S^{n-1}$. Then, we have $|\mathcal N| \leq 9^n$ by Corollary~4.2.13 in \cite{vershynin_2018}. Then, \eqref{eq:hansonwright} with the union bound argument implies 
\begin{align*}
& \P\left\{\exists \bs w\in \mathcal N:~  |\bs w^\top (VV^\top - \Tr(\Sigma)I_n)\bs w| \geq t 
\right\} \\
& \qquad \leq 2 \cdot 9^n\exp \left\{-c \min\left(\frac{t^2}{c_2^4L^4\|\Sigma\|_F^2}, \frac{t}{c_2^2L^2\|\Sigma\|}\right) \right\}.
\end{align*}
Rewriting this, we have with probability at least $1-\delta$, the following inequality holds with a universal constant $c_3$ for all $\bs w \in \mathcal N$:
\[
|\bs w^\top (VV^\top - \Tr(\Sigma)I_n)\bs w| \leq c_3 L^2\max\left\{\sqrt{n + \log(\tfrac{1}{\delta})}\|\Sigma\|_F, \{n + \log(\tfrac{1}{\delta})\}\|\Sigma\| \right\}.
\]
By applying the epsilon-net argument (Lemma~S.8. in \cite{bartlett2020benign}) and noting the assumption $n \geq \log(\tfrac{1}{\delta})$, we have with probability at least $1-\delta$, 
\[
\|VV^\top - \Tr(\Sigma) I_n\| \leq \tilde C L^2\max\left\{\sqrt{n}\|\Sigma\|_F, n\|\Sigma\| \right\} 
\]
where $\tilde C$ is a universal constant. Take $\tilde C_1  = 4\tilde C$ to complete the proof. \end{proof} 

Finally, we establish the probability bound for $\Omega_{2,2}(b_2)$ and $\Omega_{2,\infty}(b_\infty)$.

\begin{lemm}\label{lemm:esGOmega2}
Suppose model \ref{model:EsG} holds with $n\geq \log(\tfrac{1}{\delta})$. There exist universal constants $\tilde C_2, \tilde C_3$ such that $\P(\Omega_{2,2}(b_2) \cap \Omega_{2,\infty}(b_\infty)) \geq 1-2\delta$ holds for 
\[
b_2 = \tilde C_2 L \sqrt{n}\|\bs \mu\|_\Sigma, \qquad  b_{\infty} =  \tilde C_3 L\|\bs \mu\|_\Sigma \sqrt{\log (\tfrac{n}{\delta})}.
\]

\end{lemm}

\begin{proof} 
For a general vector $\bs w \neq \bs 0$ we have for a universal constant $C$
\[
\|(\bs w^\top \bs \xi)^2\|_{\psi_1} = \|\bs w^\top \bs \xi\|_{\psi_2}^2 = \|\bs w\|^2 \Big\| (\tfrac{\bs w}{\|\bs w\|})^\top \bs\xi \Big\|_{\psi_2}^2 \leq \|\bs w\|^2 \|\bs \xi\|_{\psi_2}^2 \leq C L^2 \|\bs w\|^2
\]
where we used Lemma~2.7.6 in \cite{vershynin_2018} for the first equality, the definition of the sub-Gaussian norm of a random vector in the second to last inequality and Lemma~3.4.2 in \cite{vershynin_2018} along with the assumption that $\xi$ has independent, centered, sub-Gaussian entries in the last inequality. Noting that $\langle \boldsymbol{v}_i, \boldsymbol{\mu} \rangle = (W^\top \bs\mu)^\top \bs\xi_i$ we obtain 
\[
\max_{i=1, \ldots,n} \|\langle \boldsymbol{v}_i, \boldsymbol{\mu} \rangle^2\|_{\psi_1} \leq C \|W^\top\boldsymbol{\mu}\|^2 L^2 = C \|\boldsymbol{\mu}\|_\Sigma^2 L^2.
\] 
By Exercise 2.7.10 in \cite{vershynin_2018} we have 
\[
\|\langle \boldsymbol{v}_i, \boldsymbol{\mu} \rangle^2 - \E\langle \boldsymbol{v}_i, \boldsymbol{\mu} \rangle^2 \|_{\psi_1} \leq \tilde C L^2 \|\bs \mu\|_\Sigma^2
\]
for a universal constant $\tilde C$. Thus by Corollary~S.6. in the Appendix of \cite{bartlett2020benign} we have any $x>0$, we have with probability at least $1 - 2e^{-x}$,
\[
\left| \|V\boldsymbol{\mu}\|^2 - n \|\bs \mu\|_\Sigma^2\right| \leq \frac{C}{2} L^2\|\bs \mu\|_\Sigma^2 \max\{x, \sqrt{xn}\},
\]
where $C$ is a different universal constant. Thus, by noting that we may assume $\delta \leq \tfrac{1}{2}$ since the statement would be trivial otherwise, with probability at least $1 - \delta$, we have 
\begin{align*}
\left| \|V\boldsymbol{\mu}\|^2 - n \|\bs \mu\|_\Sigma^2\right| & \leq \frac{C}{2} L^2\|\bs \mu\|_\Sigma^2 \max \left\{\log (\tfrac{2}{\delta}), \sqrt{n\log(\tfrac{2}{\delta})}\right\} \\
    & \leq CL^2\|\bs \mu\|_\Sigma^2 \max \left\{\log (\tfrac{1}{\delta}), \sqrt{n\log(\tfrac{1}{\delta})}\right\},
\end{align*}
which implies 
\begin{equation}
\|V\boldsymbol{\mu}\|^2 \leq n \|\boldsymbol{\mu}\|_\Sigma^2\left(1 + CL^2\max \left\{\frac{\log (\tfrac{1}{\delta})}{n}, \sqrt{\frac{\log(\tfrac{1}{\delta})}{n}}\right\}\right).
\end{equation}

Recalling $L\geq1$ and $n\geq \log(\tfrac{1}{\delta})$, this implies 
\begin{equation}\label{eq:bound|Zmu|2}
\|V\boldsymbol{\mu}\| \leq \tilde C_2 L\sqrt{n}\|\bs \mu\|_\Sigma
\end{equation}
with probability at least $1-\delta$, where $\tilde C_2 $ is a universal constant.

For $b_\infty$, we note that $\langle \boldsymbol{v}_i, \boldsymbol{\mu}\rangle, i=1,\ldots, n$ are mean zero sub-Gaussian random variables since 
\[
\|\langle \boldsymbol{v}_i, \boldsymbol{\mu}\rangle\|_{\psi_2}^2 = \Big\|\sum_{j=1}^p (W^\top\bs\mu)_j \xi_j \Big\|_{\psi_2}^2 \leq C \sum_{j=1}^p (W^\top\bs\mu)_j^2 L^2 = CL^2 \|W^\top\bs\mu\|^2 = CL^2 \|\bs\mu\|_\Sigma^2
\]
for a universal constant $C$ by Proposition~2.6.1 in \cite{vershynin_2018}. 

Thus, with the union bound argument, we have for any $t >0$
\begin{equation}
\begin{aligned}
\mathbb{P}(\max_i |\langle \bs{v}_i, \bs{\mu} \rangle| \geq t) 
    & \leq \sum_{i=1}^n \mathbb{P}(|\langle \bs{v}_i, \bs{\mu} \rangle| \geq t) \\
    & \leq 2n\exp\left(-c\frac{t^2}{CL^2 \|\bs \mu\|_\Sigma^2}\right),
\end{aligned}
\end{equation}
where $c$ is a universal constant. Setting the right hand side as $\delta$, we have 
\[
\mathbb{P}\left(\max_i |\langle \bs{v}_i, \bs{\mu} \rangle| \geq \tilde C_3 L\|\bs \mu\|_\Sigma\sqrt{\log (\tfrac{n}{\delta})}\right) \leq \delta,
\]
where $\tilde C_3$ is a universal constant.
\end{proof}

\begin{rmk}
We note that the condition $n\geq \log(\tfrac{1}{\delta})$ is required only for the argument over $b_2$.
\end{rmk}

As noted earlier, we can use Lemma~\ref{lemm4-ex1} to bound the probability of $\Omega_{4,1}(B_1)$ and $\Omega_{4,2}(B_2)$. However, in the case where $g=1$ almost surely, $\Omega_{4,1}(0)$ holds and a sharper characterization of $\Omega_{4,2}$ is possible and provided below. 

\begin{lemm}\label{lemm:esGOmega4}
Suppose model~\ref{model:sG} holds. Then there exists a universal constant $\tilde C_4 >1$ such that for the event $\Omega_{4,2}(B_2)$ defined in~\eqref{eq:omega42} we have $\P(\Omega_{4,2}(B_2)) \geq 1- \delta$ for $B_2 = \tilde C_4 \sqrt{n\log(\tfrac{1}{\delta})}$. 
\end{lemm}

\begin{proof}
By independence of $y_iy_{\n i}$ and since $y_iy_{\n i} \in \{-1,1\}$, Hoeffding's inequality yields with probability at least $1 - \delta$,
\begin{equation*}
\left|\sum_i y_{\n, i}y_i - (1-2\eta)n \right| \leq \tilde C_4 \sqrt{n\log(\tfrac{1}{\delta})}
\end{equation*}
where $\tilde C_4 $ is a universal constant.
\end{proof}

\subsection{Discussion of events \texorpdfstring{$E_1,\dots,E_5$}{E1,...,E5} and verification of condition (\texorpdfstring{$N_C$}{NC}) in model \ref{model:EsG}}

Combined with Lemma~\ref{lem:genboundsevents}, the bounds on the probabilities of $\Omega_1$--$\Omega_{4,2}$ obtained in Lemma~\ref{lemm3-ex1}, \ref{lemm4-ex1}, \ref{lemm:esGOmega1}, \ref{lemm:esGOmega2}, and \ref{lemm:esGOmega4} imply the following result. % 
\begin{lemm}\label{lemm:esGboundsevents}
Let $\tilde C_1$ be the constant in Lemma~\ref{lemm:esGOmega1} and $\tilde C_2, \tilde C_3$ be the constants in Lemma~\ref{lemm:esGOmega2}, $\tilde C_4$ be the constant in Lemma~\ref{lemm:esGOmega4} and $C_2(g) := 2^{2+2/k}\tfrac{\|g^{-2}\|_{L^{k/2}}}{\|g^{-2}\|_{L^1}}
$ be the constant from Lemma~\ref{lem:genboundsevents}. Suppose model \ref{model:EsG} holds with $n\geq \log (\tfrac{1}{\delta})$ and
\begin{equation}\label{eq:assboundepsEsG}
\Tr(\Sigma) \geq 2 \tilde C_1 L^2\max\left\{\sqrt{n}\|\Sigma\|_F, n\|\Sigma\| \right\}.
\end{equation}
Then, event $\bigcap_{i=1}^5 E_i$ holds with probability at least $1-5\delta$ (event $E_1 \cap E_3$ holds with probability at least $1-2\delta$ and $\bigcap_{i=1}^4 E_i$ holds with probability at least $1-4\delta$) with, 
\[
\eps = \tilde C_1 L^2\max\left\{\sqrt{n}\|\Sigma\|_F, n\|\Sigma\| \right\}/\Tr(\Sigma)\leq \frac{1}{2},
\]
\[
\alpha_2 = 2\tilde C_2 L \frac{\sqrt{n}\|\bs \mu\|_\Sigma}{\sqrt{\Tr(\Sigma)}\|\bs \mu\|}, \qquad \alpha_\infty = 2\tilde C_3 L\frac{\sqrt{\log (\tfrac{n}{\delta})}\|\bs \mu\|_\Sigma}{\sqrt{\Tr(\Sigma) }\|\bs \mu\|},
\]
\[
\beta = \gamma \;=\; \varepsilon + C_2(g)\delta^{-2/k} n^{-(1 - 2/k)}, \qquad \rho  \;=\; \E [g^{-2}]\Tr(\Sigma)^{-1},
\] 
\[
M \;=\;    (1+\eps)\|g\|_{L^\ell}(\tfrac{n}{\delta})^{1/\ell}\sqrt{\Tr(\Sigma)}.
\]
If we further assume $g=1$ almost surely, $\beta$ and $\gamma$ can be replaced with
\[
\beta = \eps, \qquad \gamma = \eps + \tilde C_4\sqrt{\tfrac{\log (\tfrac{1}{\delta})}{n}}.
\]
When $g=1$ a.s. we can take $\ell = \infty$ and interpret $(n/\delta)^{1/\ell} = 1$.
\end{lemm}

By Lemma~\ref{lemm:esGboundsevents} we obtain the following lemma, which establishes the sufficient condition for Lemma~\ref{noisy-mmc-bound1} and Theorem~\ref{detail-noisy-main-1-simple}.

\begin{lemm}\label{lemm:esG-noisy-what-bound}
Let $\tilde C_1 $ be the constant in Lemma~\ref{lemm:esGOmega1}, $\tilde C_2 , \tilde C_3 $ be the constants in Lemma~\ref{lemm:esGOmega2}, $\tilde C_4 $ be the constant in Lemma~\ref{lemm:esGOmega4}, and $C_2(g)$ be the constant in Lemma~\ref{lemm5-ex1}. Suppose \ref{model:EsG} holds in model \ref{model:M}, 
    \begin{align}
    \Tr(\Sigma) & \geq \frac{33\tilde C_1 L^2}{16\eta}  \|g\|_{L^\ell}\|g^{-1}\|_{L^2} (\tfrac{n}{\delta})^{1/\ell}\max\left\{n^{3/2} \|\Sigma\|, n \|\Sigma\|_{\textsc{f}}\right\}, \label{eq:esGtracebound} \\
    \|\boldsymbol{\mu}\|^2 & \geq 2\tilde C_2  C_{2,\eta}\frac{L}{\|g^{-1}\|_{L^2}} \|\bs \mu\|_\Sigma, 
    \end{align}
    and further one of the following conditions holds: 
    \begin{enumerate}
        \item[(i)] $\displaystyle \Tr(\Sigma) \geq \max\left\{60\tilde C_2  n, 132\tilde C_3\|g^{-1}\|_{L^2}\|g\|_{L^\ell}(\tfrac{n}{\delta})^{1/\ell}n\sqrt{\log (\tfrac{n}{\delta}}\right\} L\|g^{-1}\|_{L^2}  \|\bs \mu\|_\Sigma$,   
        \item [(ii)] $\displaystyle \|\boldsymbol{\mu}\|^2 \geq \tfrac{33}{16}\tilde C_3 C_{2,\eta}L\|g\|_{L^\ell}(\tfrac{n}{\delta})^{1/\ell}\sqrt{\log (\tfrac{n}{\delta})}\|\bs \mu\|_\Sigma$ and \\       $\displaystyle \Tr(\Sigma) \geq 4\tilde C_2  C_{2,\eta} L^2 \max\left\{\tilde C_2  n,\tfrac{33}{32}\tilde C_3\|g\|_{L^\ell}\|g^{-1}\|_{L^2}(\tfrac{n}{\delta})^{1/\ell}n\sqrt{\log(\tfrac{n}{\delta})}\right\}\frac{\|\bs \mu\|_\Sigma^2}{\|\boldsymbol{\mu}\|^2}$. 
    \end{enumerate}
    Then, we have for $n \geq \delta^{-\frac{2}{k-2}}\left(\tfrac{16C_2(g)}{\min\{\eta, 1-2\eta\}}\right)^{\frac{k}{k-2}}$, with probability at least $1-5\delta$, ($N_{C_{2,\eta}}$) and $\eps\vee \beta \vee \gamma \leq \tfrac{\min\{\eta, 1-2\eta\}}{8}$ hold.

    \begin{sloppypar}
    If we further assume $g\equiv 1$ a.s., then ($N_{C_{2,\eta}}$) and $\eps\vee \beta \vee \gamma \leq \tfrac{\min\{\eta, 1-2\eta\}}{8}$ hold for $n\geq  \log(\tfrac{1}{\delta}) \left( \tfrac{16\tilde C_4}{\min\{\eta, 1-2\eta\}}\right)^2$. Moreover, we can set $\ell = \infty$ and interpret $(n/\delta)^{1/\ell} = 1$.     
    \end{sloppypar}
\end{lemm}

\begin{proof}[Proof of Lemma~\ref{lemm:esG-noisy-what-bound}] 

Throughout this proof, let $\eps, \beta, \gamma, \rho, \alpha_{2}, \alpha_\infty, M$ be as defined in Lemma~\ref{lemm:esGboundsevents}. Noting that $L \geq 1$, $\|g\|_{L^\ell}\|g^{-1}\|_{L^2} \geq 1$ by \eqref{eq:gnorm}, and $\eta \in (0, \tfrac12)$, we see that the condition \eqref{eq:esGtracebound} implies that the condition~\eqref{eq:assboundepsEsG} of Lemma~\ref{lemm:esGboundsevents} is satisfied. 

\begin{sloppypar}
Recalling that $C_2(g) \geq 1$ by \eqref{eq:gnorm}, $\tilde C_4 \geq 1$ and $k \in (2,4]$ we further see that $n \geq \delta^{-\frac{2}{k-2}}\left(\tfrac{16C_2(g)}{\min\{\eta, 1-2\eta\}}\right)^{\frac{k}{k-2}}$ implies $n \geq \log(1/\delta)$ and similarly $n\geq  \log(\tfrac{1}{\delta}) \left( \tfrac{16\tilde C_4}{\min\{\eta, 1-2\eta\}}\right)^2$ implies $n \geq \log(1/\delta)$ so this condition of Lemma~\ref{lemm:esGboundsevents} holds as well.    
\end{sloppypar}

Next, noting that $C_2(g) \geq 1$ by \eqref{eq:gnorm} and $\tilde C_4 \geq 1$, $\log \tfrac1\delta \geq 1$ by $\delta \leq \tfrac15$ and recalling $k \in (2,4]$ it follows that 
\begin{equation}\label{eq:boundnlemmaEsG}
n \geq \Big( \delta^{-\frac{2}{k-2}}\left(\tfrac{16C_2(g)}{\min\{\eta, 1-2\eta\}}\right)^{\frac{k}{k-2}}\Big) \wedge \Big( \log(\tfrac{1}{\delta}) \left( \tfrac{16\tilde C_4 }{\min\{\eta, 1-2\eta\}}\right)^2 \Big) \geq \Big(\tfrac{16}{\min\{\eta, 1-2\eta\}} \Big)^2 > 64.
\end{equation} 
In what follows, we will thus assume $n \geq 64$. Furthermore, \eqref{eq:esGtracebound} implies 
\begin{equation}\label{eq:boundepsinlemmaEsG}
1 \geq \tfrac{33}{16\eta}\|g\|_{L^\ell}\|g^{-1}\|_{L^2} \left(\tfrac{n}{\delta}\right)^{1/\ell}\sqrt{n}\eps \geq \tfrac{33}{16\eta}\sqrt{n}\eps \geq \tfrac{33}{8}\sqrt{n}\eps.
\end{equation}
Thus, we have $\eps \leq \tfrac{8}{33\sqrt{n}} < \tfrac{1}{32}$ for $n\geq 64$. We can thus additionally assume $\eps \leq \tfrac{1}{32}$ in what follows.

We will first show that $\eps\vee \beta \vee \gamma \leq \tfrac{\min\{\eta, 1-2\eta\}}{8}$. We note that $\eps \leq \tfrac{\eta}{2\sqrt{n}}$ by the first and second inequality in~\eqref{eq:boundepsinlemmaEsG}. Since $n \geq (\tfrac{16}{\min\{\eta, 1-2\eta\}} )^2$ by the first and second inequality in~\eqref{eq:boundnlemmaEsG}, we obtain $\eps \leq \tfrac{\min\{\eta, 1-2\eta\}}{16}$. It remains to bound $\gamma \vee \beta$, which we will do separately for the case of general $g$ and $g =1 $ a.s.

In the case of $g=1$ almost surely, we have $\beta \vee \gamma = \eps + \tilde C_4\sqrt{\tfrac{\log (1/\delta)}{n}}$ and hence it suffices to show $\tilde C_4\sqrt{\tfrac{\log (1/\delta)}{n}} \leq \tfrac{\min\{\eta, 1-2\eta\}}{16}$ which follows directly from the assumption we made on $n$ in this case.

Similarly, for the case of general $g$ we have $\beta \vee \gamma = \eps + C_2(g)\delta^{-2/k}n^{-(1-2/k)}$ and the bound $C_2(g)\delta^{-2/k}n^{-(1-2/k)} \leq\tfrac{\min\{\eta, 1-2\eta\}}{16}$ again follows from our assumption on $n$ in this case.

This completes the verification of $\eps\vee \beta \vee \gamma \leq \tfrac{\min\{\eta, 1-2\eta\}}{8}$. We will next verify all other assumptions of Theorem~\ref{detail-noisy-main-1-simple}.

Under the conclusions of Lemma~\ref{lemm:esGboundsevents}, the assumption $\eps M\sqrt{n\rho}\leq \tfrac{\eta}{2}$ of ($N_{C_{2,\eta}}$) is equivalent to 
\[
\Tr(\Sigma) \geq \frac{2(1+\eps)\tilde C_1 L^2}{\eta}\|g\|_{L^\ell}\|g^{-1}\|_{L^2} (\tfrac{n}{\delta})^{1/\ell} \max\left\{n^{3/2}\|\Sigma\|, n\|\Sigma\|_{\textsc{f}}\right\}.
\]

Recalling $\eps <\tfrac{1}{32}$, we see that \eqref{eq:esGtracebound} implies $\eps M\sqrt{n\rho}\leq \tfrac{\eta}{2}$.

After some tedious algebraic manipulations, we obtain equivalent expressions of the remaining assumptions of ($N_{C_{2,\eta}}$) as follows:
\begin{align*}
& \|\boldsymbol{\mu}\|\geq C_{2,\eta}\frac{\alpha_2}{\sqrt{n\rho}} \iff \|\boldsymbol{\mu}\|^2 \geq 2\tilde C_2 C_{2,\eta}\frac{L}{\|g^{-1}\|_{L^2}} \|\bs \mu\|_\Sigma, \\ 
(i) \; & \alpha_2 \|\boldsymbol{\mu}\|\sqrt{n\rho} \leq \tfrac{1}{30} \\
& \quad \iff \Tr(\Sigma)\geq 60 \tilde C_2 L\|g^{-1}\|_{L^2}\, n \|\bs \mu\|_\Sigma, \\ 
& \alpha_\infty\|\boldsymbol{\mu}\|Mn\rho \leq \tfrac{1}{64} \\
& \quad \iff \Tr(\Sigma) \geq 128(1+\eps) \tilde C_3 L\|g\|_{L^\ell}\|g^{-1}\|_{L^2}^2 (\tfrac{n}{\delta})^{1/\ell}n\sqrt{\log (\tfrac{n}{\delta})}\|\boldsymbol{\mu}\|_\Sigma, \\ 
(ii) \; & \|\boldsymbol{\mu}\| \geq C_{2,\eta}\alpha_\infty M \\
& \quad \iff \|\boldsymbol{\mu}\|^2 \geq 2(1+\eps)\tilde C_3C_{2,\eta}L \|g\|_{L^\ell}(\tfrac{n}{\delta})^{1/\ell} \sqrt{\log (\tfrac{n}{\delta})}\|\boldsymbol{\mu}\|_\Sigma, \\ 
& C_{2,\eta}\alpha_2^2 \leq 1 \\
& \quad \iff \Tr(\Sigma)\geq 4\tilde C_2^2 C_{2,\eta}L^2 n \frac{\|\boldsymbol{\mu}\|_\Sigma^2}{\|\boldsymbol{\mu}\|^2}, \\ 
& C_{2,\eta}\alpha_2 \alpha_\infty M\sqrt{n\rho} \leq 1 \\
& \quad \iff \Tr(\Sigma) \geq 4(1+\eps) \tilde C_2 \tilde C_3C_{2,\eta}L^2  \|g\|_{L^\ell}\|g^{-1}\|_{L^2} (\tfrac{n}{\delta})^{1/\ell}n \sqrt{\log (\tfrac{n}{\delta})} \frac{\|\boldsymbol{\mu}\|_\Sigma^2}{\|\boldsymbol{\mu}\|^2}. 
\end{align*}

Recalling $\eps \leq \tfrac{1}{32}$, it is easy to see the assumptions (i) and (ii) of this Lemma~imply (i) and (ii) of ($N_{C_{2,\eta}}$), respectively. 

\end{proof}

\subsection{Proof of Theorem~\ref{thm:sGnoisymain}}\label{sec:proof:thm:sGnoisymain}

\begin{proof}[Proof of Theorem~\ref{thm:sGnoisymain}]
We first show that $\|\boldsymbol{z}\|_{\psi_2}^2 \leq cL\|\Sigma\|$ for a universal constant $c$. For any $\boldsymbol{v}\in S^{p-1}$, we have 
\[
\langle \boldsymbol{v}, \boldsymbol{z}\rangle = \langle \boldsymbol{v}, W\boldsymbol{\xi} \rangle = \sum_{j=1}^p \left(W^\top\boldsymbol{v}\right)_{j} \xi_{j}.
\]
Since $\xi_{j}, j=1, \ldots, p$ are independent and centered, Proposition~2.6.1 in \cite{vershynin_2018} implies that there exists a universal constant $c^\prime$ such that 
\begin{align*}
\|\langle \boldsymbol{v}, \boldsymbol{z}\rangle\|_{\psi_2}^2
    & \leq c^\prime \sum_{j=1}^p \left(W^\top\boldsymbol{v}\right)_{j}^2 \|\xi_{j}\|_{\psi_2}^2 \\
    & \leq c^\prime L^2 \|W^\top\boldsymbol{v}\|^2 \\
    & = c^\prime L^2 \|\boldsymbol{v}\|_\Sigma^2 \\
    & \leq c^\prime L^2 \|\Sigma\|.
\end{align*}
Therefore, by the definition of sub-Gaussian norm of a random vector, we have $\|\boldsymbol{z}\|_{\psi_2} \leq \sqrt{c^\prime}L\|\Sigma\|$.

With this, the desired conclusion follows immediately from equation~\eqref{eq:testerrornoisygenexp} in the proof of Theorem~\ref{detail-noisy-main-1-simple},  and Lemma~\ref{lemm:esG-noisy-what-bound} where we use that
$
\tfrac{1}{n\rho\|\bm \mu\|^4} + n\rho \geq \tfrac{1}{\|\bm \mu\|^2}.
$

\end{proof}

\subsection{Proof of Theorem~\ref{thm:noiseless-main-sG}}\label{sec:proof:thm:noiseless-main-sG}

\begin{proof}[Proof of Theorem~\ref{thm:noiseless-main-sG}] 
For the regime (i), first note that, by $\tilde C_5 > 2\tilde C_1$, \eqref{eq:sGtrcond} implies that the condition of Lemma~\ref{lemm:esGboundsevents} is satisfied. Thus, with probability at least $1-4\delta$ event $\bigcap_{i=1}^4 E_i$ holds with $\eps$, $\alpha_2$, $\alpha_\infty$, $M$, $\beta$, and $\rho$ specified in Lemma~\ref{lemm:esGboundsevents}.

Under the conclusions of Lemma~\ref{lemm:esGboundsevents}, we have $\beta = \eps < \tfrac{1}{2}$ by $\tilde C_5 > 2\tilde C_1$. Furthermore, the assumption $\|\bs \mu\|\sqrt{(1-\beta)n\rho} \geq C\alpha_2$ in (i) of Theorem~\ref{thm:noiseless-main} is equivalent to
\[
\|\bs \mu\|^2 \geq \tfrac{2}{\sqrt{1-\beta}} C \tilde C_2 L\|\bs \mu\|_\Sigma.
\]
Similarly, we can obtain equivalent expressions of the remaining conditions of (i) of Theorem~\ref{thm:noiseless-main} as follows:
\begin{align*}
& \alpha_2\|\bs \mu\|\sqrt{(1+\beta)n\rho} \leq \tfrac{1}{4} \\
&\iff  \Tr(\Sigma) \geq 8\sqrt{1+\beta}\tilde C_2 L n \|\bs \mu\|_\Sigma, \\
& \eps M \sqrt{(1+\beta)n\rho} \leq \tfrac{1}{4} \\
&\iff \Tr(\Sigma) \geq 4\sqrt{1+\beta}(1+\eps)\tilde C_1 L^2 \max\left\{n\|\Sigma\|_F, n^{3/2}\|\Sigma \|\right\}, \\
&M \alpha_\infty \|\bs \mu\|(1+\beta) n\rho < \tfrac{3}{32} \\
&\iff \Tr(\Sigma) > \tfrac{64}{3}(1+\beta)(1+\eps) \tilde C_3 L n\sqrt{\log (\tfrac{n}{\delta})}\|\bs \mu\|_\Sigma.
\end{align*}
By $\eps = \beta < \tfrac{1}{2}$, it is easy to see all of these are satisfied under the regime (i) of this theorem. Note that under (i) we can assume $\delta \leq 1/4$ as the statement becomes trivial otherwise. This implies $\log(n/\delta) > 1$.

For the regime (ii), again the condition of Lemma~\ref{lemm:esGboundsevents} is satisfied. Thus, with probability at least $1-2\delta$ event $\bigcap_{i=1, 3} E_i$ holds with $\eps$ and $M$ specified in Lemma~\ref{lemm:esGboundsevents}.

Then, the condition (ii) of Theorem~\ref{thm:noiseless-main} is equivalent to $\|\bs \mu\|\geq C^\prime (1+\eps)\sqrt{\Tr(\Sigma)}$. Since $\eps \leq \tfrac{1}{2}$, it is easy to see this is satisfied under the regime (ii) of this theorem.

So far we have shown both conditions (i) and (ii) of Theorem~\ref{thm:noiseless-main} are satisfied under the regime (i) and (ii) of this theorem, respectively. Therefore, by Theorem~\ref{thm:noiseless-main} and its extension in~\eqref{eq:errboundnonoise-exp}, we obtain the desired test error bound where we use that
$
\tfrac{1}{n\rho\|\bm \mu\|^4} + n\rho \geq \tfrac{1}{\|\bm \mu\|^2}.
$
\end{proof}

\section{Detailed Comparison with previous work for sub-Gaussian Mixtures}\label{sec:detailedcomparison}

In this section, we provide a more detailed comparison of our results for \ref{model:sG} to the previous work \cite{JMLR:v22:20-974,NEURIPS2021_46e0eae7,wang2022binary, minsker2025classification}. 

Before proceeding to the details, we would like to remind readers that our main motivation is to present more universal framework that can be applicable to broader scenario rather than to attain the state-of-the-art in a specific setting. Nevertheless, our analysis also provides new insights in some classical settings such as sub-Gaussian predictors in the noiseless case, and the purpose of this section is to compare the commonalities and differences of our results and existing results in the literature.

We also note that model their setting is also simpler than ours. All of them only studied \ref{model:sG} with $W = U\Lambda^{1/2}$ while our results hold for any $W \in \R^{p\times p}$ that is a solution to $WW^\top = \Sigma$, which includes the more standard choice $W = U\Lambda^{1/2}U^\top$, as noted in Section~\ref{sec:setting}.

For the case of isotropic Gaussian mixtures, a comparison of our results with existing findings is summarized in Figure~\ref{fig:summary-svp-benign1} and Table~\ref{tb:summary-svp-benign1} for the noiseless case and Figure~\ref{fig:summary-svp-benign2} and Table~\ref{tb:summary-svp-benign2} for the noisy case.

\subsection{Noiseless sub-Gaussian Mixtures}\label{sec:detailedcomparison-noiseless}

As noted in Appendix~\ref{sec:recoveringSubGauss}, our Theorem~\ref{thm:noiseless-main-sG} essentially recovers previous work in noiseless sub-Gaussian mixtures. We here provide more detailed comparison.

Since many results in existing work are obtained with probability at least $1 - n^{-1}$, we take $\delta \asymp n^{-1}$ to allow for more direct comparison:

\begin{thm}[Theorem~\ref{thm:noiseless-main-sG} with $\delta \asymp n^{-1}$]\label{thm:sGnoiseless-comparison}
Suppose \ref{model:sG} holds with $\eta = 0$ and one of the following holds for a sufficiently large universal constant $C$:
\begin{enumerate}
    \item[(i)] $\|\bs \mu\|^2 \geq C L\|\bs \mu\|_\Sigma$ and 
    \begin{align*}
    \Tr(\Sigma) & \geq C \max\left\{L^2 n\|\Sigma\|_F, L^2n^{3/2} \|\Sigma\|, L n\sqrt{\log n} \|\bs \mu\|_\Sigma \right\}, %\label{eq:sGtrcond}
    \end{align*}
    \item[(ii)] $\|\bs \mu\| \geq \tfrac{3}{2}C \sqrt{\Tr(\Sigma)}$ and $\displaystyle \Tr(\Sigma) \geq C L^2\max\left\{ \sqrt{n}\|\Sigma\|_F, n \|\Sigma\|\right\}$.
\end{enumerate}
Then, we have with probability at least $1 - n^{-1}$
    \begin{equation*}
        \P_{(\boldsymbol{x}, y)}(\langle \boldsymbol{\hat w}, y\boldsymbol{x}\rangle < 0) \leq \exp\left\{-\frac{c}{L^2 }\frac{n\|\bs \mu\|^4}{\|\Sigma\|(n\|\bs \mu\|^2 + \Tr(\Sigma))}\right\},
    \end{equation*}
where $c$ is a universal constant.
\end{thm}

We first compare with \cite{wang2022binary} whose result corresponds to the special case of Gaussian mixtures. 

\begin{thm}[Theorem~3.1 \& 4.2, \cite{wang2022binary}]\label{thm:wangThrampoulidis-aniso-noiseless}
\begin{sloppypar}
Suppose $\bs \xi \sim \mathcal{N}(0, I_p)$ in model \ref{model:sG} (Gaussian mixtures) with $\eta = 0$, $\Tr(\Sigma) \geq C\max\{n\sqrt{\log n}\|\Sigma\|_F, n^{3/2}\log n \|\Sigma\|, n\sqrt{\log n}\|\bs \mu\|_\Sigma\}$ and $\|\bs \mu\|^2 \geq C \|\bs \mu\|_\Sigma$ for a sufficiently large universal constant $C$. Then with probability $1 - n^{-1}$ we have
\end{sloppypar}
\begin{equation}
\P_{(\bs x, y)}(\langle \bs{\hat w}, y\bs x\rangle <0) \leq \exp\left(- C^\prime \frac{n\|\bs \mu\|^4}{n\|\bs \mu\|_\Sigma^2 + n\|\Sigma\|_F^2 + \frac{\|\Sigma\|_F^2}{\Tr(\Sigma)^2}n^3\|\bs \mu\|_\Sigma^2}\right),
\end{equation}
where $C^\prime$ is a universal constant.
\end{thm}

By comparing the conditions, we see that the condition of \cite{wang2022binary}'s result is stronger than condition (i) in Theorem~\ref{thm:sGnoiseless-comparison} by a $\log n$ factor. Moreover, \cite{wang2022binary}'s result does not capture condition (ii) in our result.

For the test error bound, we see quite a bit of difference in the denominator of fractions inside the exponentials. While $\|\Sigma\|\|\bs \mu\|^2 \geq \|\bs \mu\|_\Sigma^2$ holds in general, we can show that our bound is tighter when $\lambda_{max}/\lambda_{min} = O(1)$. In fact, if $\lambda_{max}/\lambda_{min} = O(1)$, then we have $\|\Sigma\|\|\bs \mu\|^2 \asymp \|\bs \mu\|_\Sigma^2$ and $\|\Tr(\Sigma)\| \ll n\lambda_{min}\Tr(\Sigma) \leq n\frac{\Tr(\Sigma)}{p}\Tr(\Sigma) \leq n\|\Sigma\|_F^2$, where the last inequality is due to Cauchy-Schwartz inequality.

\bigskip

Now we turn to comparison with \cite{NEURIPS2021_46e0eae7}'s result:

\begin{thm}[Theorem~3.1, \cite{NEURIPS2021_46e0eae7}]\label{thm:cao2021}
Suppose in model \ref{model:sG} with $\eta = 0$ and $W = U\Lambda^{1/2}$ that $\Tr(\Sigma) \geq C\max \{n\|\Sigma\|_F, n^{3/2}\|\Sigma\|, n\sqrt{\log n}\|\bs \mu\|_\Sigma\}$ and $\|\bs \mu\|^2 \geq C\|\bs \mu\|_\Sigma$ for some constant $C$. Then with probability at least $1 - n^{-1}$ we have
\begin{equation}
\P_{(\bs x, y)}(\langle \bs{\hat w}, y\bs x\rangle <0) \leq \exp\left(- C^\prime \frac{n\|\bs \mu\|^4}{n\|\bs \mu\|_\Sigma^2 + \|\Sigma\|_F^2 + n\|\Sigma\|^2}\right),
\end{equation}
where $C^\prime$ is some constant.
\end{thm}

We see that their conditions match our condition (i) in Theorem~\ref{thm:sGnoiseless-comparison} except that the dependence of the constant $C$ on sub-Gaussian norm is not explicit in their work. Our condition (ii) in Theorem~\ref{thm:sGnoiseless-comparison} is not captured by their work. As noted in Appendix~\ref{sec:recoveringSubGauss}, their bound is better in general but matches with ours when we assume $\lambda_{max}/\lambda_{min} = O(1)$, e.g. $\Sigma = I_p$.

\bigskip

Lastly, we compare our result with \cite{minsker2025classification}. They extended \cite{NEURIPS2021_46e0eae7} and \cite{wang2022binary} in a different direction from ours while only studied noiseless \ref{model:sG}. To state their result we need to define the following quantity: for a universal constant $b>1$, we define $k^*$ by 
\begin{equation}\label{eq:k*}
    k^* := \min \Big\{k\geq 0, \sum_{j>k} \lambda_j \geq b n \lambda_{k+1}\Big\},
\end{equation}
where $\lambda_1 \geq \lambda_2 \geq \dots \geq \lambda_p$ are eigenvalues of $\Sigma$.

For $k^*$, we let $\pi_{k^*}: \R^p \to \R^p$ be the orthogonal projection onto the subspace spanned by eigenvectors corresponding to eigenvalues $\lambda_1, \ldots, \lambda_{k^*}$.

\begin{thm}[Theorem~3 \& 5, \cite{minsker2025classification}]
Suppose model \ref{model:sG} with $\eta = 0$, $W = U\Lambda^{1/2}$, and positive definite $\Sigma \in \R^{p\times p}$. Suppose further that $k^* \leq \tfrac{n}{2}$, $\|\pi_{k^*}(\bs \mu)\|\leq \tfrac{\|\bs \mu\|}{\sqrt{5}}$, $k^*(\log n)^2 \leq Cn$, $\sqrt{n\log n\sum_{j > k^*}\lambda_j^2} \leq C\sum_{j > k^*} \lambda_j$, and $n \sqrt{(1+k^*) \log n} \|\bs \mu\|_\Sigma \leq C \sum_{j > k^*} \lambda_j$ for some universal constant $C>0$. Then for some universal constants $C^\prime>0$ and $c>0$, we have with probability at least $1 - 3/n$
\begin{equation}
\begin{aligned}
& \P_{(\bs x, y)}\left(\langle \bs{\hat w}, y\bs x\rangle < 0 \right) \\
& \quad \leq C^\prime\exp\left(- c \frac{n\|\bs \mu\|^4}{(1+k^*)n\|\bs \mu\|_\Sigma^2 + k^*\lambda_{k^*}^2 + \sum_{j > k^*}\lambda_j^2 + (k^*\lambda_{k^*}^2 + \lambda_{k^*+1}^2)\log n}\right).    
\end{aligned}
\end{equation}
\end{thm}

Our assumption $\Tr(\Sigma) \gtrsim n\|\Sigma\|$ in Theorem~\ref{thm:sGnoiseless-comparison} implies $k^* = 0$. Thus, they allow more generality in the sense $k^*$ can be non-zero. 

When $k^* = 0$, i.e. $\Tr(\Sigma)\gtrsim n\|\Sigma\|$, their assumptions read: 
\[
\Tr(\Sigma)\gtrsim \max\{\sqrt{n\log n}\|\Sigma\|_F, n\|\Sigma\|, n\sqrt{\log n}\|\bs \mu\|_\Sigma\},
\]
and the test error bound becomes
\begin{equation}\label{eq:minskererror_k*=0}
\P_{(\bs x, y)}\left(\langle \bs{\hat w}, y\bs x\rangle < 0 \right) \leq C^\prime \exp \left(- c\frac{n\|\bs \mu\|^4}{n\|\bs \mu\|_\Sigma^2 + \|\Sigma\|_F^2 + \log n\|\Sigma\|^2} \right).    
\end{equation}

Comparing with our condition (i) in Theorem~\ref{thm:sGnoiseless-comparison}, we see that their assumption on $\Tr(\Sigma)$ is weaker than ours by a factor of $\sqrt{n}$ while we do not require $\Sigma$ to be positive definite and our result holds for a more general choice of $W$. Moreover, our condition (ii) in Theorem~\ref{thm:sGnoiseless-comparison} is not captured by their work. Their test error bound \eqref{eq:minskererror_k*=0} is slightly tighter than that of Theorem~\ref{thm:cao2021} (Theorem~3.1, \cite{NEURIPS2021_46e0eae7}). 

We note that their proof relies on their specific choice of $W = U\Lambda^{1/2}$ for diagonal $\Lambda$ which allows the following decomposition of $ZZ^\top$ in a sum of independent random matrices associated with the eigenvalues of $\Sigma$:
\begin{equation}
    ZZ^\top = \Xi \Lambda^{1/2}U^\top U \Lambda^{1/2} \Xi^\top = \Xi \Lambda \Xi^\top = \sum_{j=1}^p \lambda_j \bs \zeta_j \bs \zeta_j^\top.
\end{equation}
where $\Xi = (\bs \xi_1, \ldots, \bs \xi_n)^\top = (\bs \zeta_1, \ldots, \bs \zeta_p) \in \R^{n\times p}$. Such a decomposition is not possible with the more standard choice, i.e. $W = U\Lambda^{1/2}U^\top$. 

\subsection{Noiseless, isotropic Gaussian Mixtures}\label{sec:detailedcomparison-noiseless-isotropic}

We now provide detailed comparison for isotropic Gaussian mixtures, for which \cite{wang2022binary} obtained tighter results than Theorem~\ref{thm:wangThrampoulidis-aniso-noiseless} taking advantage of rotational invariance of isotropic Gaussian random variables and a special property of inverse Wishart matrices.

By taking $\Sigma = I_p$, Theorem~\ref{thm:sGnoiseless-comparison} becomes:

\begin{thm}[Theorem~\ref{thm:noiseless-main-sG} with $\delta \asymp n^{-1}$ and $\Sigma = I_p$]\label{thm:sGnoiseless-isotropic-comparison}
Suppose \ref{model:sG} holds with $\eta = 0$ and one of the following holds for a sufficiently large universal constant $C$:
\begin{enumerate}
    \item[(i)] $\|\bs \mu\| \geq C L$ and 
    \begin{equation}\label{eq:sG-what=wls}
    p \geq C \max\left\{L^4 n^2, L n\sqrt{\log n} \|\bs \mu\| \right\}, 
    \end{equation}
    \item[(ii)] $\|\bs \mu\| \geq \tfrac{3}{2}C \sqrt{p}$ and $\displaystyle p \geq C L^2 n $.
\end{enumerate}
Then, we have with probability at least $1 - n^{-1}$
    \begin{equation*}
        \P_{(\boldsymbol{x}, y)}(\langle \boldsymbol{\hat w}, y\boldsymbol{x}\rangle < 0) \leq \exp\left\{-\frac{c}{L^2 }\frac{n\|\bs \mu\|^4}{n\|\bs \mu\|^2 + p}\right\},
    \end{equation*}
where $c$ is a universal constant.
\end{thm}

As noted in the previous section, this result matches with \cite{NEURIPS2021_46e0eae7}.

\begin{rmk}
For the equivalence $\bs{\hat{w}} = \bs w_{\rm LS}$ to hold, it is actually enough to assume just \eqref{eq:sG-what=wls} in condition (i). We also note that condition (ii) does not necessarily guarantee the equivalence $\bs{\hat{w}} = \bs w_{\rm LS}$. 
\end{rmk}

Theorem~3.2 and 4.3 of \cite{wang2022binary} imply the following result:

\begin{thm}[Theorem~3.2 \& 4.3, \cite{wang2022binary}]
\begin{sloppypar}
Suppose $\bs z \sim \mathcal{N}(0, I_p)$ in model \ref{model:sG} (Gaussian mixtures) with $\eta = 0$, $p \geq C\max\{n\log n, n\sqrt{\log n}\|\bs \mu\|\}$ and $\|\bs \mu\| \geq C$ for a sufficiently large universal constant $C$. Then with probability $1 - n^{-1}$ we have     
\end{sloppypar}
\begin{equation}
\P_{(\bs x, y)}(\langle \bs{\hat w}, y\bs x\rangle <0) \leq \exp\left(- C^\prime \frac{n\|\bs \mu\|^4}{n\|\bs \mu\|^2 + p}\right),
\end{equation}
where $C^\prime$ is a universal constant.
\end{thm}

First note that their result corresponds to condition (i) in Theorem~\ref{thm:sGnoiseless-isotropic-comparison}, and hence the result under condition (ii) is not captured by their result. Comparing our condition (i) and theirs, we see that their condition $p\geq C\max\{n\log n, n\sqrt{\log n}\|\bs \mu\|\}$ is weaker than our corresponding condition \eqref{eq:sG-what=wls} under (i). Their proof relies on the special properties of isotropic Gaussian as noted earlier and their proof technique is not applicable to general isotropic sub-Gaussian mixtures. We note that \cite{minsker2025classification}'s result when applied to this case recovers their result. 

Lastly, we would like to note that the test error bounds are of the same form in this case in \cite{NEURIPS2021_46e0eae7,wang2022binary,minsker2025classification}, and also same as ours namely
\[
C\exp\left(- c\frac{n\|\bs \mu\|^4}{n\|\bs \mu\|^2 + p}\right)
\]
for some constants $C, c>0$. Thus, benign overfitting occurs when $\|\bs \mu\| \gg \max\{1, \sqrt{p/n}\}$. Since $p \gg n$ is required in all the results, it reduces to $\|\bs \mu\| \gg \sqrt{p/n}$. 
The difference appears in the sufficient conditions. While \cite{wang2022binary} and \cite{minsker2025classification} have weaker sufficient condition than (i) in Theorem~\ref{thm:sGnoiseless-isotropic-comparison}, their results do not capture condition (ii) in the theorem. Fixing $\|\bs \mu\| \asymp n^a$, $p \asymp n^b$ for some constant $a, b$ and letting $n\to \infty$, the sufficient conditions of benign overfitting are summarized in Figure~\ref{fig:summary-svp-benign1} and Table~\ref{tb:summary-svp-benign1}.

% ---------------- parameters for THIS figure only ----------------
\def\AmaxB{4}
\def\BmaxB{6}
\def\aQuarterB{0.25}
\def\aHalfB{0.5}
\def\aOneB{1}
\def\epsB{0.001} % for strict inequalities like a>0

\begin{figure}[H]
\begin{tikzpicture}
\begin{axis}[
  axis lines=middle,
  axis line style={->},
  xmin=0, xmax=\AmaxB,
  ymin=1, ymax=\BmaxB,
  width=14cm, height=8cm,
  clip=true,
  xlabel={$a$},
  ylabel={$b$},
  xlabel style={at={(axis description cs:1,0)},anchor=north west,xshift=2pt,yshift=-2pt},
  ylabel style={at={(axis description cs:0,1)},anchor=south east,xshift=-3pt,yshift=2pt},
  legend style={at={(0.5,-0.18)},anchor=north,draw=none,fill=none},
  legend columns=3,
  after end axis/.code={
  \node[anchor=east] at (axis cs:-0.02,1) {$1$};
}
]

% -------- named paths (define before fill between) --------
\addplot[name path=TopB,   domain=0:\AmaxB, draw=none] {\BmaxB};
\addplot[name path=BtwoB,  domain=0:\AmaxB, draw=none] {2};
\addplot[name path=BoneB,  domain=0:\AmaxB, draw=none] {1};

\addplot[name path=L4B,    domain=0:\AmaxB, draw=none] {1 + 4*x}; % 1+4a
\addplot[name path=L1aB,   domain=0:\AmaxB, draw=none] {1 + x};   % 1+a
\addplot[name path=L2aB,   domain=0:\AmaxB, draw=none] {2*x};     % 2a

% =========================================================
% (V)  {1+a < b < 1+4a}  (base light band)
\addplot[orange!70, opacity=0.25]
  fill between[of=L4B and L1aB, soft clip={domain=\epsB:\AmaxB}];

% (II) {a >= 0, b >= 2, 1+a <= b < 1+4a}
% Piece 1: a in [1/4, 1] -> lower bound is b=2
\addplot[red!70, opacity=0.45]
  fill between[of=L4B and BtwoB, soft clip={domain=\aQuarterB:\aOneB}];
% Piece 2: a in [1, 2] -> lower bound is b=1+a
\addplot[red!70, opacity=0.45]
  fill between[of=L4B and L1aB, soft clip={domain=\aOneB:\AmaxB}];

% (I) {b>=2, b>=1+4a, a>=0}
% Piece 1: a in [0, 1/4] -> b>=2 dominates
\addplot[blue!70, opacity=0.35]
  fill between[of=TopB and BtwoB, soft clip={domain=0:\aQuarterB}];
% Piece 2: a in [1/4, 2] -> b>=1+4a dominates
\addplot[blue!70, opacity=0.35]
  fill between[of=TopB and L4B, soft clip={domain=\aQuarterB:\AmaxB}];

% (IV) {a >= 0, 1+4a <= b < 2}
% Exists only for a in [0, 1/4]
\addplot[gray!60, opacity=0.35]
  fill between[of=BtwoB and L4B, soft clip={domain=0:\aQuarterB}];

% (III) {1 <= b <= min(2a, 1+a)}
% min switches at a=1:
%  - for a in [1/2, 1]: min = 2a
%  - for a in [1, 2]:   min = 1+a
\addplot[green!60, opacity=0.35]
  fill between[of=L2aB and BoneB, soft clip={domain=\aHalfB:\aOneB}];
\addplot[green!60, opacity=0.35]
  fill between[of=L1aB and BoneB, soft clip={domain=\aOneB:\AmaxB}];

% -------- boundary lines on top --------
\addplot[thick, domain=0:1] {2}
  node[pos=0.98, anchor=east, xshift=-2pt, yshift=8pt] {};
  
\addplot[thick, dashed,domain=1:\AmaxB] {2}
  node[pos=0.98, anchor=east, xshift=-2pt, yshift=8pt] {$b=2$};

\addplot[thick, domain=0:\AmaxB] {1+4*x}
  node[pos=0.3, anchor=north west, xshift=-5pt, yshift=-3pt] {$b=1+4a$};

\addplot[thick, domain=0:\AmaxB] {1+x}
  node[pos=0.75, anchor=north west, xshift=-2pt, yshift=2pt] {$b=1+a$};

\addplot[thick, domain=0:1] {2*x}
  node[pos=0.60, anchor=north west, xshift=2pt, yshift=2pt] {};

\addplot[thick, dashed, domain=1:\AmaxB] {2*x}
  node[pos=0.60, anchor=north west, xshift=2pt, yshift=2pt] {$b=2a$};

\addplot[thick, domain=0:\AmaxB] {1}
  node[pos=0.98, anchor=east, xshift=-2pt, yshift=8pt] {$b=1$};

% -------- region labels (plain) --------
\node at (axis cs:0.5,4.5) {\textbf{(I)}};
\node at (axis cs:1.7,4.5)  {\textbf{(II)}};
\node at (axis cs:0.5,1.8)  {\textbf{(V)}};
\node at (axis cs:0.11,1.8) {\textbf{(IV)}};
\node at (axis cs:2.5,2.5) {\textbf{(III)}};
\node at (axis cs:0.45,1.2) {\textbf{(VI)}};

\end{axis}
\end{tikzpicture}

\caption{Summary of existing results on the noiseless isotropic GMM. (I) - (V) on  the $(a,b)$-plane are the regions covered by this paper and the existing literature: $\text{\textbf{(I)}} = \{b\geq 2, b \geq 1 + 4a, a \geq 0\}$, $\text{\textbf{(II)}} = \{a \geq 0, b \geq 2, 1 + a \leq b < 1 + 4a\}$, $\text{\textbf{(III)}} = \{1 \leq b \leq \min\{2a, 1 + a\}\}$, $\text{\textbf{(IV)}} = \{a \geq 0, 1 + 4a \leq b < 2\}$, and $\text{\textbf{(V)}} = \{1 + a < b < 1 + 4a\}$. 
}
\label{fig:summary-svp-benign1}
\end{figure}

\begin{table}[H]
\caption{Conditions of $\bs{\hat w} = \bs w_{\rm LS}$ and Benign Overfitting in the noiseless isotropic GMM.}
\label{tb:summary-svp-benign1}
\begin{center}
\resizebox{\textwidth}{!}{%
\begin{tabular}{c|c|c}
& $\bs{\hat w} = \bs w_{\rm LS}$ & Benign Overfitting \\ \toprule
\textbf{(I)} & Holds (\textbf{This paper}, \cite{wang2022binary}, & \textbf{Fails (This paper, \cite{NEURIPS2021_46e0eae7})} \\ 
 & \cite{NEURIPS2021_46e0eae7,minsker2025classification}) & \\ \midrule
\textbf{(II)} & Holds (\textbf{This paper}, \cite{wang2022binary}, & Holds (\textbf{This paper}, \cite{wang2022binary}, \\ 
& \cite{NEURIPS2021_46e0eae7,minsker2025classification}) & \cite{NEURIPS2021_46e0eae7,minsker2025classification}) \\ \midrule
\textbf{(III)} & Unknown & \textbf{Holds (This paper)} \\ \midrule
\textbf{(IV)} & Holds(\cite{wang2022binary,minsker2025classification}) & Unknown \\ \midrule
\textbf{(V)} & Holds(\cite{wang2022binary,minsker2025classification}) & Holds(\cite{wang2022binary,minsker2025classification}) \\ \midrule
\textbf{(VI)} & Unknown & Unknown \\
\bottomrule
\end{tabular}
}
\end{center}
\end{table}

\subsection{Noisy sub-Gaussian mixtures}\label{sec:detailedcomparison-noisy}

The noisy sub-Gaussian mixture model is studied by \cite{JMLR:v22:20-974} and \cite{wang2022binary} only for the weak signal regime. Since \cite{wang2022binary} only considers the isotropic Gaussian case and \cite{JMLR:v22:20-974} assumes an approximately isotropic case, i.e. $\|\Sigma\|\gtrsim 1$ and $\Tr(\Sigma) \asymp p$, we provide below a special case of Theorem~\ref{thm:sGnoisymain}:

\begin{thm}[Theorem~\ref{thm:sGnoisymain} with $\|\Sigma\| \asymp 1$ and $\Tr(\Sigma) \asymp p$]
Suppose \ref{model:sG} holds with $\eta \in (0, 1/2)$, $\|\Sigma\| \asymp 1$, $\Tr(\Sigma) \asymp p$, $n \geq \frac{C}{\min\{\eta, 1-2\eta\}^2}\log(\tfrac{1}{\delta})$,
\begin{align*}
    p & \geq \frac{CL^2}{\eta}\max\{n^{3/2}, n\|\Sigma\|_F\}, \\
    \|\bs \mu\|^2 & \geq CL\|\bs \mu\|_\Sigma, 
\end{align*}
and further one of the following conditions holds:
\begin{enumerate}
    \item[(i)] $p\geq CL n\sqrt{\log (\tfrac{n}{\delta})}\|\bs \mu\|_\Sigma$,
    \item[(ii)] $\|\bs \mu\|^2 \geq C\sqrt{\log (\tfrac{n}{\delta})}\|\bs \mu\|_\Sigma$ and $p \geq CL^2 n\sqrt{\log(\tfrac{n}{\delta})}\frac{\|\bs \mu\|_\Sigma^2}{\|\bs \mu\|^2}$.
\end{enumerate}
Then, we have with probability at least $1 - \delta$
\[
\P_{(\bs x, y_\n)}(\langle \bs{\hat w}, y_\n\bs x\rangle <0) \leq \eta + (1-\eta)\exp\left\{ - c\frac{(1-2\eta)^2}{L^2}\left(\frac{n}{p} + \frac{p}{n\|\bs \mu\|^4}\right)^{-1}\right\},
\]
\end{thm}

We note that sufficient conditions of the equivalence $\bs{\hat w} = \bs w_{\rm LS}$ is same as above up to a constant factor.

We first compare with \cite{JMLR:v22:20-974}. We note that \cite{JMLR:v22:20-974} actually consider a more general label flipping noise which contains the same label flipping noise as a special case. For the purpose of direct comparison, we provide their result with the same label flipping noise as ours.

\begin{thm}[Theorem~4, \cite{JMLR:v22:20-974}]
Suppose \ref{model:sG} holds with $\eta \leq C^{-1}$, $W = U\Lambda^{1/2}$, $\delta < C^{-1}$, $\|\Sigma\|\asymp 1$, $\Tr(\Sigma)\asymp p$, $n\geq C\log (\tfrac{1}{\delta})$, 
\begin{align*}
    p & \geq C \max\{n\|\bs \mu\|^2, n^2\log (\tfrac{n}{\delta})\}, \\
    \|\bs \mu\|^2 & \geq Cn\log (\tfrac{n}{\delta}),
\end{align*}
for a sufficiently large constant $C$. Then, we have with probability at least $1- \delta$
\[
\P_{(\bs x, y_\n)}(\langle \bs{\hat w}, y_\n\bs x\rangle <0) \leq \eta + \exp\left( - c\frac{\|\bs \mu\|^4}{p}\right),
\]
where $c$ is a universal constant.
\end{thm}

Since they assume $\Tr(\Sigma)\asymp p$ and $p\geq Cn\|\bs \mu\|^2$, their result only captures the weak signal regime in our analysis. Comparing the test error bound, we see that their bound is substantially looser than ours by an factor of $n$ inside the exponentials.

In terms of the assumptions, we can show that their assumptions are stronger than ours under condition (ii). To see this, note that Cauchy-Schwartz inequality implies $\Tr(\Sigma) \leq \sqrt{p}\|\Sigma\|_F$, and thus $p/\|\Sigma\|_F \geq p^{3/2}/\Tr(\Sigma) \asymp \sqrt{p}$, where the last step is due to $\Tr(\Sigma)\asymp p$. With $p\gg n^2$, this further implies $p/\|\Sigma\|_F \gg n$, which shows $p \gtrsim n\|\Sigma\|_F$ is satisfied. The other conditions under condition (ii) should be straightforward by noting $\|\Sigma\|\asymp 1$.

We now turn to comparison with the result of \cite{wang2022binary}. Since their result is given for noisy isotropic Gaussian mixtures and  with $\delta \asymp n^{-1}$, we provide below a special case of Theorem~\ref{thm:sGnoisymain} when $\Sigma = I_p$ and $\delta \asymp n^{-1}$:

\begin{thm}[Theorem~\ref{thm:sGnoisymain} with $\Sigma = I_p$ and $\delta \asymp n^{-1}$]\label{thm:sGnoisy-comparison-isotropic}
Suppose \ref{model:sG} holds $\eta \in (0, 1/2)$, $\Sigma = I_p$, $n \geq \frac{C}{\min\{\eta, 1-2\eta\}^2}$,
\begin{align*}
    p & \geq \frac{CL^4}{\eta^2}n^2, \\
    \|\bs \mu\| & \geq CL, 
\end{align*}
and further one of the following conditions holds:
\begin{enumerate}
    \item[(i)] $p\geq CL n\sqrt{\log n}\|\bs \mu\|$,
    \item[(ii)] $\|\bs \mu\| \geq C\sqrt{\log n}$ and $p \geq CL^2 n\sqrt{\log n}$.
\end{enumerate}
Then, we have with probability at least $1 - n^{-1}$
\[
\P_{(\bs x, y_\n)}(\langle \bs{\hat w}, y_\n\bs x\rangle <0) \leq \eta + (1-\eta)\exp\left\{ - c\frac{(1-2\eta)^2}{L^2}\left(\frac{n}{p} + \frac{p}{n\|\bs \mu\|^4}\right)^{-1}\right\},
\]  
\end{thm}

We note that 

We now compare our result with the corresponding results of \cite{wang2022binary}:

\begin{thm}[Theorem~7.2 \& 7.3, \cite{wang2022binary}]
Suppose \ref{model:sG} with $\bs z \sim \mathcal{N}(0, I_p)$ with $n \geq C$,
\begin{align*}
    p & \geq C \max\{n\log n, n\sqrt{\log n}\|\bs \mu\|, n\|\bs \mu\|^2\}, \\
    \|\bs \mu\| & \geq C
\end{align*}
for a sufficiently large constant $C$. Then, we have with probability at least $1 - n^{-1}$
\[
\P_{(\bs x, y_\n)}(\langle \bs{\hat w}, y_\n\bs x\rangle <0) \leq \eta + \exp\left( - c\frac{n\|\bs \mu\|^4}{p}\right),
\]
where $c$ is a universal constant.
\end{thm}

Since $p \gtrsim n\|\bs \mu\|^2$ is assumed, their result only captures the weak signal regime in our analysis. Their result corresponds to Theorem~\ref{thm:sGnoisy-comparison-isotropic} under condition (i). Comparing the lower bounds of $p$, we see that they assume $p\gtrsim n\log n$ while ours requires $p \gtrsim n^2$. The discrepancy is again due to the special properties of isotropic Gaussian random vectors, and their proof technique is not applicable to general sub-Gaussian mixtures.

Fixing $\|\bs \mu\| \asymp n^a$, $p \asymp n^b$ for some constant $a, b$ and letting $n\to \infty$, the sufficient conditions of benign overfitting when $\Sigma = I_p$ are summarized in Figure~\ref{fig:summary-svp-benign2} and Table~\ref{tb:summary-svp-benign2} 

We note that $b\geq 1+2a$ corresponds to the weak signal regime $p\gtrsim n\|\bs \mu\|^2$. While \cite{wang2022binary}'s result is sharper in the weak signal regime taking advantage of the special properties of isotropic Gaussian distribution, our results allows the strong signal regime (when $b < 1+2a$).

\def\Amax{2}
\def\Bmax{6}

\begin{figure}[H]
\begin{tikzpicture}
\begin{axis}[
  axis lines=middle,
  axis line style={->},
  xmin=0, xmax=\Amax,
  ymin=1, ymax=\Bmax,
  width=13.5cm, height=8cm,
  clip=true,
  legend style={
  at={(0.5,-0.18)},
  anchor=north,
  draw=none,
  fill=none,
  },
  legend columns=2,
  xlabel={$a$},
  ylabel={$b$},
  xlabel style={at={(axis description cs:1,0)},anchor=north west,xshift=2pt,yshift=-2pt},
  ylabel style={at={(axis description cs:0,1)},anchor=south east,xshift=-3pt,yshift=2pt},
  after end axis/.code={
  \node[anchor=east] at (axis cs:-0.015,1) {$1$};
}
]

% Key value a = 1/4
\def\aCut{0.25}

% --- Define named paths for boundaries
% Top boundary
\addplot[name path=Top, domain=0:\Amax, draw=none] {\Bmax};
% b = 2
\addplot[name path=Btwo, domain=0:\Amax, draw=none] {2};
% b = 1 + 4a
\addplot[name path=Lfour, domain=0:\Amax, draw=none] {1 + 4*x};
% b = 1 + 2a
\addplot[name path=Ltwo, domain=0:\Amax, draw=none] {1 + 2*x};

% =========================
% (IV-a): between b=1+4a and b=2 for a in [0, 1/4]
\addplot[gray!60, opacity=0.75]
  fill between[of=Btwo and Lfour, soft clip={domain=0:\aCut}];

% (IV), piece 1: between b=1+2a and b=1+4a for a in [0, 1/4]
\addplot[gray!60, opacity=0.45]
  fill between[of=Lfour and Ltwo, soft clip={domain=0:\aCut}];

% (V), piece 2: between b=1+2a and b=2 for a in [1/4, 1/2]
\addplot[gray!60, opacity=0.45]
  fill between[of=Btwo and Ltwo, soft clip={domain=\aCut:0.5}];

% Emphasize the a=0, b>1 part as a thick segment on the y-axis
\addplot[very thick, gray!70, domain=1:\Bmax] ({0},{x});

% =========================
% ----------------------------
% (II) Red ∩ Gray:
% For a in [1/4, 1/2], ALL of red is already in gray (since 1+2a <= 2)
\addplot[red!70, opacity=0.55]
  fill between[of=Lfour and Btwo, soft clip={domain=\aCut:0.5}];

% For a in [1/2, \Amax], the gray threshold is above 2, so intersection is b in [1+2a, 1+4a]
\addplot[red!70, opacity=0.55]
  fill between[of=Lfour and Ltwo, soft clip={domain=0.5:\Amax}];

% (III) Red only:
% Exists ONLY for a >= 1/2: b in [2, 1+2a]
\addplot[red!70, opacity=0.25]
  fill between[of=Ltwo and Btwo, soft clip={domain=0.5:\Amax}];

% ----------------------------

% =========================
% (1) BLUE: a>=0, b>=2, b > 1+4a
% Piece 1: a in [0, 1/4], between Top and b=2
\addplot[blue!70, opacity=0.35]
  fill between[of=Top and Btwo, soft clip={domain=0:\aCut}];

% Piece 2: a in [1/4, Amax], between Top and (1+4a)
\addplot[blue!70, opacity=0.35]
  fill between[of=Top and Lfour, soft clip={domain=\aCut:\Amax}];

% --- Draw boundary lines on top (so they’re visible)
\addplot[thick, domain=0:\Amax] {2}
  node[
    pos=0.98,
    anchor=east,   % text goes LEFT of the point
    xshift=-2pt,
    yshift=5pt     % optional: lift slightly above the line
  ] {$b=2$};

\addplot[thick, domain=0:\Amax] {1+4*x} node[pos=0.60, anchor=north west] {$b=1+4a$};
\addplot[thick, domain=0:\Amax] {1+2*x} node[pos=0.85, anchor=north west] {$b=1+2a$};

% Mark (1/4,2)
\addplot[only marks, mark=*, mark size=1.5pt] coordinates {(0.25,2)};
\node[anchor=south] at (axis cs:0.25,2) {\small $(\tfrac14,2)$};

% ---- Region labels (plain text, no boxes)
\node at (axis cs:0.4,5.0) {\textbf{(I)}};   % Blue ∩ Gray
\node at (axis cs:1.0,4.0)  {\textbf{(II)}};  % Red ∩ Gray
\node at (axis cs:1.4,3.0)  {\textbf{(III)}}; % Red only
\node at (axis cs:0.08,1.7)  {\textbf{(IV)}};  % Gray only
\node at (axis cs:0.25,1.7)  {\textbf{(V)}};  % Gray only
\node at (axis cs:1.0,1.5)  {\textbf{(VI)}};  

\end{axis}
\end{tikzpicture}
\caption{Summary of existing results on the noisy isotropic GMM. (I) - (V) on  the $(a,b)$-plane are the regions covered by this paper and the existing literature: $\text{\textbf{(I)}}  =\{a\geq 0, b\geq 2, b \geq 1 + 4a\}$, $\text{\textbf{(II)}} =\{a\geq 0, b\geq 2, 1 + 2a \leq b < 1 + 4a\}$, $\text{\textbf{(III)}} =\{2\leq b < 1 + 2a\}$, $\text{\textbf{(IV)}} =\{a\geq 0, 1 + 4a \leq b< 2\}$, $\text{\textbf{(V)}} =\{a\geq 0, 1+2a \leq b < \min\{1 + 4a, 2\}\}$. 
}
\label{fig:summary-svp-benign2}
\end{figure}

\begin{table}[H]
\caption{Conditions of $\bs{\hat w} = \bs w_{\rm LS}$ and Benign Overfitting in the noisy isotropic GMM.}
\label{tb:summary-svp-benign2}
\begin{center}
\resizebox{\textwidth}{!}{%
\begin{tabular}{c|c|c}
& $\bs{\hat w} = \bs w_{\rm LS}$ & Benign Overfitting \\ \toprule
\textbf{(I)} & Holds (\textbf{This paper}, \cite{wang2022binary}) & \textbf{Fails (This paper)} \\ \midrule
\textbf{(II)} & Holds (\textbf{This paper}, \cite{wang2022binary}) & Holds (\textbf{This paper}, \cite{wang2022binary}) \\ \midrule
\textbf{(III)} & \textbf{Holds (This paper)} & \textbf{Holds (This paper)} \\ \midrule
\textbf{(IV)} & Holds(\cite{wang2022binary}) & Unknown \\ \midrule
\textbf{(V)} & Holds(\cite{wang2022binary}) & Holds(\cite{wang2022binary}) \\ \midrule
\textbf{(VI)} & Unknown & Unknown \\
\bottomrule
\end{tabular}
}
\end{center}
\end{table}

\section{Comparison with \cite{tsigler2025benign}}

In this section, we provide comparison with the concurrent work \cite{tsigler2025benign} which was posted on arxiv after we posted the initial version of our paper. We start by introducing the basic observational model in \cite{tsigler2025benign}. 
\begin{enumerate}[label=(T), ref=(T)]
\item \label{model:T} Suppose in model \ref{model:M}, $\bs z = \Sigma^{1/2}\bs \xi$, where $\Sigma = \rm{diag}(\lambda_1, \ldots, \lambda_p)$ and $\bs \xi$ is an isotropic random vector.
\end{enumerate}

Before proceeding to the comparison, we note that there are several major differences between their setting and ours which make a full comparison impossible. We will now describe the main differences, and subsequently focus on specific scenarios where our and their results can be compared in a meaningful way.

First, we note that their primary object of interest is different from ours. Specifically, \cite{tsigler2025benign} study the ridge regression solution:
\begin{equation}\label{eq:ridge}
    \bs{w}_\textrm{ridge}:= X^\top (XX^\top + \lambda I_n)^{-1}\bs y_\n,
\end{equation}
where $\lambda >0$ is a regularization parameter. 
All results in \cite{tsigler2025benign} are derived for $\bs{w}_\textrm{ridge}$, which does not always coincide with the max margin classifier even when $\lambda = 0$. This means that they do not study conditions which ensure that $\bs{\hat w} = \bs w_{\rm LS}$. Establishing the latter is a substantial part of our analysis that leads to many of our conditions.

Second, \cite{tsigler2025benign} provide their results under the assumption that certain events introduced in Definitions~\ref{def:eventA},~\ref{def:eventB} below hold. We also follow this approach with different events $E_i$ defined in~\eqref{eq:E1}--\eqref{eq:E5}. A substantial part of our work is dedicated to proving that events $E_1$--$E_5$ hold with high probability in model~\ref{model:EM}, while \cite{tsigler2025benign} comment on their events more briefly and do not provide a formal verification in the heavy-tailed case.

Third, the results in \cite{tsigler2025benign} are stated in terms of the quantity $k^*$ defined by \eqref{eq:k*}. In this respect, \cite{tsigler2025benign} is similar to the work of \cite{minsker2025classification} and more general than ours. We assume $\Tr(\Sigma) \gtrsim n\|\Sigma\|$ under both models \ref{model:EM} and \ref{model:sG}, which implies $k^* = 0$. For this reason, the comparisons below will focus on the case $k^* = 0$.

Fourth, \cite{tsigler2025benign} provide bounds on the quantity $\frac{\|\bs{w}_{\rm LS}\|_\Sigma}{\langle \bs{w}_{\rm LS}, \bs \mu \rangle}$ rather than on the classification error. This quantity is also a key ingredient in our analysis of the classification error, so below we will compare the bounds that we obtain for this quantity in Lemma~\ref{noiseless-mmc-bound1}, \ref{noiseless-mmc-bound2}, and \ref{noisy-mmc-bound1} with the bounds from \cite{tsigler2025benign} in the case $k^* = 0$, $\lambda = 0$.

\bigskip

For the reader's convenience, we recall some key notation and results from \cite{tsigler2025benign}. We first introduce the following events. Note that our definition of $Z$ differs from theirs. %To state their results, we need to introduce two events.

\begin{defn}\label{def:eventA}
Suppose model~\ref{model:T}. For any $C\geq 1$, we denote by $\mathcal{A}(C)$ the event 
\[
\mathcal{A}(C) := \left\{\frac{1}{C}\Tr(\Sigma) \leq \lambda_{min}(Z Z^\top) \leq \lambda_{max}(Z Z^\top) \leq C \Tr(\Sigma) \right\}.
\]
\end{defn}

\begin{defn}\label{def:eventB}
Suppose model~\ref{model:T}. For any $C_B > 0$, denote by $\mathcal{B}(C_B)$ the event on which all of the following hold:
\begin{enumerate}
    \item $\|Z\bs \mu\|^2 \leq C_B n\|\bs \mu\|_\Sigma^2$, 
    \item $\Tr(Z\Sigma Z^\top) \leq C_B n\|\Sigma\|_F^2$, 
    \item $\|Z\Sigma Z^\top\| \leq C_B (\|\Sigma\|_F^2 + n\|\Sigma\|^2)$.
\end{enumerate}
\end{defn}

We note that some of the results in \cite{tsigler2025benign} are established under the events above. In this sense their approach is similar to ours, where the probabilistic and geometric analysis are separated. A direct comparison of our events $E_1$--$E_5$ and the events above is difficult since different aspects of the problem are captured. Compare for instance our event $E_2$ with their part 1 of Definition~\ref{def:eventB}. In our event, only the angles between $\bm z_i$ and $\bs \mu$ play a role, whereas the event in \cite{tsigler2025benign} also puts restriction on the norms of $\bm z_i$. This is also apparent in their event $\mathcal{A}$: if $\Sigma$ is the identity and all $\bm z_i$ are exactly orthogonal, this event would require $\|\bm z_i\|$ to all be of order $p$ while our assumptions on the norms of $\bm z_i$ are much milder. One can thus argue that our events have a more intuitive geometric interpretation.

Next, we comment on the types of bounds which \cite{tsigler2025benign} and we obtain. To this end we need to introduce the following amount which depends on the noise rate $\eta$:
\[
\sigma_\eta = \left(\ln \frac{3+\eta^{-1}}{2}\right)^{-1}.
\]

Below we state two key results from \cite{tsigler2025benign} that are closest to ours.

\begin{thm}[Theorem~8, \cite{tsigler2025benign} when $k^* = 0$]\label{thm:tsiglerupper}
Suppose in model \ref{model:M}, $\bs z = \Sigma^{1/2}\bs \xi$, where $\Sigma = \rm{diag}(\lambda_1, \ldots, \lambda_p)$ and $\bs \xi$ is an isotropic random vector. 
For any $C_B >0$ and $C>1$, there exists a constant $c$ that only depends on $c_B$ and $C$, such that the following holds.

Assume that $\eta < c^{-1}$, $\Tr(\Sigma) > c\max\{n\|\Sigma\|, \sqrt{n}\|\Sigma\|_F^2\}$, and $\|\bs \mu\|^2 > \frac{2c^2 t}{\sqrt{n}}\|\bs \mu\|_\Sigma$. Then for any $t \in (0, \sqrt{n}/c)$, conditionally on the event $\mathcal{A}(C)\cap \mathcal{B}(c_B)$, with probability at least $1 - ce^{-t^2/2}$ over the draw of $(\bs y, \bs y_\n)$, the following inequality holds:
\[
\frac{\|\bs{w}_{\rm LS}\|_\Sigma}{\langle \bs{w}_{\rm LS}, \bs \mu \rangle} \leq 4c^2\frac{(\Tr(\Sigma) + \sigma_\eta n\|\bs \mu\|^2)(\sqrt{n}\|\Sigma\|_F + \sqrt{n}t\|\Sigma\|) + n\|\bs \mu\|_\Sigma\Tr(\Sigma)}{n\|\bs \mu\|^2 \Tr(\Sigma)}.
\]
\end{thm}

In the noiseless case ($\eta =0$), we note that the upper bound in Theorem~\ref{thm:tsiglerupper} gives
\begin{equation}\label{eq:tsiglernoiselessupper}
\left(\frac{\|\bs{w}_{\rm LS}\|_\Sigma}{\langle \bs{w}_{\rm LS}, \bs \mu \rangle}\right)^2 \lesssim \frac{\|\Sigma\|_F^2 + t^2\|\Sigma\|^2 + n\|\bs \mu\|_\Sigma^2}{n\|\bs \mu\|^4}.    
\end{equation}

In the noisy case ($\eta > 0$), we note that the upper bound gives
\begin{equation}\label{eq:tsiglernoisylessupper}
\left(\frac{\|\bs{w}_{\rm LS}\|_\Sigma}{\langle \bs{w}_{\rm LS}, \bs \mu \rangle}\right)^2 \lesssim \left(\frac{1}{n\|\bs \mu\|^2} + \sigma_\eta \frac{1}{\Tr(\Sigma)}\right)^2 (n\|\Sigma\|_F^2 + nt^2\|\Sigma\|^2) + \frac{\|\bs \mu\|_\Sigma^2}{\|\bs \mu\|^4}
\end{equation}

A lower bound is only obtained for the noiseless case with additional distributional assumption:

\begin{thm}[Theorem~9, \cite{tsigler2025benign} when $k^* = 0$]
Suppose that, in model \ref{model:M}, $\bs z = \Sigma^{1/2}\bs \xi$, where $\Sigma = \rm{diag}(\lambda_1, \ldots, \lambda_p)$ and $\bs \xi$ is an isotropic sub-Gaussian random vector with $\|\bs \xi\|_{\psi_2}\leq L$. For any $C>1$, there are large constants $a, c$ that only depend on $L$ and $C$ and an absolute constant $c_y$ such that the following holds. 

Suppose that $\Tr(\Sigma) > c\left(n\|\Sigma\| + \sqrt{n}\|\Sigma\|_F\right)$, the distribution of $Z$ is symmetric, and $\|\bs \mu\|^2 \geq \frac{1}{a\sqrt{n}}\|\bs \mu\|_\Sigma$. Then for any $t \in (0, \sqrt{n}/c_y)$, the probability of the event
\begin{equation}\label{eq:tsiglernoiselesslower}
\left\{\frac{\|\bs{w}_{\rm LS}\|_\Sigma}{\langle \bs{w}_{\rm LS}, \bs \mu \rangle} \geq \frac{\sqrt{n}\|\Sigma\|_F + n\|\bs \mu\|_\Sigma}{\sqrt{2}c(1+t)n\|\bs \mu\|^2}\right\}    
\end{equation}
is at least
\[
(c_y^{-1} - c_ye^{-t^2/c_y} - c_ye^{-n/c})_+ (\P(\mathcal{A}(C)) - e^{-n/c})_+,
\]
where $u_+ = \max\{u, 0\}$ for any $u\in \R$.
\end{thm}

The lower bound takes the form,
\begin{equation}\label{eq:tsiglernoiselesslower}
\left(\frac{\|\bs{w}_{\rm LS}\|_\Sigma}{\langle \bs{w}_{\rm LS}, \bs \mu \rangle}\right)^2 \gtrsim \frac{\|\Sigma\|_F^2 + n\|\bs \mu\|_\Sigma^2}{(1+t)^2n\|\bs \mu\|^4}.    
\end{equation}

\bigskip

Before proceeding with the comparison, we once again emphasize that the setting and objectives in our work and in \cite{tsigler2025benign} are quite different and thus the comparison below is only about very specific aspects of both works.

First, note that our Lemma~\ref{noiseless-mmc-bound1}, \ref{noiseless-mmc-bound2}, and \ref{noisy-mmc-bound1} provide matching (up to constants) upper and lower bounds for a slightly different quantity: we consider $\Big(\frac{\|\hat{\bs{w}}\|}{\langle \hat{\bs{w}}, \bs \mu \rangle}\Big)^2$.

When $\Sigma$ is the identity and $\hat{\bs{w}} = \bs{w}_{\rm{LS}}$, the quantity we bound coincides with the quantity studied by \cite{tsigler2025benign}. In general, we have
\[
\lambda_{min} \left(\frac{\|\bs{\hat w}\|}{\langle \bs{\hat w}, \bs \mu \rangle}\right)^2 
\lesssim 
\left(\frac{\|\bs{\hat w}\|_\Sigma}{\langle \bs{\hat w}, \bs \mu \rangle}\right)^2 
\lesssim 
\lambda_{max} \left(\frac{\|\bs{\hat w}\|}{\langle \bs{\hat w}, \bs \mu \rangle}\right)^2.
\]
When $\lambda_{max}/\lambda_{min} = O(1)$, our results also provide upper and lower bounds that match up to constant factors under the completely general model \ref{model:M} on certain events. In contrast, the analysis in \cite{tsigler2025benign} only gives lower bounds under much stronger conditions than upper bounds and only considers lower bounds in the noiseless case. In addition, their lower bound does not necessarily hold with high probability. However, their upper bounds hold under different conditions than ours and for a different quantity. Even when $\hat{\bs w} = \bs w_{\rm{LS}}$, there are many settings that are covered by our results but not by \cite{tsigler2025benign} and vice versa.

We now compare the corresponding upper and lower bounds in our work (Lemma~\ref{noiseless-mmc-bound1}, \ref{noiseless-mmc-bound2}, and \ref{noisy-mmc-bound1}), which are given for our most general model \ref{model:M}. 

Under the assumptions of any one of Lemma~\ref{noiseless-mmc-bound1} or \ref{noiseless-mmc-bound2} or \ref{noisy-mmc-bound1} we have
\begin{equation*}
\begin{aligned}
\frac{\lambda_{min}(\Sigma)}{(1-2\eta)^2}\left\{\eta n\rho + \frac{1}{\|\bs \mu\|^2} + \frac{1}{n\rho\|\bs \mu\|^4}\right\} \lesssim \left(\frac{\|\bs{\hat w}\|_\Sigma}{\langle \bs{\hat w}, \bs \mu \rangle}\right)^2  \lesssim \frac{\lambda_{max}(\Sigma)}{(1-2\eta)^2}\left\{\eta n\rho + \frac{1}{\|\bs \mu\|^2} + \frac{1}{n\rho\|\bs \mu\|^4}\right\}.    
\end{aligned}
\end{equation*}
If additionally $g \equiv 1$ and the norms of $\|\bs z_i\|$, we have $\rho = 1/\Tr(\Sigma)$, see Lemma~\ref{lemm5-ex1} for specific conditions that ensure this under model \ref{model:EM}. When $\Sigma = I_p$, we see that our upper bound matches the upper bounds of~\cite{tsigler2025benign} in~\eqref{eq:tsiglernoiselessupper} and~\eqref{eq:tsiglernoisylessupper} up to constants in both, the noisy and the noiseless case. This bound also matches results in \cite{minsker2025classification} in the isotropic noiseless model.
When there is a substantial deviation between $\lambda_{\max}$ and $\lambda_{min}$, the bounds differ. In this case,~\eqref{eq:tsiglernoiselessupper} and~\cite{minsker2025classification} capture some aspects of the alignment between $\bs \mu$ and $\Sigma$ that our analysis does not capture. It would be an interesting future direction to obtain sufficient conditions for $\bs{\hat w} = \bs w_{\rm LS}$ which allow for $k^* >0$ and a matching lower bound which holds in more generality.

\section{Simulation results}

In this section, we present simulations that illustrate our main results: (i) phase transition between the weak and signal regime and (ii) theoretical benign overfitting threshold.

\subsection{Isotropic Gaussian Mixtures}

All the simulation results here are for isotropic Gaussian mixtures with $\|\bs \mu\| = n^a$ and $p = n^b$. In this setting, the phase transition between the weak signal regime $a < \frac{b-1}{2}$ and the strong signal regime $a > \frac{b-1}{2}$ occurs at $a = \frac{b-1}{2}$.
Benign overfitting provably fails if $a \leq \frac{b-1}{4}$ (see the last paragraph of Section~\ref{subsec:main-noisy}). 

For each value of $n$, a linear classifier is trained using gradient descent \eqref{eq:gdsc}. Gradient descent is run until directional convergence, defined as stabilization of the normalized iterate $\bs w_t/\|\bs w_t\|$, which is detected numerically by monitoring the cosine similarity between consecutive normalized iterates. This stopping criterion detects stabilization of the classifier direction, corresponding to convergence toward the max-margin solution, while maintaining computational efficiency. The experiment is repeated independently 10 times, and the empirical test errors are averaged across repetitions. 

Figure~\ref{fig:simulation-gmm-b12} shows the simulation results for $b=1.2$ with different values of $a$, which confirm that benign overfitting occurs if $a=0.2$ (under strong signal regime) while fails if $a = 0.05$. 

Similarly, Figure~\ref{fig:simulation-gmm-b18} shows the results for $b = 1.45$ with different values of $a$, which confirm that benign overfitting occurs if $a = 0.2$ (weak signal regime) and fails if $a = 0.1125$.

\begin{figure}[H]
    \centering
    \includegraphics[width=1.0\linewidth]{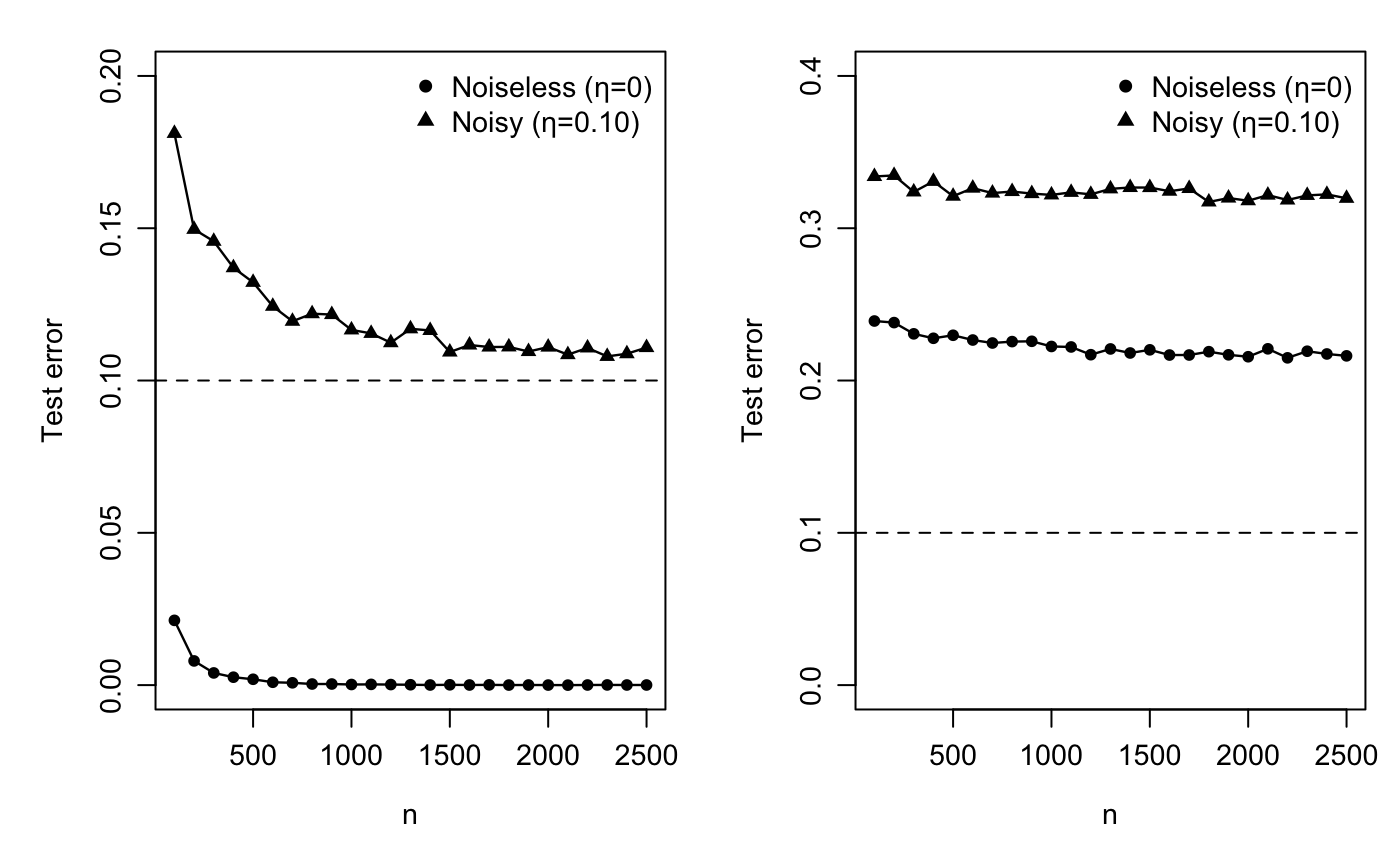}
    \caption{Empirical test error for $b = 1.2$ with different values of $a$: $a= 0.2$ (left) and $a=0.05$ (right). When $a = 0.2$, the model is in the strong signal regime, and $\bs{\hat w}$ is expected to behave differently between the noiseless and noisy case. This difference appears in the theoretical upper test error bounds which are in the form of $C\exp(-c n^{0.4})$ for the noiseless case and $0.1 + C\exp(-cn^{0.2})$ for the noisy case where the constants $C,c$ can differ between the noisy and noiseless case. The simulation result confirms faster decay in the noiseless case as expected by the theoretical bounds. When $a = 0.05$, benign overfitting is expected to fail.}
    \label{fig:simulation-gmm-b12}
\end{figure}

\begin{figure}[H]
    \centering
    \includegraphics[width=1.0\linewidth]{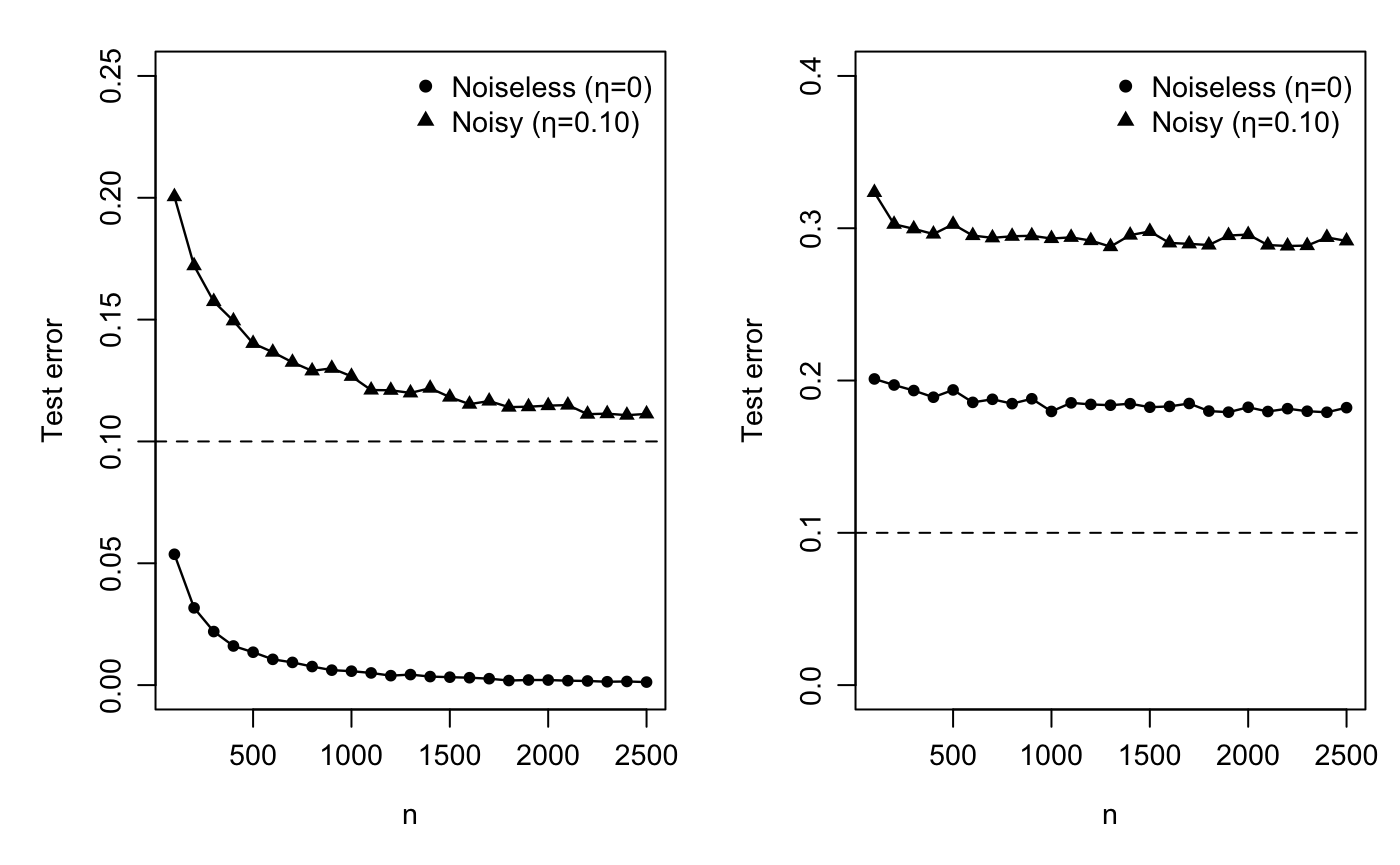}
    \caption{Empirical test error for $b = 1.45$ with different values of $a$: $a= 0.2$ (left) and $a=0.1125$ (right). When $a = 0.2$, the model is in the weak signal regime, and the theoretical upper test error bounds are in the form of $\eta + C\exp(-c n^{0.35})$ where the constants $c,C$ can differ between the noisy and noiseless case. 
    When $a = 0.1125$, benign overfitting is expected to fail.}
    \label{fig:simulation-gmm-b18}
\end{figure}

\newpage

\subsection{Heavy tailed Mixtures}

In this subsection, we present simulation results under a heavy-tailed mixture model. We consider model \ref{model:EM} with $\Sigma = I_p$ and $\xi_i$ follow the Student's $t$-distribution scaled by $\tfrac{1}{4}$ with degrees of freedom $df = r + 0.001$, i.e. $r$-th moment is finite while $(r+0.001)$-th moment diverges. 

Figure~\ref{fig:simulation-heavy} shows the experimental results for $a=0.3, b = 1.8$ (weak signal regime) with varying  values of $r$. For $r = 2.1, 2.4, 3.0$, we confirm that benign overfitting occurs as expected by our main results. On the contrary, benign overfitting seems to fail with $r=1.8, 1.9$. This result seems to indicate that the finite moment condition $r>2$ is necessary for benign overfitting and cannot be relaxed further. 

\begin{figure}[H]
    \centering
    \includegraphics[width=1.0\linewidth]{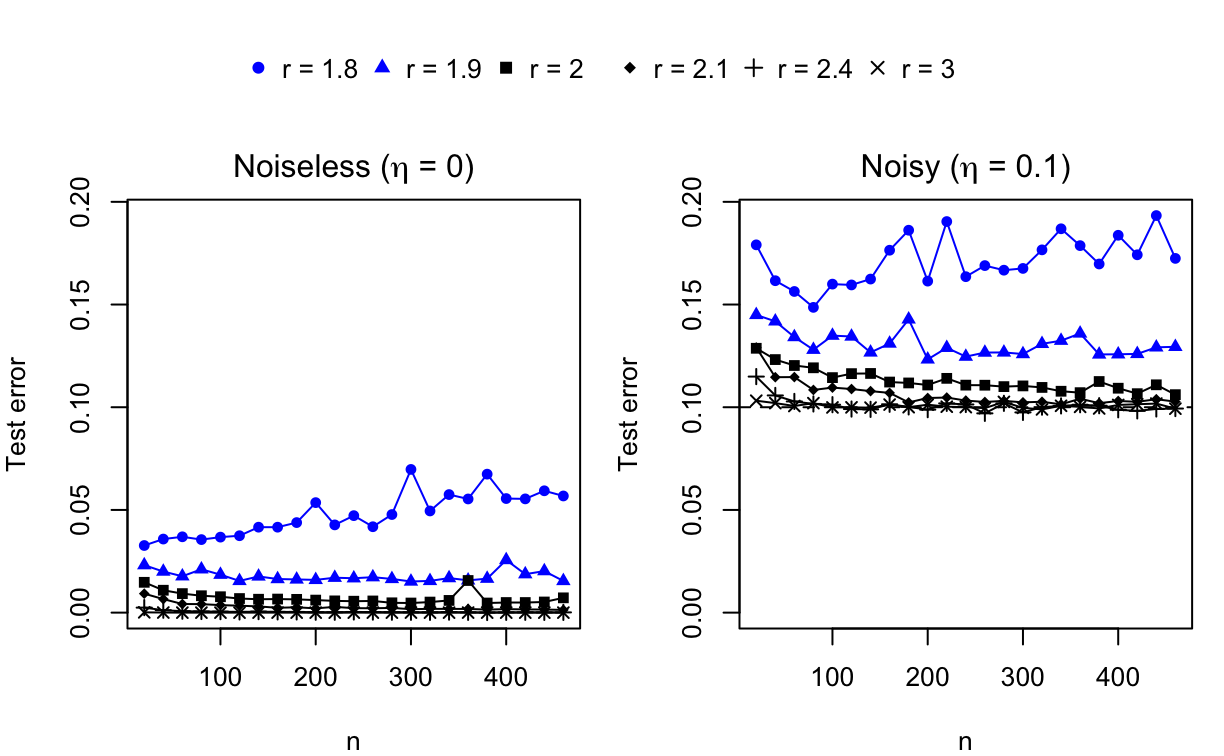}
    \caption{Empirical test error for the noiseless model (left) and noisy model (right) with $a = 0.3, b = 1.8$, and varying values of $r$. When $r > 2$, we confirm that the empirical test error decreases toward $\eta$ as $n$ increases. When $r< 2$ we see that both the noisy and noiseless model fail to generalize well.}
    \label{fig:simulation-heavy}
\end{figure}

\end{document}